\documentclass{article}

\usepackage[final]{neurips_2021}

\usepackage[utf8]{inputenc} 
\usepackage[T1]{fontenc}    
\usepackage{hyperref}       
\usepackage{url}            
\usepackage{booktabs}       
\usepackage{amsfonts}       
\usepackage{nicefrac}       
\usepackage{microtype}      
\usepackage{xcolor}         

\definecolor{kagreen}{rgb}{0.0, 0.5, 0.0}

\usepackage{wrapfig}

\usepackage{tikz}
\usetikzlibrary{backgrounds}
\usetikzlibrary{arrows}
\usetikzlibrary{shapes,shapes.geometric,shapes.misc}

\tikzstyle{tikzfig}=[baseline=-0.25em,scale=0.5]

\pgfkeys{/tikz/tikzit fill/.initial=0}
\pgfkeys{/tikz/tikzit draw/.initial=0}
\pgfkeys{/tikz/tikzit shape/.initial=0}
\pgfkeys{/tikz/tikzit category/.initial=0}

\pgfdeclarelayer{edgelayer}
\pgfdeclarelayer{nodelayer}
\pgfsetlayers{background,edgelayer,nodelayer,main}

\tikzstyle{none}=[inner sep=0mm]

\newcommand{\tikzfig}[1]{%
{\tikzstyle{every picture}=[tikzfig]
\IfFileExists{#1.tikz}
  {\input{#1.tikz}}
  {%
    \IfFileExists{./figures/#1.tikz}
      {\input{./figures/#1.tikz}}
      {\tikz[baseline=-0.5em]{\node[draw=red,font=\color{red},fill=red!10!white] {\textit{#1}};}}%
  }}%
}

\tikzstyle{every loop}=[]

\tikzstyle{big box}=[fill=white, fill opacity=0, draw=black, shape=rectangle, minimum width=2.5cm, minimum height=3cm]
\tikzstyle{small box}=[fill=white, draw={rgb,255: red,171; green,171; blue,171}, shape=rectangle, minimum width=2cm, minimum height=1cm, tikzit shape=rectangle]

\tikzstyle{new edge style 0}=[->, draw=blue]
\tikzstyle{new edge style 1}=[draw=red, ->]
\tikzstyle{new edge style 2}=[draw={rgb,255: red,128; green,0; blue,128}, ->]
\tikzstyle{new edge style 3}=[->, dashed, draw=red, opacity=0.5]

\input{sample.tikzdefs}

\RequirePackage{algorithm}
\RequirePackage{algorithmic}
\usepackage{bbm}            
\usepackage{todonotes} 

\usepackage{graphicx}

\usepackage{amsmath}
\usepackage{amsthm}
\usepackage{stmaryrd}
\usepackage{amssymb}
\usepackage{comment}
\usepackage{setspace}

\usepackage[space]{grffile}
\usepackage{subcaption}

\newcommand{\bc}[1]{\left\{{#1}\right\}}
\newcommand{\br}[1]{\left({#1}\right)}
\newcommand{\bs}[1]{\left[{#1}\right]}
\newcommand{\abs}[1]{\left| {#1} \right|}

\newcommand{\ip}[2]{\left\langle{#1},{#2}\right\rangle}
\newcommand{\norm}[1]{\left\| {#1} \right\|}

\renewcommand{\P}[1]{\mathbb{P}\bs{{#1}}}

\newcommand{\E}[1]{\mathbb{E}\bs{{#1}}}
\newcommand{\Ee}[2]{\underset{#1}{\mathbb{E}}\bs{{#2}}}

\newtheorem{lemma}{Lemma}
\newtheorem{theorem}{Theorem}

\newtheorem{fact}{Fact}
\newtheorem{corollary}{Corollary}

\DeclareMathOperator*{\argmin}{arg\,min}
\DeclareMathOperator*{\argmax}{arg\,max}

\makeatletter
\newcommand{\newreptheorem}[2]{\newtheorem*{rep@#1}{\rep@title} 
	\newenvironment{rep#1}[1]{\def\rep@title{#2 \ref*{##1}}\begin{rep@#1}}{\end{rep@#1}}
}
\makeatother

\newreptheorem{lemma}{Lemma}
\newreptheorem{theorem}{Theorem}
\newreptheorem{claim}{Claim}
\newreptheorem{proposition}{Proposition}
\newreptheorem{corollary}{Corollary}
\newreptheorem{remark}{Remark}

\allowdisplaybreaks

\title{Robust Inverse Reinforcement Learning under Transition Dynamics Mismatch}

\author{%
  Luca Viano \\
  LIONS, EPFL \\
  \And
  Yu-Ting Huang \\
  EPFL \\
  \AND
  Parameswaran Kamalaruban\thanks{Correspondence to: Parameswaran Kamalaruban <\texttt{kparameswaran@turing.ac.uk}>} \\
  The Alan Turing Institute \\
  \And
  Adrian Weller \\
  University of Cambridge \\
  \& The Alan Turing Institute \\
  \And
  Volkan Cevher \\
  LIONS, EPFL \\
}

\begin{document}

\maketitle

\begin{abstract}
We study the inverse reinforcement learning (IRL) problem under a transition dynamics mismatch between the expert and the learner. Specifically, we consider the Maximum Causal Entropy (MCE) IRL learner model and provide a tight upper bound on the learner's performance degradation based on the $\ell_1$-distance between the transition dynamics of the expert and the learner. Leveraging insights from the Robust RL literature, we propose a robust MCE IRL algorithm, which is a principled approach to help with this mismatch. Finally, we empirically demonstrate the stable performance of our algorithm compared to the standard MCE IRL algorithm under transition dynamics mismatches in both finite and continuous MDP problems.
\end{abstract}


\section{Introduction}\label{sec:intro}

Recent advances in Reinforcement Learning (RL)~\cite{sutton2000policy,silver2014deterministic,schulman2015trust,schulman2017proximal} have demonstrated impressive performance in games~\cite{mnih2015human,silver2017mastering}, continuous control~\cite{lillicrap2015continuous}, and robotics~\cite{levine2016end}. Despite these successes, a broader application of RL in real-world domains is hindered by the difficulty of designing a proper reward function. Inverse Reinforcement Learning (IRL) addresses this issue by inferring a reward function from a given set of demonstrations of the desired behavior~\cite{russell1998learning,ng2000algorithms}. IRL has been extensively studied, and many algorithms have already been proposed~\cite{abbeel2004apprenticeship,ratliff2006maximum,ziebart2008maximum,syed2008game,boularias2011relative,osa2018algorithmic}.

Almost all IRL algorithms 
assume that the expert demonstrations are collected from the same environment as the one in which the IRL agent is trained. However, this assumption rarely holds in real world because of many possible factors identified by~\cite{dulacarnold2019challenges}. For example, consider an autonomous car that should learn by observing  expert demonstrations performed on another 
car with possibly different technical characteristics. There is often a mismatch between the learner and the expert's transition dynamics, resulting in poor performance that are critical in healthcare \cite{yu2019reinforcement} or autonomous driving \cite{kiran2020deep}. Indeed, the performance degradation of an IRL agent due to transition dynamics mismatch has 
been noted empirically~\cite{reddy2018you,Gong2020WhatII,gangwani2020stateonly,liu2019state}, but 
without theoretical guidance. 

To this end, our work first provides a theoretical study on the effect of such mismatch in the context of the infinite horizon Maximum Causal Entropy (MCE) IRL framework~\cite{ziebart2010modeling,ziebart2013principle,zhou2017infinite}. Specifically, we bound the potential decrease in the IRL learner's performance as a function of the $\ell_1$-distance between the expert and the learner's transition dynamics. We then propose a robust variant of the MCE IRL algorithm to effectively recover a reward function under transition dynamics mismatch, mitigating degradation. 
There is precedence to our robust IRL approach, such as~\cite{tessler2019action} that employs an adversarial training method to learn a robust policy against adversarial changes in the learner's environment. The novel idea of our work is to incorporate this method within our IRL context, by viewing the expert's transition dynamics as a perturbed version of the learner's one. 

Our robust MCE IRL algorithm leverages techniques from the robust RL literature~\cite{iyengar2005robust,nilim2005robust,pinto2017robust,tessler2019action}. A few recent works~\cite{reddy2018you,Gong2020WhatII,Herman2016InverseRL} attempt to infer the expert's transition dynamics from the demonstration set or via additional information, and then apply the standard IRL method to recover the reward function based on the learned dynamics. Still, the transition dynamics can be estimated only up to a certain accuracy, i.e., a mismatch between the learner's belief and the dynamics of the expert's environment remains. Our robust IRL approach can be incorporated into this research vein to further improve the IRL agent's performance.

To our knowledge, this is the first work that rigorously reconciles model-mismatch in IRL with only one shot access to the expert environment. We highlight the following contributions: 
\begin{enumerate}
\itemsep0em 
    \item We provide a tight upper bound for the suboptimality of an IRL learner that receives expert demonstrations from an MDP with different transition dynamics 
    compared to a learner that receives demonstrations from an MDP with the same transition dynamics 
    (Section~\ref{sec:mce-irl-tight-upper}). 
    \item We find suitable conditions under which a solution exists to the MCE IRL optimization problem with model mismatch (Section~\ref{sec:existence}).
    \item We propose a robust variant of the MCE IRL algorithm to learn a policy from expert demonstrations under transition dynamics mismatch (Section~\ref{sec:two_players_entropy}). 
    \item We demonstrate our method's robust performance compared to the standard MCE IRL in a broad set of experiments under both linear and non-linear reward settings (Section~\ref{sec:experiments}). 
    \item We extend our robust IRL method to the high dimensional continuous MDP setting with appropriate practical relaxations, and empirically demonstrate its effectiveness (Section~\ref{sec:experiments_continuous_control_main}). 
\end{enumerate}

\section{Problem Setup}\label{sec:Setup}

This section formalizes the IRL problem with an emphasis on the learner and expert environments. We use bold notation to represent vectors. A glossary of notation is given in Appendix~\ref{appendix:notation}.

\subsection{Environment and Reward} 

We formally represent the environment by a Markov decision process (MDP) $M_{\boldsymbol{\theta}} := \bc{\mathcal{S}, \mathcal{A}, T, \gamma, P_0, R_{\boldsymbol{\theta}}}$, parameterized by $\boldsymbol{\theta} \in \mathbb{R}^d$. The state and action spaces are denoted as $\mathcal{S}$ and $\mathcal{A}$, respectively. We assume that $\abs{\mathcal{S}}, \abs{\mathcal{A}} < \infty$. $T: \mathcal{S} \times \mathcal{S} \times \mathcal{A} \rightarrow [0, 1]$ represents the transition dynamics, i.e., $T(s'|s,a)$ is the probability of transitioning to state $s'$ by taking action $a$ from state $s$. The discount factor is given by $\gamma \in (0,1)$, and $P_0$ is the initial state distribution. We consider a linear reward function $R_{\boldsymbol{\theta}}:\mathcal{S} \rightarrow \mathbb{R}$ of the form $R_{\boldsymbol{\theta}}(s) = \ip{\boldsymbol{\theta}}{\boldsymbol{\phi}(s)}$, where $\boldsymbol{\theta} \in \mathbb{R}^d$ is the reward parameter, and $\boldsymbol{\phi}: \mathcal{S} \rightarrow \mathbb{R}^d$ is a feature map. We use a one-hot feature map $\boldsymbol{\phi}: \mathcal{S} \rightarrow \bc{0,1}^{\abs{\mathcal{S}}}$, where the $s^\text{th}$ element of $\boldsymbol{\phi}\br{s}$ is 1 and 0 elsewhere. Our results can be extended to any general feature map (see empirical evidence in Fig.~\ref{fig:low-dim-exp-short}), but we use this particular choice as a running example for concreteness. 

We focus on the state-only reward function since the state-action reward function is not that useful in the robustness context. Indeed, as~\cite{gangwani2020stateonly} pointed out, the actions to achieve a specific goal under different transition dynamics will not necessarily be the same and, consequently, should not be imitated. Analogously, in the IRL context, the reward for taking a particular action should not be recovered since the quality of that action depends on the transition dynamics. We denote an MDP without a reward function by $M = M_{\boldsymbol{\theta}} \backslash R_{\boldsymbol{\theta}} = \bc{\mathcal{S}, \mathcal{A}, T, \gamma, P_0}$. 

\subsection{Policy and Performance}

A policy $\pi: \mathcal{S} \rightarrow \Delta_\mathcal{A} $ is a mapping from a state to a probability distribution over actions. The set of all valid stochastic policies is denoted by $\Pi := \bc{ \pi : \sum_a \pi(a | s) = 1 , \forall s \in \mathcal{S};  
\pi(a | s) \geq 0 , \forall (s,a) \in \mathcal{S} \times \mathcal{A}}$. We are interested in two different performance measures of any policy $\pi$ acting in the MDP $M_{\boldsymbol{\theta}}$: (i) the expected discounted return $V^{\pi}_{M_{\boldsymbol{\theta}}} := \E{\sum^{\infty}_{t=0} \gamma^t R_{\boldsymbol{\theta}}\br{s_t} \mid \pi, M}$, and (ii) its entropy regularized variant $V^{\pi,\mathrm{soft}}_{M_{\boldsymbol{\theta}}} := \E{\sum^{\infty}_{t=0} \gamma^t \bc{R_{\boldsymbol{\theta}}\br{s_t} - \log \pi \br{a_t | s_t}} \mid \pi, M}$. The state occupancy measure of a policy $\pi$ in the MDP $M$ is defined as $\rho^{\pi}_{M}(s) := \br{1-\gamma} \sum_{t=0}^\infty \gamma^t \P{s_t = s \mid \pi, M}$, where $\P{s_t = s \mid \pi, M}$ denotes the probability of visiting the state $s$ after $t$ steps by following the policy $\pi$ in $M$. Note that $\rho^{\pi}_{M}(s)$ does not depend on the reward function. Let $\boldsymbol{\rho}^{\pi}_{M} \in \mathbb{R}^{\abs{\mathcal{S}}}$ be a vector whose $s^\text{th}$ element is $\rho^{\pi}_{M}(s)$. For the one-hot feature map $\boldsymbol{\phi}$, we have that $V^{\pi}_{M_{\boldsymbol{\theta}}} =\frac{1}{1 - \gamma} \sum_{s} \rho^{\pi}_{M} (s) R_{\boldsymbol{\theta}}(s) = \frac{1}{1 - \gamma} \ip{\boldsymbol{\theta}}{\boldsymbol{\rho}^{\pi}_{M}}$. A policy $\pi$ is \emph{optimal} for the MDP $M_{\boldsymbol{\theta}}$ if $\pi \in \argmax_{\pi'} V^{\pi'}_{M_{\boldsymbol{\theta}}}$, in which case we denote it by $\pi^*_{M_{\boldsymbol{\theta}}}$. Similarly, the \emph{soft-optimal} policy (always unique~\cite{geist2019regmdp}) in $M_{\boldsymbol{\theta}}$ is defined as $\pi^{\mathrm{soft}}_{M_{\boldsymbol{\theta}}} := \argmax_{\pi'} V^{\pi',\mathrm{soft}}_{M_{\boldsymbol{\theta}}}$ (see Appendix~\ref{appendix:sec2} for a parametric form of this policy). 

\subsection{Learner and Expert} 

\begin{wrapfigure}{r}{0.5\textwidth}
\vspace{-\intextsep}
\begin{tikzpicture}[scale=0.65, every node/.style={scale=0.75}]
	\begin{pgfonlayer}{nodelayer}
		\node [style=big box, label={above:Expert}] (0) at (4.5, 3.25) {};
		\node [style=big box, label={above:Learner}] (1) at (9.25, 3.25) {};
		\node [style=small box] (2) at (4.5, 4) {$M^{L}_{\boldsymbol{\theta^*}} : \pi^{*}_{M^{L}_{\boldsymbol{\theta^*}}}$};
		\node [style=small box] (3) at (4.5, 2.5) {$M^{E}_{\boldsymbol{\theta^*}} : \pi^{*}_{M^{E}_{\boldsymbol{\theta^*}}}$};
		\node [style=small box] (4) at (9.25, 4) {$M^{L}_{\boldsymbol{\theta^*}} \texttt{\textbackslash} R_{\boldsymbol{\theta^*}}$};
		\node [style=none, label={right: $\br{\boldsymbol{\theta_{L}}, \pi^{\mathrm{soft}}_{M^{L}_{\boldsymbol{\theta_{L}}}}}$}] (5) at (11.25, 3.75) {};
		\node [style=none, label={right: $\br{\boldsymbol{\theta_{E}}, \pi^{\mathrm{soft}}_{M^{L}_{\boldsymbol{\theta_{E}}}}}$}] (6) at (11.25, 2.75) {};
		\node [style=none, label={$\boldsymbol{\rho} = \boldsymbol{\rho}^{\pi^{*}_{M_{\boldsymbol{\theta^*}}^L}}_{M^L}$}] (7) at (7, 4) {};
		\node [style=none, label={$\boldsymbol{\rho} = \boldsymbol{\rho}^{\pi^{*}_{M_{\boldsymbol{\theta^*}}^E}}_{M^E}$}] (8) at (7, 2.35) {};
        \node [style=none] (9) at (10.47, 3.75) {};
        \node [style=none] (10) at (10.47, 2.75) {};
	\end{pgfonlayer}
	\begin{pgfonlayer}{edgelayer}
		\draw [style=new edge style 0] (2) to (4);
		\draw [style=new edge style 1] (3) to (4);
		\draw [style=new edge style 0] (9) to (5.center);
		\draw [style=new edge style 1] (10) to (6.center);
	\end{pgfonlayer}
\end{tikzpicture}
\vspace{-\intextsep}
\caption{An illustration of the IRL problem under transition dynamics mismatch: See Section 2.} 
\vspace{-\intextsep}
\label{fig:mismatch}
\end{wrapfigure}
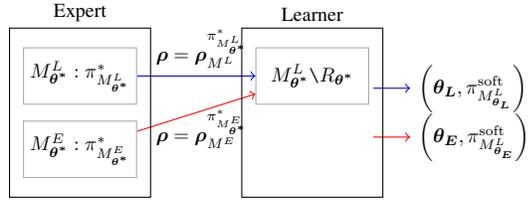

Our setting has two entities: a learner implementing the MCE IRL algorithm, and an expert. We consider two MDPs, $M_{\boldsymbol{\theta}}^L = \bc{\mathcal{S}, \mathcal{A}, T^L, \gamma, P_0, R_{\boldsymbol{\theta}}}$ and $M_{\boldsymbol{\theta}}^E = \bc{\mathcal{S}, \mathcal{A}, T^E, \gamma, P_0, R_{\boldsymbol{\theta}}}$, that differ only in the transition dynamics. The true reward parameter $\boldsymbol{\theta} = \boldsymbol{\theta^*}$ is known only to the expert. The expert provides demonstrations to the learner: (i) by following policy $\pi^{*}_{M_{\boldsymbol{\theta^*}}^E}$ in $M^E$ when there is a \emph{transition dynamics mismatch} between the learner and the expert, or (ii) by following policy $\pi^{*}_{M_{\boldsymbol{\theta^*}}^L}$ in $M^L$ otherwise. The learner always operates in the MDP $M^L$ and is not aware of the true reward parameter and of the expert dynamics $T^E$\footnote{The setting with $T^E$ known to the learner has been studied under the name of \emph{imitation learning across embodiments}~\cite{chen2016adversarial}.}, i.e., it only has access to $M_{\boldsymbol{\theta^*}}^L \backslash R_{\boldsymbol{\theta^*}}$. It learns a reward parameter $\boldsymbol{\theta}$ and the corresponding soft-optimal policy $\pi^{\mathrm{soft}}_{M_{\boldsymbol{\theta}}^L}$, based on the state occupancy measure $\boldsymbol{\rho}$ received from the expert. Here, $\boldsymbol{\rho}$ is either $\boldsymbol{\rho}^{\pi^{*}_{M_{\boldsymbol{\theta^*}}^E}}_{M^E}$ or $\boldsymbol{\rho}^{\pi^{*}_{M_{\boldsymbol{\theta^*}}^L}}_{M^L}$ depending on the case. Our results  can be extended to the stochastic estimate of $\boldsymbol{\rho}$ using  concentration inequalities
~\cite{abbeel2004apprenticeship}.  

Our learner model builds on the MCE IRL~\cite{ziebart2010modeling,ziebart2013principle,zhou2017infinite} framework that matches the expert's state occupancy measure $\boldsymbol{\rho}$. In particular, the learner policy is obtained by maximizing its causal entropy while matching the expert's state occupancy:
\begin{equation}
\max_{\pi \in \Pi} ~ \E{\sum_{t = 0}^\infty - \gamma^t \log \pi(a_t | s_t) \biggm| \pi, M^L} ~~ \text{subject to} ~~ \boldsymbol{\rho}^{\pi}_{M^L} = \boldsymbol{\rho} .
\label{opt_start}
\end{equation}
Note that this optimization problem only requires access to $M_{\boldsymbol{\theta}}^L \backslash R_{\boldsymbol{\theta}}$. The constraint in~\eqref{opt_start} follows from our choice of the one-hot feature map. We denote the optimal solution of the above problem by $\pi^{\mathrm{soft}}_{M_{\boldsymbol{\theta}}^L}$ with a corresponding reward parameter: (i) $\boldsymbol{\theta} = \boldsymbol{\theta_E}$, when we use $\boldsymbol{\rho}^{\pi^{*}_{M_{\boldsymbol{\theta^*}}^E}}_{M^E}$ as $\boldsymbol{\rho}$, or (ii) $\boldsymbol{\theta} = \boldsymbol{\theta_L}$, when we use $\boldsymbol{\rho}^{\pi^{*}_{M_{\boldsymbol{\theta^*}}^L}}_{M^L}$ as $\boldsymbol{\rho}$. Here, the parameters $\boldsymbol{\theta_E}$ and $\boldsymbol{\theta_L}$ are obtained by solving the corresponding dual problems of~\eqref{opt_start}. Finally, we are interested in the performance of the learner policy $\pi^{\mathrm{soft}}_{M_{\boldsymbol{\theta}}^L}$ in the MDP $M_{\boldsymbol{\theta^*}}^L$. Our problem setup is illustrated in Figure~\ref{fig:mismatch}.

\section{MCE IRL under Transition Dynamics Mismatch}
\label{sec:bounds}

This section analyses the MCE IRL learner's suboptimality when there is a transition dynamics mismatch between the expert and the learner, as opposed to an ideal learner without this mismatch. The proofs of the theoretical statements of this section can be found in Appendix~\ref{appendix:maxentirl}.

\subsection{Upper bound on the Performance Gap}
\label{sec:mce-irl-tight-upper}

First, we introduce an auxiliary lemma to be used later in our analysis. We define the distance between the two transition dynamics $T$ and $T'$, and the distance between the two policies $\pi$ and $\pi'$ as follows, respectively: $d_\mathrm{dyn} \br{T, T'} := \max_{s,a} \norm{T\br{\cdot \mid s,a} - T'\br{\cdot \mid s,a}}_1$, and $d_\mathrm{pol} \br{\pi, \pi'} := \max_{s} \norm{\pi(\cdot|s) - \pi'(\cdot|s)}_1$. Consider the two MDPs $M_{\boldsymbol{\theta}} = \bc{\mathcal{S}, \mathcal{A}, T, \gamma, P_0, R_{\boldsymbol{\theta}}}$ and $M'_{\boldsymbol{\theta}} = \bc{\mathcal{S}, \mathcal{A}, T', \gamma, P_0, R_{\boldsymbol{\theta}}}$. We assume that the reward function is bounded, i.e., $R_{\boldsymbol{\theta}} \br{s} \in \bs{R^{\mathrm{min}}_{\boldsymbol{\theta}},R^{\mathrm{max}}_{\boldsymbol{\theta}}}, \forall{s \in \mathcal{S}}$. Also, we define the following two constants: $\kappa_{\boldsymbol{\theta}} := \sqrt {\gamma \cdot \max \bc{ R^{\mathrm{max}}_{\boldsymbol{\theta}} + \log \abs{\mathcal{A}},  - \log \abs{\mathcal{A}} - R^{\mathrm{min}}_{\boldsymbol{\theta}}}}$ and $\abs{R_{\boldsymbol{\theta}}}^{\mathrm{max}} := \max \bc{\abs{R^{\mathrm{min}}_{\boldsymbol{\theta}}}, \abs{R^{\mathrm{max}}_{\boldsymbol{\theta}}}}$. 
\begin{lemma}
\label{thm:first_week}
Let $\pi := \pi^\mathrm{soft}_{M_{\boldsymbol{\theta}}}$ and $\pi' := \pi^\mathrm{soft}_{M'_{\boldsymbol{\theta}}}$ be the soft optimal policies for the MDPs $M_{\boldsymbol{\theta}}$ and $M'_{\boldsymbol{\theta}}$ respectively. Then, the distance between $\pi$ and $\pi'$ is bounded as follows: $d_\mathrm{pol} \br{\pi', \pi} \leq 2 \min \bc{\frac{\kappa_{\boldsymbol{\theta}} \sqrt{d_\mathrm{dyn} \br{T', T}}}{(1 - \gamma)}, \frac{\kappa^2_{\boldsymbol{\theta}} d_\mathrm{dyn} \br{T', T}}{(1 - \gamma)^2}}$.
\end{lemma}
\looseness-1The above result is obtained by bounding the KL divergence between the two soft optimal policies, and involves a non-standard derivation compared to the well-established performance difference theorems in the literature (see Appendix~\ref{app:proof-soft-lemma}). The lemma above bounds the maximum total variation distance between two soft optimal policies obtained by optimizing the same reward under different transition dynamics. It serves as a prerequisite result for our later theorems (Theorem~\ref{thm:regret_different_experts} for soft optimal experts and Theorem~\ref{thm:new-reward-transfer-bound}). In addition, it may be a result of independent interest for entropy regularized MDP.

Now, we turn to our objective. Let $\pi_1 := \pi^{\mathrm{soft}}_{M_{\boldsymbol{\theta_L}}^L}$ be the policy returned by the MCE IRL algorithm when there is no transition dynamics mismatch. Similarly, let $\pi_2 := \pi^{\mathrm{soft}}_{M_{\boldsymbol{\theta_E}}^L}$ be the policy returned by the MCE IRL algorithm when there is a mismatch. Note that $\pi_1$ and $\pi_2$ are the corresponding solutions to the optimization problem~\eqref{opt_start}, when $\boldsymbol{\rho} \gets \boldsymbol{\rho}^{\pi^{*}_{M_{\boldsymbol{\theta^*}}^L}}_{M^L}$ and $\boldsymbol{\rho} \gets \boldsymbol{\rho}^{\pi^{*}_{M_{\boldsymbol{\theta^*}}^E}}_{M^E}$, respectively. The following theorem bounds the performance degradation of the policy $\pi_2$ compared to the policy $\pi_1$ in the MDP $M_{\boldsymbol{\theta^*}}^L$, where the learner operates on:

\begin{theorem}
\label{thm:regret_different_experts}
The performance gap between the policies $\pi_1$ and $\pi_2$ on the MDP $M_{\boldsymbol{\theta^*}}^L$ is bounded as follows: $\abs{V^{\pi_1}_{M_{\boldsymbol{\theta^*}}^L} - V^{\pi_2}_{M_{\boldsymbol{\theta^*}}^L}} ~\leq~ \frac{\gamma \cdot \abs{R_{\boldsymbol{\theta^*}}}^{\mathrm{max}}}{(1 - \gamma)^2} \cdot d_\mathrm{dyn} \br{T^L, T^E}$.
\end{theorem}

The above result is obtained from the optimality conditions of the problem~\eqref{opt_start}, and using Theorem~7 from~\cite{zhang2020multi}. In Section~\ref{sec:robust-mce-irl-bound}, we show that the above bound is indeed tight. When the expert policy is soft-optimal, we can use Lemma~\ref{thm:first_week} and Simulation Lemma~\cite{kearns1998near,even2003approximate} to obtain an upper bound on the performance gap (see Appendix~\ref{app:mce-irl-upper-bound}). For an application of Theorem~\ref{thm:regret_different_experts}, consider an IRL learner that first learns a simulator of the expert environment, and then matches the expert behavior in the simulator. In this case, our upper bound provides an estimate (sufficient condition) of the accuracy required for the simulator.

\subsection{Existence of Solution under Mismatch} 
\label{sec:existence}

The proof of the existence of a unique solution to the optimization problem~\eqref{opt_start}, presented in~\cite{bloem2014infinite}, relies on the fact that both expert and learner environments are the same. This assumption implies that the expert policy is in the feasible set that is consequently non-empty. Theorem~\ref{thm:occ_states} presented in this section poses a condition under which we can ensure that the feasible set is non-empty when the expert and learner environments are not the same. 

Given $M^L$ and $\boldsymbol{\rho}$, we define the following quantities 
useful for stating our theorem. We define, for each state $s \in \mathcal{S}$, the probability flow matrix $\boldsymbol{F}(s) \in \mathbb{R}^{\abs{\mathcal{S}} \times \abs{\mathcal{A}} }$ as follows: $\left[\boldsymbol{F}(s)\right]_{i,j} ~:=~ \rho(s) T^L_{s_i,s,a_j}$, where $T^L_{s_i, s, a_j} := T^L(s_i|s,a_j)$ for $i= 1,\ldots,\abs{\mathcal{S}}$ and $j= 1,\ldots,\abs{\mathcal{A}}$. Let $\boldsymbol{B}(s) \in \mathbb{R}^{\abs{\mathcal{S}} \times \abs{\mathcal{A}} }$ be 
a row matrix that contains only ones in row $s$ and zero elsewhere. Then, we define the matrix $\boldsymbol{T} \in \mathbb{R}^{2\abs{\mathcal{S}} \times \abs{\mathcal{S}}\abs{\mathcal{A}}}$ by stacking the probability flow and the row matrices as follows: $\boldsymbol{T} ~:=~ \begin{bmatrix}
\boldsymbol{F}(s_1) & \boldsymbol{F}(s_2) & \dots & \boldsymbol{F}(s_{\abs{\mathcal{S}}}) \\
\boldsymbol{B}(s_1) & \boldsymbol{B}(s_2) & \dots & \boldsymbol{B}(s_{\abs{\mathcal{S}}}) \\
\end{bmatrix}$. In addition, we define the vector $\boldsymbol{v} \in \mathbb{R}^{2\abs{\mathcal{S}}}$ as follows: $\boldsymbol{v}_i = \rho(s_i) - (1 - \gamma)P_0(s_i)$ if $i \leq \abs{\mathcal{S}}$, and $1$ otherwise.


\begin{theorem}
\label{thm:occ_states}
The feasible set of the optimization problem~\eqref{opt_start} is non-empty iff the rank of the matrix $\boldsymbol{T}$ is equal to the rank of the augmented matrix $(\boldsymbol{T} | \boldsymbol{v})$. 
\end{theorem}
The proof of the above theorem leverages the fact that the Bellman flow constraints~\cite{boularias2011relative} must hold for any policy in an MDP. This requirement leads to the formulation of a linear system whose solutions set corresponds to the feasible set of~\eqref{opt_start}. The Rouché-Capelli theorem \cite{shafarevich2014linear}[Theorem~2.38] states that the solutions set is non-empty if and only if the condition in Theorem~\ref{thm:occ_states} holds. We note that the construction of the matrix $\boldsymbol{T}$ does not assume any restriction on the MDP structure since it leverages only on the Bellman flow constraints. Theorem ~\ref{thm:occ_states} allows us to develop a robust MCE IRL scheme in Section~\ref{sec:two_players_entropy} by ensuring the absence of duality gap. To this end, the following corollary provides a simple sufficient condition for the existence of a solution under transition dynamics mismatch. 
\begin{corollary}
\label{corollary}
Let $\abs{\mathcal{A}} > 1$. Then, a sufficient condition for the non-emptiness of the feasible set of the optimization problem~\eqref{opt_start} is given by $\boldsymbol{T}$ being full rank. 
\end{corollary} 

\subsection{Reward Transfer under Mismatch}
\label{sec:mismatch-learning}

Consider a class $\mathcal{M}$ of MDPs such that it contains both the learner and the expert environments, i.e., $M^L, M^E \in \mathcal{M}$ (see Figure~\ref{fig:MDPschematics}). We are given the expert's state occupancy measure $\boldsymbol{\rho} = \boldsymbol{\rho}^{\pi^{*}_{M_{\boldsymbol{\theta^*}}^E}}_{M^E}$; but the expert's policy $\pi^{*}_{M_{\boldsymbol{\theta^*}}^E}$ and the MDP $M^E$ are unknown. Further, we assume that every MDP $M \in \mathcal{M}$ satisfies the condition in Theorem~\ref{thm:occ_states}. 

We aim to find a policy $\pi^L$ that performs well in the MDP $M_{\boldsymbol{\theta^*}}^L$, i.e., $V^{\pi^L}_{M_{\boldsymbol{\theta^*}}^L}$ is high. To this end, we can choose any MDP $M^{\mathrm{train}} \in \mathcal{M}$, and solve the MCE IRL problem~\eqref{opt_start} with the constraint given by $\boldsymbol{\rho} = \boldsymbol{\rho}^{\pi}_{M^{\mathrm{train}}}$. Then, we always obtain a reward parameter $\boldsymbol{\theta}^\mathrm{train}$ s.t. $\boldsymbol{\rho} = \boldsymbol{\rho}^{\pi^{\mathrm{soft}}_{M^{\mathrm{train}}_{\boldsymbol{\theta}^\mathrm{train}}}}_{M^{\mathrm{train}}}$, since $M^{\mathrm{train}}$ satisfies the condition in Theorem~\ref{thm:occ_states}. We can use this reward parameter $\boldsymbol{\theta}^\mathrm{train}$ to learn a good policy $\pi^L$ in the MDP $M_{\boldsymbol{\theta}^\mathrm{train}}^L$, i.e., $\pi^L := \pi^*_{M_{\boldsymbol{\theta}^\mathrm{train}}^L}$ or $\pi^L := \pi^{\mathrm{soft}}_{M_{\boldsymbol{\theta}^\mathrm{train}}^L}$. Using Lemma~\ref{thm:first_week}, we obtain a bound on the performance gap between $\pi^L$ and $\pi_1 := \pi^{\mathrm{soft}}_{M_{\boldsymbol{\theta_L}}^L}$ (see Theorem~\ref{thm:new-reward-transfer-bound} in Appendix~\ref{app:reward-transfer}).

\begin{wrapfigure}{r}{0.5\textwidth}
\vspace{-\intextsep}
\includegraphics[width=0.5\textwidth]{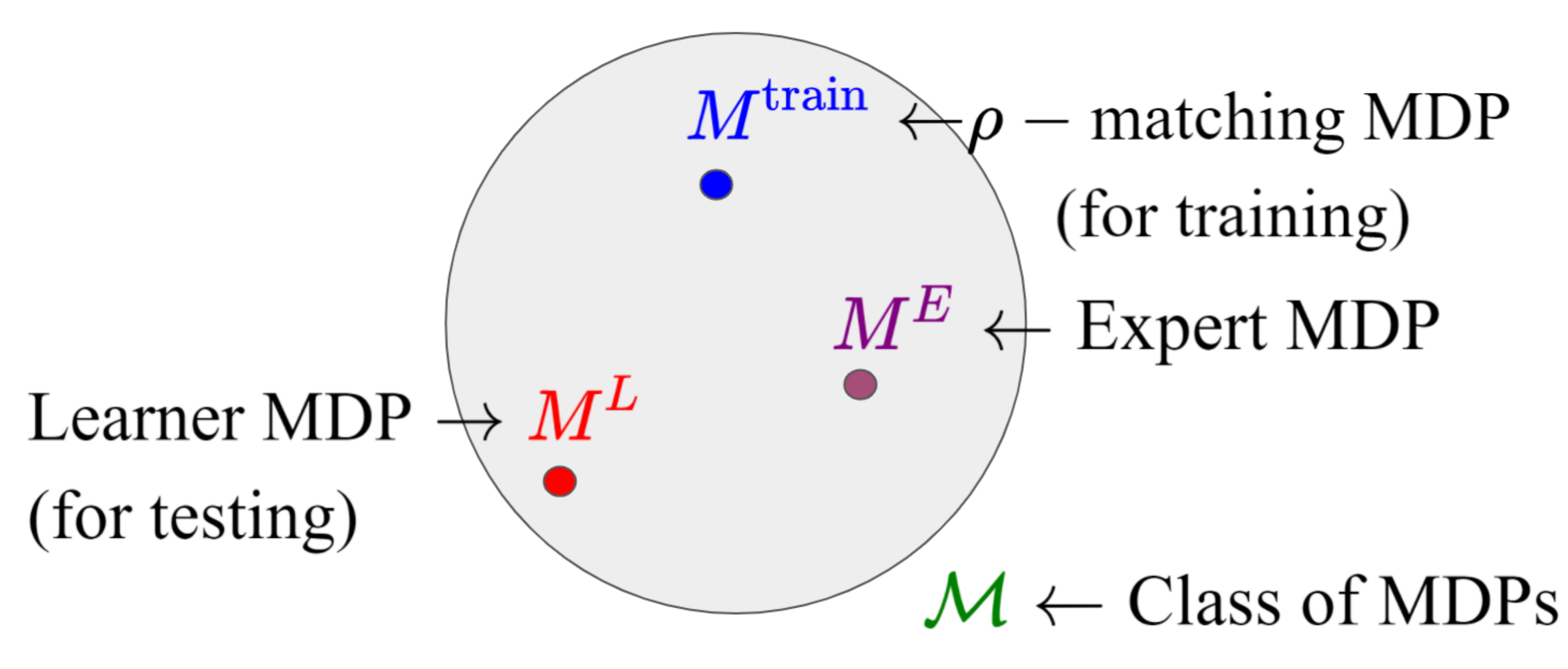}
\vspace{-\intextsep}
\caption{Illustrative example of learning a policy $\pi^L$ to act in one MDP $M^L$, given the expert occupancy measure $\boldsymbol{\rho}$. }
\vspace{-\intextsep}
\label{fig:MDPschematics}
\end{wrapfigure}

However, there are two problems with this approach: (i) it requires access to multiple environments $M^{\mathrm{train}}$, and (ii) unless $M^{\mathrm{train}}$ happened to be closer to the expert's MDP $M^E$, we cannot recover the true intention of the expert. Since the MDP $M^E$ is unknown, one cannot compare the different reward parameters $\boldsymbol{\theta}^\mathrm{train}$'s obtained with different MDPs $M^{\mathrm{train}}$'s. Thus, with 
$\boldsymbol{\theta}^\mathrm{train}$, it is impossible to ensure that the performance of $\pi^L$ is high in the MDP $M_{\boldsymbol{\theta^*}}^L$. Instead, we try to learn a robust policy $\pi^L$ over the class $\mathcal{M}$, while aligning with the expert's occupancy measure $\boldsymbol{\rho}$, and acting only in $M^L$. By doing this, we ensure that $\pi^L$ performs reasonably well on any MDP $M_{\boldsymbol{\theta^*}} \in \mathcal{M}$ including $M_{\boldsymbol{\theta^*}}^L$. We further build upon this idea in the next section.

\section{Robust MCE IRL via Two-Player Markov Game}
\label{sec:two_players_entropy}

\subsection{Robust MCE IRL Formulation}

This section focuses on recovering a learner policy via MCE IRL framework in a robust manner, under transition dynamics mismatch, i.e., $\boldsymbol{\rho} = \boldsymbol{\rho}^{\pi^{\mathrm{soft}}_{M_{\boldsymbol{\theta^*}}^E}}_{M^E}$ in Eq.~\eqref{opt_start}. In particular, our learner policy matches the expert state occupancy measure $\boldsymbol{\rho}$ under the most adversarial transition dynamics belonging to a set described as follows for a given $\alpha > 0$: $\mathcal{T}^{L,\alpha} := \bc{ \alpha T^L + (1 - \alpha) \bar{T}, \forall \bar{T} \in \Delta_{\mathcal{S} \mid \mathcal{S},\mathcal{A}}}$, where $\Delta_{\mathcal{S} \mid \mathcal{S},\mathcal{A}}$ is the set of all the possible transition dynamics $T: \mathcal{S} \times \mathcal{S} \times \mathcal{A} \rightarrow \bs{0,1}$. Note that the set $\mathcal{T}^{L,\alpha}$ is equivalent to the $(s,a)$-rectangular uncertainty set~\cite{iyengar2005robust} centered around $T^L$, i.e., $\mathcal{T}^{L,\alpha} = \bc{T: d_\mathrm{dyn} \br{T, T^L} \leq 2 (1-\alpha)}$. We need this set $\mathcal{T}^{L,\alpha}$ for establishing the equivalence between robust MDP and action-robust MDP formulations. The action-robust MDP formulation allows us to learn a robust policy while accessing only the MDP $M^L$. 

We define a class of MDPs as follows: $\mathcal{M}^{L,\alpha} ~:=~ \bc{\bc{\mathcal{S}, \mathcal{A}, T^{L,\alpha}, \gamma, P_0}, \forall T^{L,\alpha} \in \mathcal{T}^{L,\alpha}}$. Then, based on the discussions in Section~\ref{sec:mismatch-learning}, we propose the following robust MCE IRL problem:  
\begin{equation}
\max_{\pi^{\mathrm{pl}} \in \Pi} \min_{M \in \mathcal{M}^{L,\alpha}} ~ \E{\sum_{t = 0}^\infty - \gamma^t \log \pi^{\mathrm{pl}}(a_t | s_t) \biggm| \pi^{\mathrm{pl}}, M} ~~ \text{subject to} ~~ \boldsymbol{\rho}^{\pi^{\mathrm{pl}}}_{M} = \boldsymbol{\rho}
\label{eq:robust-mce-irl-primal-form}
\end{equation}
The corresponding dual problem is given by:
\begin{equation}
\min_{\boldsymbol{\theta}} \max_{\pi^{\mathrm{pl}} \in \Pi} \min_{M \in \mathcal{M}^{L,\alpha}}   \E{\sum_{t = 0}^{\infty} - \gamma^t \log \pi^{\mathrm{pl}}(a_t | s_t) \biggm| \pi^{\mathrm{pl}}, M} +  \boldsymbol{\theta}^\top \br{\boldsymbol{\rho}^{\pi^{\mathrm{pl}}}_{M} - \boldsymbol{\rho}} \label{eq:robust-mce-irl-dual-form}
\end{equation}
In the dual problem, for any $\boldsymbol{\theta}$, we attempt to learn a robust policy over the class $\mathcal{M}^{L,\alpha}$ with respect to the entropy regularized reward function. The parameter $\boldsymbol{\theta}$ plays the role of aligning the learner's policy with the expert's occupancy measure via constraint satisfaction. 

\subsection{Existence of Solution}
\label{sec:exist-robust-sol}

We start by formulating the IRL problem for any MDP $M^{L,\alpha} \in \mathcal{M}^{L,\alpha}$, with transition dynamics $T^{L,\alpha} = \alpha T^L + (1 - \alpha) \bar{T} \in \mathcal{T}^{L,\alpha}$, as follows:
\begin{equation}
\max_{\pi^{\mathrm{pl}} \in \Pi} ~ \E{\sum_{t = 0}^{\infty} - \gamma^t \log \pi^{\mathrm{pl}}(a_t | s_t) \biggm| \pi^{\mathrm{pl}} , M^{L,\alpha}} ~~ \text{subject to} ~~ \boldsymbol{\rho}^{\pi^{\mathrm{pl}}}_{M^{L,\alpha}} = \boldsymbol{\rho}
\label{primal_start_2}
\end{equation}
By introducing the Lagrangian vector $\boldsymbol{\theta} \in \mathbb{R}^{\abs{\mathcal{S}}}$, we get: 
\begin{equation} 
\max_{\pi^{\mathrm{pl}} \in \Pi} ~~ \E{\sum_{t = 0}^{\infty} - \gamma^t \log \pi^{\mathrm{pl}}(a_t | s_t) \biggm| \pi^{\mathrm{pl}} , M^{L,\alpha}} + \boldsymbol{\theta}^\top \br{\boldsymbol{\rho}^{\pi^{\mathrm{pl}}}_{M^{L,\alpha}} - \boldsymbol{\rho}} \label{primal_relaxed_2}
\end{equation}
For any fixed $\boldsymbol{\theta}$, the problem \eqref{primal_relaxed_2} is feasible since $\Pi$ is a closed and bounded set. We define $U(\boldsymbol{\theta})$ as the value of the program \eqref{primal_relaxed_2} for a given $\boldsymbol{\theta}$. By weak duality, $U(\boldsymbol{\theta})$ provides an upper bound on the optimization problem \eqref{primal_start_2}. Consequently, we introduce the dual problem aiming to find the value of $\boldsymbol{\theta}$ corresponding to the lowest upper bound, which can be written as 
\begin{equation}
 \min_{\boldsymbol{\theta}} U(\boldsymbol{\theta}) := \max_{\pi^{\mathrm{pl}} \in \Pi} ~ \E{\sum_{t = 0}^{\infty} - \gamma^t \log \pi^{\mathrm{pl}}(a_t | s_t) \biggm| \pi^{\mathrm{pl}} , M^{L,\alpha}} + \boldsymbol{\theta}^\top \br{\boldsymbol{\rho}^{\pi^{\mathrm{pl}}}_{M^{L,\alpha}} - \boldsymbol{\rho}}. \label{dual_2}
\end{equation}
Given $\boldsymbol{\theta}$, we define  $\pi^{\mathrm{pl}, \ast} := \pi^\mathrm{soft}_{M^{L,\alpha}_{\boldsymbol{\theta}}}$. Due to~\cite{geist2019regmdp}[Theorem~1], for any fixed $M^{L,\alpha}_{\boldsymbol{\theta}}$, the policy $\pi^{\mathrm{pl}, \ast}$ exists and it is unique. We can compute the gradient\footnote{In Appendix~\ref{sec:gradient}, we proved that this is indeed the gradient update under the transition dynamics mismatch.} $\nabla_{\boldsymbol{\theta}} U = \boldsymbol{\rho}^{\pi^{\mathrm{pl}, \ast}}_{M^{L,\alpha}} - \boldsymbol{\rho}$, and update the parameter via gradient descent: $\boldsymbol{\theta} \gets \boldsymbol{\theta} - \nabla_{\boldsymbol{\theta}} U$. Note that, if the condition in Theorem~\ref{thm:occ_states} holds, the feasible set of~\eqref{primal_start_2} is non-empty. Then, according to~\cite{bloem2014infinite}[Lemma~2], there is no duality gap between the programs~\eqref{primal_start_2} and \eqref{dual_2}. Based on these observations, we argue that the program~\eqref{eq:robust-mce-irl-primal-form} is well-posed and admits a unique solution. 

\subsection{Solution via Markov Game}
\label{sec:sol-markov-game}

\begin{algorithm}[t]
	\caption{Robust MCE IRL via Markov Game}
	\label{alg:MaxEntIRL}
	\begin{spacing}{0.8}
	\begin{algorithmic}
	    \STATE \textbf{Input:} opponent strength $1-\alpha$
	    \STATE \textbf{Initialize:} player policy $\pi^{\mathrm{pl}}$, opponent policy $\pi^{\mathrm{op}}$, and parameter $\boldsymbol{\theta}$
	    
		\WHILE{not converged}
		\STATE compute $\boldsymbol{\rho}^{\alpha \pi^{\mathrm{pl}} + (1 - \alpha) \pi^{\mathrm{op}}}_{ M^L}$  by dynamic programming~\cite{bloem2014infinite}[Section V.C].
		
		\STATE update $\boldsymbol{\theta}$ with Adam~\cite{kingma2014adam} using the gradient $\br{\boldsymbol{\rho}^{\alpha \pi^{\mathrm{pl}} + (1 - \alpha) \pi^{\mathrm{op}}}_{ M^L} - \boldsymbol{\rho}}$. 
		
		\STATE use Algorithm~\ref{alg:TwoplayersDynProg} with $R = R_{\boldsymbol{\theta}}$ to update $\pi^{\mathrm{pl}}$ and $\pi^{\mathrm{op}}$ s.t. they solve the problem~\eqref{objective}. 
    \ENDWHILE
    \STATE \textbf{Output:} player policy $\pi^{\mathrm{pl}}$
	\end{algorithmic}
	\end{spacing}
\end{algorithm}

In the following, we outline a method (see Algorithm~\ref{alg:MaxEntIRL}) to solve the robust MCE IRL dual problem~\eqref{eq:robust-mce-irl-dual-form}. To this end, for any given $\boldsymbol{\theta}$, we need to solve the inner max-min problem of~\eqref{eq:robust-mce-irl-dual-form}. First, we express the entropy term $\E{\sum_{t = 0}^{\infty} - \gamma^t \log \pi^{\mathrm{pl}}(a_t | s_t) \big| \pi^{\mathrm{pl}}, M}$ as follows:
\[
\sum_{s \in \mathcal{S}} \rho^{\pi^{\mathrm{pl}}}_{M}(s) \sum_{a \in \mathcal{A}} \bc{-\pi^{\mathrm{pl}}(a|s) \log \pi^{\mathrm{pl}}(a|s)} ~=~ \sum_{s \in \mathcal{S}} \rho^{\pi^{\mathrm{pl}}}_{M}(s) H^{\pi^{\mathrm{pl}}}\br{A \mid S = s} ~=~ \br{\boldsymbol{H}^{\pi^{\mathrm{pl}}}}^\top \boldsymbol{\rho}^{\pi^{\mathrm{pl}}}_{M}, 
\]
where $\boldsymbol{H}^{\pi^{\mathrm{pl}}} \in \mathbb{R}^{\abs{\mathcal{S}}}$ a vector whose $s^\text{th}$ element is the entropy of the player policy given the state $s$. Since the quantity $\boldsymbol{H}^{\pi^{\mathrm{pl}}} + \boldsymbol{\theta}$ depends only on the states, to solve the dual problem, we can utilize the equivalence between the \emph{robust MDP}~\cite{iyengar2005robust,nilim2005robust} formulation and the \emph{action-robust MDP}~\cite{pinto2017robust,tessler2019action,kamalaruban2020robust} formulation shown in~\cite{tessler2019action}. We can interpret the minimization over the environment class as the minimization over a set of opponent policies that with probability $1 - \alpha$ take control of the agent and perform the worst possible move from the current agent state. Indeed, interpreting $\br{\boldsymbol{H}^{\pi^{\mathrm{pl}}} +\boldsymbol{\theta}}^\top \boldsymbol{\rho}^{\pi^{\mathrm{pl}}}_{M}$ as an entropy regularized value function, i.e., $\boldsymbol{\theta}$ as a reward parameter, we can write:
\begin{align}
\max_{\pi^{\mathrm{pl}} \in \Pi} \min_{M \in \mathcal{M}^{L,\alpha}} \br{\boldsymbol{H}^{\pi^{\mathrm{pl}}} +\boldsymbol{\theta}}^\top \boldsymbol{\rho}^{\pi^{\mathrm{pl}}}_{M} ~=~& \max_{\pi^{\mathrm{pl}} \in \Pi} \min_{\bar{T}}  \E{G \bigm| \pi^{\mathrm{pl}}, P_0, \alpha T^L + (1 - \alpha) \bar{T}} \label{eq:T_minimizer}\\
~\leq~&\max_{\pi^{\mathrm{pl}} \in \Pi} \min_{\pi^{\mathrm{op}} \in \Pi} \E{G \bigm| \alpha \pi^{\mathrm{pl}} + (1 - \alpha) \pi^{\mathrm{op}}, M^L} , \label{equivalence_new}
\end{align}
where $G := \sum_{t=0}^{\infty} \gamma^t \bc{R_{\boldsymbol{\theta}}(s_t) + H^{\pi^{\mathrm{pl}}}\br{A \mid S = s_t}}$. 
The above inequality holds due to the derivation in section 3.1 of \cite{tessler2019action}. Further details are in Appendix~\ref{app:details-equivalence-new}.

Finally, we can formulate the problem~\eqref{equivalence_new} as a two-player zero-sum Markov game~\cite{littman1994markov} with transition dynamics given by $T^{\mathrm{two},L,\alpha}(s' | s, a^{\mathrm{pl}}, a^{\mathrm{op}}) = \alpha T^L(s' | s, a^{\mathrm{pl}}) + (1 - \alpha) T^L(s' | s, a^{\mathrm{op}})$, where $a^{\mathrm{pl}}$ is an action chosen according to the player policy and $a^{\mathrm{op}}$ according to the opponent policy. Note that the opponent is restricted to take the worst possible action from the state of the player, i.e., there is no additional state variable for the opponent. As a result, we reach a two-player Markov game with a regularization term for the player as follows:
\begin{equation}
\argmax_{\pi^{\mathrm{pl}} \in \Pi} \min_{\pi^{\mathrm{op}} \in \Pi} \E{G \bigm| \pi^{\mathrm{pl}}, \pi^{\mathrm{op}}, M^{\mathrm{two},L,\alpha}} , 
\label{objective}
\end{equation}
where $M^{\mathrm{two},L,\alpha} = \bc{\mathcal{S}, \mathcal{A}, \mathcal{A}, T^{\mathrm{two},L,\alpha}, \gamma, P_0, R_{\boldsymbol{\theta}}}$ is the two-player MDP associated with the above game. The repetition of the action space $\mathcal{A}$ denotes the fact that player and adversary share the same action space. Inspired from~\cite{grau2018balancing}, we propose a dynamic programming approach to find the player and opponent policies (see Algorithm~\ref{alg:TwoplayersDynProg} in Appendix~\ref{app:softQ}). 

\subsection{Performance Gap of Robust MCE IRL}
\label{sec:robust-mce-irl-bound}

Let $\pi^\mathrm{pl}$ be the policy returned by our Algorithm~\ref{alg:MaxEntIRL} when there is a transition dynamics mismatch. Recall that $\pi_1 := \pi^{\mathrm{soft}}_{M_{\boldsymbol{\theta_L}}^L}$ is the policy recovered without this mismatch. Then, we obtain the following upper-bound\footnote{This bound is worst than the one given in Theorem \ref{thm:regret_different_experts}. When the condition in Theorem~\ref{thm:occ_states} does not hold, the robust MCE IRL achieves a tighter bound than the MCE IRL for a proper choice of $\alpha$ (see Appendix~\ref{app:robust-mce-irl-upper-infeasible}).} for the performance gap of our algorithm via the triangle inequality:
\begin{theorem}
\label{thm:new-robust-mce-irl-bound}
The performance gap between the policies $\pi_1$ and $\pi^\mathrm{pl}$ on the MDP $M_{\boldsymbol{\theta^*}}^L$ is bounded as follows: $\abs{V^{\pi_1}_{M_{\boldsymbol{\theta^*}}^L} - V^{\pi^\mathrm{pl}}_{M_{\boldsymbol{\theta^*}}^L}} ~\leq~ \frac{\abs{R_{\boldsymbol{\theta^*}}}^{\mathrm{max}}}{(1 - \gamma)^2} \cdot \bc{\gamma \cdot d_\mathrm{dyn} \br{T^L, T^E} + 2 \cdot (1 - \alpha)}$. 
\end{theorem}

\begin{wrapfigure}{r}{0.5\textwidth}
\vspace{-\intextsep}
\includegraphics[width=0.5\textwidth]{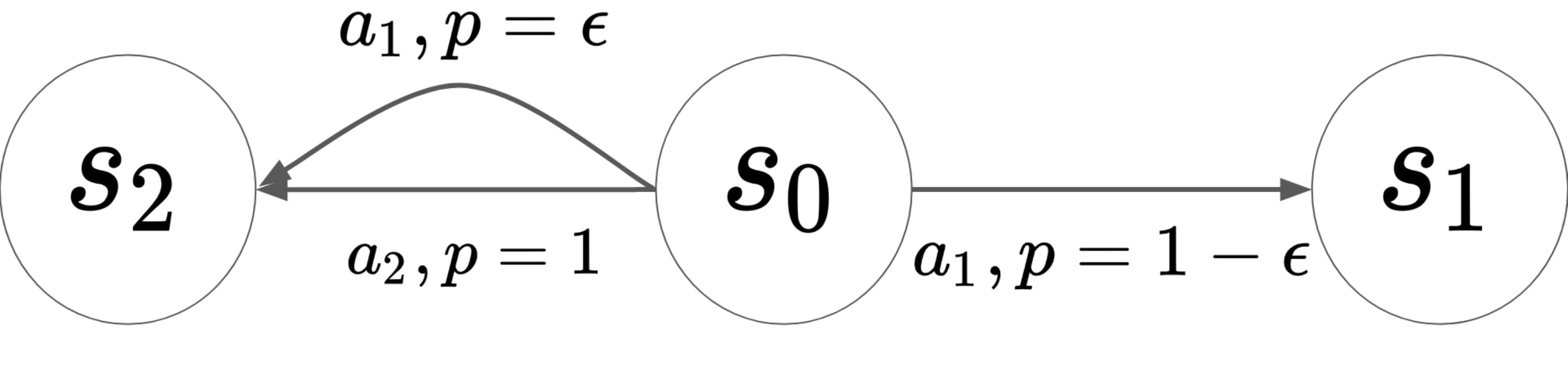}
\vspace{-\intextsep}
\caption{Constructive example to study the performance gap of Algorithm~\ref{alg:MaxEntIRL} and the MCE IRL.}
\vspace{-\intextsep}
\label{fig:worst_case_mdp}
\end{wrapfigure}

However, we now provide a constructive example, in which, by choosing the appropriate value for $\alpha$, the performance gap of our Algorithm~\ref{alg:MaxEntIRL} vanishes. In contrast, the performance gap of the standard MCE IRL is proportional to the mismatch. Note that our Algorithm~\ref{alg:MaxEntIRL} with $\alpha = 1$ corresponds to the standard MCE-IRL algorithm. 

Consider a reference MDP $M^{\br{\epsilon}} = \bc{\mathcal{S}, \mathcal{A}, T^{\br{\epsilon}}, \gamma, P_0}$ with variable $\epsilon$ (see Figure~\ref{fig:worst_case_mdp}). The state space is $\mathcal{S} = \bc{s_0, s_1, s_1}$, where $s_1$ and $s_2$ are absorbing states. The action space is $\mathcal{A} = \bc{ a_1, a_2}$ and the initial state distribution is $P_0\br{s_0} = 1$. The transition dynamics is defined as: $T^{\br{\epsilon}}(s_1| s_0, a_1) = 1 - \epsilon$, $T^{\br{\epsilon}}(s_2| s_0, a_1) = \epsilon$, $T^{\br{\epsilon}}(s_1| s_0, a_2) = 0$, and $T^{\br{\epsilon}}(s_2| s_0, a_2) = 1$. The true reward function is given by: $R_{\boldsymbol{\theta^*}}\br{s_0} = 0$, $R_{\boldsymbol{\theta^*}}\br{s_1} = 1$, and $R_{\boldsymbol{\theta^*}}\br{s_2} = -1$. We define the learner and the expert environment as: $M^L := M^{\br{0}}$ and $M^L := M^{\br{\epsilon_E}}$. Note that the distance between the two transition dynamics is $d_\mathrm{dyn} \br{T^L, T^E} = 2 \epsilon_E$. Let $\pi^\mathrm{pl}$ and $\pi_2 := \pi^{\mathrm{soft}}_{M_{\boldsymbol{\theta_E}}^L}$ be the policies returned by Algorithm~\ref{alg:MaxEntIRL} and the MCE IRL algorithm, under the above mismatch. Recall that $\pi_1$ is the policy recovered by the MCE IRL algorithm without this mismatch. Then, the following holds: 

\begin{theorem}
\label{theorem-tightness}
For this example, the performance gap of Algorithm~\ref{alg:MaxEntIRL} vanishes by choosing $\alpha = 1 - \frac{d_\mathrm{dyn} (T^L, T^E)}{2}$, i.e., $\abs{V^{\pi_1}_{M_{\boldsymbol{\theta^*}}^L} - V^{\pi^\mathrm{pl}}_{M_{\boldsymbol{\theta^*}}^L}} = 0$. Whereas, the performance gap of the standard MCE IRL is given by: $\abs{V^{\pi_1}_{M^L_{\boldsymbol{\theta^*}}} - V^{\pi_2}_{M^L_{\boldsymbol{\theta^*}}}} ~=~ \frac{\gamma}{1 - \gamma} \cdot d_\mathrm{dyn} (T^L, T^E)$.
\end{theorem}

\section{Experiments}
\label{sec:experiments}

This section demonstrates the superior performance of our Algorithm~\ref{alg:MaxEntIRL} compared to the standard MCE IRL algorithm, when there is a transition dynamics mismatch between the expert and the learner. All the missing figures and hyper-parameter details are reported in Appendix~\ref{appendix:experiments}. 


\textbf{Setup.} 
Let $M^{\mathrm{ref}}_{\boldsymbol{\theta^*}} = \br{\mathcal{S}, \mathcal{A}, T^{\mathrm{ref}}, \gamma, P_0, R_{\boldsymbol{\theta^*}}}$ be a reference MDP. Given a \emph{learner noise} $\epsilon_L \in \bs{0,1}$, we introduce a learner MDP without reward function as $M^{L,\epsilon_L} = \br{\mathcal{S}, \mathcal{A}, T^{L, \epsilon_L}, \gamma, P_0}$, where $T^{L, \epsilon_L} \in \Delta_{\mathcal{S}|\mathcal{S},\mathcal{A}}$ is defined as $T^{L, \epsilon_L} := (1 - \epsilon_L) T^{\mathrm{ref}} + \epsilon_L \bar{T}$ with $\bar{T} \in \Delta_{\mathcal{S}|\mathcal{S},\mathcal{A}}$. Similarly, given an \emph{expert noise} $\epsilon_E \in \bs{0,1}$, we define an expert MDP $M^{E,\epsilon_E}_{\boldsymbol{\theta
^*}} = \br{\mathcal{S}, \mathcal{A}, T^{E, \epsilon_E}, \gamma, P_0, R_{\boldsymbol{\theta^*}}}$, where $T^{E, \epsilon_E} \in \Delta_{\mathcal{S}|\mathcal{S},\mathcal{A}}$ is defined as $T^{E, \epsilon_E} := (1 - \epsilon_E) T^{\mathrm{ref}} + \epsilon_E \bar{T}$ with $\bar{T} \in \Delta_{\mathcal{S}|\mathcal{S},\mathcal{A}}$. Note that a pair $(\epsilon_E,\epsilon_L)$ corresponds to an IRL problem under dynamics mismatch, where the expert acts in the MDP $M^{E,\epsilon_E}_{\boldsymbol{\theta
^*}}$ and the learner in $M^{L,\epsilon_L}$. In our experiments, we set $T^{\mathrm{ref}}$ to be deterministic, and  $\bar{T}$ to be uniform. Then, one can easily show that $d_\mathrm{dyn} \br{T^{L,\epsilon_L}, T^{E, \epsilon_E}} = 2\br{1 - \frac{1}{\abs{\mathcal{S}}}}\abs{\epsilon_L - \epsilon_E}$. The learned policies are evaluated in the MDP $M^{L, \epsilon_L}_{\boldsymbol{\theta^*}}$, i.e., $M^{L, \epsilon_L}$ endowed with the true reward function $R_{\boldsymbol{\theta^*}}$.

\textbf{Baselines.}
We are not aware of any comparable prior IRL work that exactly matches our setting: (i) only one shot access to the expert environment, and (ii) do not explicitly model the expert environment. Note that Algorithm 2 in~\cite{chen2016adversarial} requires online access to $T^E$ (or the expert environment) to empirically estimate the gradient for every (time step) adversarial expert policy $\check{\pi}^*$, whereas we do not access the expert environment after obtaining a batch of demonstrations, i.e., $\boldsymbol{\rho}$. Thus, for each pair $(\epsilon_E,\epsilon_L)$, we compare the performance of the following: (i) our robust MCE IRL algorithm with different values of $\alpha \in \bc{0.8, 0.85, 0.9, 0.95}$, (ii) the standard MCE IRL algorithm, and (iii) the ideal baseline that utilizes the knowledge of the true reward function, i.e, $\pi^*_{M^{L,\epsilon_L}_{\boldsymbol{\theta^*}}}$.

\begin{figure*}[t] 
\centering
\vspace{-\intextsep}
\begin{subfigure}{0.24\textwidth}
\includegraphics[width=\linewidth]{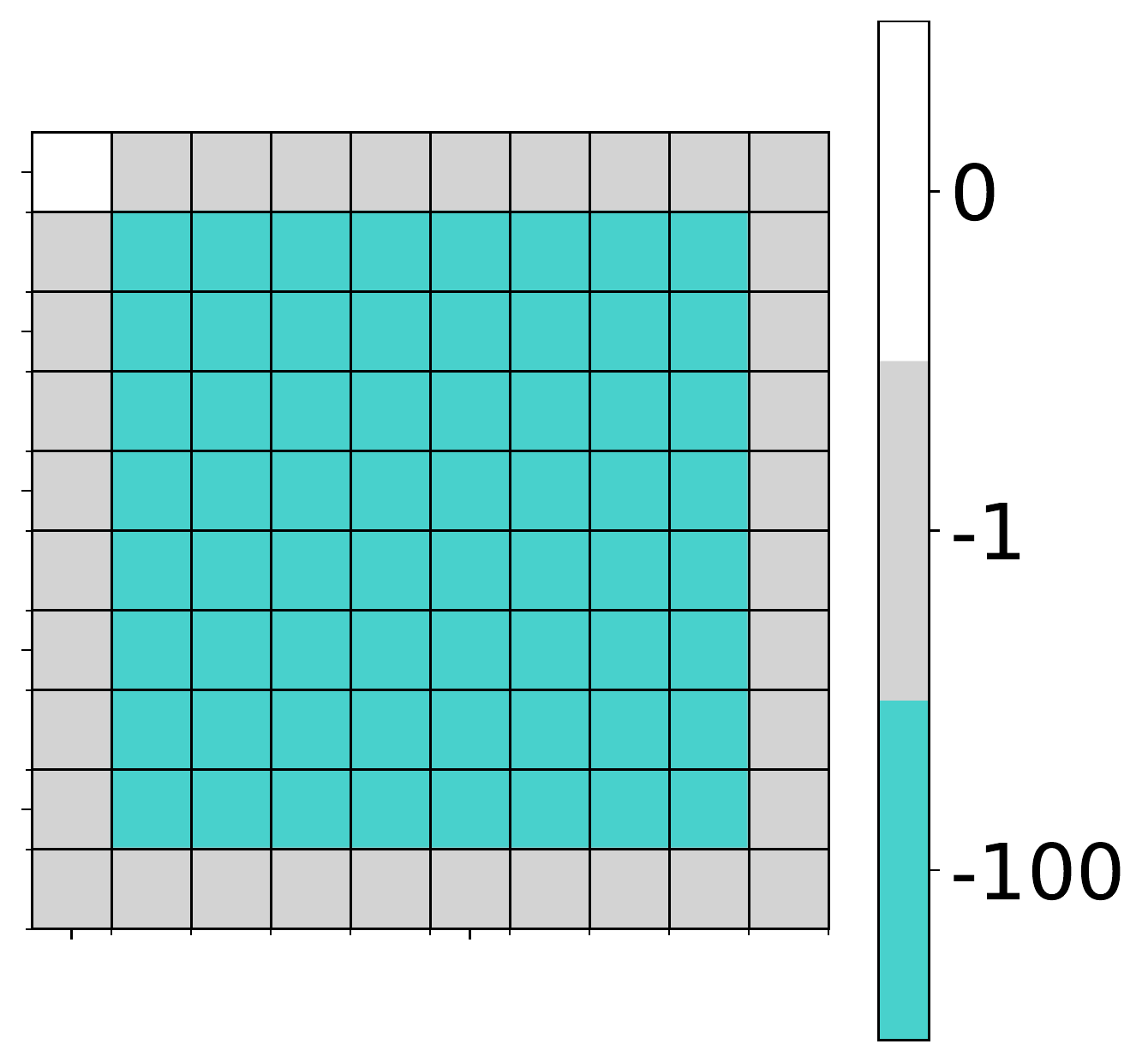}
\caption{\textsc{GridWorld-1}} \label{fig:grid1}
\end{subfigure}\hspace*{\fill}
\begin{subfigure}{0.24\textwidth}
\includegraphics[width=\linewidth]{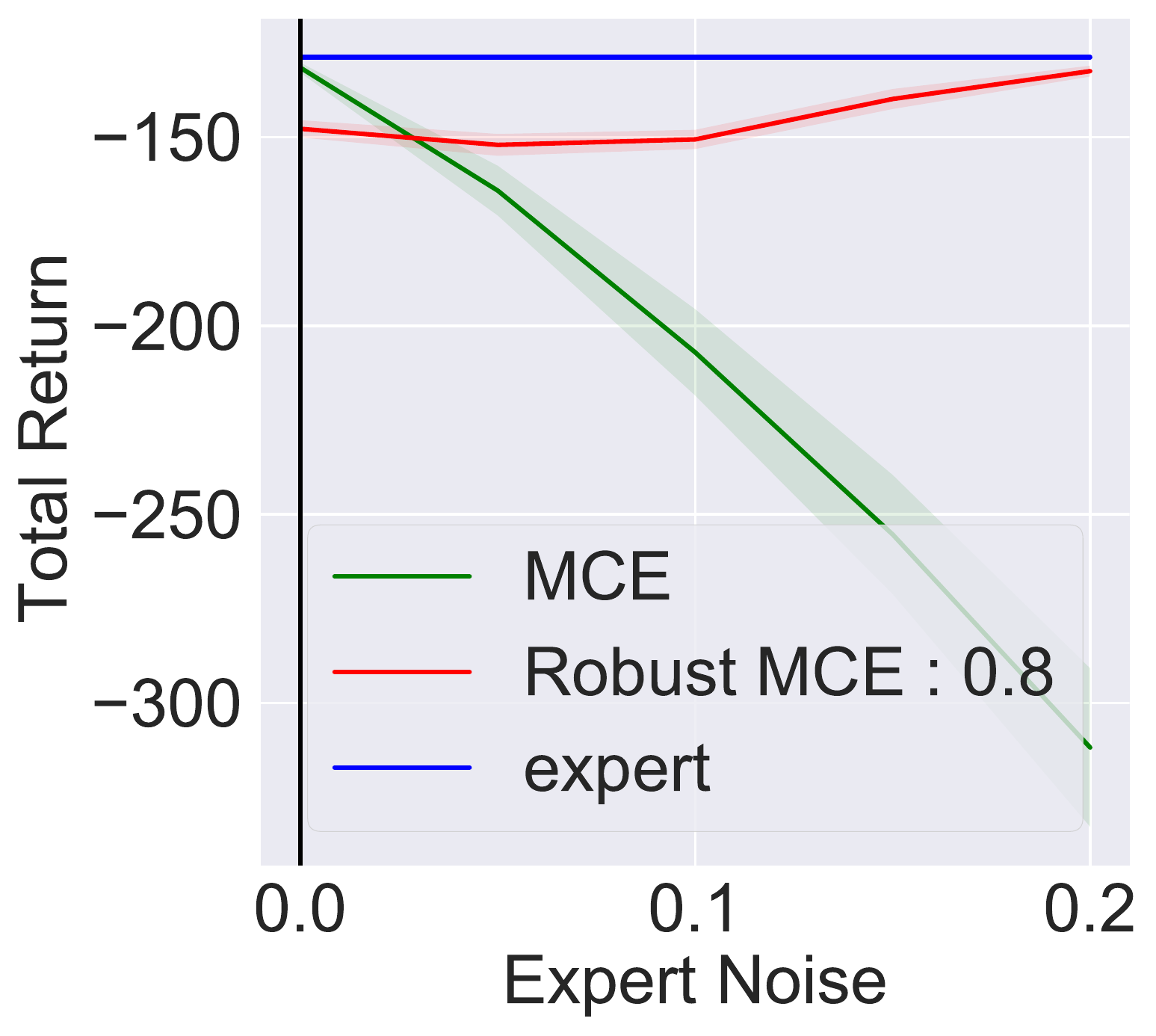}
\caption{\textsc{GrW} $\epsilon_L = 0$} \label{fig:grid1_noise_0}
\end{subfigure}\hspace*{\fill}
\begin{subfigure}{0.24\textwidth}
\includegraphics[width=\linewidth]{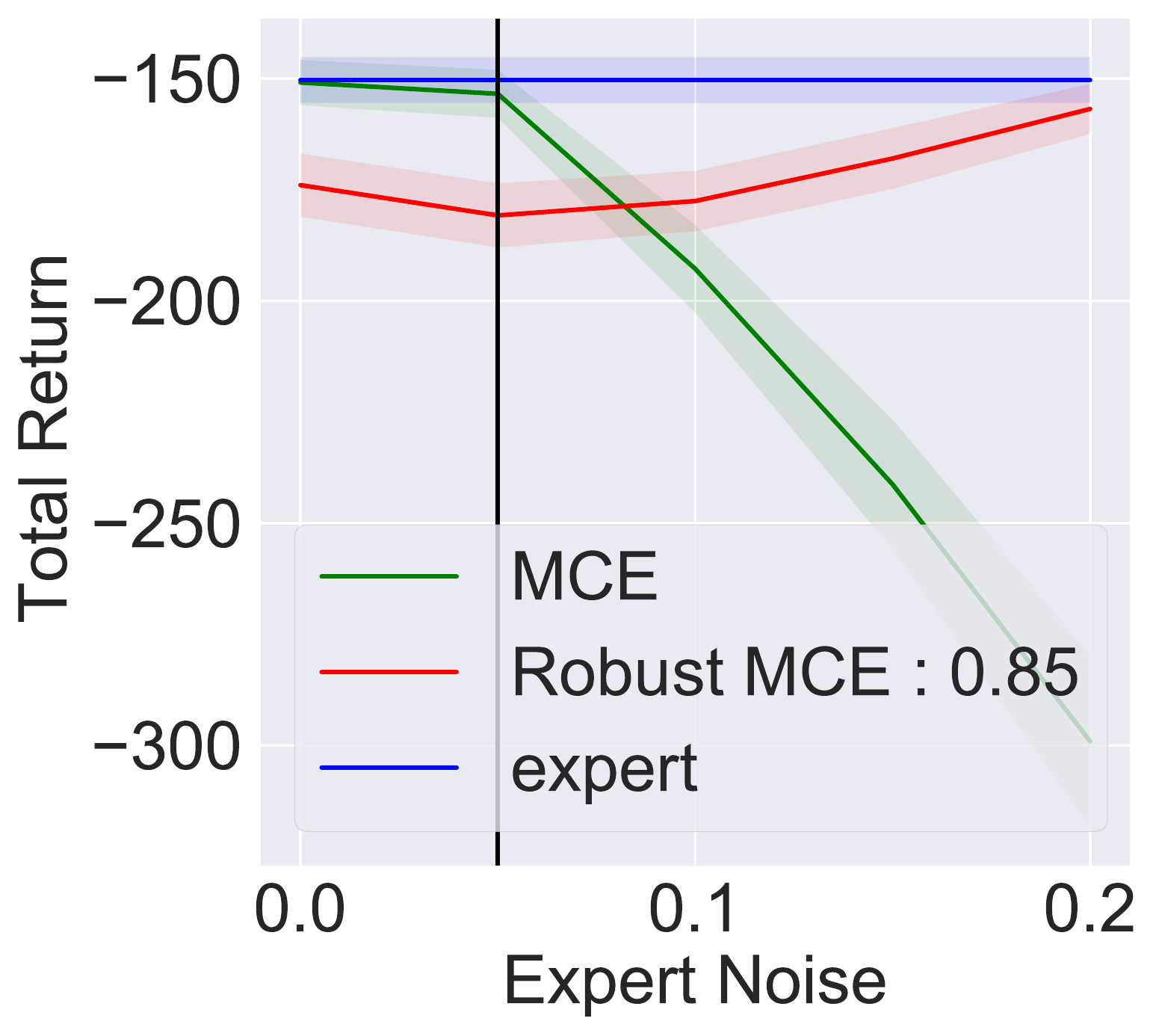}
\caption{\textsc{GrW} $\epsilon_L = 0.05$} \label{c}
\end{subfigure}\hspace*{\fill}
\begin{subfigure}{0.24\textwidth}
\includegraphics[width=\linewidth]{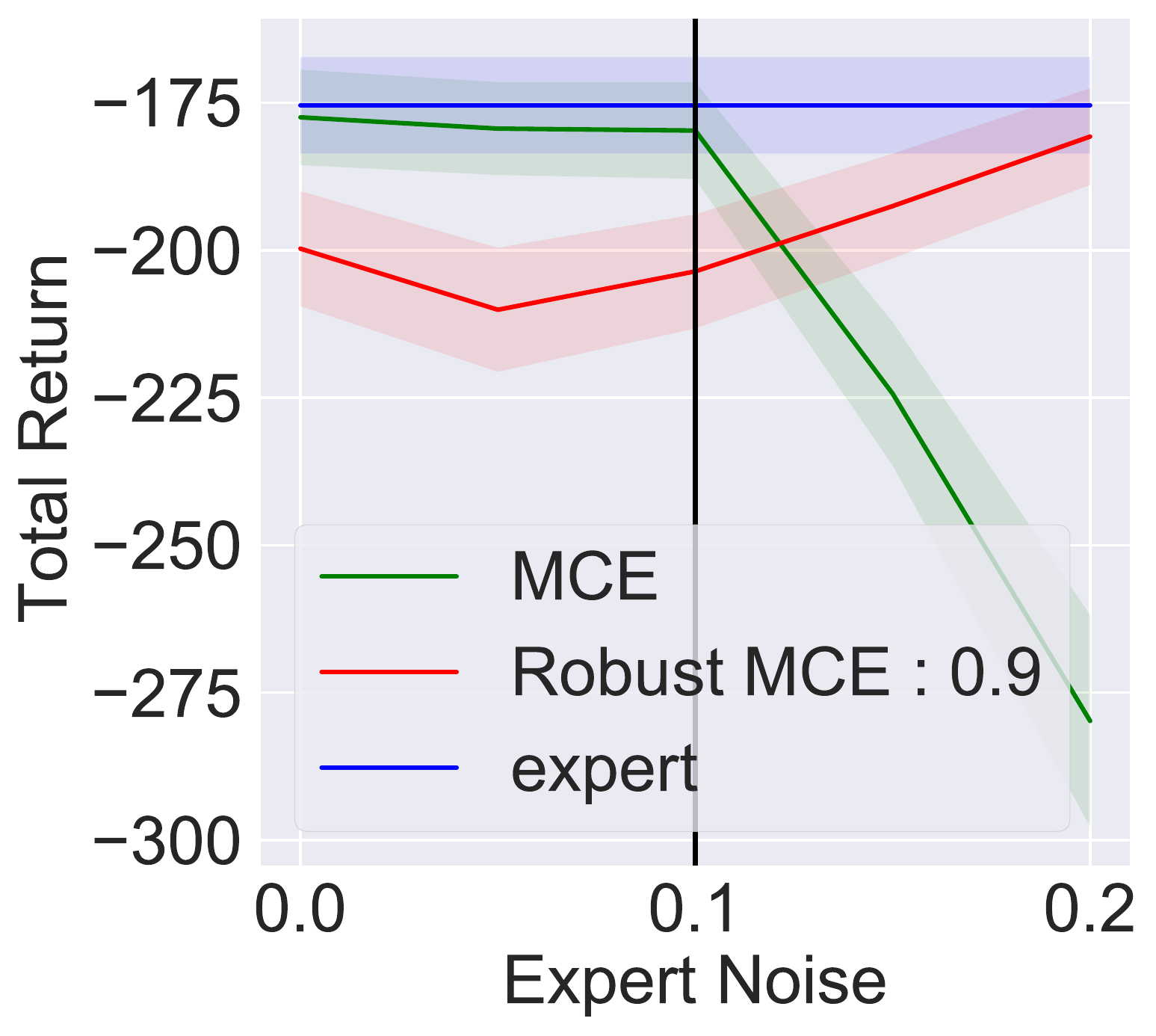}
\caption{\textsc{GrW} $\epsilon_L = 0.1$} \label{d}
\end{subfigure}
\medskip
\begin{subfigure}{0.24\textwidth}
\includegraphics[width=\linewidth]{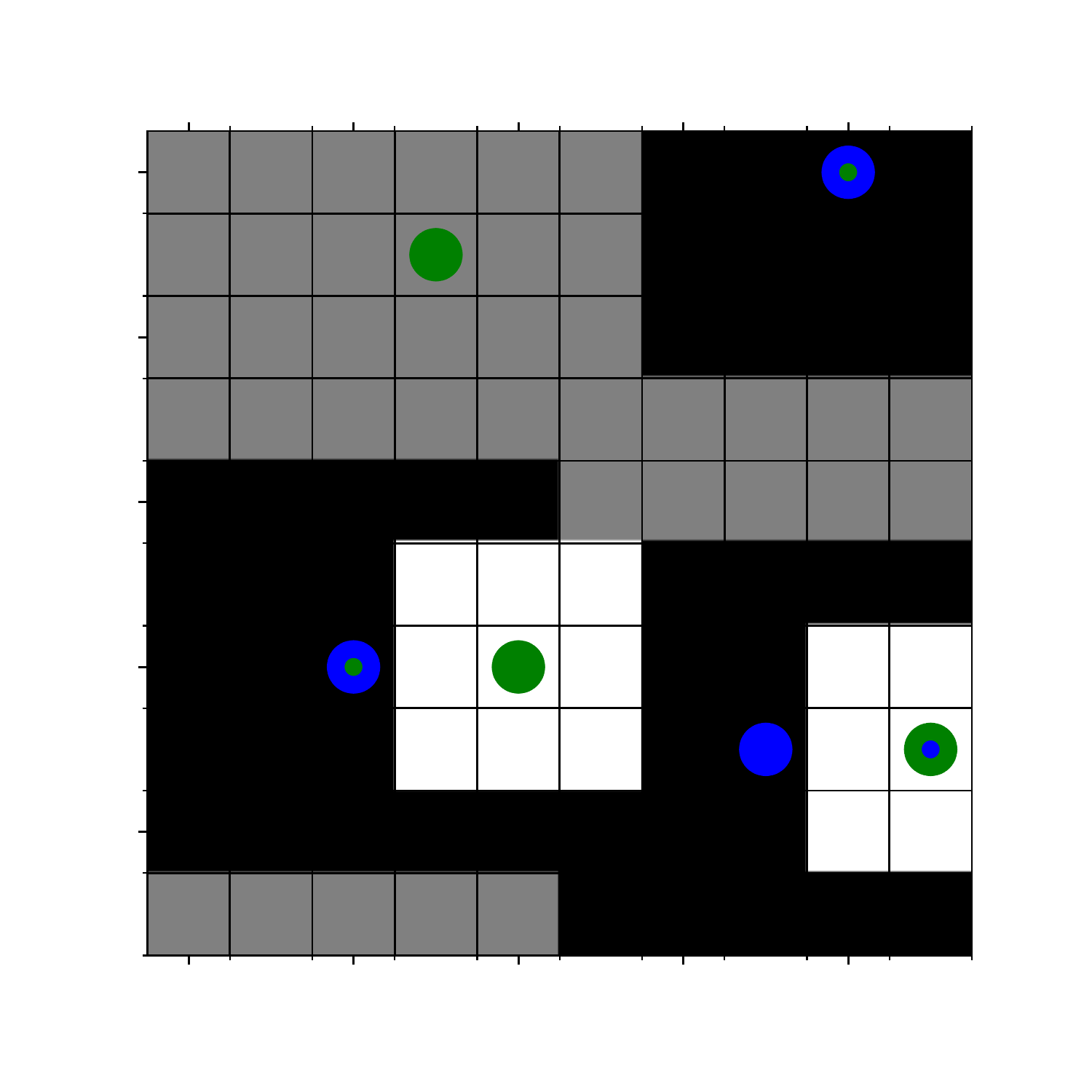}
\caption{\textsc{ObjectWorld}} \label{fig:obj_world_main}
\end{subfigure}\hspace*{\fill}
\begin{subfigure}{0.24\textwidth}
\includegraphics[width=\linewidth]{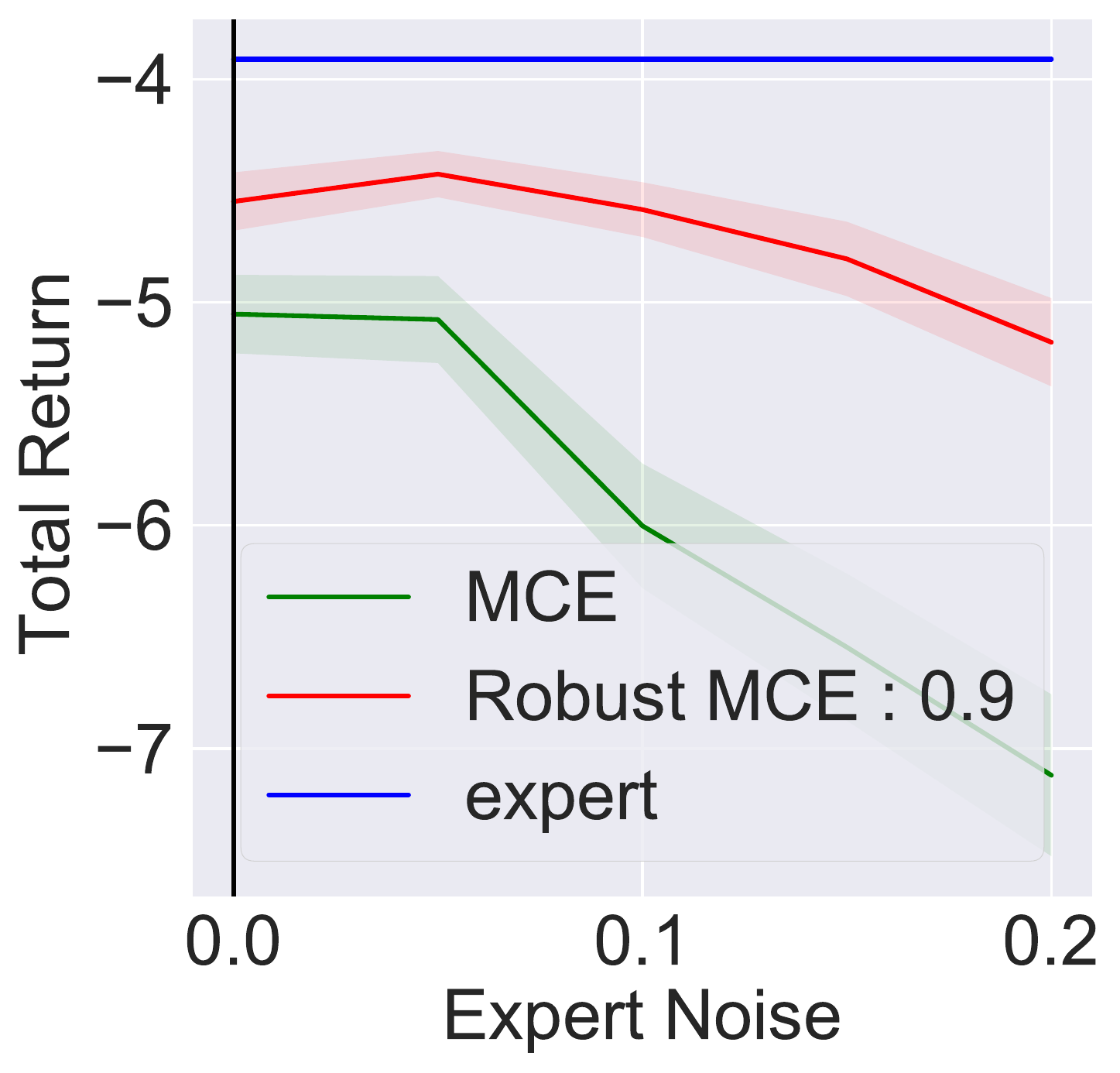}
\caption{\textsc{ObW} $\epsilon_L = 0$} \label{f}
\end{subfigure}\hspace*{\fill}
\begin{subfigure}{0.24\textwidth}
\includegraphics[width=\linewidth]{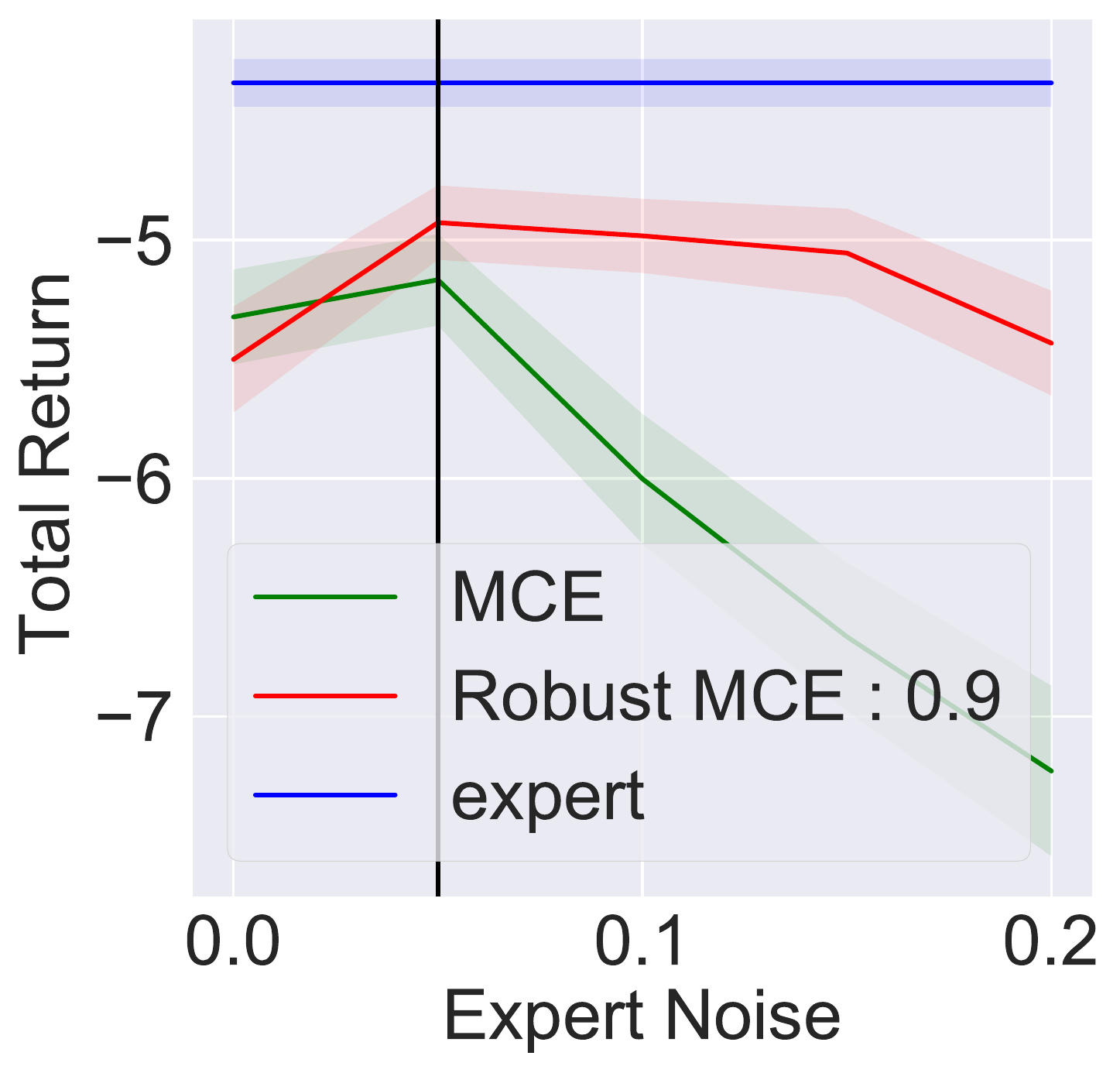}
\caption{\textsc{ObW} $\epsilon_L = 0.05$} \label{g}
\end{subfigure}\hspace*{\fill}
\begin{subfigure}{0.24\textwidth}
\includegraphics[width=\linewidth]{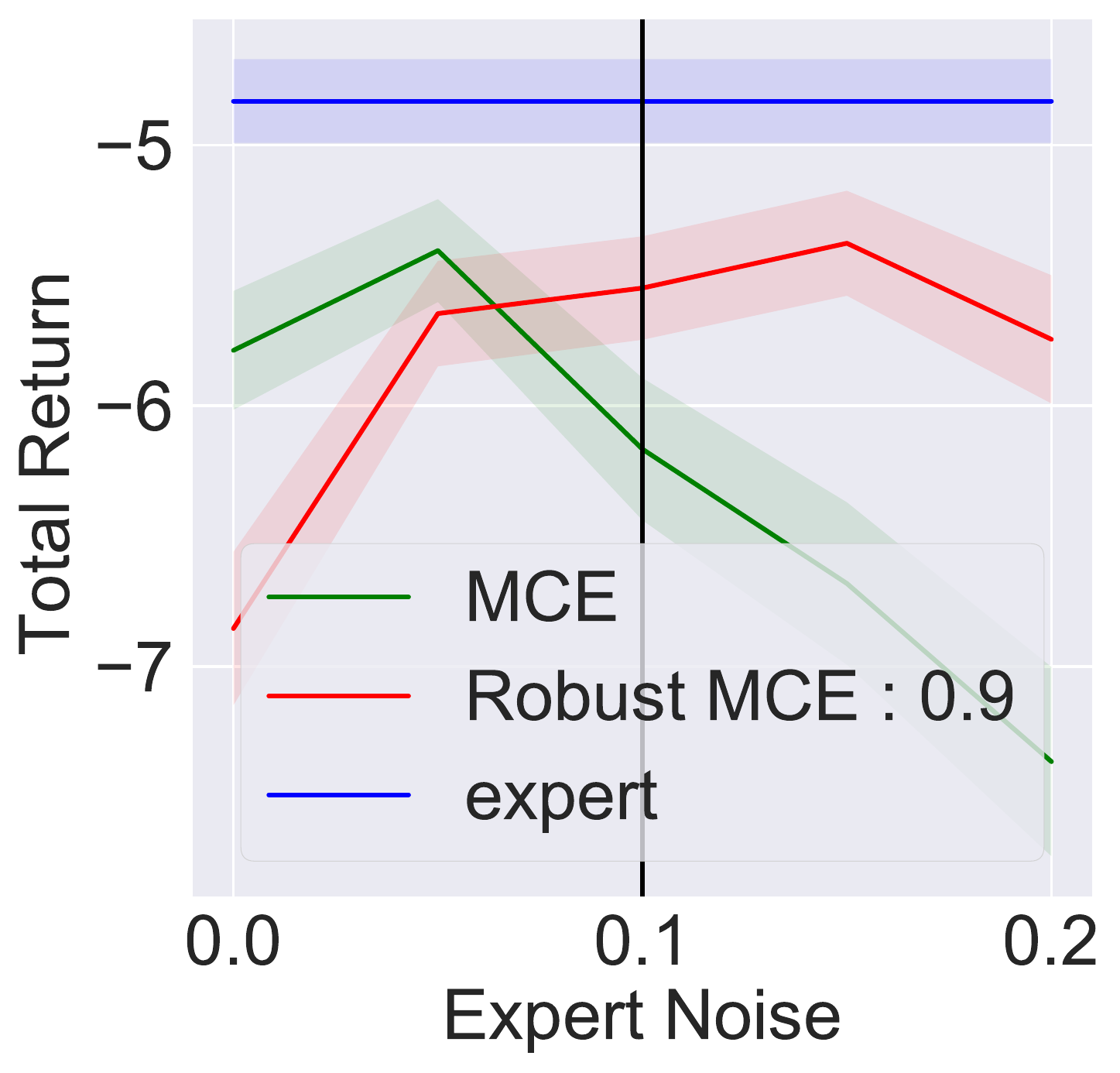}
\caption{\textsc{ObW} $\epsilon_L = 0.1$} \label{h}
\end{subfigure}
\caption{Comparison of the performance our Algorithm~\ref{alg:MaxEntIRL} against the baselines, under different levels of mismatch: $\br{\epsilon_E, \epsilon_L} \in \bc{0.0, 0.05, 0.1, 0.15, 0.2} \times \bc{ 0.0, 0.05, 0.1}$. Each plot corresponds to a fixed leaner environment $M^{L,\epsilon_L}$ with $\epsilon_L \in \bc{ 0.0, 0.05, 0.1}$. The values of $\alpha$ used for Algorithm~\ref{alg:MaxEntIRL} are reported in the legend. The vertical line indicates the position of the learner environment in the x-axis. We abbreviated the environment names as \textsc{GrW}, and \textsc{ObW}. Note that our Robust MCE IRL outperforms standard MCE IRL when the expert noise increases along the x-axis. At the same time, Robust MCE IRL might perform slightly worse in the low expert noise regime. This observation aligns with the overly conservative nature of robust training methods.}
\vspace{-\intextsep}
\label{fig:main_paper_best_alpha}
\end{figure*}

\textbf{Environments.}  
We consider four \textsc{GridWorld} environments and an \textsc{ObjectWorld}~\cite{levine2011nonlinear} environment. All of them are $N \times N$ grid, where a cell represents a state. There are four actions per state, corresponding to steps in one of the four cardinal directions; $T^{\mathrm{ref}}$ is defined accordingly. \textsc{GridWorld} environments are endowed with a linear reward function $R_{\boldsymbol{\theta^*}}(s) = \ip{\boldsymbol{\theta^*}}{\boldsymbol{\phi}(s)}$, where $\boldsymbol{\phi}$ is a one-hot feature map. The entries $\boldsymbol{\theta^*_s}$ of the parameter $\boldsymbol{\theta^*}$ for each state $s \in \mathcal{S}$ are shown in Figures~\ref{fig:grid1},~\ref{fig:grid2},~\ref{fig:grid3},~and~\ref{fig:grid4}. \textsc{ObjectWorld} is endowed with a non-linear reward function, determined by the distance of the agent to the objects that are randomly placed in the environment. Each object has an outer and an inner color; however, only the former plays a role in determining the reward while the latter serves as a distractor. The reward is $-2$ in positions within three cells to an outer blue object (black areas of Figure~\ref{fig:obj_world_main}), $0$ if they are also within two cells from an outer green object (white areas), and $-1$ otherwise (gray areas). We shift the rewards originally proposed by~\cite{levine2011nonlinear} to non-positive values, and we randomly placed the goal state in a white area. We also modify the reward features by augmenting them with binary features indicating whether the goal state has been reached. These changes simplify the application of the MCE IRL algorithm in the infinite horizon setting. For this non-linear reward setting, we used the deep MCE IRL algorithm from~\cite{wulfmeier2015maximum}, where the reward function is parameterized by a neural network. 


\textbf{Results.} 
\looseness-1In Figure~\ref{fig:main_paper_best_alpha}, we have presented the results for two of the environments, and the complete results can be found in Figure~\ref{fig:all_gridworld_best_alpha}. Also, in Figure~\ref{fig:main_paper_best_alpha}, we have reported the results of our algorithm with the best performing value of $\alpha$; and the performance of our algorithm with different values of $\alpha$ are presented in Figure~\ref{fig:gridworld_diff_alpha}. In all the plots, every point in the x-axis corresponds to a pair $\br{\epsilon_E, \epsilon_L}$. For example, consider Figure~\ref{fig:grid1_noise_0}, for a fixed learner environment $M^{L,\epsilon_L}$ with $\epsilon_L = 0$, and different expert environments $M^{E,\epsilon_E}$ by varying $\epsilon_E$ along the x-axis. Note that, in this figure, the distance $d_\mathrm{dyn} \br{T^{L,\epsilon_L}, T^{E, \epsilon_E}} \propto \abs{\epsilon_L - \epsilon_E}$ increases along the x-axis. For each pair $\br{\epsilon_E, \epsilon_L}$, in the y-axis, we present the performance of the learned polices in the MDP $M^{L, \epsilon_L}_{\boldsymbol{\theta^*}}$, i.e., $V^\pi_{M^{L, \epsilon_L}_{\boldsymbol{\theta^*}}}$. In alignment with our theory, the performance of the standard MCE IRL algorithm degrades along the x-axis. Whereas, our Algorithm~\ref{alg:MaxEntIRL} resulted in robust performance (even closer to the ideal baseline) across different levels of mismatch. These results confirm the efficacy of our method under mismatch. However, one has to carefully choose the value of $1-\alpha$ (s.t. $T^{E, \epsilon_E} \in \mathcal{T}^{L, \alpha}$): (i) underestimating it would lead to a linear decay in the performance, similar to the MCE IRL, (ii) overestimating it would also slightly hinder the performance, and (iii) given a rough estimate $\widehat T^E$ of the expert dynamics, choosing $1-\alpha \approx \frac{d_\mathrm{dyn} \br{T^L, \widehat T^E}}{2}$ would lead to better performance in practice. The potential drop in the performance of our Robust MCE IRL method under the low expert noise regime (see Figures~\ref{c},~\ref{d},~and~\ref{h}) can be related to the overly conservative nature of robust training. See Appendix~\ref{app:alpha-choice} for more discussion on the choice of $1-\alpha$. In addition, we have tested our method on a setting with low-dimensional feature mapping $\boldsymbol{\phi}$, where we observed significant improvement over the standard MCE IRL (see Appendix~\ref{app:low-dim-exp}).  

\section{Extension to Continuous MDP Setting}
\label{sec:experiments_continuous_control_main}

In this section, we extend our ideas to the continuous MDP setting, i.e., the environments with continuous state and action spaces. In particular, we implement a robust variant of the Relative Entropy IRL (RE IRL)~\cite{boularias2011relative} algorithm (see Algorithm~\ref{alg:RobustREIRL} in Appendix~\ref{sec:experiments_continuous_control}). We cannot use the dynamic programming approach to find the player and opponent policies in the continuous MDP setting. Therefore, we solve the two-player Markov game in a model-free manner using the policy gradient methods (see Algorithm~\ref{alg:TwoplayersPolicyGrad} in Appendix~\ref{sec:experiments_continuous_control}).

We evaluate the performance of our Robust RE IRL method on a continuous gridworld environment that we called \textsc{GaussianGrid}. The details of the environment and the experimental setup are given in Appendix~\ref{sec:experiments_continuous_control}. The results are reported in Figure~\ref{fig:gaussian-gridworld}, where we notice that our Robust RE IRL method outperforms standard RE IRL.

\begin{figure}[t!] 
\centering
\vspace{-\intextsep}
\begin{subfigure}{0.24\textwidth}
\includegraphics[width=0.99\linewidth]{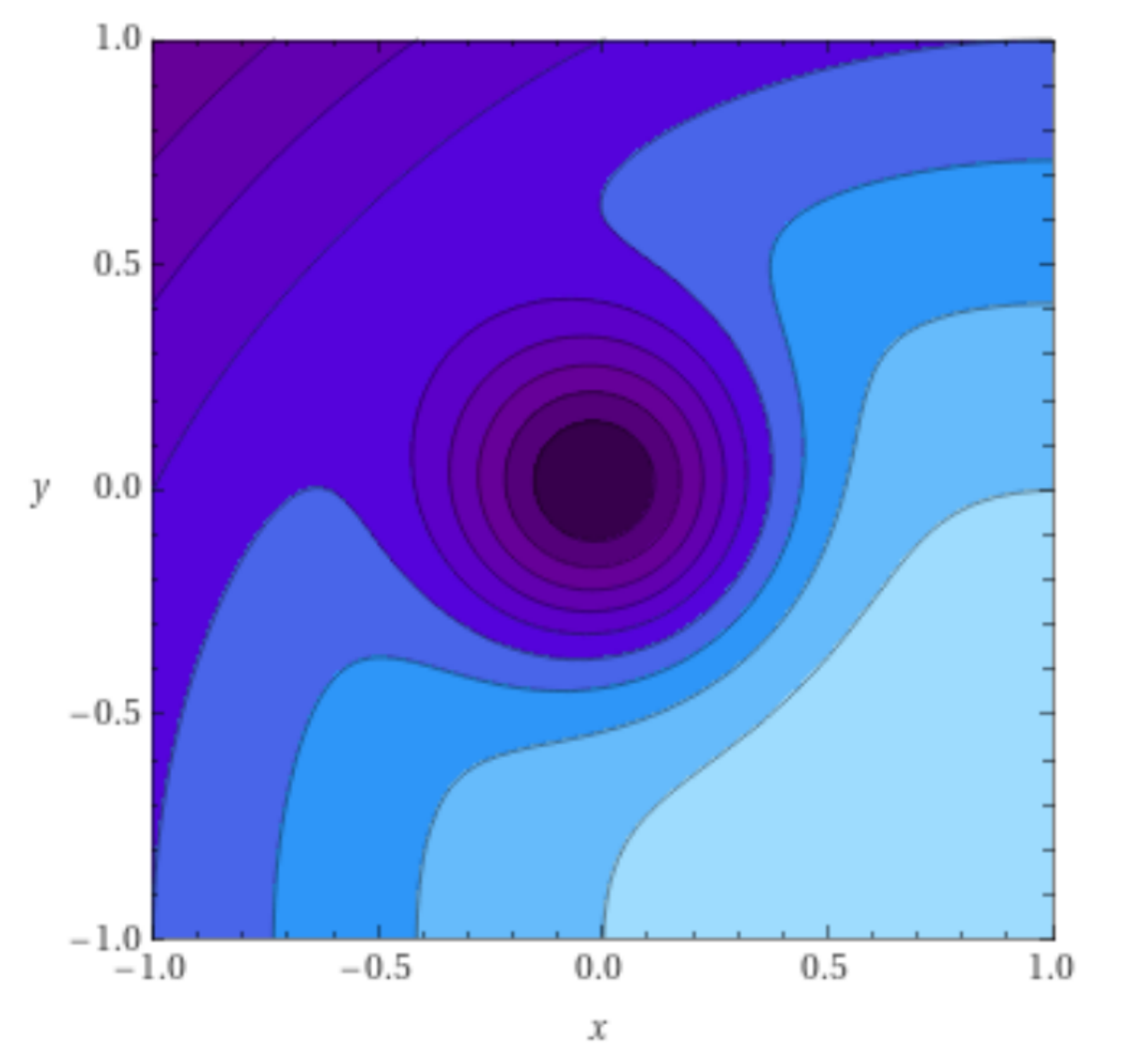}
\caption{\textsc{GaussianGrid}} \label{fig:gauss_grid}
\end{subfigure}
\begin{subfigure}{0.24\textwidth}
\includegraphics[width=\linewidth]{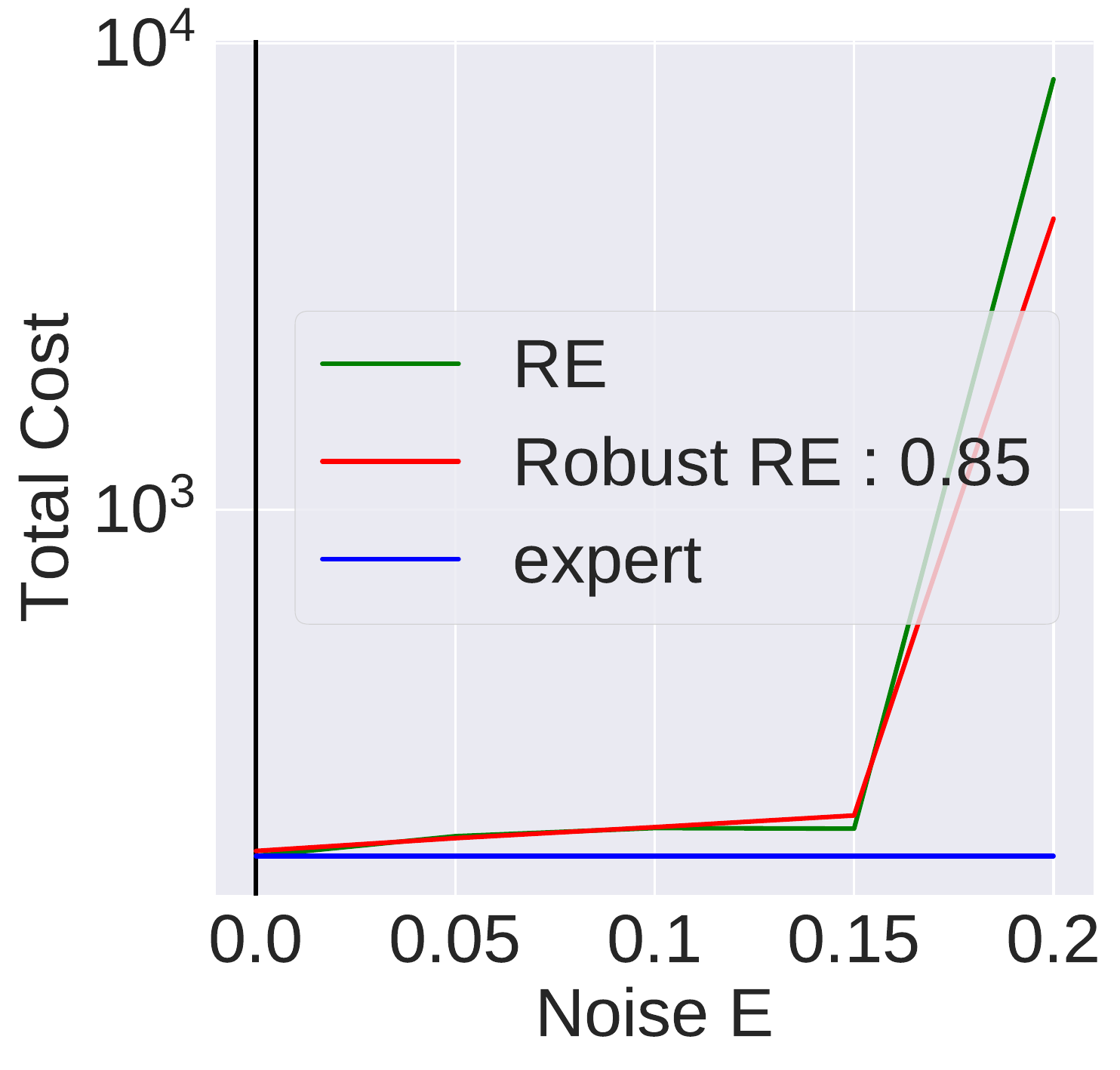}
\caption{$M^{L,\epsilon_L}$ with $\epsilon_L = 0$} \label{fig:gauss_gridl0.0}
\end{subfigure}\hspace*{\fill}
\begin{subfigure}{0.24\textwidth}
\includegraphics[width=\linewidth]{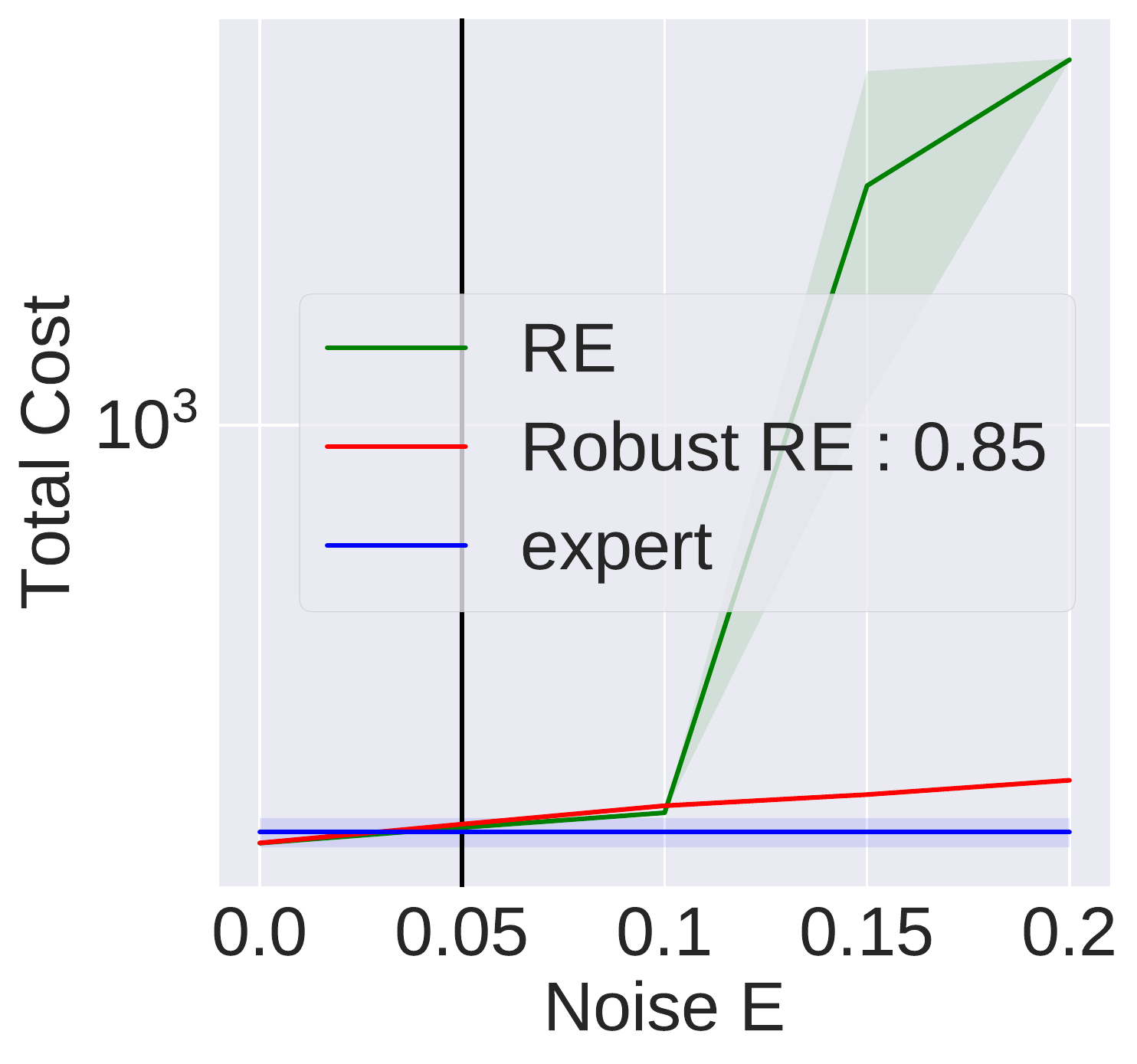}
\caption{$M^{L,\epsilon_L}$ with $\epsilon_L = 0.05$} \label{fig:gauss_gridl0.05}
\end{subfigure}\hspace*{\fill}
\begin{subfigure}{0.24\textwidth}
\includegraphics[width=\linewidth]{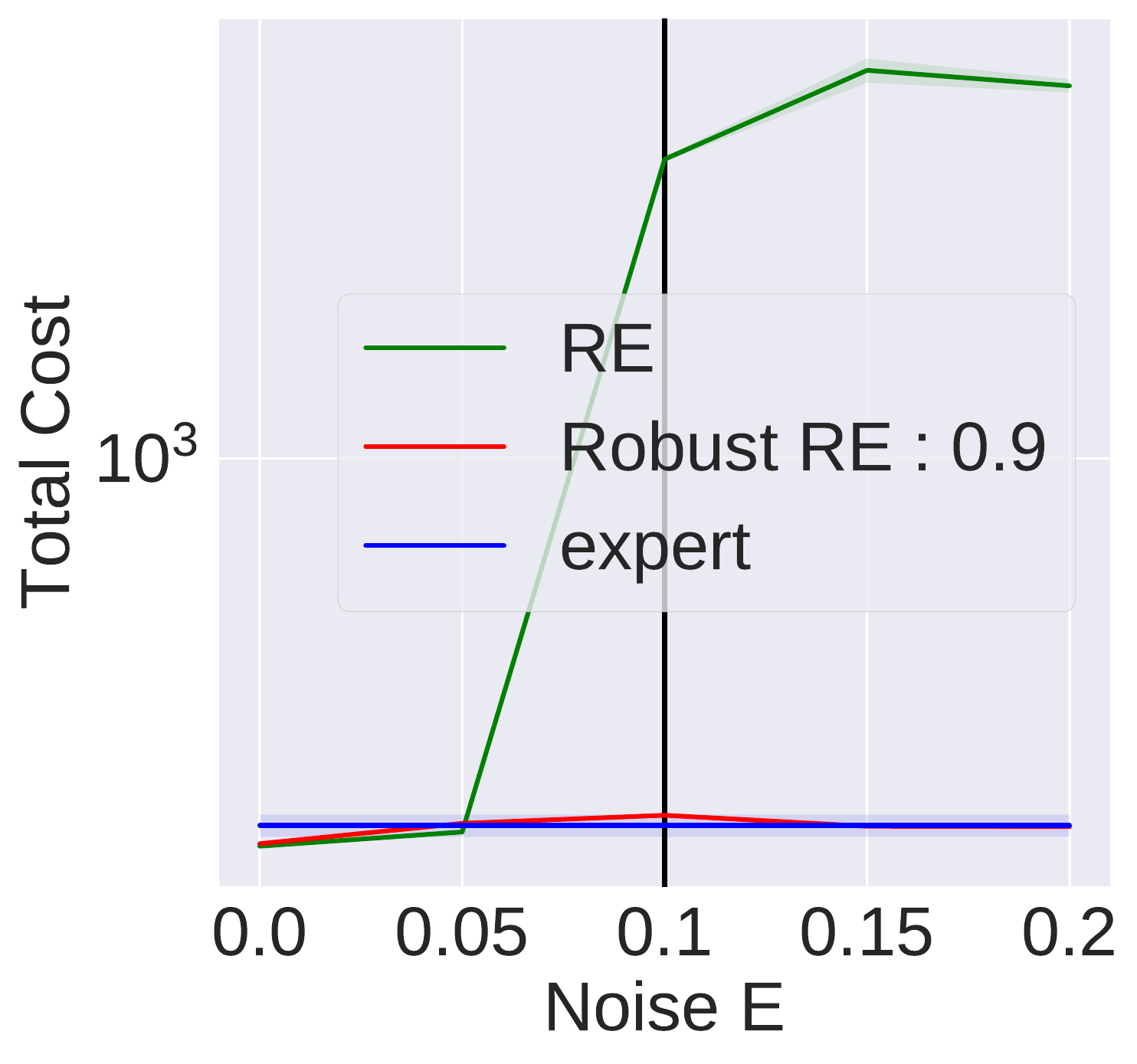}
\caption{$M^{L,\epsilon_L}$ with $\epsilon_L = 0.1$} \label{fig:gauss_gridl0.1}
\end{subfigure}
\caption{Comparison of the performance our Robust RE IRL (Algorithm~\ref{alg:RobustREIRL}) against the standard RE IRL, under different levels of mismatch: $\br{\epsilon_E, \epsilon_L} \in \bc{0.0, 0.05, 0.1, 0.15, 0.2} \times \bc{ 0.0, 0.05, 0.1}$. Each plot corresponds to a fixed leaner environment $M^{L,\epsilon_L}$ with $\epsilon_L \in \bc{ 0.0, 0.05, 0.1}$. The values of $\alpha$ used for Algorithm~\ref{alg:RobustREIRL} are reported in the legend. The vertical line indicates the position of the learner environment in the x-axis. The results are averaged across $5$ seeds.} 
\vspace{-\intextsep}
\label{fig:gaussian-gridworld}
\end{figure}

\section{Related Work}\label{sec:relatedwork}

In the context of forward RL, there are works that build on the robust MDP framework~\cite{iyengar2005robust,nilim2005robust,wiesemann2013robust}, for example,~\cite{shashua2017deeprobust,peng2018sim,mankowitz2019robust}. However, our work is closer to the line of work that leverages on the equivalence between action-robust and robust MDPs~\cite{morimoto2005robust,doyle2013feedback,pinto2017robust,tessler2019action,kamalaruban2020robust}. To our knowledge, this is the first work to adapt the robust RL methods in the IRL context. Other works study the IRL problem under a mismatch between the learner and the expert's worldviews~\cite{haug2018teaching,tschiatschek2019learner}. However, these works do not consider the dynamics mismatch.

Generative Adversarial Imitation Learning (GAIL)~\cite{ho2016generative} and its variants are IRL methods that use a GAN-based reward to align the distribution of the state-action pairs between the expert and the learner. When there is a transition dynamics mismatch, the expert's actions are not quite useful for imitation. \cite{torabi2018generative,sun2019provably}~have considered state only distribution matching when the expert actions are not observable. Building on these works, \cite{gangwani2020stateonly,liu2019state} have studied the imitation learning problem under transition dynamics mismatch. These works propose model-alignment based imitation learning algorithms in the high dimensional settings to address the dynamics mismatch. Finally, our work has the following important differences with AIRL~\cite{fu2018learning}. In AIRL, the learner has access to the expert environment during the training phase, i.e., there is no transition dynamics mismatch during the training phase but only at test time. In contrast, we consider a different setting where the learner can not access the expert environment during the training phase. In addition, AIRL requires input demonstrations containing both states and actions, while our algorithm requires state-only demonstrations.

\section{Conclusions}\label{sec:conclusions}

In this work, we theoretically analyze the MCE IRL algorithm under the transition dynamics mismatch: (i) we derive necessary and sufficient conditions for the existence of solution, and (ii) we provide a tight upper bound on the performance degradation. We propose a robust MCE IRL algorithm and empirically demonstrate its significant improvement over the standard MCE IRL under dynamics mismatch. Even though our Algorithm~\ref{alg:MaxEntIRL} is not essentially different from the standard robust RL methods, it poses additional theoretical challenges in the IRL context compared to the RL setup. In particular, we have proved: (i) the existence of solution for the robust MCE IRL formulation, and (ii) the performance gap improvement of our algorithm compared to the non-robust MCE IRL in a constructive example. We present empirical results for the settings not covered by our theory: MDPs with non-linear reward function and continuous state and action spaces. 


\section*{Code Repository}

\url{https://github.com/lviano/RobustMCE_IRL/tree/master/robustIRLcode}

\begin{ack}
This project has received funding from the European Research Council (ERC) under the European Union's Horizon 2020 research and innovation programme (grant agreement n° 725594 - time-data). Research was sponsored by the Army Research Office and was accomplished under Grant Number W911NF-19-1-0404, by the Department of the Navy, Office of Naval Research (ONR)  under a grant number N62909-17-1-2111 and by Hasler Foundation Program: Cyber Human Systems (project number 16066).

Parameswaran Kamalaruban acknowledges support from The Alan Turing Institute. He carried out part of this work while at LIONS, EPFL.

Adrian Weller acknowledges support from a Turing AI Fellowship under grant EP/V025379/1, The Alan Turing Institute, and the Leverhulme Trust via CFI.
\end{ack}


\bibliographystyle{unsrt}
\bibliography{robust-irl}

\newpage
\appendix


\section{Appendix structure}
\label{app:structure}
Here, we provide an overview on the organization of the appendix:
\begin{itemize}
\item Appendix~\ref{appendix:remarks} summarizes the scope and contributions of the paper.
\item Appendix~\ref{appendix:notation} provides a glossary of notation.
\item Appendix~\ref{appendix:sec2} provides further details of Section~\ref{sec:Setup}. In particular, we show that the expected feature count with one-hot feature map is proportional to the state occupancy measure.  
\item Appendix~\ref{appendix:maxentirl} provides further details of Section~\ref{sec:bounds}. In particular:
\begin{enumerate}
\item In Appendix~\ref{app:proof-soft-lemma}, we provide the proof of Lemma~\ref{thm:first_week} (performance difference between two soft optimal policies).
\item In Appendix~\ref{app:mce-irl-upper-bound}, we provide the proof of Theorem~\ref{thm:regret_different_experts} (performance gap of MCE IRL under model mismatch).
\item In Appendix~\ref{app:state-only-reward}, we explain why state-action reward function is not useful under model mismatch.
\item In Appendix~\ref{app:occ_states}, we provide the proof of Theorem~\ref{thm:occ_states} (existence of solution for MCE IRL under model mismatch).
\item In Appendix~\ref{app:reward-transfer}, we study the performance gap of the reward transfer strategy explained in Section~\ref{sec:mismatch-learning}.
\end{enumerate}
\item Appendix~\ref{app:solution-more-details} provides further details of Section~\ref{sec:two_players_entropy}. In particular:
\begin{enumerate}
\item In Appendix~\ref{sec:gradient}, we derive the gradient update for MCE IRL under model mismatch. 
\item In Appendix~\ref{app:softQ}, we present Algorithm~\ref{alg:TwoplayersDynProg}, with theoretical support, to solve the Markov Game in Section~\ref{sec:sol-markov-game}.
\item In Appendix~\ref{app:robust-mce-irl-upper}, we provide the proof of Theorem~\ref{thm:new-robust-mce-irl-bound} (performance gap of Algorithm~\ref{alg:MaxEntIRL} under model mismatch).
\item In Appendix~\ref{app:robust-mce-irl-upper-infeasible}, we study the performance gap of Algorithm~\ref{alg:MaxEntIRL} under model mismatch in the infeasible case (when exact occupancy measure matching is not possible). 
\item In Appendix~\ref{app:tightness-proof}, we provide the proof of Theorem~\ref{theorem-tightness} (constructive example comparing MCE IRL and Algorithm~\ref{alg:MaxEntIRL}).
\end{enumerate}
\item Appendix~\ref{appendix:experiments} provides further details of Section~\ref{sec:experiments}. In particular:
\begin{enumerate}
\item In Appendix~\ref{app:hyper-figs}, we report all the hyperparameter details, and present the figures mentioned in the main text.
\item In Appendix~\ref{app:low-dim-exp}, we demonstrate superior performance of Algorithm~\ref{alg:MaxEntIRL} on a low-dimensional feature setting.
\item In Appendix~\ref{app:alpha-choice}, we study the impact of the opponent strength parameter $1-\alpha$ on Robust MCE IRL.
\end{enumerate}
\item Appendix~\ref{sec:experiments_continuous_control} provides further details of Section~\ref{sec:experiments_continuous_control_main}. In particular, we present a high-dimensional continuous control extension of our robust IRL method, and demonstrates its efficacy on a domain with continuous state and spaces under dynamics mismatch.
\end{itemize}

\newpage
\section{Scope and Contributions}
\label{appendix:remarks}

Our work is intended to:
\begin{enumerate}
\item provide a theoretical investigation of the transition dynamics mismatch issue in the standard MCE IRL formulation, including:
\begin{enumerate}
\item an upper bound on the performance gap due to dynamics mismatch (Theorem~\ref{thm:regret_different_experts}) + the tightness of the bound (Theorem~\ref{theorem-tightness})
\item existence of solution under dynamics mismatch (Theorem~\ref{thm:occ_states})
\end{enumerate}
\item illustrate the issues with the reward transfer scheme under transition dynamics mismatch (Theorem~\ref{thm:new-reward-transfer-bound} + Lemma~\ref{thm:first_week}; see Section~\ref{sec:mismatch-learning}, and Appendix~\ref{app:reward-transfer})
\item understand the role of robust RL methods in mitigating the mismatch issue
\begin{enumerate}
\item validity (existence of solution using Theorem~\ref{thm:occ_states}) of the robust MCE IRL formulation (see Section~\ref{sec:exist-robust-sol})
\item an upper bound on the performance gap of robust MCE IRL (Theorem~\ref{thm:new-robust-mce-irl-bound}) + improvement over standard MCE IRL (Theorem~\ref{thm:occ_states})
\item an upper bound on the performance gap of robust MCE IRL when exact occupancy measure matching is not possible (Theorem~\ref{thm:new-robust-mce-irl-bound-infeasible})
\item different effect of over and underestimating the robustness parameter alpha (see Appendix~\ref{app:alpha-choice})
\end{enumerate}
\item empirically validate our claims in a setting (finite MDP) without theory-practice gap (see Section~\ref{sec:experiments}, and Appendix~\ref{appendix:experiments})
\item extend our robust IRL method to the high dimensional continuous MDP setting with appropriate practical relaxations, and empirically demonstrate its effectiveness (see Appendix~\ref{sec:experiments_continuous_control}).
\end{enumerate}


\section{Glossary of Notation}
\label{appendix:notation}

We have carefully developed the notation based on the best practices prescribed by the RL theory community~\cite{agarwal2019reinforcement}, and do not want to compromise its rigorous nature. To help the reader, we provide a glossary of notation. 

\begin{table}[h!]
\centering
\begin{tabular}{l|l}
\hline
$\pi^{*}_{M_{\boldsymbol{\theta^*}}^L}$ & optimal policy in the MDP $M_{\boldsymbol{\theta^*}}^L = \bc{\mathcal{S}, \mathcal{A}, T^L, \gamma, P_0, R_{\boldsymbol{\theta^*}}}$ \\
$\pi^{*}_{M_{\boldsymbol{\theta^*}}^E}$ & optimal policy in the MDP $M_{\boldsymbol{\theta^*}}^E = \bc{\mathcal{S}, \mathcal{A}, T^E, \gamma, P_0, R_{\boldsymbol{\theta^*}}}$ \\
$\boldsymbol{\rho}^{\pi^{*}_{M_{\boldsymbol{\theta^*}}^L}}_{M^L}$ & state occupancy measure of $\pi^{*}_{M_{\boldsymbol{\theta^*}}^L}$ in the MDP $M^L = \bc{\mathcal{S}, \mathcal{A}, T^L, \gamma, P_0}$ \\
$\boldsymbol{\rho}^{\pi^{*}_{M_{\boldsymbol{\theta^*}}^E}}_{M^E}$ & state occupancy measure of $\pi^{*}_{M_{\boldsymbol{\theta^*}}^E}$ in the MDP $M^E = \bc{\mathcal{S}, \mathcal{A}, T^E, \gamma, P_0}$ \\
$\boldsymbol{\theta_L}$ & reward parameter recovered when there is no transition dynamics mismatch \\
$\boldsymbol{\theta_E}$ & reward parameter recovered under transition dynamics mismatch \\
$\pi_1 = \pi^{\mathrm{soft}}_{M^{L}_{\boldsymbol{\theta_{L}}}}$ & soft optimal policy in the MDP $M_{\boldsymbol{\theta_{L}}}^L = \bc{\mathcal{S}, \mathcal{A}, T^L, \gamma, P_0, R_{\boldsymbol{\theta_{L}}}}$ \\ 
$\pi_2 = \pi^{\mathrm{soft}}_{M^{L}_{\boldsymbol{\theta_{E}}}}$ & soft optimal policy in the MDP $M_{\boldsymbol{\theta_{E}}}^L = \bc{\mathcal{S}, \mathcal{A}, T^L, \gamma, P_0, R_{\boldsymbol{\theta_{E}}}}$ \\
$V^{\pi_1}_{M_{\boldsymbol{\theta^*}}^L}$ & total expected return of $\pi_1$ in the MDP $M_{\boldsymbol{\theta^*}}^L = \bc{\mathcal{S}, \mathcal{A}, T^L, \gamma, P_0, R_{\boldsymbol{\theta^*}}}$ \\
$V^{\pi_2}_{M_{\boldsymbol{\theta^*}}^L}$ & total expected return of $\pi_2$ in the MDP $M_{\boldsymbol{\theta^*}}^L = \bc{\mathcal{S}, \mathcal{A}, T^L, \gamma, P_0, R_{\boldsymbol{\theta^*}}}$ \\
$\boldsymbol{\rho}^{\pi^{\mathrm{pl}}}_{M^{L,\alpha}}$ & state occupancy measure of $\pi^{\mathrm{pl}}$ in the MDP $M^{L,\alpha} = \bc{\mathcal{S}, \mathcal{A}, T^{L,\alpha}, \gamma, P_0}$ \\
$\boldsymbol{\rho}^{\alpha \pi^{\mathrm{pl}} + (1 - \alpha) \pi^{\mathrm{op}}}_{ M^L}$ & state occupancy measure of $\alpha \pi^{\mathrm{pl}} + (1 - \alpha) \pi^{\mathrm{op}}$ in the MDP $M^{L} = \bc{\mathcal{S}, \mathcal{A}, T^{L}, \gamma, P_0}$ \\ \hline
\end{tabular}
\vspace{2mm}
\caption{A glossary of notation.}
\label{tab:glossary}
\end{table}

\newpage
\section{Further Details of Section~\ref{sec:Setup}}
\label{appendix:sec2}

An optimal policy $\pi^*_{M_{\boldsymbol{\theta}}}$ in the MDP $M_{\boldsymbol{\theta}}$ satisfies the following \emph{Bellman optimality equations} for all the state-action pairs $(s,a) \in \mathcal{S}\times\mathcal{A}$:
\begin{align*}
\pi^*_{M_{\boldsymbol{\theta}}}(s) ~=~& \argmax_a Q^{*}_{M_{\boldsymbol{\theta}}}(s,a) \nonumber \\
Q^{*}_{M_{\boldsymbol{\theta}}}(s,a) ~=~& R_{\boldsymbol{\theta}}(s) + \gamma \sum_{s'}T(s'|s,a) V^{*}_{M_{\boldsymbol{\theta}}}(s') \nonumber \\
V^{*}_{M_{\boldsymbol{\theta}}}(s) ~=~& \max_a Q^{*}_{M_{\boldsymbol{\theta}}}(s,a)~
 \end{align*}
 
The soft-optimal policy $\pi^{\mathrm{soft}}_{M_{\boldsymbol{\theta}}}$ in the MDP $M_{\boldsymbol{\theta}}$ satisfies the following \emph{soft Bellman optimality equations} for all the state-action pairs $(s,a) \in \mathcal{S}\times\mathcal{A}$:
 \begin{align*}
\pi^{\mathrm{soft}}_{M_{\boldsymbol{\theta}}}(a|s) ~=~& \exp \br{Q^{ \mathrm{soft}}_{M_{\boldsymbol{\theta}}}(s,a) - V^{\mathrm{soft}}_{M_{\boldsymbol{\theta}}}(s)} \nonumber \\
Q^{\mathrm{soft}}_{M_{\boldsymbol{\theta}}}(s,a) ~=~& R_{\boldsymbol{\theta}}(s) + \gamma \sum_{s'}T(s'|s,a) V^{\mathrm{soft}}_{M_{\boldsymbol{\theta}}}(s') \nonumber \\
V^{\mathrm{soft}}_{M_{\boldsymbol{\theta}}}(s) ~=~& \log \sum_a \exp Q^{\mathrm{soft}}_{M_{\boldsymbol{\theta}}}(s,a)
\end{align*}

The expected feature count of a policy $\pi$ in the MDP $M$ is defined as $\bar{\boldsymbol{\phi}}^{\pi}_{M} := \Ee{\pi, M}{\sum^{\infty}_{t=0} \gamma^t \boldsymbol{\phi}(s_t)}$. 

\begin{fact}
If $\forall s \in \mathcal{S}$, $\boldsymbol{\phi}(s) \in \mathbb{R}^{\abs{\mathcal{S}}}$ is a one-hot vector with only the element in position $s$ being $1$, then the expected feature count of a policy $\pi$ in the MDP $M$ is proportional to its state occupancy measure vector in the MDP $M$. 
\end{fact}
\begin{proof}
For any $M, \pi$, we have:
\begin{align*}
\bar{\boldsymbol{\phi}}^{\pi}_{M} &~=~ \Ee{\pi, M}{\sum^{\infty}_{t=0} \gamma^t \boldsymbol{\phi}(s_t)} \\
&~=~ \Ee{\pi, M}{\sum^{\infty}_{t=0} \gamma^t \sum_{s \in \mathcal{S}} \boldsymbol{\phi}(s)\mathbbm{1}\bs{s = s_t}} \\
&~=~ \sum_{s \in \mathcal{S}} \boldsymbol{\phi}(s) \Ee{\pi, M}{\sum^{\infty}_{t=0} \gamma^t \mathbbm{1}\bs{s = s_t}} \\
&~=~ \sum_{s \in \mathcal{S}} \boldsymbol{\phi}(s) \sum^{\infty}_{t=0} \gamma^t \Ee{\pi, M}{\mathbbm{1}\bs{s = s_t}} \\
&~=~ \frac{1}{1 - \gamma} \sum_{s \in \mathcal{S}} \rho^{\pi}_{M}(s) \boldsymbol{\phi}(s)
\end{align*}
For the one-hot feature map, ignoring the normalizing factor, the above sum of vectors can be written as follows: 
\begin{equation*}
\bs{\rho^{\pi}_{M}(s_1), \rho^{\pi}_{M}(s_2), \cdots}^\top ~=~ \boldsymbol{\rho}^{\pi}_{M} .
\end{equation*}
\end{proof}
Leveraging on this fact, we formulate the MCE IRL problem~\eqref{opt_start} with the state occupancy measure $\boldsymbol{\rho}$ match rather than the usual expected feature count match. Note that if the occupancy measure match is attained, then the match of any expected feature count is also attained.

\newpage
\section{Further Details of Section~\ref{sec:bounds}}
\label{appendix:maxentirl}

\subsection{Proof of Lemma~\ref{thm:first_week}}
\label{app:proof-soft-lemma}

\begin{proof}
The soft-optimal policy of the MDP $M'_{\boldsymbol{\theta}}$ satisfies the following soft Bellman optimality equations:
\begin{align}
\pi'(a|s) ~=~& \frac{Z'_{a|s}}{Z'_s}
\label{MaxEntPolicyL} \\
\log Z'_s ~=~& \log \sum_a  Z'_{a|s} \nonumber \\
\log Z'_{a|s} ~=~& R_{\boldsymbol{\theta}}(s) + \gamma \sum_{s^\prime} T'(s^\prime| a, s) \log Z'_{s^\prime} \label{MaxEntQL}
\end{align}
Analogously, the soft-optimal policy of the MDP $M_{\boldsymbol{\theta}}$ satisfies the following soft Bellman optimality equations:
\begin{align}
\pi(a|s) ~=~& \frac{Z_{a|s}}{Z_s}
\label{MaxEntPolicyE} \\
\log Z_s ~=~& \log \sum_a  Z_{a|s} \nonumber \\
\log Z_{a|s} ~=~& R_{\boldsymbol{\theta}}(s) + \gamma \sum_{s^\prime} T(s^\prime| a, s) \log Z_{s^\prime} \label{MaxEntQE}
\end{align}

For any $s \in \mathcal{S}$, we have:
\begin{align}
D_{\mathrm{KL}}\br{\pi'(\cdot|s), \pi(\cdot|s)} &~=~ \sum_{a} \pi'(a|s) \log \frac{\pi'(a|s)}{\pi(a|s)} \nonumber \\
    &~=~ \sum_{a} \frac{Z'_{a|s}}{Z'_s} \br{\log \frac{Z'_{a|s}}{Z_{a|s}}+ \log \frac{Z_{s}}{Z'_{s}}} \nonumber \\
    &~=~ \sum_{a} \frac{Z'_{a|s}}{Z'_s} \log \frac{Z'_{a|s}}{Z_{a|s}} + \log \frac{Z_{s}}{Z'_{s}} \label{KL_bound}
\end{align}
By using the log-sum inequality on the term depending on the states only:
\begin{align}
\log \frac{Z_{s}}{Z'_{s}} &~=~ \underbrace{\sum_{a} \frac{Z_{a|s}}{Z_s}}_{1} \log \frac{Z_{s}}{Z'_{s}} \nonumber \\ &~=~  \sum_{a} \frac{Z_{a|s}}{Z_s} \log \frac{\sum_a Z_{a|s}}{\sum_a Z'_{a|s}} \nonumber \\
 &~\leq~ \frac{1}{Z_s}  \sum_{a} Z_{a|s}\log \frac{Z_{a|s}}{ Z'_{a|s}} \label{log-sum}
 \end{align}
Consequently, replacing \eqref{log-sum} in \eqref{KL_bound}, and using the definitions \eqref{MaxEntPolicyE} and \eqref{MaxEntPolicyL}, we have:
\begin{align}
     D_{\mathrm{KL}}\br{\pi'(\cdot|s), \pi(\cdot|s)}  &~\leq~ \sum_{a}\br{ \frac{Z'_{a|s}}{Z'_s} - \frac{Z_{a|s}}{Z_s} }\log \frac{Z'_{a|s}}{Z_{a|s}} \nonumber \\
     &~=~ \sum_{a}\br{ \pi'(a|s) - \pi(a|s)}\log \frac{Z'_{a|s}}{Z_{a|s}} \nonumber \\
     &~\leq~ \sum_{a}\abs{ \pi'(a|s) - \pi(a|s)} \cdot \abs{\log \frac{Z'_{a|s}}{Z_{a|s}}} \nonumber \\
     &~\leq~ \sum_{a'}\abs{ \pi'(a'|s) - \pi(a'|s)} \cdot \max_a \abs{\log \frac{Z'_{a|s}}{Z_{a|s}}} \nonumber \\
     &~=~ \norm{ \pi'(\cdot|s) - \pi(\cdot|s)}_1 \cdot \max_a \abs{\log \frac{Z'_{a|s}}{Z_{a|s}}}  \nonumber 
\end{align} 
Then, by taking $\max$ over $s$, we have:
 \begin{equation}
     \max_s D_{\mathrm{KL}}\br{\pi'(\cdot|s), \pi(\cdot|s)} ~\leq~ \max_{s} \norm{ \pi'(\cdot|s) - \pi(\cdot|s)}_1 \cdot \max_{s,a}  \abs{\log \frac{Z'_{a|s}}{Z_{a|s}}} \label{LogToBound}
 \end{equation}
Further, we exploit the following fact:
 \begin{equation}
     \max_{s,a}\abs{\log \frac{Z'_{a|s}}{Z_{a|s}}} ~=~ \max \bc{ \log \frac{Z'_{\bar{a}|\bar{s}}}{Z_{\bar{a}|\bar{s}}},  \log \frac{Z_{\underline{a}|\underline{s}}}{Z'_{\underline{a}|\underline{s}}}} \label{SplitMax} ,
 \end{equation}
where we adopted the following notation:
\begin{align}
     (\bar{s},\bar{a}) &~=~ \argmax_{s,a} \log \frac{Z'_{a|s}}{Z_{a|s}} \label{a_high}  \\
     (\underline{s},\underline{a}) &~=~ \argmin_{s,a} \log \frac{Z'_{a|s}}{Z_{a|s}}    \label{a_low} 
\end{align}
At this point, we can bound separately the two arguments of the max in \eqref{SplitMax}. Starting from \eqref{a_high}:
\begin{align*}
     \log \frac{Z'_{\bar{a}|\bar{s}}}{Z_{\bar{a}|\bar{s}}} &~=~ \log Z'_{\bar{a}|\bar{s}} - \log Z_{\bar{a}|\bar{s}} \\
     &~=~ \underbrace{R_{\boldsymbol{\theta}}(\bar{s}) - R_{\boldsymbol{\theta}}(\bar{s})}_{0} + \gamma \bc{\sum_{s^\prime} T'(s^\prime|\bar{s},\bar{a})\log Z'_{s^\prime} - T(s^\prime|\bar{s},\bar{a})\log Z_{s^\prime}} \\
     &~=~ \gamma \bc{\sum_{s^\prime} T'(s^\prime|\bar{s},\bar{a})\log \frac{Z'_{s^\prime}}{Z_{s^\prime}} + \br{T'(s^\prime|\bar{s},\bar{a}) - T(s^\prime|\bar{s},\bar{a})}\log Z_{s^\prime}} \\ &~\leq~ \gamma \bc{\sum_{s^\prime} T'(s^\prime|\bar{s},\bar{a}) \br{\sum_{a} \pi'(a|s^\prime) \log \frac{Z'_{a|s^\prime}}{Z_{a|s^\prime}}} + \br{T'(s^\prime|\bar{s},\bar{a}) - T(s^\prime|\bar{s},\bar{a})}\log Z_{s^\prime}} \\
     &~\leq~ \gamma \log \frac{Z'_{\bar{a}|\bar{s}}}{Z_{\bar{a}|\bar{s}}} + \gamma \sum_{s^\prime} \br{T'(s^\prime|\bar{s},\bar{a}) - T(s^\prime|\bar{s},\bar{a})}\log Z_{s^\prime}
\end{align*}
By rearranging the terms, we get:
\begin{align}
        \log \frac{Z'_{\bar{a}|\bar{s}}}{Z_{\bar{a}|\bar{s}}} &~\leq~ \frac{\gamma}{1-\gamma} \cdot \sum_{s^\prime} \br{T'(s^\prime|\bar{s},\bar{a}) - T(s^\prime|\bar{s},\bar{a})}\log Z_{s^\prime} \nonumber \\
        &~\leq~  \frac{\gamma}{1-\gamma} \cdot \sum_{s^\prime} \abs{T'(s^\prime|\bar{s},\bar{a}) - T(s^\prime|\bar{s},\bar{a})} \cdot \abs{ \log Z_{s^\prime}} \nonumber \\
        &~\leq~ \frac {\gamma}{1-\gamma} \cdot \max_{s^\prime}\abs{ \log Z_{s^\prime}} \cdot \sum_{s^\prime} \abs{T'(s^\prime|\bar{s},\bar{a}) - T(s^\prime|\bar{s},\bar{a})} \label{argmaxBound}
\end{align}
Then, with analogous calculations for the second argument of the max operator in \eqref{SplitMax}, we have
\begin{align*}
         \log \frac{Z_{\underline{a}|\underline{s}}}{Z'_{\underline{a}|\underline{s}}} &~=~ \log Z_{\underline{a}|\underline{s}} 
         - \log Z'_{\underline{a}|\underline{s}} \\
     &~=~ \underbrace{R_{\boldsymbol{\theta}}(\underline{s}) - R_{\boldsymbol{\theta}}(\underline{s})}_{0} + \gamma \bc{\sum_{s^\prime} T(s^\prime|\underline{s},\underline{a})\log Z_{s^\prime} - T'(s^\prime|\underline{s},\underline{a})\log Z'_{s^\prime}} \\
     &~=~ \gamma \bc{\sum_{s^\prime} T(s^\prime|\underline{s},\underline{a})\log \frac{Z_{s^\prime}}{Z'_{s^\prime}} + \br{T(s^\prime|\underline{s},\underline{a})- T'(s^\prime|\underline{s},\underline{a})}\log Z'_{s^\prime}} \\
     &~\leq~ \gamma \log \frac{Z_{\underline{a}|\underline{s}}}{Z'_{\underline{a}|\underline{s}}} + \gamma \sum_{s^\prime}  \br{T(s^\prime|\underline{s},\underline{a})- T'(s^\prime|\underline{s},\underline{a})}\log Z'_{s^\prime}
\end{align*}
It follows that:
\begin{align}
     \log \frac{Z_{\underline{a}|\underline{s}}}{Z'_{\underline{a}|\underline{s}}} &~\leq~ \frac{\gamma}{1 - \gamma} \cdot \sum_{s^\prime}  \br{T(s^\prime|\underline{s},\underline{a})- T'(s^\prime|\underline{s},\underline{a})}\log Z'_{s^\prime} \nonumber \\
     &~\leq~  \frac{\gamma}{1 - \gamma} \cdot \sum_{s^\prime}  \abs{T(s^\prime|\underline{s},\underline{a})- T'(s^\prime|\underline{s},\underline{a})} \cdot \abs{\log Z'_{s^\prime}} \nonumber \\
     &~\leq~ \frac{\gamma}{1 - \gamma} \cdot \max_{s^\prime} \abs{\log Z'_{s^\prime}} \cdot \sum_{s^\prime}  \abs{T(s^\prime|\underline{s},\underline{a})- T'(s^\prime|\underline{s},\underline{a})} \label{argminBound}
\end{align}
We can plug in the bounds obtained in \eqref{argminBound} and \eqref{argmaxBound} in \eqref{SplitMax}:
\begin{equation}
      \max_{s,a}\abs{\log \frac{Z'_{a|s}}{Z_{a|s}}} ~\leq~ \frac{\gamma}{1 - \gamma} \cdot \max \bc{\max_{s^\prime} \abs{\log Z_{s^\prime}}, \max_{s^\prime} \abs{\log Z'_{s^\prime}}} \cdot  \max_{s,a} \sum_{s^\prime}  \abs{T'(s^\prime|s,a)- T(s^\prime|s,a)} \label{Main}
\end{equation}
 It still remains to bound the term $\max \bc{\max_{s^\prime} \abs{\log Z_{s^\prime}}, \max_{s^\prime} \abs{\log Z'_{s^\prime}}}$. It can be done by a splitting procedure similar to the one in \eqref{SplitMax}. Indeed:
 \begin{equation}
     \max_{s^\prime} \abs{\log Z_{s^\prime}} ~=~ \max \bc{ \log Z_{\bar{s}},  \log \frac{1}{Z_{\underline{s}}} }
     \label{SplitMax2}
 \end{equation}
 where, changing the previous definitions of $\bar{s}$ and $\underline{s}$, we set:
 \begin{align}
     \bar{s} &~=~ \argmax_{s} \log Z_{s} \label{argmax} \\
      \underline{s} &~=~ \argmin_{s} \log Z_{s} \label{argmin}
 \end{align}
Starting from the first term in \eqref{SplitMax2} and applying \eqref{MaxEntQE}:
\begin{align}
    \log Z_{\bar{s}} &~=~ \log \sum_{a}  Z_{a|\bar{s}} \nonumber \\
    &~\leq~ \log \br{\abs{\mathcal{A}} \max_{a} Z_{a|\bar{s}}} \nonumber \\
    &~=~  \log \abs{\mathcal{A}} +   \log \max_{a} Z_{a|\bar{s}} \nonumber \\
    &~=~ \log \abs{\mathcal{A}} +  \max_{a} \log Z_{a|\bar{s}}  \label{BarBound}
\end{align}
where the last equality follows from the fact that $\log$ is a monotonically increasing function.
Furthermore, \eqref{argmax} implies that $\log Z_{s^\prime} \leq \log Z_{\bar{s}}, \quad \forall s^\prime \in \mathcal{S}$:
\begin{align}
    \max_{a} \log Z_{a|\bar{s}} &~\leq~  \max_{a} \br{R_{\boldsymbol{\theta}}(\bar{s}) + \gamma \log Z_{\bar{s}} \sum_{s^\prime} T(s^\prime|\bar{s},a)} \nonumber \\
    &~\leq~ R^{\mathrm{max}}_{\boldsymbol{\theta}} + \gamma \log Z_{\bar{s}} \label{ABarBound}
    \end{align}
In the last inequality we have used the quantity $R^{\mathrm{max}}_{\boldsymbol{\theta}}$ that satisfies $R_{\boldsymbol{\theta}}(s) \leq R^{\mathrm{max}}_{\boldsymbol{\theta}}, \quad \forall s \in \mathcal{S}$. In a similar fashion, we will use $R^{\mathrm{min}}_{\boldsymbol{\theta}}$ such that $R_{\boldsymbol{\theta}}(s)  \geq R^{\mathrm{min}}_{\boldsymbol{\theta}}, \quad \forall s \in \mathcal{S}$.
Finally, plugging \eqref{ABarBound} into \eqref{BarBound}, we get:
\begin{equation}
    \log Z_{\bar{s}} ~\leq~ \frac{R^{\mathrm{max}}_{\boldsymbol{\theta}} + \log \abs{\mathcal{A}}}{1 - \gamma} \label{SBound1}
\end{equation}
We can proceed bounding the second argument of the max operator in \eqref{SplitMax2}. To this scope, we observe that $\sum_{a} \frac{1}{\abs{\mathcal{A}}} = 1$, and, then, we apply the log-sum inequality as follows:
\begin{align}
    \log \frac{1}{Z_{\underline{s}}} &~=~ \sum_{a} \frac{1}{\abs{\mathcal{A}}} \log \frac{ \sum_{a} \frac{1}{\abs{\mathcal{A}}}}{\sum_{a}Z_{a|\underline{s}}} \nonumber \\
    &~\leq~  \sum_{a} \frac{1}{\abs{\mathcal{A}}} \log \frac{ \frac{1}{\abs{\mathcal{A}}}}{Z_{a|\underline{s}}}  \nonumber \\
    &~=~ \log \frac {1}{\abs{\mathcal{A}}} + \sum_{a} \frac {1}{\abs{\mathcal{A}}} \log \frac{1}{Z_{a|\underline{s}}} \nonumber \\
    &~\leq~ \log \frac {1}{\abs{\mathcal{A}}} + \max_{a} \log \frac{1}{Z_{a|\underline{s}}} \label{UnderBound}
\end{align}
Similarly to \eqref{ABarBound}, we can apply one step of the soft Bellman equation to bound the term $ \log \frac{1}{Z_{a|\underline{s}}} $:
\begin{align}
    \log \frac{1}{Z_{a|\underline{s}}} &~=~ - \log Z_{a|\underline{s}} \nonumber \\
    &~=~ - R_{\boldsymbol{\theta}}(\underline{s}) - \gamma \sum_{s^\prime} T(s^\prime|\underline{s},a) \log Z_{s^\prime} \nonumber \\
    &~=~ - R_{\boldsymbol{\theta}}(\underline{s}) + \gamma \sum_{s^\prime} T(s^\prime|\underline{s},a) \log \frac{1}{Z_{s^\prime}} \nonumber \\
    &~\leq~ - R^{\mathrm{min}}_{\boldsymbol{\theta}} + \gamma  \log \frac{1}{Z_{\underline{s}}} \underbrace{\sum_{s^\prime} T(s^\prime|\underline{s},a)}_{1} \label{AUnderBound}
\end{align}
where in the last inequality we used \eqref{argmin}, $R_{\boldsymbol{\theta}}(s)  \geq R^{\mathrm{min}}_{\boldsymbol{\theta}}, \quad \forall s \in \mathcal{S}$. Since the upper bound in \eqref{AUnderBound} does not depend on $a$, we have:
\begin{equation}
     \max_{a} \log \frac{1}{Z_{a|\underline{s}}} ~\leq~  - R^{\mathrm{min}}_{\boldsymbol{\theta}} + \gamma  \log \frac{1}{Z_{\underline{s}}} \label{AUnderBound2}
\end{equation}
Replacing \eqref{AUnderBound2} into \eqref{UnderBound}, we have:
\begin{equation*}
        \log \frac{1}{Z_{\underline{s}}} ~\leq~ \log \frac {1}{\abs{\mathcal{A}}} - R^{\mathrm{min}}_{\boldsymbol{\theta}} + \gamma  \log \frac{1}{Z_{\underline{s}}}
\end{equation*}
and, consequently:
\begin{equation}
    \log \frac{1}{Z_{\underline{s}}} ~\leq~ \frac{- \log \abs{\mathcal{A}} - R^{\mathrm{min}}_{\boldsymbol{\theta}}}{1 - \gamma} \label{SBound2}
\end{equation}
Finally, using \eqref{SBound1} and \eqref{SBound2} in \eqref{SplitMax2}:
\begin{equation}
     \max_{s^\prime}  \abs{\log Z_{s^\prime}} ~\leq~ \frac{1}{1 - \gamma} \cdot \max \bc{ R^{\mathrm{max}}_{\boldsymbol{\theta}} + \log \abs{\mathcal{A}},  - \log \abs{\mathcal{A}} - R^{\mathrm{min}}_{\boldsymbol{\theta}}} \label{logBound}
\end{equation}
In addition, one can notice that the bound \eqref{logBound} holds also for $\max_{s^\prime}  \abs{\log Z'_{s^\prime}}$:
\begin{equation*}
     \max_{s^\prime}  \abs{\log Z'_{s^\prime}} ~\leq~ \frac{1}{1 - \gamma} \cdot \max \bc{ R^{\mathrm{max}}_{\boldsymbol{\theta}} + \log \abs{\mathcal{A}},  - \log \abs{\mathcal{A}} - R^{\mathrm{min}}_{\boldsymbol{\theta}}}
\end{equation*}
Thus, we can finally replace \eqref{logBound} in \eqref{Main} that gives:
\begin{equation}
      \max_{s,a}\abs{\log \frac{Z'_{a|s}}{Z_{a|s}}} ~\leq~ \frac{\gamma}{\br{1 - \gamma}^2} \cdot
       \max \bc{ R^{\mathrm{max}}_{\boldsymbol{\theta}} + \log \abs{\mathcal{A}},  - \log \abs{\mathcal{A}} - R^{\mathrm{min}}_{\boldsymbol{\theta}}} \cdot \max_{s,a}
 \sum_{s^\prime}  \abs{T'(s^\prime|s,a)- T(s^\prime|s,a)}
 \label{five}
\end{equation}
We can now go back through the inequality chain to eventually state the bound in the Theorem. First, plugging in \eqref{five} into \eqref{LogToBound} gives:
\begin{equation}
\max_s D_{\mathrm{KL}}\br{\pi'(\cdot|s), \pi(\cdot|s)} ~\leq~ \frac{\max_s  \norm{ \pi'(\cdot|s) - \pi(\cdot|s)}_1 \cdot \kappa_{\boldsymbol{\theta}}^2}{(1 - \gamma)^2} \cdot d_\mathrm{dyn} \br{T', T}
 \label{six}
\end{equation}
First, by using Pinsker's inequality and the fact that $\max_s  \norm{ \pi'(\cdot|s) - \pi(\cdot|s)}_1 \leq 2$, we get:
\[
\max_{s} \norm{\pi'(\cdot|s) - \pi(\cdot|s)}_1 ~\leq~ \sqrt{2 \max_s D_{\mathrm{KL}}\br{\pi'(\cdot|s), \pi(\cdot|s)}} ~\leq~ \frac{2 \cdot \kappa_{\boldsymbol{\theta}}}{(1 - \gamma)} \cdot \sqrt{d_\mathrm{dyn} \br{T', T}} 
\]
Similarly, by using Pinsker's inequality, we get:
\[
\max_{s} \norm{\pi'(\cdot|s) - \pi(\cdot|s)}_1 ~\leq~ \sqrt{2 \max_s D_{\mathrm{KL}}\br{\pi'(\cdot|s), \pi(\cdot|s)}} ~\leq~ \frac{ \kappa_{\boldsymbol{\theta}}}{(1 - \gamma)} \cdot \sqrt{2 \max_{s} \norm{\pi'(\cdot|s) - \pi(\cdot|s)}_1 d_\mathrm{dyn} \br{T', T}} 
\]
Thus, we have:
\[
\max_{s} \norm{\pi'(\cdot|s) - \pi(\cdot|s)}_1 ~\leq~ \frac{ 2 \cdot \kappa^2_{\boldsymbol{\theta}}}{(1 - \gamma)^2} \cdot d_\mathrm{dyn} \br{T', T}
\]
Finally, we get:
\[
d_\mathrm{pol} \br{\pi', \pi} ~\leq~ 2 \min \bc{\frac{\kappa_{\boldsymbol{\theta}} \cdot \sqrt{d_\mathrm{dyn} \br{T', T}}}{(1 - \gamma)}, \frac{\kappa^2_{\boldsymbol{\theta}} \cdot d_\mathrm{dyn} \br{T', T}}{(1 - \gamma)^2}}
\]
\end{proof}

\subsection{Proof of Theorem~\ref{thm:regret_different_experts}}
\label{app:mce-irl-upper-bound}

\begin{proof}
Consider the following:
\begin{align*}
\abs{V^{\pi_1}_{M_{\boldsymbol{\theta^*}}^L} - V^{\pi_2}_{M_{\boldsymbol{\theta^*}}^L}} ~\leq~& \abs{V^{\pi_1}_{M_{\boldsymbol{\theta^*}}^L} - V^{\pi^{*}_{M_{\boldsymbol{\theta^*}}^L}}_{M_{\boldsymbol{\theta^*}}^L}} + \abs{V^{\pi^{*}_{M_{\boldsymbol{\theta^*}}^L}}_{M_{\boldsymbol{\theta^*}}^L} - V^{\pi^{*}_{M_{\boldsymbol{\theta^*}}^E}}_{M_{\boldsymbol{\theta^*}}^E}} + \abs{V^{\pi^{*}_{M_{\boldsymbol{\theta^*}}^E}}_{M_{\boldsymbol{\theta^*}}^E} - V^{\pi_2}_{M_{\boldsymbol{\theta^*}}^L}} \\
~=~& \frac{1}{1-\gamma} \abs{\ip{\boldsymbol{\theta^*}}{\boldsymbol{\rho}^{\pi_1}_{M^L} - \boldsymbol{\rho}^{\pi^{*}_{M_{\boldsymbol{\theta^*}}^L}}_{M^L}}} + \abs{V^{\pi^{*}_{M_{\boldsymbol{\theta^*}}^L}}_{M_{\boldsymbol{\theta^*}}^L} - V^{\pi^{*}_{M_{\boldsymbol{\theta^*}}^E}}_{M_{\boldsymbol{\theta^*}}^E}} + \frac{1}{1-\gamma} \abs{\ip{\boldsymbol{\theta^*}}{\boldsymbol{\rho}^{\pi^{*}_{M_{\boldsymbol{\theta^*}}^E}}_{M^E} - \boldsymbol{\rho}^{\pi_2}_{M^L}}} \\
~=~& \abs{V^{\pi^{*}_{M_{\boldsymbol{\theta^*}}^L}}_{M_{\boldsymbol{\theta^*}}^L} - V^{\pi^{*}_{M_{\boldsymbol{\theta^*}}^E}}_{M_{\boldsymbol{\theta^*}}^E}} \\
~\leq~& \frac{\gamma \cdot \abs{R_{\boldsymbol{\theta^*}}}^{\mathrm{max}}}{(1 - \gamma)^2} \cdot d_\mathrm{dyn} \br{T^L, T^E}
\end{align*}
The first and third terms vanish, since:
\begin{enumerate}
\item $\pi_1$ is the optimal (thus feasible) solution to the optimization problem~\eqref{opt_start} with $\boldsymbol{\rho} \gets \boldsymbol{\rho}^{\pi^{*}_{M_{\boldsymbol{\theta^*}}^L}}_{M^L}$, and 
\item $\pi_2$ is the optimal (thus feasible) solution to the optimization problem~\eqref{opt_start} with $\boldsymbol{\rho} \gets \boldsymbol{\rho}^{\pi^{*}_{M_{\boldsymbol{\theta^*}}^E}}_{M^E}$.
\end{enumerate}
The last inequality is obtained from the Bellman optimality condition (see Theorem~7 in~\cite{zhang2020multi}). 
\end{proof}

For completeness, we restate Theorem~7 in~\cite{zhang2020multi} adapting the notation to our framework and considering bounded rewards instead of normalized rewards as in~\cite{zhang2020multi}.
\begin{theorem}[Theorem~7 in~\cite{zhang2020multi}]
Consider two MDPs $M_1 = \{ \mathcal{S}, \mathcal{A}, T_1, \gamma, P_0, R \}$ and $M_2 = \{ \mathcal{S}, \mathcal{A}, T_2, \gamma, P_0, R \}$ with bounded reward function $|R| \leq |R|^{\mathrm{max}} $ and policies $\pi^*_1$ optimal in $M_1$ and $\pi^*_2$ optimal in $M_2$. Then, we have that:
\begin{equation}
    |V^{\pi^*_1}_{M_1} - V^{\pi^*_2}_{M_2}| \leq \frac{\gamma \cdot \abs{R}^{\mathrm{max}}}{(1 - \gamma)^2} \cdot d_\mathrm{dyn} \br{T_1, T_2} .
\end{equation}
\end{theorem}

When the expert policy is soft-optimal, we use Lemma~\ref{thm:first_week} and Simulation Lemma~\cite{kearns1998near,even2003approximate} to obtain the following bound on the performance gap:

\begin{align*}
\abs{V^{\pi_1}_{M_{\boldsymbol{\theta^*}}^L} - V^{\pi_2}_{M_{\boldsymbol{\theta^*}}^L}} ~\leq~& \abs{V^{\pi_1}_{M_{\boldsymbol{\theta^*}}^L} - V^{\pi^{\mathrm{soft}}_{M_{\boldsymbol{\theta^*}}^L}}_{M_{\boldsymbol{\theta^*}}^L}} + \abs{V^{\pi^{\mathrm{soft}}_{M_{\boldsymbol{\theta^*}}^L}}_{M_{\boldsymbol{\theta^*}}^L} - V^{\pi^{\mathrm{soft}}_{M_{\boldsymbol{\theta^*}}^E}}_{M_{\boldsymbol{\theta^*}}^E}} + \abs{V^{\pi^{\mathrm{soft}}_{M_{\boldsymbol{\theta^*}}^E}}_{M_{\boldsymbol{\theta^*}}^E} - V^{\pi_2}_{M_{\boldsymbol{\theta^*}}^L}} \\
~=~& \abs{\ip{\boldsymbol{\theta^*}}{\boldsymbol{\rho}^{\pi_1}_{M^L} - \boldsymbol{\rho}^{\pi^{\mathrm{soft}}_{M_{\boldsymbol{\theta^*}}^L}}_{M^L}}} + \abs{V^{\pi^{\mathrm{soft}}_{M_{\boldsymbol{\theta^*}}^L}}_{M_{\boldsymbol{\theta^*}}^L} - V^{\pi^{\mathrm{soft}}_{M_{\boldsymbol{\theta^*}}^E}}_{M_{\boldsymbol{\theta^*}}^E}} + \abs{\ip{\boldsymbol{\theta^*}}{\boldsymbol{\rho}^{\pi^{\mathrm{soft}}_{M_{\boldsymbol{\theta^*}}^E}}_{M^E} - \boldsymbol{\rho}^{\pi_2}_{M^L}}} \\
~=~& \abs{V^{\pi^{\mathrm{soft}}_{M_{\boldsymbol{\theta^*}}^L}}_{M_{\boldsymbol{\theta^*}}^L} - V^{\pi^{\mathrm{soft}}_{M_{\boldsymbol{\theta^*}}^E}}_{M_{\boldsymbol{\theta^*}}^E}} \\
~\leq~& \abs{V^{\pi^{\mathrm{soft}}_{M_{\boldsymbol{\theta^*}}^L}}_{M_{\boldsymbol{\theta^*}}^L} - V^{\pi^{\mathrm{soft}}_{M_{\boldsymbol{\theta^*}}^L}}_{M_{\boldsymbol{\theta^*}}^E}} + \abs{V^{\pi^{\mathrm{soft}}_{M_{\boldsymbol{\theta^*}}^L}}_{M_{\boldsymbol{\theta^*}}^E} - V^{\pi^{\mathrm{soft}}_{M_{\boldsymbol{\theta^*}}^E}}_{M_{\boldsymbol{\theta^*}}^E}} \\
~\leq~& \frac{\gamma \cdot \abs{R_{\boldsymbol{\theta^*}}}^{\mathrm{max}}}{(1 - \gamma)^2} \cdot d_\mathrm{dyn} \br{T^L, T^E} + \frac{2 \cdot \kappa_{\boldsymbol{\theta^*}} \cdot \abs{R_{\boldsymbol{\theta^*}}}^{\mathrm{max}}}{(1 - \gamma)^3} \cdot \sqrt{d_\mathrm{dyn} \br{T^L, T^E}}
\end{align*}

\subsection{Proof of Theorem~\ref{thm:occ_states}}
\label{app:existence pi star}

\subsubsection{Impossibility to match the State-action Occupancy Measure}
\label{app:state-only-reward}

We overload the notation $\rho^{\pi}_{M}$ to denote the state-action occupancy measure as well, which is defined as follows:
\begin{equation*}
\rho^{\pi}_{M}(s,a) ~:=~ \pi(a|s) \rho^{\pi}_{M}(s) .
\end{equation*}
Before proving the theorem, we show that finding the policy $\pi^L$ whose state-action occupancy measure matches the state-action visitation frequency $\boldsymbol{\rho}$ of the expert policy\footnote{In this proof, the expert policy is denoted by $\pi^E$. In the specific case of our paper, it stands for either $\pi^*_{M_{\boldsymbol{\theta}}^L}$ or $\pi^*_{M_{\boldsymbol{\theta}}^E}$. However, the result holds for every valid expert policy.} $\pi^E$ is impossible in case of model mismatch. 
Consider:
\begin{align*}
\rho(s,a) ~=~& \rho^{\pi^L}_{M^L}(s,a) \\
\rho(s) \pi^E(a|s) ~=~& \rho^{\pi^L}_{M^L}(s) \pi^L(a|s) \\
\pi^L(a|s) ~=~& \pi^E(a|s) \frac{\rho(s)}{\rho^{\pi^L}_{M^L}(s)}
\end{align*}
Notice that the policy $\pi^L$ is normalized only if we require that $\frac{\rho(s)}{\rho^{\pi^L}_{M^L}(s)} = 1$.
This implies that $\pi^L(s|a) = \pi^E(s|a)$. However, the same policy can not induce the same state occupancy measure under different transition dynamics, it follows that $\frac{\rho(s)}{\rho^{\pi^L}_{M^L}(s)} \neq 1$. We reached a contradiction that allows us to conclude that $\pi^L$ can match the state-action occupancy measure only in absence of model mismatch. Therefore, when there is a model mismatch, the feasible set of~\eqref{opt_start} would be empty if state-action occupancy measures were used in posing the constraint. In addition, even if the two environments were the same, only the expert policy would have been in the feasible set because there exists an injective mapping from state-action visitation frequencies to policies as already noted in~\cite{bloem2014infinite, boularias2010bootstrapping}.

\subsubsection{Theorem Proof} \label{app:occ_states}

\begin{proof}
If there exists a policy $\pi^L$ that matches the expert state occupancy measure $\boldsymbol{\rho}$ in the environment $M^L$, the Bellman flow constraints~\cite{boularias2010bootstrapping} lead to the following equation for each state $s \in \mathcal{S}$:
\begin{equation}
  \rho(s) - (1 - \gamma)P_0(s) ~=~ \gamma \sum_{s^\prime , a^\prime} \rho (s^\prime) \pi^L(a^\prime|s^\prime) T^L(s | s^\prime, a^\prime)  \label{LE}
\end{equation}
This can be seen by writing the Bellman flow constraints for the expert policy $\pi^E$ with transition dynamics $T^E$, and for the policy $\pi^L$ with transition dynamics $T^L$:
\begin{align}
\rho(s) - (1 - \gamma)P_0(s) ~=~& \gamma \sum_{s^\prime , a^\prime} \rho (s^\prime) \pi^E(a^\prime|s^\prime) T^E(s | s^\prime, a^\prime) \label{E}\\
\rho^{\pi^L}_{M^L}(s) - (1 - \gamma)P_0(s) ~=~& \gamma \sum_{s^\prime , a^\prime} \rho^{\pi^L}_{M^L} (s^\prime) \pi^L(a^\prime|s^\prime) T^L(s | s^\prime, a^\prime) \label{L}
\end{align}
By definition of $\pi^L$, the two occupancy measures are equal, so we can equate the LHS of \eqref{E} to the RHS of \eqref{L}, obtaining:
\begin{equation*}
\rho(s) - (1 - \gamma)P_0(s) ~=~ \gamma \sum_{s^\prime , a^\prime} \rho^{\pi^L}_{M^L} (s^\prime) \pi^L(a^\prime|s^\prime) T^L(s | s^\prime, a^\prime)
\end{equation*}
Finally, replacing $\rho$ in the RHS, one obtains the equation in \eqref{LE}. In addition, for each state we have the condition on the normalization of the policy:
\begin{equation*}
1 ~=~ \sum_a \pi^L(a|s) ,
\quad \forall s \in \mathcal{S}
\end{equation*}
All these conditions can be seen as an underdetermined system with $2 \abs{\mathcal{S}}$ equations ($\abs{\mathcal{S}}$ for normalization, and $\abs{\mathcal{S}}$ for the Bellman flow constraints). The unknown is the policy $\pi^\ast$ represented by the $\abs{S}\abs{A}$ entries of the vector $\boldsymbol{\pi^\ast}$, formally defined in \eqref{pi}.

We introduce the matrix $\boldsymbol{T}$. In the first $\abs{\mathcal{S}}$ rows, the entry in the $s^{\mathrm{th}}$ row and $\br{s' \abs{\mathcal{A}} + a'}^\mathrm{th}$ column is the element $\rho (s^\prime) T^L(s | s^\prime, a^\prime)$. In the last $\abs{\mathcal{S}}$ rows, the entries are instead given by $1$ from position $s'\abs{\mathcal{A}}$ to position $s'\abs{\mathcal{A}} + \abs{\mathcal{A}}$. These rows of the matrix serves to impose the normalization condition for each possible state. A clearer block structure representation is given in Section~\ref{sec:existence}.

We can thus write the underdetermined system as:
\begin{equation}
    \begin{bmatrix}
     \boldsymbol{\rho} - (1 - \gamma)\boldsymbol{P}_{0} \\
        \boldsymbol{1}_{\abs{\mathcal{S}}}
    \end{bmatrix}
 ~=~ \boldsymbol{T} \boldsymbol{\pi}^L ,
 \label{undersystem}
\end{equation}
where the left hand side is a vector whose first $\abs{\mathcal{S}}$ positions are the element-wise difference between the state occupancy measure and the initial probability distribution for each state, and the second half are all ones. Recognising that this matches the vector $\boldsymbol{v}$ described in Section~\ref{sec:existence}, we can rewrite the system as:
\begin{equation}
     \boldsymbol{v}
 ~=~ \boldsymbol{T} \boldsymbol{\pi}^L
\end{equation}
The right hand side is instead written using the matrix $\boldsymbol{T}$, and the unknown matching policy vector $\boldsymbol{\pi}^L$.
A direct application of the Rouché-Capelli theorem gives that a linear system admits solutions if and only if the rank of the coeffient matrix is equal to the rank of the coefficient matrix augmented with the known vector. In our case it is:
\begin{equation}
    \mathrm{rank}\br{\boldsymbol{T}} = \mathrm{rank}\br{\boldsymbol{T}|\boldsymbol{v} }
    \label{iff_cond}
\end{equation}

This fact limits the class of perturbation in the dynamics that can be considered still achieving perfect matching.
Corollary ~\ref{corollary} follows because in the case of determined or underdetermined system, i.e. when $\abs{\mathcal{A}} > 1$, the matrix $\boldsymbol{T}$ has rank no larger than $\min(2\abs{\mathcal{S}}, \abs{\mathcal{S}}\abs{\mathcal{A}}) = 2\abs{\mathcal{S}}$  that is the number of rows of the matrix.
It follows that under this assumption, $\boldsymbol{T}$ is full rank when its rank is equal to $2\abs{\mathcal{S}}$. The augmented matrix $(\boldsymbol{T}|\boldsymbol{v})$ will also have a rank upper bounded by $\min(2\abs{\mathcal{S}}, \abs{\mathcal{S}}\abs{\mathcal{A}} + 1) = 2\abs{\mathcal{S}}$ since it has constructed adding one column.
This implies that, when $\boldsymbol{T}$ is full rank, equation \eqref{iff_cond} holds.

\paragraph{Block Representation of the Matching Policy Vector $\boldsymbol{\pi}^L$.}
For each state $s \in \mathcal{S}$, we can define a local matching policy vector $\boldsymbol{\pi}^L(s) \in \mathbb{R}^{\abs{\mathcal{A}}}$ as:
\begin{equation*}
    \boldsymbol{\pi}^L(s) ~=~ \begin{bmatrix}
    \pi(a_1| s) \\
    \pi(a_2| s) \\
    \vdots \\
    \pi(a_{\abs{\mathcal{A}}}| s)
    \end{bmatrix}
\end{equation*}

Then, the matching policy vector $\boldsymbol{\pi}^L \in \mathbb{R}^{\abs{\mathcal{S}}\abs{\mathcal{A}}}$ is given by the vertical stacking of the local matching vectors:

\begin{equation}
    \boldsymbol{\pi}^L ~=~ \begin{bmatrix}
    \boldsymbol{\pi^L}(s_1) \\
    \boldsymbol{\pi^L}(s_2) \\   
    \vdots \\
    \boldsymbol{\pi^L}(s_{\abs{\mathcal{S}}}) \\ 
    \end{bmatrix}
    \label{pi}
\end{equation}
\end{proof}

\subsection{Upper bound for the Reward Transfer Strategy}
\label{app:reward-transfer}

Let $\pi^L$ be the policy obtained from the reward transfer strategy explained in Section~\ref{sec:mismatch-learning}, and $\pi_1 := \pi^{\mathrm{soft}}_{M_{\boldsymbol{\theta_L}}^L}$.

\begin{theorem}
\label{thm:new-reward-transfer-bound}
The performance gap between the policies $\pi_1$ and $\pi^L$ on the MDP $M_{\boldsymbol{\theta^*}}^L$ is bounded as follows:
\begin{align*}
& \abs{V^{\pi_1}_{M_{\boldsymbol{\theta^*}}^L} - V^{\pi^L}_{M_{\boldsymbol{\theta^*}}^L}} \\
~\leq~& \frac{\abs{R_{\boldsymbol{\theta^*}}}^{\mathrm{max}}}{(1 - \gamma)^2} \cdot \bc{\gamma \cdot d_\mathrm{dyn} \br{T^L, T^E} + \frac{2 \cdot \kappa_{\boldsymbol{\theta}^\mathrm{train}}}{1 - \gamma} \cdot \sqrt{d_\mathrm{dyn} \br{T^\mathrm{train}, T^L}} + \gamma \cdot d_\mathrm{dyn} \br{T^{\mathrm{train}}, T^L} + d_\mathrm{pol} \br{\pi_4, \pi^L}}
\end{align*}
\end{theorem}

\begin{proof}
We define $\pi_3 := \pi^{\mathrm{soft}}_{M^{\mathrm{train}}_{\boldsymbol{\theta}^\mathrm{train}}}$ and $\pi_4 := \pi^{\mathrm{soft}}_{M_{\boldsymbol{\theta}^\mathrm{train}}^L}$. First, consider the following:
\begin{align*}
\abs{V^{\pi_3}_{M_{\boldsymbol{\theta^*}}^{\mathrm{train}}} - V^{\pi_4}_{M_{\boldsymbol{\theta^*}}^{\mathrm{train}}}} &~=~ \frac{1}{1 - \gamma} \cdot \abs{\sum_{s} \bc{\rho^{\pi_3}_{M^{\mathrm{train}}} (s) - \rho^{\pi_4}_{M^{\mathrm{train}}} (s)} R_{\boldsymbol{\theta^*}}(s)} \\
&~\leq~ \frac{1}{1 - \gamma} \cdot  \sum_{s} \abs{\rho^{\pi_3}_{M^{\mathrm{train}}} (s) - \rho^{\pi_4}_{M^{\mathrm{train}}} (s)} \cdot \abs{ R_{\boldsymbol{\theta^*}}(s)} \\
&~\leq~ \frac{\abs{R_{\boldsymbol{\theta^*}}}^{\mathrm{max}}}{1 - \gamma} \cdot \sum_{s} \abs{\rho^{\pi_3}_{M^{\mathrm{train}}} (s) - \rho^{\pi_4}_{M^{\mathrm{train}}} (s)}  \\
&~=~ \frac{\abs{R_{\boldsymbol{\theta^*}}}^{\mathrm{max}}}{1 - \gamma} \cdot  \norm{\boldsymbol{\rho}^{\pi_3}_{M^{\mathrm{train}}} - \boldsymbol{\rho}^{\pi_4}_{M^{\mathrm{train}}}}_1 \\
&~\stackrel{\mathrm{a}}{\leq}~ \frac{\abs{R_{\boldsymbol{\theta^*}}}^{\mathrm{max}}}{(1 - \gamma)^2} \cdot  d_\mathrm{pol} \br{\pi_3, \pi_4} \\
&~\stackrel{\mathrm{b}}{\leq}~ \frac{2 \cdot \kappa_{\boldsymbol{\theta}^\mathrm{train}} \cdot \abs{R_{\boldsymbol{\theta^*}}}^{\mathrm{max}}}{(1 - \gamma)^3} \cdot \sqrt{d_\mathrm{dyn} \br{T^\mathrm{train}, T^L}} ,
\end{align*}
where $\mathrm{a}$ is due to Lemma A.1 in~\cite{sun2018dual}, and $\mathrm{b}$ is due to Lemma~\ref{thm:first_week}. Then, consider the following:
\begin{align*}
& \abs{V^{\pi_1}_{M_{\boldsymbol{\theta^*}}^L} - V^{\pi^L}_{M_{\boldsymbol{\theta^*}}^L}} \\
~\leq~& \abs{V^{\pi_1}_{M_{\boldsymbol{\theta^*}}^L} - V^{\pi^{*}_{M_{\boldsymbol{\theta^*}}^L}}_{M_{\boldsymbol{\theta^*}}^L}} + \abs{V^{\pi^{*}_{M_{\boldsymbol{\theta^*}}^L}}_{M_{\boldsymbol{\theta^*}}^L} - V^{\pi^{*}_{M_{\boldsymbol{\theta^*}}^E}}_{M_{\boldsymbol{\theta^*}}^E}} + \abs{V^{\pi^{*}_{M_{\boldsymbol{\theta^*}}^E}}_{M_{\boldsymbol{\theta^*}}^E} - V^{\pi_3}_{M_{\boldsymbol{\theta^*}}^{\mathrm{train}}}} + \\& \abs{V^{\pi_3}_{M_{\boldsymbol{\theta^*}}^{\mathrm{train}}} - V^{\pi_4}_{M_{\boldsymbol{\theta^*}}^{\mathrm{train}}}} + \abs{V^{\pi_4}_{M_{\boldsymbol{\theta^*}}^{\mathrm{train}}} - V^{\pi_4}_{M_{\boldsymbol{\theta^*}}^{L}}} + \abs{V^{\pi_4}_{M_{\boldsymbol{\theta^*}}^{L}} - V^{\pi^L}_{M_{\boldsymbol{\theta^*}}^{L}}} \\
~\stackrel{\mathrm{a}}{=}~& \abs{V^{\pi^{*}_{M_{\boldsymbol{\theta^*}}^L}}_{M_{\boldsymbol{\theta^*}}^L} - V^{\pi^{*}_{M_{\boldsymbol{\theta^*}}^E}}_{M_{\boldsymbol{\theta^*}}^E}} + \abs{V^{\pi_3}_{M_{\boldsymbol{\theta^*}}^{\mathrm{train}}} - V^{\pi_4}_{M_{\boldsymbol{\theta^*}}^{\mathrm{train}}}} + \abs{V^{\pi_4}_{M_{\boldsymbol{\theta^*}}^{\mathrm{train}}} - V^{\pi_4}_{M_{\boldsymbol{\theta^*}}^{L}}} + \abs{V^{\pi_4}_{M_{\boldsymbol{\theta^*}}^{L}} - V^{\pi^L}_{M_{\boldsymbol{\theta^*}}^{L}}} \\
~\leq~& \abs{V^{\pi^{*}_{M_{\boldsymbol{\theta^*}}^L}}_{M_{\boldsymbol{\theta^*}}^L} - V^{\pi^{*}_{M_{\boldsymbol{\theta^*}}^E}}_{M_{\boldsymbol{\theta^*}}^E}} + \abs{V^{\pi_3}_{M_{\boldsymbol{\theta^*}}^{\mathrm{train}}} - V^{\pi_4}_{M_{\boldsymbol{\theta^*}}^{\mathrm{train}}}} + \abs{V^{\pi_4}_{M_{\boldsymbol{\theta^*}}^{\mathrm{train}}} - V^{\pi_4}_{M_{\boldsymbol{\theta^*}}^{L}}} + \frac{\abs{R_{\boldsymbol{\theta^*}}}^{\mathrm{max}}}{(1 - \gamma)^2} \cdot  d_\mathrm{pol} \br{\pi_4, \pi^L} \\
~\stackrel{\mathrm{b}}{\leq}~& \frac{\gamma \cdot \abs{R_{\boldsymbol{\theta^*}}}^{\mathrm{max}}}{(1 - \gamma)^2} \cdot d_\mathrm{dyn} \br{T^L, T^E} + \abs{V^{\pi_3}_{M_{\boldsymbol{\theta^*}}^{\mathrm{train}}} - V^{\pi_4}_{M_{\boldsymbol{\theta^*}}^{\mathrm{train}}}} + \abs{V^{\pi_4}_{M_{\boldsymbol{\theta^*}}^{\mathrm{train}}} - V^{\pi_4}_{M_{\boldsymbol{\theta^*}}^{L}}} + \frac{\abs{R_{\boldsymbol{\theta^*}}}^{\mathrm{max}}}{(1 - \gamma)^2} \cdot  d_\mathrm{pol} \br{\pi_4, \pi^L} \\
~\stackrel{\mathrm{c}}{\leq}~& \frac{\gamma \cdot \abs{R_{\boldsymbol{\theta^*}}}^{\mathrm{max}}}{(1 - \gamma)^2} \cdot d_\mathrm{dyn} \br{T^L, T^E} + \abs{V^{\pi_3}_{M_{\boldsymbol{\theta^*}}^{\mathrm{train}}} - V^{\pi_4}_{M_{\boldsymbol{\theta^*}}^{\mathrm{train}}}} + \frac{\gamma \cdot \abs{R_{\boldsymbol{\theta^*}}}^{\mathrm{max}}}{(1 - \gamma)^2} \cdot d_\mathrm{dyn} \br{T^{\mathrm{train}}, T^L} + \frac{\abs{R_{\boldsymbol{\theta^*}}}^{\mathrm{max}}}{(1 - \gamma)^2} \cdot  d_\mathrm{pol} \br{\pi_4, \pi^L} \\
~\leq~& \frac{\abs{R_{\boldsymbol{\theta^*}}}^{\mathrm{max}}}{(1 - \gamma)^2} \cdot \bc{\gamma \cdot d_\mathrm{dyn} \br{T^L, T^E} + \frac{2 \cdot \kappa_{\boldsymbol{\theta}^\mathrm{train}}}{1 - \gamma} \cdot \sqrt{d_\mathrm{dyn} \br{T^\mathrm{train}, T^L}} + \gamma \cdot d_\mathrm{dyn} \br{T^{\mathrm{train}}, T^L} + d_\mathrm{pol} \br{\pi_4, \pi^L}} ,
\end{align*}
where $\mathrm{a}$ is due to the fact that $\boldsymbol{\rho}^{\pi_1}_{M^L} = \boldsymbol{\rho}^{\pi^{*}_{M_{\boldsymbol{\theta^*}}^L}}_{M^L}$ and $\boldsymbol{\rho}^{\pi^{*}_{M_{\boldsymbol{\theta^*}}^E}}_{M^E} = \boldsymbol{\rho}^{\pi_3}_{M^{\mathrm{train}}}$; $\mathrm{b}$ is due to Theorem~7 in~\cite{zhang2020multi}; and $\mathrm{c}$ is due to Simulation Lemma~\cite{kearns1998near,even2003approximate}. 
\end{proof}

When $M^\mathrm{train} = M^E$ and $\pi^L = \pi_4$, the above bound simplifies to:
\[
\abs{V^{\pi_1}_{M_{\boldsymbol{\theta^*}}^L} - V^{\pi^L}_{M_{\boldsymbol{\theta^*}}^L}} ~\leq~ \frac{2 \cdot \abs{R_{\boldsymbol{\theta^*}}}^{\mathrm{max}}}{(1 - \gamma)^2} \cdot \bc{\gamma \cdot d_\mathrm{dyn} \br{T^L, T^E} + \frac{\kappa_{\boldsymbol{\theta}^\mathrm{E}}}{1 - \gamma} \cdot \sqrt{d_\mathrm{dyn} \br{T^L, T^E}}} .
\]

\newpage
\section{Further Details of Section~\ref{sec:two_players_entropy}}
\label{app:solution-more-details}

\subsection{Relation between Robust MDP and Markov Games}
\label{app:details-equivalence-new}
This section gives a proof for the inequality in equation~\eqref{equivalence_new}:
\begin{proof}
We first introduce the set:
\begin{equation*}
    \underline{\mathcal{T}}^{L,\alpha} = \bc{T \quad | \quad T(s^\prime|s,a) = \alpha T^L(s^\prime|s,a) + (1 - \alpha) \bar{T}(s^\prime|s), \quad \bar{T}(s^\prime|s) = \sum_a \pi(a|s)T(s^\prime| s,a), \forall \pi \in \Delta_{A|S} }
\end{equation*}
Clearly, it holds that: $\underline{\mathcal{T}}^{L,\alpha} \subset \mathcal{T}^{L,\alpha}$ that implies:
\begin{equation*}
\max_{\pi^{\mathrm{pl}} \in \Pi} \min_{T \in \mathcal{T}^{L,\alpha}}  \E{G \bigm| \pi^{\mathrm{pl}}, P_0, T} \leq
\max_{\pi^{\mathrm{pl}} \in \Pi} \min_{T \in \underline{\mathcal{T}}^{L,\alpha}}  \E{G \bigm| \pi^{\mathrm{pl}}, P_0, T}
\end{equation*}
Finally, from~\cite[Section 3.1]{tessler2019action} we have:
\begin{equation*}
    \max_{\pi^{\mathrm{pl}} \in \Pi} \min_{T \in \underline{\mathcal{T}}^{L,\alpha}}  \E{G \bigm| \pi^{\mathrm{pl}}, P_0, T} = \max_{\pi^{\mathrm{pl}} \in \Pi} \min_{\pi^{\mathrm{op}} \in \Pi} \E{G \bigm| \alpha \pi^{\mathrm{pl}} + (1 - \alpha) \pi^{\mathrm{op}}, M^L}
\end{equation*}
We conclude that:
\begin{equation*}
    \max_{\pi^{\mathrm{pl}} \in \Pi} \min_{T \in \mathcal{T}^{L,\alpha}}  \E{G \bigm| \pi^{\mathrm{pl}}, P_0, T} \leq \max_{\pi^{\mathrm{pl}} \in \Pi} \min_{\pi^{\mathrm{op}} \in \Pi} \E{G \bigm| \alpha \pi^{\mathrm{pl}} + (1 - \alpha) \pi^{\mathrm{op}}, M^L}
\end{equation*}
Therefore the inequality in~\eqref{equivalence_new} holds.
\end{proof}
A natural question is whether the tightness of the bound can be controlled. An affirmative answer come from the following theorem relying on Lemma~\ref{thm:first_week}.
\begin{theorem}
Let $T^*$ be a saddle point when the $\mathrm{min}$ acts over the set $\mathcal{T}^{L,\alpha}$ and $\underline{T}^*$ be a saddle point when the $\mathrm{min}$ acts over the set $\underline{\mathcal{T}}^{L,\alpha}$. Then, the following holds:
\begin{align*}
\max_{\pi^{\mathrm{pl}} \in \Pi} \min_{\pi^{\mathrm{op}} \in \Pi} \E{G \bigm| \alpha \pi^{\mathrm{pl}} + (1 - \alpha) \pi^{\mathrm{op}}, M^L} -& \max_{\pi^{\mathrm{pl}} \in \Pi} \min_{T \in \mathcal{T}^{L,\alpha}}  \E{G \bigm| \pi^{\mathrm{pl}}, P_0, T} \\ &\leq \frac{2|R_{\theta}^{\mathrm{max}}|}{(1 - \gamma)^2} \min \bc{\frac{\kappa_{\theta} \sqrt{d(T^*,\underline{T}^*)}}{(1 - \gamma)}, \frac{\kappa^2_{\theta} d(T^*,\underline{T}^*)}{(1 - \gamma)^2}}
\end{align*}
\end{theorem}
\begin{proof}
\begin{align*}
    \max_{\pi^{\mathrm{pl}} \in \Pi} \min_{\pi^{\mathrm{op}} \in \Pi} &\E{G \bigm| \alpha \pi^{\mathrm{pl}} + (1 - \alpha) \pi^{\mathrm{op}}, M^L} - \max_{\pi^{\mathrm{pl}} \in \Pi} \min_{T \in \mathcal{T}^{L,\alpha}}  \E{G \bigm| \pi^{\mathrm{pl}}, P_0, T} \\ &= \max_{\pi^{\mathrm{pl}} \in \Pi} \min_{T \in \underline{\mathcal{T}}^{L,\alpha}}  \E{G \bigm| \pi^{\mathrm{pl}}, P_0, T} - \max_{\pi^{\mathrm{pl}} \in \Pi} \min_{T \in \mathcal{T}}^{L,\alpha}  \E{G \bigm| \pi^{\mathrm{pl}}, P_0, T} \\ &= \max_{\pi^{\mathrm{pl}} \in \Pi} \E{G \bigm| \pi^{\mathrm{pl}}, P_0, \underline{T}^*} - \max_{\pi^{\mathrm{pl}} \in \Pi}\E{G \bigm| \pi^{\mathrm{pl}}, P_0, T^*} \\ & \leq \frac{|R_\theta|^{\mathrm{max}}}{(1 - \gamma)^2} d_{\mathrm{pol}}\br{\pi^{\mathrm{soft}}_{\underline{T}^*}, \pi^{\mathrm{soft}}_{T^*}} \leq \frac{2|R_{\theta}^{\mathrm{max}}|}{(1 - \gamma)^2} \min \bc{\frac{\kappa_{\theta} \sqrt{d(T^*,\underline{T}^*)}}{(1 - \gamma)}, \frac{\kappa^2_{\theta} d(T^*,\underline{T}^*)}{(1 - \gamma)^2}}
\end{align*}
Where the second last inequality holds with similar steps of the proof of Theorem~\ref{thm:new-reward-transfer-bound} and the last inequality applies thanks to Lemma~\ref{thm:first_week}.
\end{proof}

\subsection{Deriving Gradient-based Method from Worst-case Predictive Log-loss} \label{sec:gradient}

We consider again in this section the optimization problem given in \eqref{opt_start} with model mismatch, i.e., using $\boldsymbol{\rho}^{\pi^*_{M_{\boldsymbol{\theta^*}}^E}}_{M^E}$ as $\boldsymbol{\rho}$. The aim of this section is to give an alternative point of view on this program based on a proper adaptation of the worst-case predictive log-loss~\cite{ziebart2010modeling}[Corollary~6.3] to the model mismatch case. 

~\cite{ziebart2010modeling} proved that the maximum causal entropy policy satisfying the optimization constraints is also the distribution that minimizes the worst-case predictive log-loss. However, the proof leverages on the fact that learner and expert MDPs coincide, an assumption that fails in the scenario of our work. 

This section extends the result to the general case, where expert and learner MDP do not coincide, thanks to the two following contributions:  (i) we show that the MCE constrained maximization given in \eqref{primal_start_2} in the main text can be recast as a worst-case predictive log-loss constrained minimization and (ii) that this alternative problem leads to the same reward weights update found in the main text for the dual of the program \eqref{primal_start_2}.
We start reporting again the optimization problem of interest:
\begin{align}
\argmax_{\pi \in \Pi} \quad& \E{\sum_{t = 0}^\infty - \gamma^t \log \pi(a_t | s_t) \biggm| \pi, M^L} \label{primal_start}\\
\text{subject to} \quad& \boldsymbol{\rho}^{\pi^*_{M_{\boldsymbol{\theta^*}}^E}}_{M^E} ~=~ \boldsymbol{\rho}^{\pi}_{M^L}
\label{primal_end}
\end{align}
An alternative interpretation of the entropy is given by the following property:
\begin{equation*}
\E{\sum_{t = 0}^\infty - \gamma^t \log \pi(a_t| s_t) \biggm| \pi, M^L} ~=~ \inf_{\bar{\pi}} \E{\sum_{t = 0}^\infty - \gamma^t \log \bar{\pi}(a_t| s_t) \biggm| \pi, M^L} , \quad \forall \pi
\end{equation*}
Thus, it holds also for $\pi^{\mathrm{soft}}_{M_{\boldsymbol{\theta_E}}^L}$ solution of the primal optimization problem~\eqref{primal_start}-~\eqref{primal_end}, that exists if Theorem~\ref{thm:occ_states} is satisfied. In addition, to maintain the equivalence with the program~\eqref{primal_start}-\eqref{primal_end}, we restrict the $\inf$ search space to the feasible set of~\eqref{primal_start}-\eqref{primal_end} that we denote $\widetilde{\Pi}$. 
\begin{equation*}
    \E{\sum_{t = 0}^\infty - \gamma^t \log \pi^{\mathrm{soft}}_{M_{\boldsymbol{\theta_E}}^L}(a_t| s_t) \biggm| \pi^{\mathrm{soft}}_{M_{\boldsymbol{\theta_E}}^L}, M^L } ~=~ \inf_{\bar{\pi} \in \widetilde{\Pi}} \E{\sum_{t = 0}^\infty - \gamma^t \log \bar{\pi}(a_t| s_t) \biggm| \pi^{\mathrm{soft}}_{M_{\boldsymbol{\theta_E}}^L}, M^L}
\end{equation*}
Notice that since $\pi^{\mathrm{soft}}_{M_{\boldsymbol{\theta_E}}^L}$ is solution of the maximization problem, we can indicate the the previous equality as:
\begin{align}
\sup_{\tilde{\pi} \in \widetilde{\Pi}} \E{\sum_{t = 0}^\infty - \gamma^t \log \tilde{\pi}(a_t| s_t) \biggm| \tilde{\pi}, M^L} &~=~  \sup_{\tilde{\pi} \in \widetilde{\Pi}} \inf_{\bar{\pi} \in \widetilde{\Pi}} \E{\sum_{t = 0}^\infty - \gamma^t \log \bar{\pi}(a_t| s_t) \biggm| \tilde{\pi}, M^L} 
\label{log-loss}\\ \nonumber &~=~ \inf_{\bar{\pi} \in \widetilde{\Pi}} \sup_{\tilde{\pi} \in \widetilde{\Pi}} \E{\sum_{t = 0}^\infty - \gamma^t \log \bar{\pi}(a_t| s_t) \biggm| \tilde{\pi}, M^L}
\end{align}
The last equality follows by min-max equality that holds since the objective is convex in $\bar{\pi}$ and concave in $\tilde{\pi}$.
It is thus natural to interpret the quantity:
\begin{equation}
    c(\pi) = \E{\sum_{t = 0}^\infty - \gamma^t \log \pi(a_t| s_t) \biggm| \pi^{\mathrm{soft}}_{M_{\boldsymbol{\theta_E}}^L}, M^L}
    \label{cost}
\end{equation} 
as the cost function associated to the policy $\pi$ because, according to \eqref{log-loss}, this quantity is equivalent to the worst-case predictive log-loss among the policies of the feasible set $\widetilde{\Pi}$. It can be seen that the loss inherits the feasible set of the original MCE maximization problem as search space for the $\inf$ and $\sup$ operations. It follows that in case of model mismatch, the loss studied in~\cite{ziebart2010modeling}[Corollary~6.3] is modified because a different set must be used as search space for the $\inf$ and $\sup$.

In the following, we develop a gradient based method to minimize this cost and, thus, the worst case predictive log-loss.
\footnote{
If we used $\boldsymbol{\rho}^{\pi^{\mathrm{soft}}_{M_{\boldsymbol{\theta^*}}^L}}_{M^L}$ as $\boldsymbol{\rho}$, we would have obtained the cost
  $c(\pi) = \E{\sum_{t = 0}^\infty - \gamma^t \log \pi(a_t| s_t) \biggm| \pi^{\mathrm{soft}}_{M_{\boldsymbol{\theta^*}}^L}, M^L} $. In this case, the gradient is known see~\cite{parameswaran2019interactive}.
  }

Furthermore, we can already consider that $\pi$ belongs to the family of soft Bellman policies parametrized by the parameter $\boldsymbol{\theta}$ in the environment $M_{\boldsymbol{\theta}}^L$ because they are the family of distributions attaining maximum discounted causal entropy (see~\cite{bloem2014infinite}[~Lemma 3]). The cost is, in this case, expressed for the parameter $\boldsymbol{\theta}$:
\begin{equation}
    c(\boldsymbol{\theta}) = \E{\sum_{t = 0}^\infty - \gamma^t \log \pi^{\mathrm{soft}}_{M_{\boldsymbol{\theta}}^L}(a_t| s_t) \biggm| \pi^{\mathrm{soft}}_{M_{\boldsymbol{\theta_E}}^L}, M^L}
    \label{cost_softmax}
\end{equation} 
\begin{theorem} \label{thm:gradient}
If $\pi^{\mathrm{soft}}_{M_{\boldsymbol{\theta_E}}^L}$ exists, the gradient of the cost function given in \eqref{cost_softmax} is equal to:
\begin{equation*}
    \nabla_{\boldsymbol{\theta}} c(\boldsymbol{\theta}) = \sum_{s} \br{ \rho^{\pi^{\mathrm{soft}}_{M_{\boldsymbol{\theta}}^L}}_{M_L}(s) - \rho^{\pi^*_{M_{\boldsymbol{\theta^*}}^E}}_{M_E}(s)}\nabla_{\boldsymbol{\theta}} R_{\boldsymbol{\theta}}(s)
\end{equation*}
In addition, this result generalizes when the expectation in the cost function is taken with respect to any of the policies in the feasible set of the primal problem \eqref{primal_start}-\eqref{primal_end}. 
\end{theorem}
Note that choosing one-hot features, we have $\nabla_{\boldsymbol{\theta}} c(\boldsymbol{\theta}) =   \boldsymbol{\rho}^{\pi^{\mathrm{soft}}_{M_{\boldsymbol{\theta}}^L}}_{M_L} - \boldsymbol{\rho}^{\pi^*_{M_{\boldsymbol{\theta^*}}^E}}_{M_E}$ as used in Section~\ref{sec:two_players_entropy}.

\paragraph{Uniqueness of the Solution.} The cost in equation \eqref{cost_softmax} is strictly convex in the soft max policy $\pi^{\mathrm{soft}}_{M_{\boldsymbol{\theta}}^L}$ because $- \log (\cdot)$ is a strictly convex function and the cost consists in a linear composition of these strictly convex functions. Thus the gradient descent converges to a unique soft optimal policy. In addition, the fact that for each possible $\boldsymbol{\theta}$, the quantity $\log \pi^{\mathrm{soft}}_{M_{\boldsymbol{\theta}}^L} = Q^{ \mathrm{soft}}_{M_{\boldsymbol{\theta}}}(s,a) - V^{\mathrm{soft}}_{M_{\boldsymbol{\theta}}}(s)$ is convex in $\boldsymbol{\theta}$ since the soft value functions ($Q^{ \mathrm{soft}}_{M_{\boldsymbol{\theta}}}(s,a)$ and $V^{\mathrm{soft}}_{M_{\boldsymbol{\theta}}}(s)$) are given by a sum of rewards that are linear in $\boldsymbol{\theta}$ and LogSumExp funtions that are convex. It follows that $\log \pi^{\mathrm{soft}}_{M_{\boldsymbol{\theta}}^L}$ is a composition of linear and convex functions for each state actions pairs. Consequently the cost given in \eqref{cost_softmax} is convex in $\boldsymbol{\theta}$. It follows that alternating an update of the parameter $\boldsymbol{\theta}$ using a gradient descent scheme based on the gradient given by Theorem~\ref{thm:gradient} with a derivation of the corresponding soft-optimal policy by Soft-Value-Iteration, one can converge to $\boldsymbol{\theta_E}$ whose corresponding soft optimal policy is $\pi^{\mathrm{soft}}_{M_{\boldsymbol{\theta_E}}^L}$. However, considering that the function LogSumExp is convex but not strictly convex there is no unique $\boldsymbol{\theta_E}$ corresponding to the soft optimal policy $\pi^{\mathrm{soft}}_{M_{\boldsymbol{\theta_E}}^L}$. 

\subsubsection{Proof of Theorem~\ref{thm:gradient}} \label{thm:gradient_proof}

\begin{proof}
We will make use of the following quantities:
\begin{itemize}
    \item $P^{\pi^{\mathrm{soft}}_{M_{\boldsymbol{\theta}}^L}}_t(s)$ defined as the probability of visiting state $s$ at time $t$ by the policy $\pi^{\mathrm{soft}}_{M_{\boldsymbol{\theta}}^L}$ acting in $M_{\boldsymbol{\theta}}^L$
    \item $P^{\pi^{\mathrm{soft}}_{M_{\boldsymbol{\theta}}^L}}_t(s, a)$ defined as the probability of visiting state $s$ and taking action $a$ from state $s$ at time $t$ by the policy $\pi^{\mathrm{soft}}_{M_{\boldsymbol{\theta}}^L}$ acting in $M_{\boldsymbol{\theta}}^L$
        \item $P^{\pi^{\mathrm{soft}}_{M_{\boldsymbol{\theta_E}}^L}}_t(s)$ defined as the probability of visiting state $s$ at time $t$ by the policy $\pi^{\mathrm{soft}}_{M_{\boldsymbol{\theta_E}}^L}$ acting in $M_{\boldsymbol{\theta}}^L$
    \item $P^{\pi^{\mathrm{soft}}_{M_{\boldsymbol{\theta_E}}^L}}_t(s, a)$ defined as the probability of visiting state $s$ and taking action $a$ from state $s$ at time $t$ by the policy $\pi^{\mathrm{soft}}_{M_{\boldsymbol{\theta_E}}^L}$ acting in $M_{\boldsymbol{\theta}}^L$
\end{itemize}
The cost can be rewritten as:
\begin{align}
    c(\boldsymbol{\theta}) ~=~ &- \sum^{\infty}_{t=0} \gamma^t \sum_{s \in \mathcal{S}} \sum_{a \in \mathcal{A}} P^{\pi^{\mathrm{soft}}_{M_{\boldsymbol{\theta_E}}^L}}_{t} (s,a) \log \pi^{\mathrm{soft}}_{M_{\boldsymbol{\theta}}^L}(a|s) \nonumber \\
    ~=~ &- \sum_{s \in \mathcal{S}} \sum_{a \in \mathcal{A}} P^{\pi^{\mathrm{soft}}_{M_{\boldsymbol{\theta_E}}^L}}_{0} (s,a) \br{ Q^{ \mathrm{soft}}_{M^L_{\boldsymbol{\theta}}}(s,a) - V^{ \mathrm{soft}}_{M^L_{\boldsymbol{\theta}}}(s)} \nonumber \\
    &- \sum_{s \in \mathcal{S}} \sum_{a \in \mathcal{A}} P^{\pi^{\mathrm{soft}}_{M_{\boldsymbol{\theta_E}}^L}}_{1} (s,a) \gamma \br{ Q^{ \mathrm{soft}}_{M^L_{\boldsymbol{\theta}}}(s,a) - V^{ \mathrm{soft}}_{M^L_{\boldsymbol{\theta}}}(s)} \nonumber \\
    &- \sum_{s \in \mathcal{S}} \sum_{a \in \mathcal{A}} P^{\pi^{\mathrm{soft}}_{M_{\boldsymbol{\theta_E}}^L}}_{2} (s,a) \gamma^2 \br{ Q^{ \mathrm{soft}}_{M^L_{\boldsymbol{\theta}}}(s,a) - V^{ \mathrm{soft}}_{M^L_{\boldsymbol{\theta}}}(s)} \nonumber \\ 
    &- \sum_{s \in \mathcal{S}} \sum_{a \in \mathcal{A}} P^{\pi^{\mathrm{soft}}_{M_{\boldsymbol{\theta_E}}^L}}_{3} (s,a) \gamma^3 \br{ Q^{ \mathrm{soft}}_{M^L_{\boldsymbol{\theta}}}(s,a) - V^{ \mathrm{soft}}_{M^L_{\boldsymbol{\theta}}}(s)} \nonumber \\
    & \dots \nonumber \\
    ~=~ & \sum_{s,a} P_0(s) \pi^{\mathrm{soft}}_{M_{\boldsymbol{\theta_E}}^L}(a|s) V^{ \mathrm{soft}}_{M^L_{\boldsymbol{\theta}}}(s)  \label{one}\\
    &- \sum_{s,a} P_0(s) \pi^{\mathrm{soft}}_{M_{\boldsymbol{\theta_E}}^L}(a|s) Q^{ \mathrm{soft}}_{M^L_{\boldsymbol{\theta}}}(s,a) + \gamma \sum_{s,a} P^{\pi^{\mathrm{soft}}_{M_{\boldsymbol{\theta_E}}^L}}_1(s) \pi^{\mathrm{soft}}_{M_{\boldsymbol{\theta_E}}^L}(a|s) V^{ \mathrm{soft}}_{M^L_{\boldsymbol{\theta}}}(s) \label{two}\\
    &- \gamma \sum_{s,a} P^{\pi^{\mathrm{soft}}_{M_{\boldsymbol{\theta_E}}^L}}_1(s) \pi^{\mathrm{soft}}_{M_{\boldsymbol{\theta_E}}^L}(a|s) Q^{ \mathrm{soft}}_{M^L_{\boldsymbol{\theta}}}(s,a) + \gamma^2 \sum_{s,a} P^{\pi^{\mathrm{soft}}_{M_{\boldsymbol{\theta_E}}^L}}_2(s) \pi^{\mathrm{soft}}_{M_{\boldsymbol{\theta_E}}^L}(a|s) V^{ \mathrm{soft}}_{M^L_{\boldsymbol{\theta}}}(s) \label{three} \\    
    &- \gamma^2 \sum_{s,a} P^{\pi^{\mathrm{soft}}_{M_{\boldsymbol{\theta_E}}^L}}_2(s) \pi^{\mathrm{soft}}_{M_{\boldsymbol{\theta_E}}^L}(a|s) Q^{ \mathrm{soft}}_{M^L_{\boldsymbol{\theta}}}(s,a) + \gamma^3 \sum_{s,a} P^{\pi^{\mathrm{soft}}_{M_{\boldsymbol{\theta_E}}^L}}_3(s) \pi^{\mathrm{soft}}_{M_{\boldsymbol{\theta_E}}^L}(a|s) V^{ \mathrm{soft}}_{M^L_{\boldsymbol{\theta}}}(s) \nonumber \\ 
    & \dots \nonumber
\end{align}

The gradient of the term in \eqref{one} has already been derived in~\cite{parameswaran2019interactive} and it is given by:
\begin{equation*}
    \nabla_{\boldsymbol{\theta}} \sum_{s,a} P_0(s) \pi^{\mathrm{soft}}_{M_{\boldsymbol{\theta_E}}^L}(a|s) V^{ \mathrm{soft}}_{M^L_{\boldsymbol{\theta}}}(s) ~=~     \nabla_{\boldsymbol{\theta}} \sum_{s} P_0(s) V^{ \mathrm{soft}}_{M^L_{\boldsymbol{\theta}}}(s) ~=~ \sum_{s,a} \rho^{\pi^{\mathrm{soft}}_{M_{\boldsymbol{\theta}}^L}}_{M_L}(s,a) \nabla_{\boldsymbol{\theta}} R_{\boldsymbol{\theta}}(s,a)
\end{equation*}
Now, we compute the gradient of the following terms starting from \eqref{two}. We notice that this term can be simplified as follows:
\begin{align*}
    &- \sum_{s,a} P_0(s) \pi^{\mathrm{soft}}_{M_{\boldsymbol{\theta_E}}^L}(a|s) Q^{ \mathrm{soft}}_{M^L_{\boldsymbol{\theta}}}(s,a) + \gamma \sum_{s,a} P^{\pi^{\mathrm{soft}}_{M_{\boldsymbol{\theta_E}}^L}}_1(s) \pi^{\mathrm{soft}}_{M_{\boldsymbol{\theta_E}}^L}(a|s) V^{ \mathrm{soft}}_{M^L_{\boldsymbol{\theta}}}(s) \\
    ~=~ &- \sum_{s,a} P_0(s) \pi^{\mathrm{soft}}_{M_{\boldsymbol{\theta_E}}^L}(a|s) \br{R_{\boldsymbol{\theta}}(s,a) + \gamma \sum_{s^\prime} T^L(s^\prime| s,a) V^{ \mathrm{soft}}_{M^L_{\boldsymbol{\theta}}}(s^\prime)} + \gamma \sum_{s,a} P^{\pi^{\mathrm{soft}}_{M_{\boldsymbol{\theta_E}}^L}}_1(s) \pi^{\mathrm{soft}}_{M_{\boldsymbol{\theta_E}}^L}(a|s) V^{ \mathrm{soft}}_{M^L_{\boldsymbol{\theta}}}(s) \\
    ~=~ &- \sum_{s,a} P_0(s) \pi^{\mathrm{soft}}_{M_{\boldsymbol{\theta_E}}^L}(a|s) R_{\boldsymbol{\theta}}(s,a) - \gamma \sum_{s^\prime} \sum_{s,a} T^L(s^\prime| s,a) P_0(s) \pi^{\mathrm{soft}}_{M_{\boldsymbol{\theta_E}}^L}(a|s) V^{ \mathrm{soft}}_{M^L_{\boldsymbol{\theta}}}(s^\prime) + \gamma \sum_{s} P^{\pi^{\mathrm{soft}}_{M_{\boldsymbol{\theta_E}}^L}}_1(s) V^{ \mathrm{soft}}_{M^L_{\boldsymbol{\theta}}}(s) \\
    ~=~ &- \sum_{s,a} P_0(s) \pi^{\mathrm{soft}}_{M_{\boldsymbol{\theta_E}}^L}(a|s) R_{\boldsymbol{\theta}}(s,a) - \gamma \sum_{s^\prime} P^{\pi^{\mathrm{soft}}_{M_{\boldsymbol{\theta_E}}^L}}_1(s^\prime) V^{ \mathrm{soft}}_{M^L_{\boldsymbol{\theta}}}(s^\prime) + \gamma \sum_{s} P^{\pi^{\mathrm{soft}}_{M_{\boldsymbol{\theta_E}}^L}}_1(s) V^{ \mathrm{soft}}_{M^L_{\boldsymbol{\theta}}}(s)
    \\
    ~=~ & - \sum_{s,a} P_0(s) \pi^{\mathrm{soft}}_{M_{\boldsymbol{\theta_E}}^L}(a|s) R_{\boldsymbol{\theta}}(s,a)
\end{align*}
With similar steps, all the terms except the first one are given by
\begin{equation*}
    - \sum^{\infty}_{t = 0} \sum_{s,a} P^{\pi^{\mathrm{soft}}_{M_{\boldsymbol{\theta_E}}^L}}_t(s,a) \gamma^t R_{\boldsymbol{\theta}}(s,a) ~=~ - \sum_{s,a} \rho^{\pi^{\mathrm{soft}}_{M_{\boldsymbol{\theta_E}}^L}}_{M^L} (s,a) R_{\boldsymbol{\theta}}(s,a)
\end{equation*}
If the reward is state only, then and we can marginalize the sum over the action and then exploiting the fact that $\pi^{\mathrm{soft}}_{M_{\boldsymbol{\theta_E}}^L}$ is in the feasible set of the primal problem \eqref{primal_start}-\eqref{primal_end}:
\begin{equation*}
    - \sum^{\infty}_{t = 0} \sum_{s,a} P^{\pi^{\mathrm{soft}}_{M_{\boldsymbol{\theta_E}}^L}}_t(s,a) \gamma^t R_{\boldsymbol{\theta}}(s) ~=~ - \sum_{s} \rho^{\pi^{\mathrm{soft}}_{M_{\boldsymbol{\theta_E}}^L}}_{M^L} (s) R_{\boldsymbol{\theta}}(s) ~=~ - \sum_{s} \rho^{\pi^*_{M_{\boldsymbol{\theta^*}}^E}}_{M^E} (s) R_{\boldsymbol{\theta}}(s)
\end{equation*}

It follows that the gradient of all the terms but the first term \eqref{one} is given by:
\begin{equation*}
    - \sum_{s} \rho^{\pi^*_{M_{\boldsymbol{\theta^*}}^E}}_{M^E} (s) R_{\boldsymbol{\theta}}(s)
\end{equation*}
Finally, the proof is concluded by summing the latest result to the gradient of \eqref{one} that gives:
\begin{equation*}
    \nabla_{\boldsymbol{\theta}} c(\boldsymbol{\theta}) ~=~ \sum_{s} \br{ \rho^{\pi^{\mathrm{soft}}_{M_{\boldsymbol{\theta}}^L}}_{M^L}(s) - \rho^{\pi^*_{M_{\boldsymbol{\theta^*}}^E}}_{M^E}(s)}\nabla_{\boldsymbol{\theta}} R_{\boldsymbol{\theta}}(s)
\end{equation*}

It can be noticed that the computation of this gradient exploits only the fact that $\pi^{\mathrm{soft}}_{M_{\boldsymbol{\theta_E}}^L}$ is in the primal feasible set and not the fact that it maximizes the discounted causal entropy. It follows that all the policies in the primal feasible set share this gradient. This means that this gradients aim to move the learner policy towards the primal feasible set while the causal entropy is then maximized by Soft-Value-Iteration. 
\end{proof}

\subsection{Solving the Two-Player Markov Game}
\label{app:softQ}

\begin{algorithm}[h]
    \caption{Value Iteration for Two-Player Markov Game}
    \label{alg:TwoplayersDynProg}
    \begin{spacing}{0.8}
    \begin{algorithmic}
    \STATE \textbf{Initialize:} $Q(s, a^{\mathrm{pl}}, a^{\mathrm{op}}) \gets 0$, $V(s) \gets 0$
        \WHILE{not converged}
            \FOR {$s \in \mathcal{S}$}
                \FOR {$(a^{\mathrm{pl}},a^{\mathrm{op}}) \in \mathcal{A} \times \mathcal{A}$}
                    \STATE update joint Q-function as follows:
                    \begin{equation}
                        Q(s, a^{\mathrm{pl}}, a^{\mathrm{op}}) ~=~ R(s) + \gamma \sum_{s'} T^{\mathrm{two},L,\alpha} (s'|s, a^{\mathrm{pl}}, a^{\mathrm{op}}) V(s')
                        \label{SGstep1}
                    \end{equation}
                \ENDFOR
            \STATE update joint V-function as follows:
                \begin{equation}
                V(s) ~=~ \log \sum_{a^{\mathrm{pl}}} \exp \br{\min_{a^{\mathrm{op}}} Q(s, a^{\mathrm{pl}}, a^{\mathrm{op}})}
                \label{saddle_point_update}
                \end{equation}
            \ENDFOR
        \ENDWHILE
        \STATE compute the marginal Q values for player and opponent, for all $(s, a^{\mathrm{pl}}, a^{\mathrm{op}}) \in \mathcal{S} \times \mathcal{A} \times \mathcal{A}$:
            \begin{align*}
            Q^{\mathrm{pl}}(s, a^{\mathrm{pl}}) ~=~& \min_{a^{\mathrm{op}}} Q(s, a^{\mathrm{pl}}, a^{\mathrm{op}}) \quad \text{and} \\
            Q^{\mathrm{op}}(s, a^{\mathrm{op}}) ~=~& \log \sum_{a^{\mathrm{pl}}} \exp{Q(s, a^{\mathrm{pl}}, a^{\mathrm{op}})}
            \end{align*}
        \STATE compute the player (soft-max) and opponent (greedy) policies, for all $(s, a^{\mathrm{pl}}, a^{\mathrm{op}}) \in \mathcal{S} \times \mathcal{A} \times \mathcal{A}$:
        \begin{align*}
            \pi^{\mathrm{pl}}(a^{\mathrm{pl}}|s) ~=~& \frac{\exp Q^{\mathrm{pl}}(s, a^{\mathrm{pl}})}{\sum_{a'}{Q^{\mathrm{pl}}(s, a')}} \quad \text{and} \\ 
            \pi^{\mathrm{op}}(a^{\mathrm{op}}|s) ~=~& \mathbbm{1}\bs{a^{\mathrm{op}} \in \argmin_{a'} Q^{\mathrm{op}}(s, a')}
            \label{opp_policy}
            \end{align*}
        \STATE \textbf{Output:} player policy $\pi^{\mathrm{pl}}$, opponent policy $\pi^{\mathrm{op}}$
    \end{algorithmic}
    \end{spacing}
\end{algorithm}

Here, we prove that the optimization problem in \eqref{objective} can be solved by the Algorithm~\ref{alg:TwoplayersDynProg}.
First of all, one can rewrite \eqref{objective} as:
\begin{equation*}
\Ee{s \sim P_0}{\E{\sum_{t=0}^{\infty} \gamma^t \bc{R_{\boldsymbol{\theta}}(s_t) + H^{\pi^{\mathrm{pl}}}\br{A \mid S = s_t}} \biggm| \pi^{\mathrm{pl}}, \pi^{\mathrm{op}}, M^{\mathrm{two},L,\alpha}, s_0 = s} }
\end{equation*}
The quantity inside the expectation over $P_0$ is usually known as free energy, and for each state $s \in \mathcal{S}$, it is equal to:
\begin{equation*}
    F(\pi^{\mathrm{pl}}, \pi^{\mathrm{op}},s) ~=~ \E{\sum_{t=0}^{\infty} \gamma^t \bc{R_{\boldsymbol{\theta}}(s_t) + H^{\pi^{\mathrm{pl}}}\br{A \mid S = s_t}} \biggm| \pi^{\mathrm{pl}}, \pi^{\mathrm{op}}, M^{\mathrm{two},L,\alpha}, s_0 = s}
\end{equation*}
Separating the first term of the sum over temporal steps, one can observe a recursive relation that is useful for the development of the algorithm:
\begin{align*}
    &F(\pi^{\mathrm{pl}}, \pi^{\mathrm{op}},s) \\ ~=~ 
    &R_{\boldsymbol{\theta}}(s) + H^{\pi^{\mathrm{pl}}}(A|S = s) \\ ~+~
    &\Ee{a^{\mathrm{pl}} \sim \pi^{\mathrm{pl}}, a^{\mathrm{op}} \sim \pi^{\mathrm{op}}}{\Ee{s^\prime \sim T^{\mathrm{two}, L,\alpha}(\cdot|s, a^{\mathrm{pl}}, a^{\mathrm{op}})}{\E{\sum_{t=1}^{\infty} \gamma^t \bc{R_{\boldsymbol{\theta}}(s_t) + H^{\pi^{\mathrm{pl}}}\br{A \mid S = s_t}} \biggm| \pi^{\mathrm{pl}}, \pi^{\mathrm{op}}, M^{\mathrm{two},L,\alpha}, s_1 = s^\prime}}} \\
    ~=~ 
    &R_{\boldsymbol{\theta}}(s) + H^{\pi^{\mathrm{pl}}}(A|S = s) \\ ~+~&
    \gamma \Ee{a^{\mathrm{pl}} \sim \pi^{\mathrm{pl}}, a^{\mathrm{op}} \sim \pi^{\mathrm{op}}}{\Ee{s^\prime \sim T^{\mathrm{two}, L,\alpha}(\cdot|s, a^{\mathrm{pl}}, a^{\mathrm{op}})}{\E{\sum_{t=0}^{\infty} \gamma^t \bc{R_{\boldsymbol{\theta}}(s_t) + H^{\pi^{\mathrm{pl}}}\br{A \mid S = s_t}} \biggm| \pi^{\mathrm{pl}}, \pi^{\mathrm{op}}, M^{\mathrm{two},L,\alpha}, s_0 = s^\prime}}} \\
    ~=~ &R_{\boldsymbol{\theta}}(s) + H^{\pi^{\mathrm{pl}}}(A|S = s) +
    \gamma \Ee{a^{\mathrm{pl}} \sim \pi^{\mathrm{pl}}, a^{\mathrm{op}} \sim \pi^{\mathrm{op}}}{\Ee{s^\prime \sim T^{\mathrm{two}, L,\alpha}(\cdot|s, a^{\mathrm{pl}}, a^{\mathrm{op}})}{F(\pi^{\mathrm{pl}}, \pi^{\mathrm{op}},s^\prime)}} \\
    ~=~ &\Ee{a^{\mathrm{pl}} \sim \pi^{\mathrm{pl}}, a^{\mathrm{op}} \sim \pi^{\mathrm{op}}}{ R_{\boldsymbol{\theta}}(s) - \log \pi^{\mathrm{pl}}(a^{\mathrm{pl}}|s) +
    \gamma \Ee{s^\prime \sim T^{\mathrm{two}, L,\alpha}(\cdot|s, a^{\mathrm{pl}}, a^{\mathrm{op}})}{F(\pi^{\mathrm{pl}}, \pi^{\mathrm{op}},s^\prime)}} \\
\end{align*}
Then, our aim is to find the saddle point:
\begin{equation*}
    V(s) ~=~ \max_{\pi^{\mathrm{pl}}} \min_{\pi^{\mathrm{op}}} F(\pi^{\mathrm{pl}}, \pi^{\mathrm{op}}, s)
\end{equation*}
and the policies attaining it.
Define the joint quality function for a triplet $(s, a^{\mathrm{pl}}, a^{\mathrm{op}})$ as:
\begin{equation*}
    Q(s, a^{\mathrm{pl}}, a^{\mathrm{op}}) ~=~ R_{\boldsymbol{\theta}}(s) + \gamma \Ee{s^{\prime} \sim T(\cdot| s, a^{\mathrm{pl}}, a^{\mathrm{op}})}{V(s^{\prime})}
\end{equation*}
In a dynamic programming context, the previous equation gives the quality function based on the observed reward and the current estimate of the saddle point $V$.
This is done by step \eqref{SGstep1} in the Algorithm~\ref{alg:TwoplayersDynProg}. It remains now to motivate the update of the saddle point estimate $V$ in \eqref{saddle_point_update}. Consider:
\begin{align*}
    &\max_{\pi^{\mathrm{pl}}} \min_{\pi^{\mathrm{op}}} F(\pi^{\mathrm{pl}}, \pi^{\mathrm{op}}, s)  \\
    &~=~ \max_{\pi^{\mathrm{pl}}} \min_{\pi^{\mathrm{op}}} \Ee{a^{\mathrm{pl}} \sim \pi^{\mathrm{pl}}(\cdot|s), a^{\mathrm{op}} \sim \pi^{\mathrm{op}}(\cdot|s)}{Q(s, a^{\mathrm{pl}}, a^{\mathrm{op}}) -\log \pi^{\mathrm{pl}}(a^{\mathrm{pl}}|s)} \\
    &~=~ \max_{\pi^{\mathrm{pl}}} \min_{\pi^{\mathrm{op}}} \Ee{a^{\mathrm{pl}} \sim \pi^{\mathrm{pl}}(\cdot|s)}{ \Ee{a^{\mathrm{op}} \sim \pi^{\mathrm{op}}(\cdot|s)}{Q(s, a^{\mathrm{pl}}, a^{\mathrm{op}}) - \log \pi^{\mathrm{pl}}(a^{\mathrm{pl}}|s)| a^{\mathrm{pl}}}} \\
    &~=~ \max_{\pi^{\mathrm{pl}}} \Ee{a^{\mathrm{pl}} \sim \pi^{\mathrm{pl}}(\cdot|s)}{ \min_{\pi^{\mathrm{op}}}  \Ee{a^{\mathrm{op}} \sim \pi^{\mathrm{op}}(\cdot|s)}{Q(s, a^{\mathrm{pl}}, a^{\mathrm{op}}) -  \log \pi^{\mathrm{pl}}(a^{\mathrm{pl}}|s)| a^{\mathrm{pl}}}} \\
    &~=~ \max_{\pi^{\mathrm{pl}}} \Ee{a^{\mathrm{pl}} \sim \pi^{\mathrm{pl}}(\cdot|s)}{\underbrace{\min_{a^{\mathrm{op}}} Q(s, a^{\mathrm{pl}}, a^{\mathrm{op}})}_{Q^{\mathrm{pl}}(s, a^{\mathrm{pl}})} - \log \pi^{\mathrm{pl}}(a^{\mathrm{pl}}|s)} \\
    &~=~ \log \sum_{a^{\mathrm{pl}}} \exp{Q^{\mathrm{pl}}(s, a^{\mathrm{pl}})} ,
\end{align*}
where the second last equality follows choosing a greedy policy $\pi^{\mathrm{op}}$ that selects the opponent action that minimizes the joint quality function $Q(s, a^{\mathrm{pl}}, a^{\mathrm{op}})$.

The last equality is more involved and it is explained in the following lines:
\begin{align*}
    \Ee{a^{\mathrm{pl}} \sim \pi^{\mathrm{pl}}(\cdot|s)}{Q^{\mathrm{pl}}(s, a^{\mathrm{pl}}) - \log \pi^{\mathrm{pl}}(a^{\mathrm{pl}}|s)} ~=~ \sum_{a^{\mathrm{pl}}} \pi^{\mathrm{pl}}(a^{\mathrm{pl}}|s) \br {Q^{\mathrm{pl}}(s, a^{\mathrm{pl}}) -  \log \pi^{\mathrm{pl}}(a^{\mathrm{pl}}|s)}
\end{align*}
The latter expression is a strictly concave with respect to each decision variable $\pi(a|s)$. So if the derivative with respect to each decision variable $\pi^{\mathrm{pl}}(a^{\mathrm{pl}}|s)$ is zero, we have found the desired global maximum. The normalization is imposed once the maximum has been found. Taking the derivative for a particular decision variable, and equating to zero, we have:
\begin{equation*}
    \br {Q^{\mathrm{pl}}(s, a^{\mathrm{pl}}) -  \log \pi^{\mathrm{pl}}(a^{\mathrm{pl}}|s)} - 1 ~=~ 0
\end{equation*}
It follows that:
\begin{equation*}
    \pi^{\mathrm{pl}}(a|s) \propto \exp Q^{\mathrm{pl}}(s, a^{\mathrm{pl}})
\end{equation*}
and imposing the proper normalization, we obtain the maximizing policy $\pi^{\mathrm{pl}, \ast}$ with the form:
\begin{equation*}
    \pi^{\mathrm{pl}, \ast}(a^{\mathrm{pl}}|s) ~=~ \frac{\exp Q^{\mathrm{pl}}(s, a^{\mathrm{pl}})}{\sum_{a^{\mathrm{pl}}}\exp Q^{\mathrm{pl}}(s, a^{\mathrm{pl}})}
\end{equation*}
Finally, computing the expectation with respect to the maximizing policy:
\begin{align}
    &  \Ee{a^{\mathrm{pl}} \sim \pi^{\mathrm{pl}, \ast}(\cdot|s)}{{Q^{\mathrm{pl}}(s, a^{\mathrm{pl}})} - \log \pi^{\mathrm{pl}}(a^{\mathrm{pl}}|s)} \nonumber \\
    &~=~ \sum_{a^{\mathrm{pl}}} \pi^{\mathrm{pl}, \ast}(a^{\mathrm{pl}}|s) \br {Q^{\mathrm{pl}}(s, a^{\mathrm{pl}}) - \log \pi^{\mathrm{pl}, \ast}(a^{\mathrm{pl}}|s)} \nonumber \\
    &~=~ \sum_{a^{\mathrm{pl}}} \frac{\exp{Q^{\mathrm{pl}}(s, a^{\mathrm{pl}})}}{\sum_{a^{\mathrm{pl}}}\exp{Q^{\mathrm{pl}}(s, a^{\mathrm{pl}})}} \br {Q^{\mathrm{pl}}(s, a^{\mathrm{pl}}) - \log \frac{\exp{Q^{\mathrm{pl}}(s, a^{\mathrm{pl}})}}{\sum_{a^{\mathrm{pl}}}\exp{Q^{\mathrm{pl}}(s, a^{\mathrm{pl}})}}} \nonumber \\
    &~=~ \sum_{a^{\mathrm{pl}}} \frac{\exp{Q^{
    \mathrm{pl}}(s, a^{\mathrm{pl}})}}{\sum_{a^{\mathrm{pl}}}\exp{Q^{\mathrm{pl}}(s, a^{\mathrm{pl}})}} \br {Q^{\mathrm{pl}}(s, a^{\mathrm{pl}}) - Q^{\mathrm{pl}}(s, a^{\mathrm{pl}}) + \log \sum_{a^{\mathrm{pl}}}\exp{Q^{\mathrm{pl}}(s, a^{\mathrm{pl}})}} \nonumber \\
    &~=~ \sum_{a^{\mathrm{pl}}} \frac{\exp{Q^{\mathrm{pl}}(s, a^{\mathrm{pl}})}}{\sum_{a^{\mathrm{pl}}}\exp{Q^{\mathrm{pl}}(s, a^{\mathrm{pl}})}} \br {\log \sum_{a^{\mathrm{pl}}}\exp{Q^{\mathrm{pl}}(s, a^{\mathrm{pl}})}} \nonumber \\
    &~=~ \log \sum_{a^{\mathrm{pl}}}\exp{Q^{\mathrm{pl}}(s, a^{\mathrm{pl}})} \label{value_func}
\end{align}
Basically, we have shown that the optimization problem is solved when the player follows a soft-max policy with respect to the quality function $Q^{\mathrm{pl}}(a^{\mathrm{pl}}|s) = \min_{a^{\mathrm{op}}} Q(s, a^{\mathrm{pl}}, a^{\mathrm{op}})$. This explains the steps for the player policy in Algorithm~\ref{alg:TwoplayersDynProg}.
In addition, replacing the definition $Q^{\mathrm{pl}}(a^{\mathrm{pl}}|s) = \min_{a^{\mathrm{op}}} Q(s, a^{\mathrm{pl}}, a^{\mathrm{op}})$ in \eqref{value_func}, one gets the saddle point update \eqref{saddle_point_update} in Algorithm~\ref{alg:TwoplayersDynProg}.

We still need to proceed similarly to motivate the opponent policy derivation from the quality function \eqref{SGstep1}. To this end, we maximize with respect to the player before minimizing for the opponent, we have:
\begin{align*}
    &\min_{\pi^{\mathrm{op}}} \max_{\pi^{\mathrm{pl}}} F(\pi^{\mathrm{pl}}, \pi^{\mathrm{op}}, s) \\
    &~=~ \min_{\pi^{\mathrm{op}}} \max_{\pi^{\mathrm{pl}}} \Ee{a^{\mathrm{pl}} \sim \pi^{\mathrm{pl}}(\cdot|s), a^{\mathrm{op}} \sim \pi^{\mathrm{op}}(\cdot|s)}{Q(s, a^{\mathrm{pl}}, a^{\mathrm{op}}) - \log \pi^{\mathrm{pl}}(a^{\mathrm{pl}}|s)} \\
    &~=~ \min_{\pi^{\mathrm{op}}} \max_{\pi^{\mathrm{pl}}} \Ee{a^{\mathrm{op}} \sim \pi^{\mathrm{op}}(\cdot|s)}{ \Ee{a^{\mathrm{pl}} \sim \pi^{\mathrm{pl}}(\cdot|s)}{Q(s, a^{\mathrm{pl}}, a^{\mathrm{op}}) - \log \pi^{\mathrm{pl}}(a^{\mathrm{pl}}|s)| a^{\mathrm{op}}}} \\
    &~=~ \min_{\pi^{\mathrm{op}}} \Ee{a^{\mathrm{op}} \sim \pi^{\mathrm{op}}(\cdot|s)}{ \max_{\pi^{\mathrm{pl}}} \Ee{a^{\mathrm{pl}} \sim \pi^{\mathrm{pl}}(\cdot|s)}{Q(s, a^{\mathrm{pl}}, a^{\mathrm{op}}) - \log \pi^{\mathrm{pl}}(a^{\mathrm{pl}}|s)| a^{\mathrm{op}}}}
\end{align*}
The innermost maximization is solved again by observing that it is a concave function in the decision variables, normalizing one obtains the maximizer policy, and plugging that in the expectation gives the soft-max function with respect to the player action $a^{\mathrm{pl}}$. We define this function as the quality function of the opponent, because it is the amount of information that can be used by the opponent to decide its move.
\begin{equation*}
    Q^{\mathrm{op}}(s,a^{\mathrm{op}}) ~=~ \log \sum_{a^{\mathrm{pl}}} \exp Q(s, a^{\mathrm{pl}}, a^{\mathrm{op}})
\end{equation*}
It remains to face the external minimization with respect to the opponent policy. This is trivial, the opponent can simply act greedly  since it is not regularized :
\begin{equation*}
    \min_{\pi^{\mathrm{op}}} \Ee{a^{\mathrm{op}} \sim \pi^{\mathrm{op}}(\cdot|s)}{ Q^{\mathrm{op}}(s, a^{\mathrm{op}})}
~=~ \min_{a^{\mathrm{op}}} Q^{\mathrm{op}}(s, a^{\mathrm{op}})
\end{equation*}
This second part clarifies the updates relative to the opponent in Algorithm~\ref{alg:TwoplayersDynProg}.

Notice that the algorithm iterates in order to obtain a more and more precise estimate of the joint quality function $Q(s, a^{\mathrm{pl}}, a^{\mathrm{op}})$. When it converges, the quality functions for the player and the agent respectively are obtained, thanks to the transformations illustrated here and in the body of Algorithm~\ref{alg:TwoplayersDynProg}.

\subsection{Proof of Theorem~\ref{thm:new-robust-mce-irl-bound}}
\label{app:robust-mce-irl-upper}

\begin{proof}
Consider the following:
\begin{align*}
\abs{V^{\pi_1}_{M_{\boldsymbol{\theta^*}}^L} - V^{\pi^\mathrm{pl}}_{M_{\boldsymbol{\theta^*}}^L}} ~\leq~& \abs{V^{\pi_1}_{M_{\boldsymbol{\theta^*}}^L} - V^{\pi^{*}_{M_{\boldsymbol{\theta^*}}^L}}_{M_{\boldsymbol{\theta^*}}^L}} + \abs{V^{\pi^{*}_{M_{\boldsymbol{\theta^*}}^L}}_{M_{\boldsymbol{\theta^*}}^L} - V^{\pi^{*}_{M_{\boldsymbol{\theta^*}}^E}}_{M_{\boldsymbol{\theta^*}}^E}} + \abs{V^{\pi^{*}_{M_{\boldsymbol{\theta^*}}^E}}_{M_{\boldsymbol{\theta^*}}^E} - V^{\alpha \pi^\mathrm{pl} + (1-\alpha) \pi^\mathrm{op}}_{M_{\boldsymbol{\theta^*}}^L}} + \abs{V^{\alpha \pi^\mathrm{pl} + (1-\alpha) \pi^\mathrm{op}}_{M_{\boldsymbol{\theta^*}}^L} - V^{\pi^\mathrm{pl}}_{M_{\boldsymbol{\theta^*}}^L}} \\
~\stackrel{\mathrm{a}}{=}~& \abs{V^{\pi^{*}_{M_{\boldsymbol{\theta^*}}^L}}_{M_{\boldsymbol{\theta^*}}^L} - V^{\pi^{*}_{M_{\boldsymbol{\theta^*}}^E}}_{M_{\boldsymbol{\theta^*}}^E}} + \abs{V^{\alpha \pi^\mathrm{pl} + (1-\alpha) \pi^\mathrm{op}}_{M_{\boldsymbol{\theta^*}}^L} - V^{\pi^\mathrm{pl}}_{M_{\boldsymbol{\theta^*}}^L}} \\
~\stackrel{\mathrm{b}}{\leq}~& \frac{\gamma \cdot \abs{R_{\boldsymbol{\theta^*}}}^{\mathrm{max}}}{(1 - \gamma)^2} \cdot d_\mathrm{dyn} \br{T^L, T^E} + \abs{V^{\alpha \pi^\mathrm{pl} + (1-\alpha) \pi^\mathrm{op}}_{M_{\boldsymbol{\theta^*}}^L} - V^{\pi^\mathrm{pl}}_{M_{\boldsymbol{\theta^*}}^L}} \\
~\leq~& \frac{\gamma \cdot \abs{R_{\boldsymbol{\theta^*}}}^{\mathrm{max}}}{(1 - \gamma)^2} \cdot d_\mathrm{dyn} \br{T^L, T^E} + \frac{\abs{R_{\boldsymbol{\theta^*}}}^{\mathrm{max}}}{1 - \gamma} \cdot  \norm{\boldsymbol{\rho}^{\alpha \pi^\mathrm{pl} + (1-\alpha) \pi^\mathrm{op}}_{M^L} - \boldsymbol{\rho}^{\pi^\mathrm{pl}}_{M^L}}_1 \\
~\stackrel{\mathrm{c}}{\leq}~& \frac{\gamma \cdot \abs{R_{\boldsymbol{\theta^*}}}^{\mathrm{max}}}{(1 - \gamma)^2} \cdot d_\mathrm{dyn} \br{T^L, T^E} + \frac{\abs{R_{\boldsymbol{\theta^*}}}^{\mathrm{max}}}{(1 - \gamma)^2} \cdot \max_{s} \norm{\alpha \pi^\mathrm{pl} + (1-\alpha) \pi^\mathrm{op}(\cdot|s) - \pi^\mathrm{pl}(\cdot|s)}_1 \\
~=~& \frac{\gamma \cdot \abs{R_{\boldsymbol{\theta^*}}}^{\mathrm{max}}}{(1 - \gamma)^2} \cdot d_\mathrm{dyn} \br{T^L, T^E} + \frac{\abs{R_{\boldsymbol{\theta^*}}}^{\mathrm{max}}}{(1 - \gamma)^2} \cdot (1-\alpha) \cdot \max_{s} \norm{\pi^\mathrm{op}(\cdot|s) - \pi^\mathrm{pl}(\cdot|s)}_1 \\
~\leq~& \frac{\gamma \cdot \abs{R_{\boldsymbol{\theta^*}}}^{\mathrm{max}}}{(1 - \gamma)^2} \cdot d_\mathrm{dyn} \br{T^L, T^E} + \frac{\abs{R_{\boldsymbol{\theta^*}}}^{\mathrm{max}}}{(1 - \gamma)^2} \cdot (1-\alpha) \cdot 2
\end{align*}

where $\mathrm{a}$ is due to the fact that $\boldsymbol{\rho}^{\pi_1}_{M^L} = \boldsymbol{\rho}^{\pi^{*}_{M_{\boldsymbol{\theta^*}}^L}}_{M^L}$ and $\boldsymbol{\rho}^{\pi^{*}_{M_{\boldsymbol{\theta^*}}^E}}_{M^E} = \boldsymbol{\rho}^{\alpha \pi^\mathrm{pl} + (1-\alpha) \pi^\mathrm{op}}_{M^L}$; $\mathrm{b}$ is due to Theorem~7 in~\cite{zhang2020multi}; and $\mathrm{c}$ is due to Lemma A.1 in~\cite{sun2018dual}.

\end{proof}

\subsection{Suboptimality gap for the Robust MCE-IRL in the infeasible case}
\label{app:robust-mce-irl-upper-infeasible}


In the main text, we always assume that the condition of Theorem~\ref{thm:occ_states} holds. In that case, the problem~\eqref{opt_start} is feasible, and the performance gap guarantee of Robust MCE IRL provided by Theorem~\ref{thm:new-robust-mce-irl-bound} is weaker than that of the standard MCE IRL. Here, instead we consider the case where the condition of Theorem \ref{thm:occ_states} does not hold\footnote{It follows that the policy output by Algorithm~\ref{alg:MaxEntIRL} is not in the feasible set of the problem~\ref{opt_start}}. 

\begin{theorem}
\label{thm:new-robust-mce-irl-bound-infeasible}
When the condition in Theorem~\ref{thm:occ_states} does not hold, the performance gap between the policies $\pi_1$ and $\pi^\mathrm{pl}$ in the MDP $M_{\boldsymbol{\theta^*}}^L$ is bounded as follows:
\begin{align*}
\abs{V^{\pi_1}_{M_{\boldsymbol{\theta^*}}^L} - V^{\pi^\mathrm{pl}}_{M_{\boldsymbol{\theta^*}}^L}} ~\leq~& \frac{\gamma \cdot \abs{R_{\boldsymbol{\theta^*}}}^{\mathrm{max}}}{(1 - \gamma)^2} \cdot d_\mathrm{dyn} \br{T^L, T^E} + \\
& \frac{\gamma \cdot \abs{R_{\boldsymbol{\theta^*}}}^{\mathrm{max}}}{(1 - \gamma)^2}  2(1 - \alpha)^2 + \frac{ \abs{R_{\boldsymbol{\theta^*}}}^{\mathrm{max}}}{(1 - \gamma)^2} d_\mathrm{pol} \br{\pi^*_{M_{\boldsymbol{\theta^*}}^E}, \pi^{\mathrm{pl}}} + \\
& \frac{ \gamma \cdot \abs{R_{\boldsymbol{\theta^*}}}^{\mathrm{max}}}{(1 - \gamma)^2} \bs{\alpha \cdot d_\mathrm{dyn} \br{T^E, T^L} + (1 - \alpha) \cdot d_\mathrm{dyn} \br{T^E, T^*}} ,
\end{align*}
where $T^*$ minimizes \eqref{eq:T_minimizer}.
\end{theorem}
\begin{proof}
\begin{align*}
\abs{V^{\pi_1}_{M_{\boldsymbol{\theta^*}}^L} - V^{\pi^\mathrm{pl}}_{M_{\boldsymbol{\theta^*}}^L}} ~\leq~& \abs{V^{\pi_1}_{M_{\boldsymbol{\theta^*}}^L} - V^{\pi^{*}_{M_{\boldsymbol{\theta^*}}^L}}_{M_{\boldsymbol{\theta^*}}^L}} + \abs{V^{\pi^{*}_{M_{\boldsymbol{\theta^*}}^L}}_{M_{\boldsymbol{\theta^*}}^L} - V^{\pi^{*}_{M_{\boldsymbol{\theta^*}}^E}}_{M_{\boldsymbol{\theta^*}}^E}} + \abs{V^{\pi^{*}_{M_{\boldsymbol{\theta^*}}^E}}_{M_{\boldsymbol{\theta^*}}^E} - V^{\alpha \pi^\mathrm{pl} + (1-\alpha) \pi^\mathrm{op}}_{M_{\boldsymbol{\theta^*}}^L}} + \abs{V^{\alpha \pi^\mathrm{pl} + (1-\alpha) \pi^\mathrm{op}}_{M_{\boldsymbol{\theta^*}}^L} - V^{\pi^\mathrm{pl}}_{M_{\boldsymbol{\theta^*}}^L}} \\
~\stackrel{\mathrm{a}}{=}~& \underbrace{\abs{V^{\pi^{*}_{M_{\boldsymbol{\theta^*}}^L}}_{M_{\boldsymbol{\theta^*}}^L} - V^{\pi^{*}_{M_{\boldsymbol{\theta^*}}^E}}_{M_{\boldsymbol{\theta^*}}^E}}}_{\text{Demonstration difference}} + \underbrace{\abs{V^{\alpha \pi^\mathrm{pl} + (1-\alpha) \pi^\mathrm{op}}_{M_{\boldsymbol{\theta^*}}^L} - V^{\pi^\mathrm{pl}}_{M_{\boldsymbol{\theta^*}}^L}}}_{\text{Transfer difference}} + \underbrace{\abs{V^{\pi^{*}_{M_{\boldsymbol{\theta^*}}^E}}_{M_{\boldsymbol{\theta^*}}^E} - V^{\alpha \pi^\mathrm{pl} + (1 - \alpha)\pi^\mathrm{op}}_{M_{\boldsymbol{\theta^*}}^L}}}_{\text{infeasibility error}} \\
\end{align*}
The Demonstration difference is bounded using Theorem~7 in~\cite{zhang2020multi}, i.e.
\begin{equation}
     \abs{V^{\pi^{*}_{M_{\boldsymbol{\theta^*}}^L}}_{M_{\boldsymbol{\theta^*}}^L} - V^{\pi^{*}_{M_{\boldsymbol{\theta^*}}^E}}_{M_{\boldsymbol{\theta^*}}^E}} \leq \frac{\gamma \cdot \abs{R_{\boldsymbol{\theta^*}}}^{\mathrm{max}}}{(1 - \gamma)^2} \cdot d_\mathrm{dyn} \br{T^L, T^E}
\end{equation}
The transfer error can be bound as:
\begin{align*}
    \abs{V^{\alpha \pi^\mathrm{pl} + (1-\alpha) \pi^\mathrm{op}}_{M_{\boldsymbol{\theta^*}}^L} - V^{\pi^\mathrm{pl}}_{M_{\boldsymbol{\theta^*}}^L}} ~\stackrel{\mathrm{a}}{=}~&\abs{V^{\pi^\mathrm{pl}}_{\boldsymbol{\theta}^*, \alpha T^L + (1 - \alpha) T^*} - V^{\pi^\mathrm{pl}}_{M_{\boldsymbol{\theta^*}}^L}} \\
    ~\stackrel{\mathrm{b}}{\leq}~& \frac{\gamma \cdot \abs{R_{\boldsymbol{\theta^*}}}^{\mathrm{max}}}{(1 - \gamma)^2} \cdot d_\mathrm{dyn} \br{\alpha T^L + (1 - \alpha)T^*, T^L} \\
    ~\stackrel{\mathrm{c}}{=}~& \frac{\gamma \cdot \abs{R_{\boldsymbol{\theta^*}}}^{\mathrm{max}}}{(1 - \gamma)^2}  (1 - \alpha)\cdot d_\mathrm{dyn} \br{ T^*, T^L} \\
    ~\stackrel{\mathrm{d}}{=}~& \frac{\gamma \cdot \abs{R_{\boldsymbol{\theta^*}}}^{\mathrm{max}}}{(1 - \gamma)^2}  2(1 - \alpha)^2
\end{align*}
Finally, for the infeasibility error
\begin{align*}
    \abs{V^{\pi^{*}_{M_{\boldsymbol{\theta^*}}^E}}_{M_{\boldsymbol{\theta^*}}^E} - V^{\alpha \pi^\mathrm{pl} + (1 - \alpha)\pi^\mathrm{op}}_{M_{\boldsymbol{\theta^*}}^L}}~\stackrel{\mathrm{a}}{=}~&\abs{V^{\pi^{*}_{M_{\boldsymbol{\theta^*}}^E}}_{M_{\boldsymbol{\theta^*}}^E} - V^{\pi^\mathrm{pl}}_{\boldsymbol{\theta}^*, \alpha T^L + (1 - \alpha) T^*}} \\
    ~\stackrel{\mathrm{b}}{\leq}~& \abs{V^{\pi^{*}_{M_{\boldsymbol{\theta^*}}^E}}_{M_{\boldsymbol{\theta^*}}^E} - V^{\pi^\mathrm{pl}}_{M_{\boldsymbol{\theta^*}}^E}} + \abs{V^{\pi^\mathrm{pl}}_{M_{\boldsymbol{\theta^*}}^E} - V^{\pi^\mathrm{pl}}_{\boldsymbol{\theta}^*, \alpha T^L + (1 - \alpha) T^*}} \\
    ~\stackrel{\mathrm{c}}{\leq}~&\frac{ \abs{R_{\boldsymbol{\theta^*}}}^{\mathrm{max}}}{(1 - \gamma)^2} d_\mathrm{pol} \br{\pi^*_{M_{\boldsymbol{\theta^*}}^E}, \pi^{\mathrm{pl}}} + \frac{ \gamma \cdot \abs{R_{\boldsymbol{\theta^*}}}^{\mathrm{max}}}{(1 - \gamma)^2} d_\mathrm{dyn} \br{T^E, \alpha T^L + (1 - \alpha) T^*}
    \\
    ~\stackrel{\mathrm{d}}{\leq}~&\frac{ \abs{R_{\boldsymbol{\theta^*}}}^{\mathrm{max}}}{(1 - \gamma)^2} d_\mathrm{pol} \br{\pi^*_{M_{\boldsymbol{\theta^*}}^E}, \pi^{\mathrm{pl}}} + \\& \frac{ \gamma \cdot \abs{R_{\boldsymbol{\theta^*}}}^{\mathrm{max}}}{(1 - \gamma)^2} \bs{\alpha \cdot d_\mathrm{dyn} \br{T^E, T^L} + (1 - \alpha) \cdot d_\mathrm{dyn} \br{T^E, T^*}}
\end{align*}
It can be seen that in case of MCE IRL $\alpha=1$, the infeasibility term can be bounded adding an additional term scaling linearly with the mismatch $d_{\mathrm{dyn}}(T^E, T^L)$, however when $\alpha<1$, the bound dependent on the linear combination of the mismatches $\alpha \cdot d_{\mathrm{dyn}}(T^E, T^L) + (1 - \alpha) \cdot d_{\mathrm{dyn}}(T^E, T^*) $ where $T^*$ is a minimizer of \eqref{eq:T_minimizer}. Therefore the bound is tighter for problems such that $d_{\mathrm{dyn}}(T^E, T^*) < d_{\mathrm{dyn}}(T^E, T^L)$. However, our bounds also explains that for $\alpha < 1$, we have nonzero bound on the transfer error that arises from the fact that the matching policy $\alpha \pi^{\mathrm{pl}} + (1-\alpha)\pi^{\mathrm{op}}$ is not equal to the evaluated policy $\pi^{\mathrm{pl}}$.
\end{proof}

The following corollary provides a value of $\alpha$ for which we can attain better bound on the performance gap of Robust MCE IRL.
\begin{corollary}
\label{thm:best_alpha_upper_bound}
When the condition in Theorem~\ref{thm:occ_states} does not hold, the upper bound on the performance gap between the policies $\pi_1$ and $\pi^\mathrm{pl}$ in the MDP $M_{\boldsymbol{\theta^*}}^L$ given in Theorem~\ref{thm:new-robust-mce-irl-bound-infeasible} is minimized for the following choice of $\alpha$:
\[
\alpha = \min \br{1, 1 - \frac{d_{\mathrm{dyn}}{(T^E, T^L)}}{4} + \frac{d_{\mathrm{dyn}}{(T^*, T^E)}}{4}} ,
\]
where $T^*$ minimizes \eqref{eq:T_minimizer}.
\end{corollary}
The suggested choice of $\alpha$ follows the intuition of having a decreasing $\alpha$ as the distance $d_{\mathrm{dyn}}{(T^E, T^L)}$ increases.
However, it should be closer to $1$ as the distance $d_{\mathrm{dyn}}{(T^E, T^*)}$ increases, i.e., a less powerful opponent should work better if the expert transition dynamics are not close to the worst ones (the ones that minimize \eqref{eq:T_minimizer}).

\subsection{Proof of Theorem~\ref{theorem-tightness}}
\label{app:tightness-proof}

\begin{proof}
For any policy $\pi$ acting in the expert environment $M^E$, we can compute the state occupancy measures, as follows:
\begin{align}
\rho^{\pi}_{M^E} (s_0) ~=~& 1 - \gamma \label{eq:occ_expert_start}\\
\rho^{\pi}_{M^E} (s_1) ~=~& (1 - \epsilon_E) \cdot \gamma \cdot \pi(a_1 | s_0) \\
\rho^{\pi}_{M^E} (s_2) ~=~& \epsilon_E \cdot \gamma \cdot \pi(a_1 | s_0) + \gamma \cdot \pi(a_2 | s_0) \label{eq:occ_expert_end}
\end{align}
Then, for the MDP $M_{\boldsymbol{\theta^*}}^E$ endowed with the true reward function $R_{\boldsymbol{\theta^*}}$, we have:
\begin{equation}
V^{\pi}_{M_{\boldsymbol{\theta^*}}^E} ~=~ \frac{\gamma}{1 - \gamma} \cdot \bc{2 \cdot (1 - \epsilon_E) \cdot \pi(a_1 | s_0) - 1} ,
\end{equation}
which is maximized when $\pi(a_1 | s_0) = 1$. Therefore, the optimal expert policy is given by: $\pi^*_{M^E_{\boldsymbol{\theta^*}}}(a_1 | s_0) = 1$ and $\pi^*_{M^E_{\boldsymbol{\theta^*}}}(a_2 | s_0) = 0$, with the corresponding optimal value $V^{\pi^*_{M^E_{\boldsymbol{\theta^*}}}}_{M_{\boldsymbol{\theta^*}}^E} = \frac{\gamma}{1 - \gamma} \cdot \br{1 - 2\epsilon_E}$.

On the learner side ($M^L$), Algorithm~\ref{alg:MaxEntIRL} converges when the occupancy measure of the mixture policy $\alpha \pi^{\mathrm{pl}} + (1 - \alpha) \pi^{\mathrm{op}}$ matches the expert's occupancy measure. First, we compute the occupancy measures for the mixture policy:
\begin{align*}
\rho^{\alpha \pi^{\mathrm{pl}} + (1 - \alpha) \pi^{\mathrm{op}}}_{M^L} (s_0) ~=~& 1 - \gamma \\
\rho^{\alpha \pi^{\mathrm{pl}} + (1 - \alpha) \pi^{\mathrm{op}}}_{M^L} (s_1) ~=~& \gamma \cdot \bc{\alpha \cdot \pi^{\mathrm{pl}}(a_1 | s_0) + (1 - \alpha) \cdot \pi^{\mathrm{op}}(a_1 | s_0)} \\
\rho^{\alpha \pi^{\mathrm{pl}} + (1 - \alpha) \pi^{\mathrm{op}}}_{M^L} (s_2) ~=~&  \gamma \cdot \bc{\alpha \cdot \pi^{\mathrm{pl}}(a_2 | s_0) + (1 - \alpha) \cdot \pi^{\mathrm{op}}(a_2 | s_0)}   
\end{align*}
Here, the worst-case opponent is given by $\pi^{\mathrm{op}}(a_1 | s_0) = 0$ and $\pi^{\mathrm{op}}(a_2 | s_0) = 1$. Note that the choice of the opponent does not rely on the unknown reward function. Instead, we choose as opponent the policy that takes the action leading to the state where the demonstrated occupancy measure is lower. Then, the above expressions reduce to:
\begin{align*}
\rho^{\alpha \pi^{\mathrm{pl}} + (1 - \alpha) \pi^{\mathrm{op}}}_{M^L} (s_0) ~=~& 1 - \gamma\\
\rho^{\alpha \pi^{\mathrm{pl}} + (1 - \alpha) \pi^{\mathrm{op}}}_{M^L} (s_1) ~=~& \gamma \cdot \alpha \cdot \pi^{\mathrm{pl}}(a_1 | s_0)\\
\rho^{\alpha \pi^{\mathrm{pl}} + (1 - \alpha) \pi^{\mathrm{op}}}_{M^L} (s_2) ~=~&  \gamma \cdot \bc{\alpha \cdot \pi^{\mathrm{pl}}(a_2 | s_0) + (1 - \alpha)}
\end{align*}

Now, we match the above occupancy measures with the expert occupancy measures (Eqs.~\eqref{eq:occ_expert_start}-\eqref{eq:occ_expert_end} with $\pi \gets \pi^*_{M^E_{\boldsymbol{\theta^*}}}$):
\begin{align*}
1 - \epsilon_E ~=~& \alpha \cdot \pi^{\mathrm{pl}}(a_1 | s_0)\\
\epsilon_E ~=~& \alpha \cdot \pi^{\mathrm{pl}}(a_2 | s_0) + (1 - \alpha) 
\end{align*}
Thus, we get: $\pi^{\mathrm{pl}}(a_1 | s_0) = \frac{1 - \epsilon_E}{\alpha}$ and $\pi^{\mathrm{pl}}(a_2 | s_0) = \frac{\alpha - (1 - \epsilon_E)}{\alpha}$. Note that $\pi^{\mathrm{pl}}$ is well-defined when $\alpha \geq 1 - \epsilon_E$. 

Given $\alpha \geq 1 - \epsilon_E$, the state occupancy measure of $\pi^{\mathrm{pl}}$ in the MDP $M^L$ is given by:
\begin{align*}
\rho^{\pi^{\mathrm{pl}}}_{M^L} (s_0) ~=~& 1 - \gamma \\
\rho^{\pi^{\mathrm{pl}}}_{M^L} (s_1) ~=~& \gamma \cdot \pi^{\mathrm{pl}}(a_1 | s_0) ~=~ \gamma \cdot \frac{1 - \epsilon_E}{\alpha} \\
\rho^{\pi^{\mathrm{pl}}}_{M^L} (s_2) ~=~&  \gamma \cdot \pi^{\mathrm{pl}}(a_2 | s_0) ~=~ \gamma \cdot \frac{\alpha - (1 - \epsilon_E)}{\alpha}
\end{align*}
Then, the expected return of $\pi^{\mathrm{pl}}$ in the MDP $M_{\boldsymbol{\theta^*}}^L$ is given by:
\begin{equation*}
V^{\pi^{\mathrm{pl}}}_{M_{\boldsymbol{\theta^*}}^L} ~=~ \frac{\gamma}{1 - \gamma} \cdot \frac{2 \cdot (1 - \epsilon_E) - \alpha}{\alpha} .
\end{equation*}
Consider the MCE IRL learner receiving the expert occupancy measure $\boldsymbol{\rho}$ from the learner environment $M^L$ itself, i.e., $\boldsymbol{\rho} = \boldsymbol{\rho}^{\pi^*_{M^L_{\boldsymbol{\theta^*}}}}_{M^L}$. Note that $\pi^*_{M^L_{\boldsymbol{\theta^*}}}(a_1 | s_0) = 1$, and $\pi^*_{M^L_{\boldsymbol{\theta^*}}}(a_2 | s_0) = 0$. In this case, the learner recovers a policy $\pi_1 := \pi^{\mathrm{soft}}_{M_{\boldsymbol{\theta_L}}^L}$ such that $\boldsymbol{\rho}^{\pi_1}_{M^L} = \boldsymbol{\rho}^{\pi^*_{M^L_{\boldsymbol{\theta^*}}}}_{M^L}$. Thus, we have $V^{\pi_1}_{M^L_{\boldsymbol{\theta^*}}} = V^{\pi^*_{M^L_{\boldsymbol{\theta^*}}}}_{M^L_{\boldsymbol{\theta^*}}} = \frac{\gamma}{1 - \gamma}$. Consequently, for this example, the performance gap is given by:
\[
\abs{V^{\pi_1}_{M^L_{\boldsymbol{\theta^*}}} - V^{\pi^{\mathrm{pl}}}_{M^L_{\boldsymbol{\theta^*}}}} ~=~ \abs{\frac{\gamma}{1 - \gamma} \cdot \bc{1 - \frac{2 \cdot (1 - \epsilon_E) - \alpha}{\alpha}}} ~=~ \frac{2 \cdot \gamma}{1 - \gamma} \cdot \abs{\frac{\alpha - (1 - \epsilon_E)}{\alpha}} .
\]

The following two cases are of particular interest:
\begin{itemize}
\item For $\alpha = 1 - \epsilon_E = 1 - \frac{d_\mathrm{dyn} \br{T^L, T^E}}{2}$, the performance gap vanishes. This indicates that our Algorithm~\ref{alg:MaxEntIRL} can recover the optimal performance even under dynamics mismatch. 
\item For $\alpha = 1$ (corresponding to the standard MCE IRL), the performance gap is given by:
\[
\abs{V^{\pi_1}_{M^L_{\boldsymbol{\theta^*}}} - V^{\pi^{\mathrm{pl}}}_{M^L_{\boldsymbol{\theta^*}}}} ~=~ \frac{2 \cdot \gamma \cdot \epsilon_E}{1 - \gamma} ~=~ \frac{\gamma}{1 - \gamma} \cdot d_\mathrm{dyn} \br{T^L, T^E} .
\]
\end{itemize}
\end{proof}

\newpage
\section{Further Details of Section~\ref{sec:experiments}}
\label{appendix:experiments}

\subsection{Hyperparameter Details and Additional Results}
\label{app:hyper-figs}

Here, we present the Figures~\ref{fig:all_gridworld_best_alpha},~and~\ref{fig:gridworld_diff_alpha}, mentioned in the main text. All the hyperparameter details are reported in Tables~\ref{tab:Adamtable},~\ref{tab:InfHorNonlinearAdamtable} and~\ref{tab:MDPtable}. We consider a uniform initial distribution $P_0$. For the performance evaluation of the learned policies, we compute the average reward of $1000 \times \abs{\mathcal{S}}$ trajectories; along with this mean, we have reported the SD as well. 

\subsection{Low Dimensional Features}
\label{app:low-dim-exp}

We consider a \textsc{GridWorld-L} environment with a low dimensional (of dimension 3) binary feature mapping $\boldsymbol{\phi}: \mathcal{S} \to \bc{0,1}^3$. For any state $s \in \mathcal{S}$, the first two entries of the vector $\boldsymbol{\phi}\br{s}$ are defined as follows: 
\[
\boldsymbol{\phi}\br{s}_i = \begin{cases}
1 & \text{the danger is of type-i in the state } s \\
0 & \text{otherwise}
\end{cases}
\]
Whereas, the last entry of the vector $\boldsymbol{\phi}\br{s} = 1$ for non-terminal states. The true reward function is given by $R_{\mathbf{w}}\br{s} = \ip{\mathbf{w}}{\boldsymbol{\phi}\br{s}}$, where $\mathbf{w} = \bs{-2, -6, -1}$. In this low dimensional setting, our Algorithm~\ref{alg:MaxEntIRL} significantly outperforms the standard MCE IRL algorithm (see Figures~\ref{fig:low-dim-exp-short},~and~\ref{fig:low-dim-exp-full}).  

\begin{figure*}[h!] 
\centering
\begin{subfigure}{0.24\textwidth}
\includegraphics[width=0.87\linewidth]{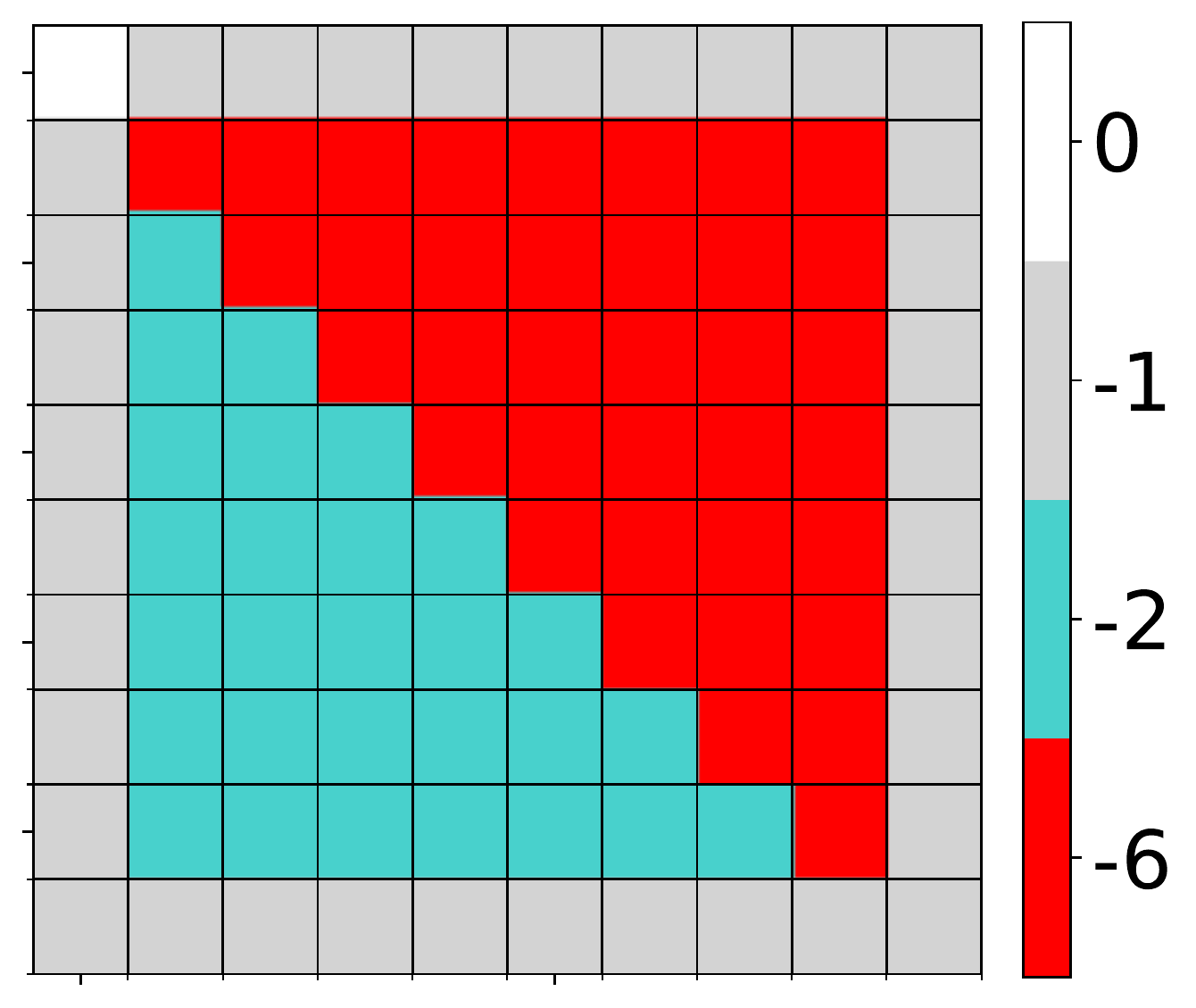}
\caption{\textsc{GridWorld-L}} \label{fig:trtdw1pres}
\end{subfigure}
\begin{subfigure}{0.24\textwidth}
\includegraphics[width=\linewidth]{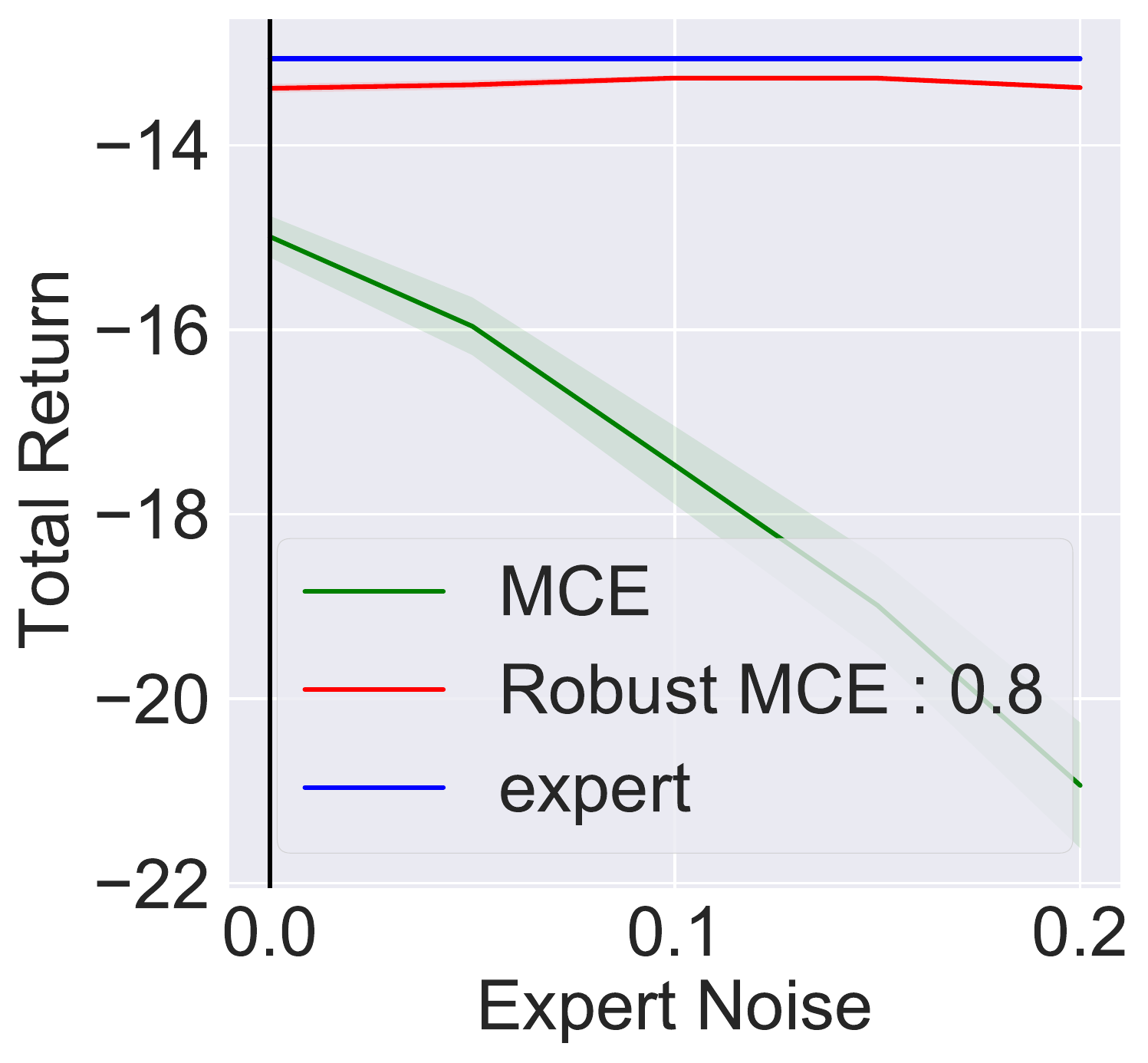}
\caption{$M^{L,\epsilon_L}$ with $\epsilon_L = 0$} \label{fig:etdw1l0.0pres}
\end{subfigure}\hspace*{\fill}
\begin{subfigure}{0.24\textwidth}
\includegraphics[width=\linewidth]{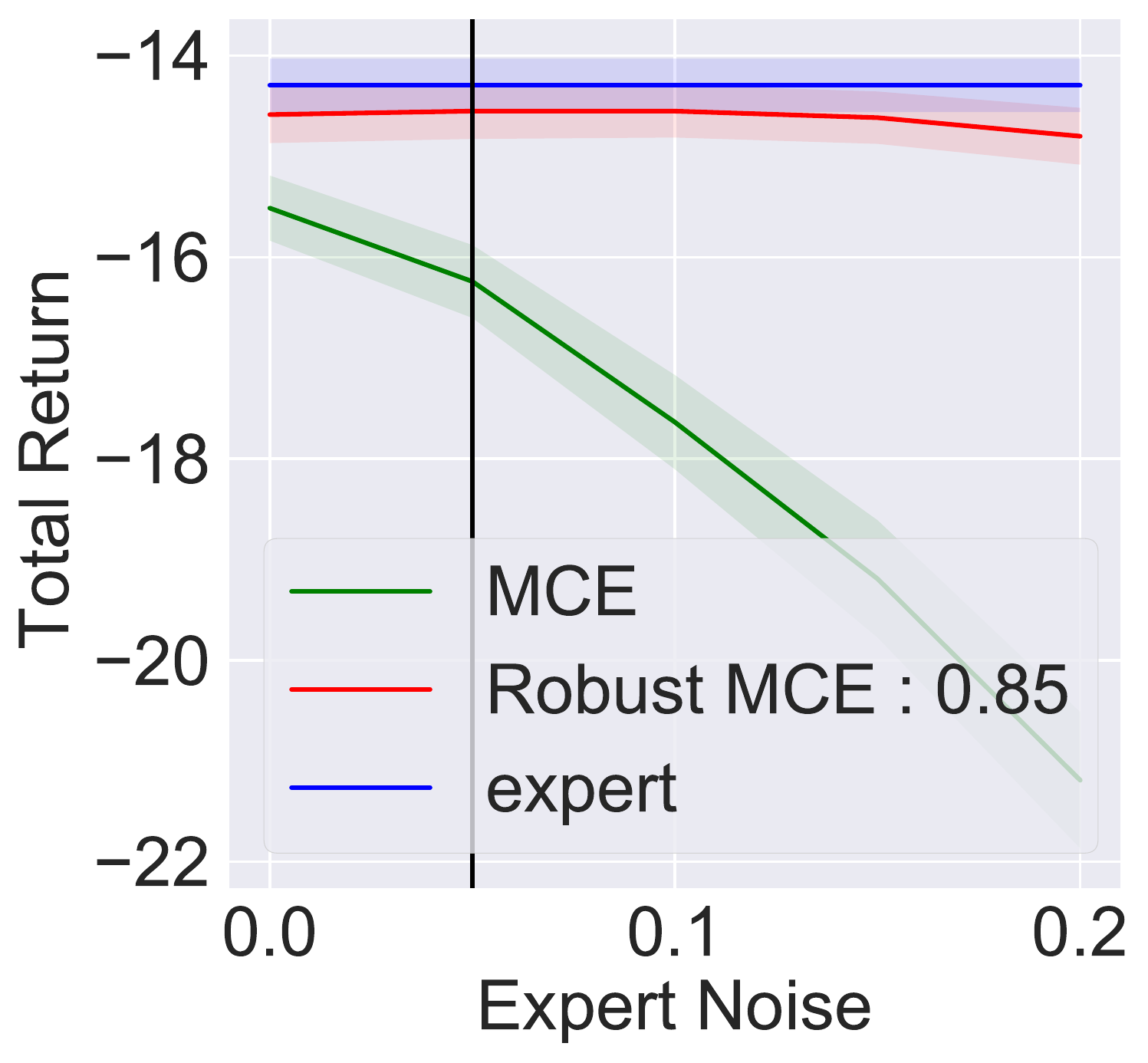}
\caption{$M^{L,\epsilon_L}$ with $\epsilon_L = 0.05$} \label{fig:etdw1l0.05pres}
\end{subfigure}\hspace*{\fill}
\begin{subfigure}{0.24\textwidth}
\includegraphics[width=\linewidth]{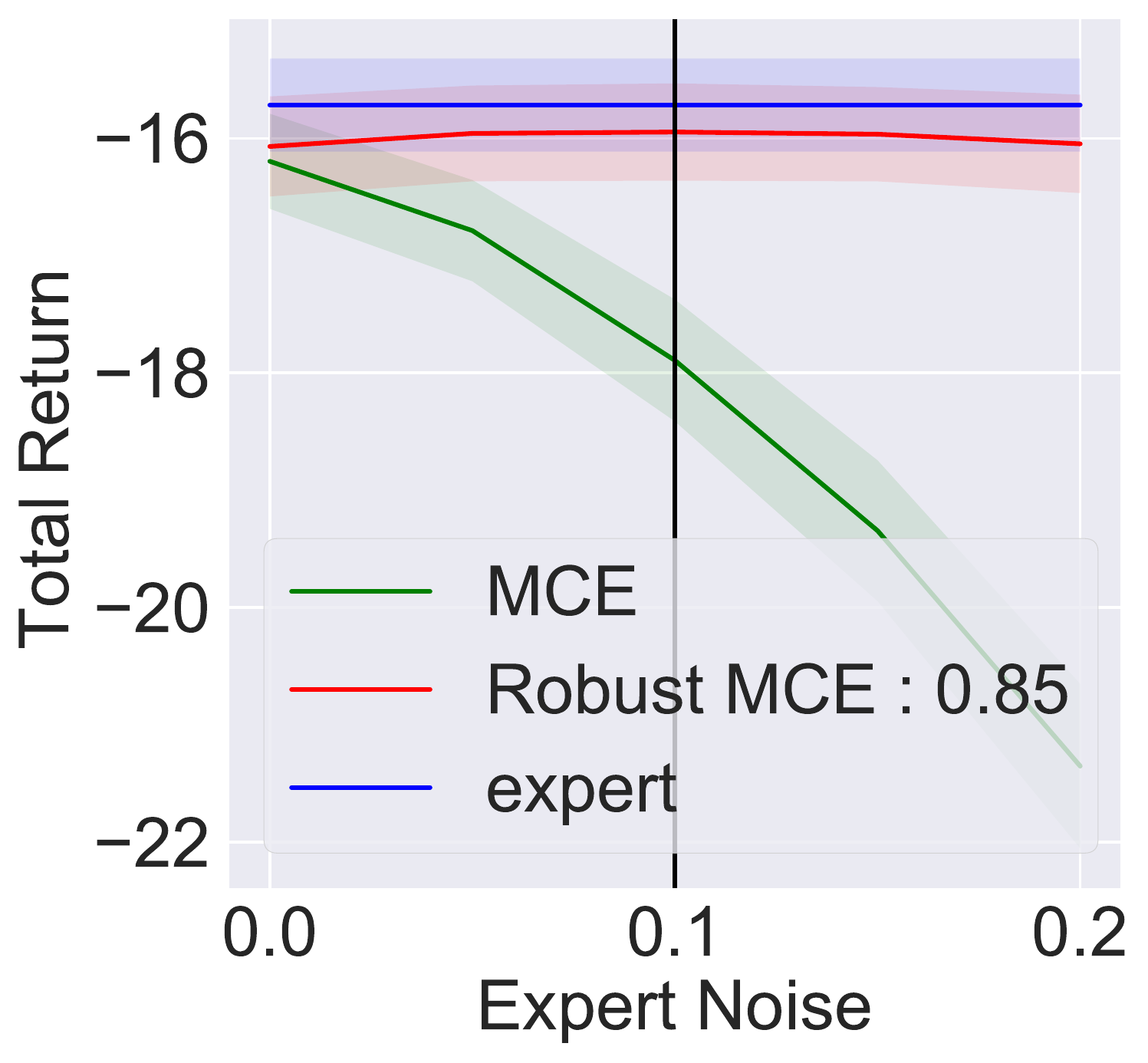}
\caption{$M^{L,\epsilon_L}$ with $\epsilon_L = 0.1$} \label{fig:etdw1l0.1pres}
\end{subfigure}
\caption{Comparison of the performance our Algorithm~\ref{alg:MaxEntIRL} against the baselines, under different levels of mismatch: $\br{\epsilon_E, \epsilon_L} \in \bc{0.0, 0.05, 0.1, 0.15, 0.2} \times \bc{ 0.0, 0.05, 0.1}$. Each plot corresponds to a fixed leaner environment $M^{L,\epsilon_L}$ with $\epsilon_L \in \bc{ 0.0, 0.05, 0.1}$. The values of $\alpha$ used for our Algorithm~\ref{alg:MaxEntIRL} are reported in the legend. The vertical line indicates the position of the learner environment in the x-axis.} 
\label{fig:low-dim-exp-short}
\end{figure*}

\begin{figure}[h!] 
\centering
\begin{subfigure}{0.24\textwidth}
\includegraphics[width=0.87\linewidth]{plots/TDW_reward1.pdf}
\caption{\textsc{GridWorld-L}} \label{fig:trtdw1abl}
\end{subfigure}
\begin{subfigure}{0.24\textwidth}
\includegraphics[width=\linewidth]{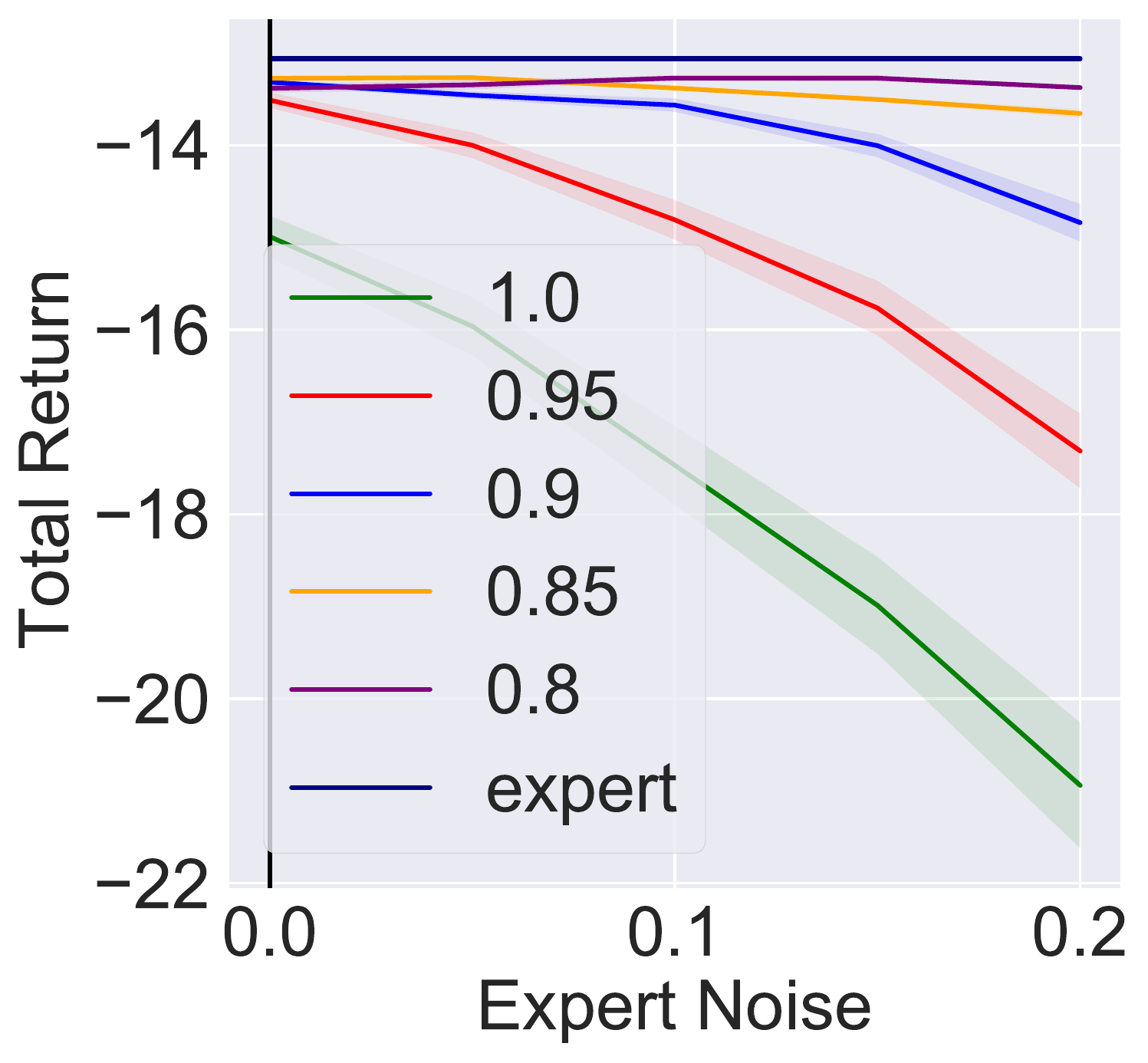}
\caption{$M^{L,\epsilon_L}$ with $\epsilon_L = 0$} \label{fig:etdw1l0.0abl}
\end{subfigure}\hspace*{\fill}
\begin{subfigure}{0.24\textwidth}
\includegraphics[width=\linewidth]{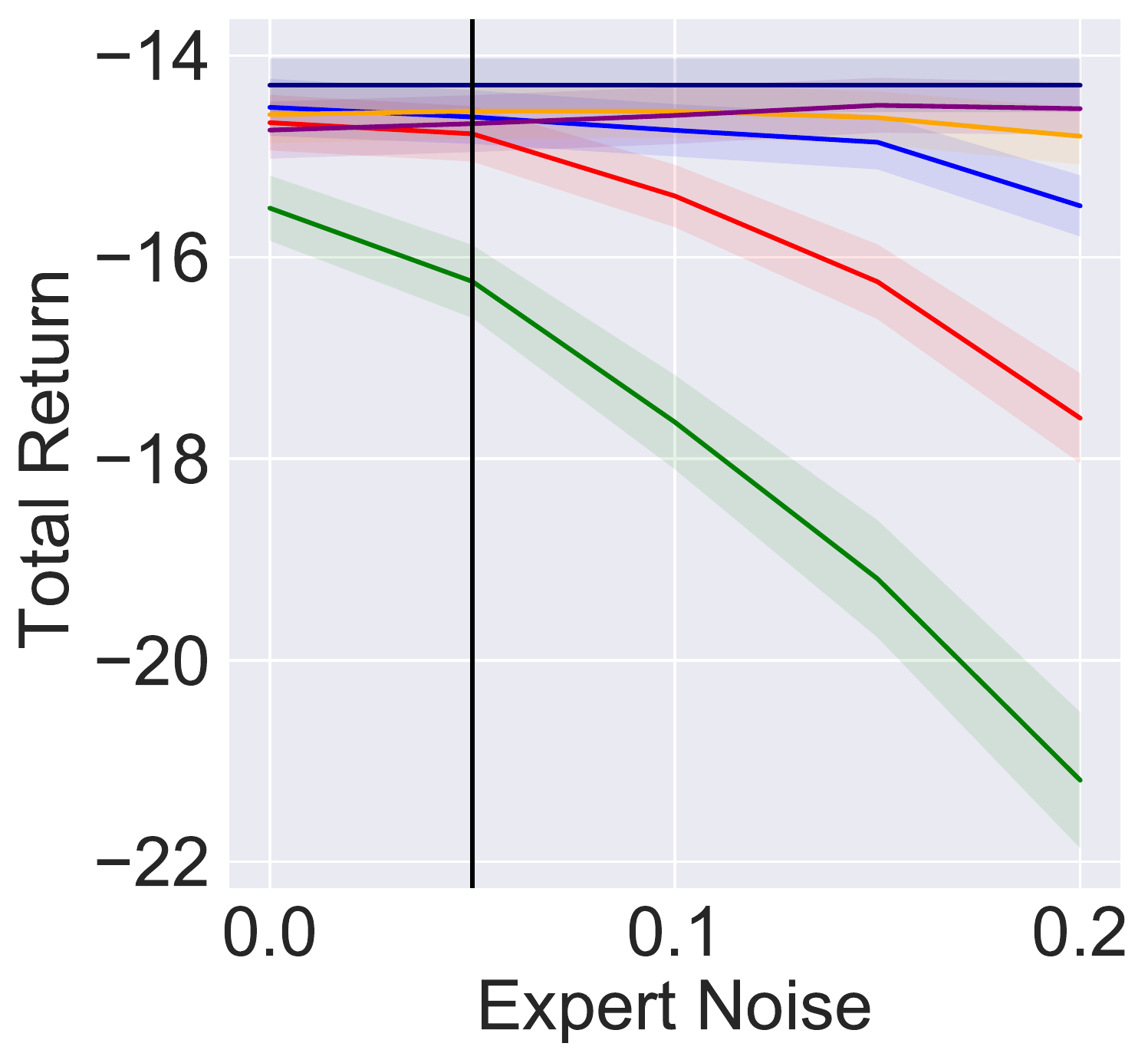}
\caption{$M^{L,\epsilon_L}$ with $\epsilon_L = 0.05$} \label{fig:etdw1l0.05abl}
\end{subfigure}\hspace*{\fill}
\begin{subfigure}{0.24\textwidth}
\includegraphics[width=\linewidth]{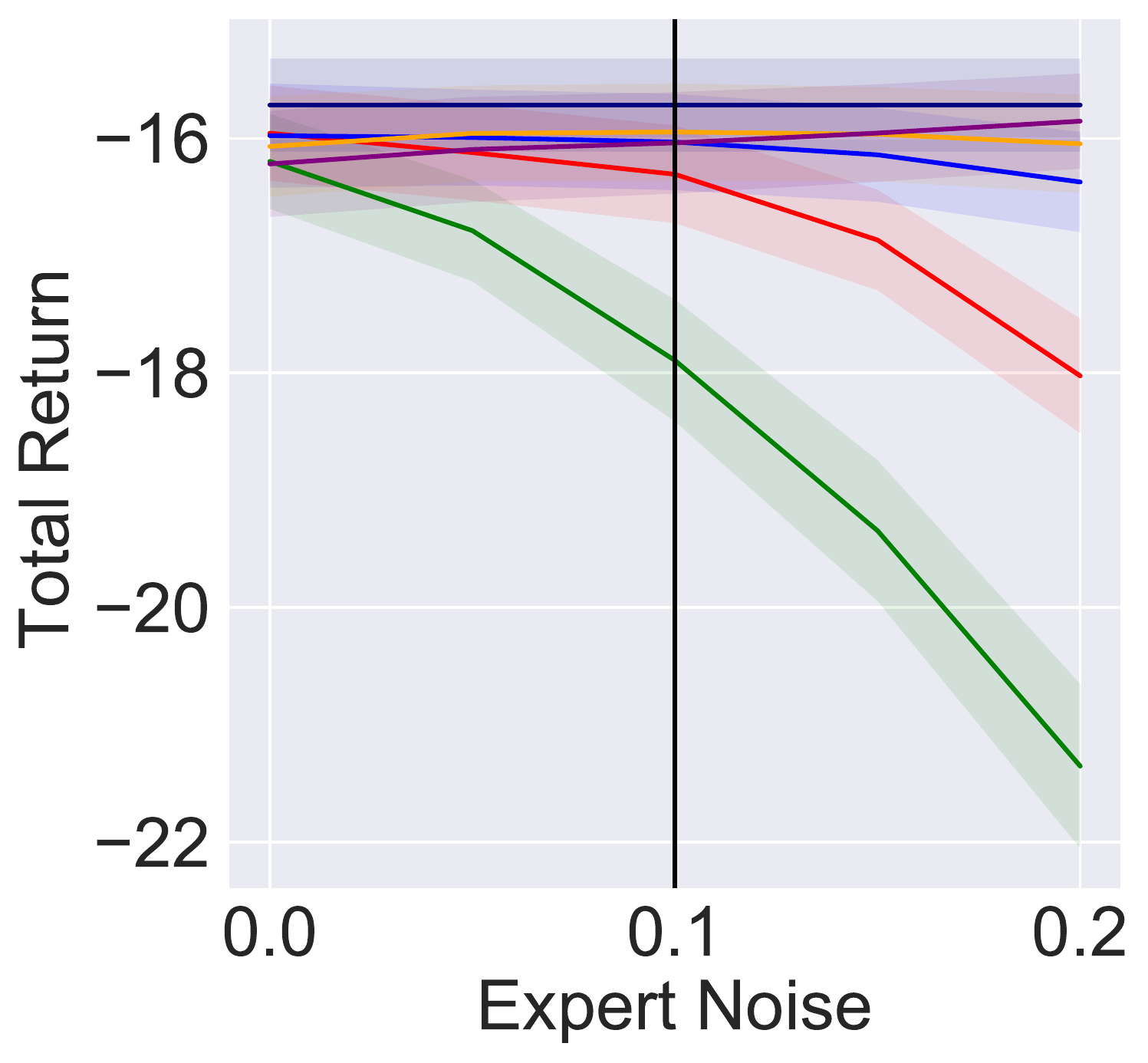}
\caption{$M^{L,\epsilon_L}$ with $\epsilon_L = 0.1$} \label{fig:etdw1l0.1abl}
\end{subfigure}
\caption{Comparison of the performance our Algorithm~\ref{alg:MaxEntIRL} with different values of $\alpha$, under different levels of mismatch: $\br{\epsilon_E, \epsilon_L} \in \bc{0.0, 0.05, 0.1, 0.15, 0.2} \times \bc{ 0.0, 0.05, 0.1}$. Each plot corresponds to a fixed leaner environment $M^{L,\epsilon_L}$ with $\epsilon_L \in \bc{ 0.0, 0.05, 0.1}$. The values of $\alpha$ used for our Algorithm~\ref{alg:MaxEntIRL} are reported in the legend. The vertical line indicates the position of the learner environment in the x-axis.} 
\label{fig:low-dim-exp-full}
\end{figure}

\newpage
\subsection{Impact of the Opponent Strength Parameter $1 - \alpha$ on Robust MCE IRL}
\label{app:alpha-choice}

\begin{figure*}[h!]
\centering
\begin{subfigure}{0.3\textwidth}
\includegraphics[width=\linewidth]{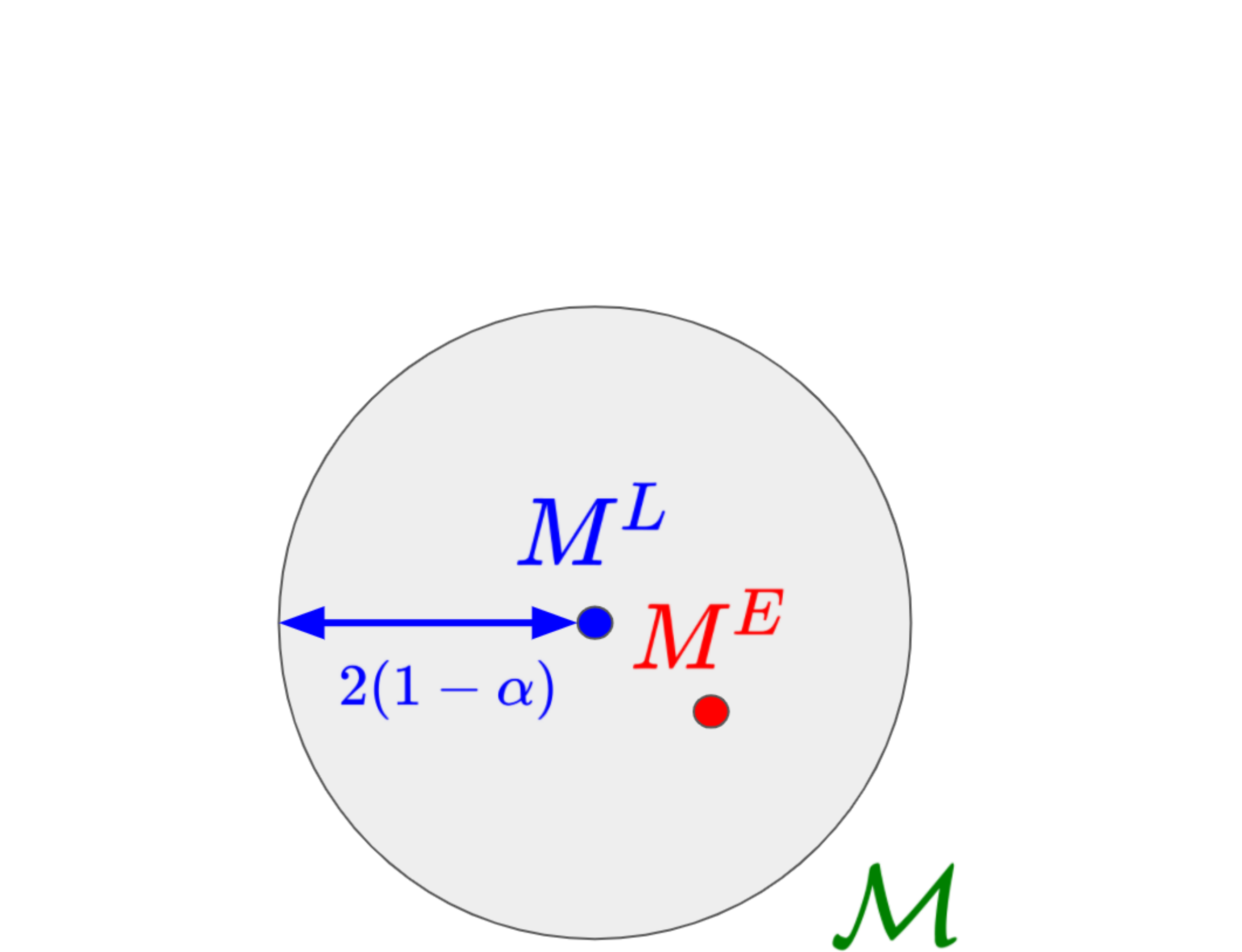}
\caption{Overestimating $1-\alpha$} 
\end{subfigure}\hspace*{\fill}
\begin{subfigure}{0.3\textwidth}
\includegraphics[width=\linewidth]{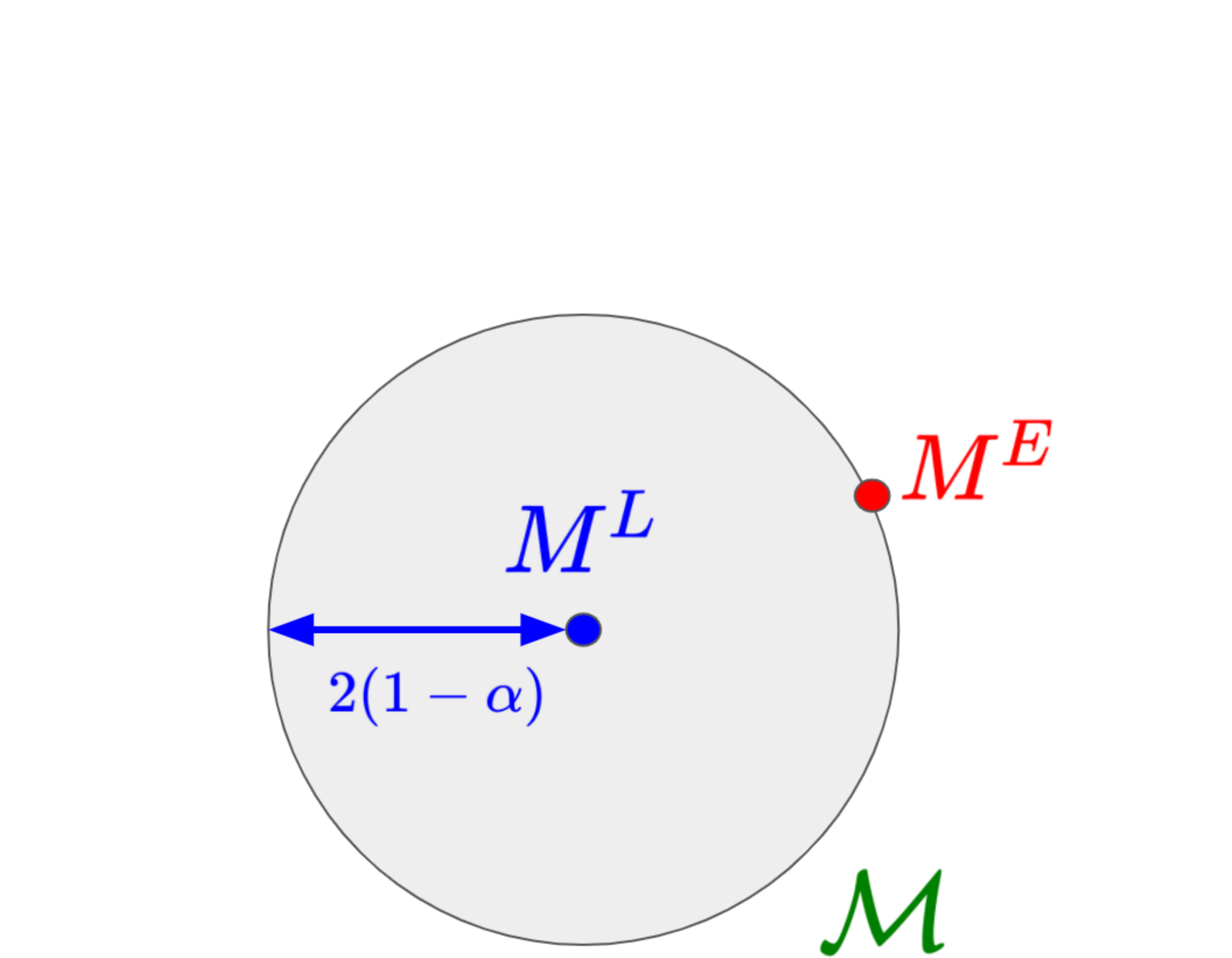}
\caption{Perfect estimation of $1-\alpha$} 
\end{subfigure}\hspace*{\fill}
\begin{subfigure}{0.3\textwidth}
\includegraphics[width=\linewidth]{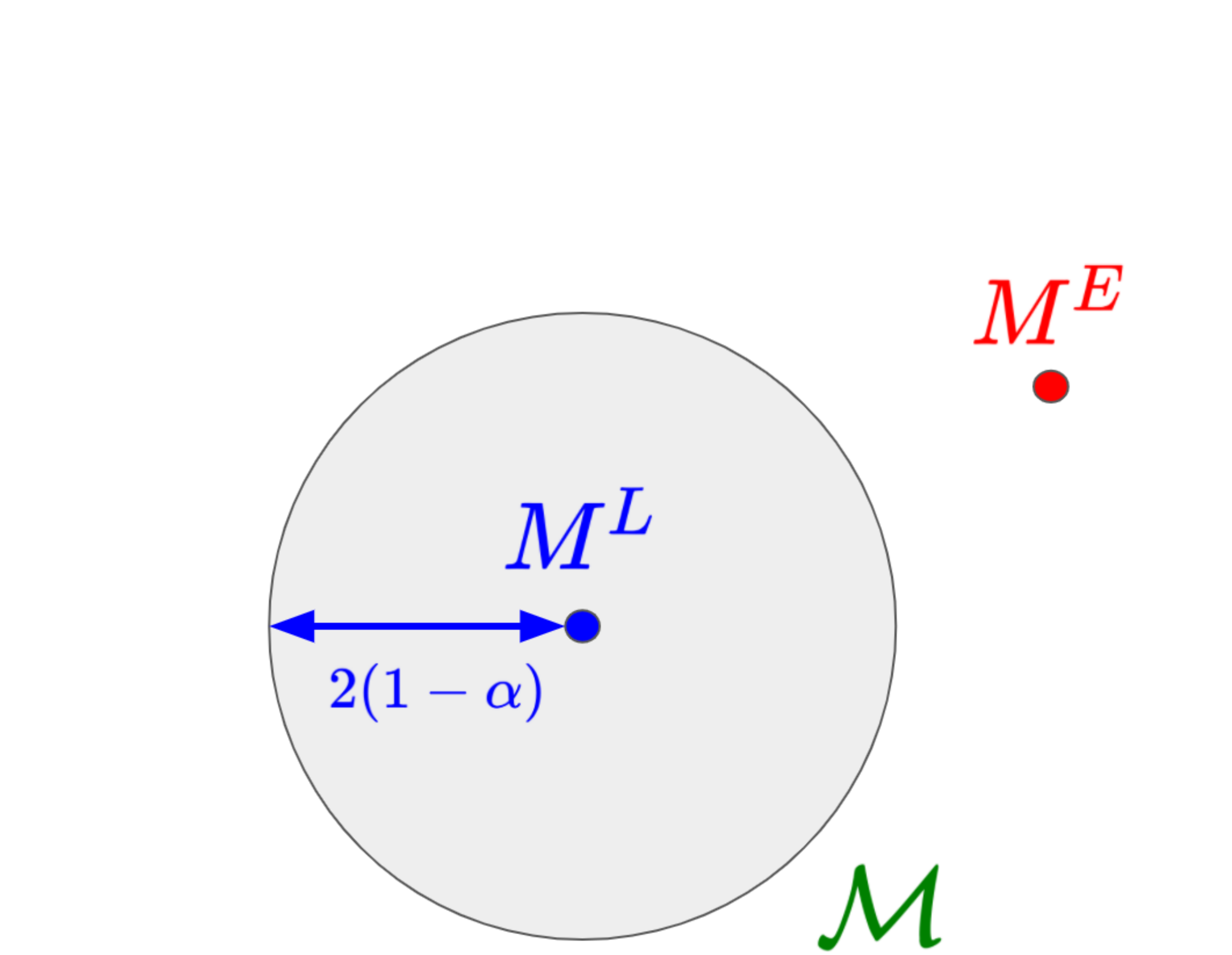}
\caption{Underestimating $1-\alpha$} 
\end{subfigure}
\caption{Illustration of the three cases related to the choice of the opponent strength parameter $1-\alpha$.}
\label{fig:select_alpha}
\end{figure*}

Here, we study the effect of the opponent strength parameter $(1 - \alpha)$ on the performance of our Algorithm~\ref{alg:MaxEntIRL}. Consider the uncertainty set associated with our Algorithm~\ref{alg:MaxEntIRL}:
\[
\mathcal{T}^{L,\alpha} = \bc{T: d_\mathrm{dyn} \br{T, T^L} \leq 2 (1-\alpha)} .
\]
Ideally, we prefer to choose the smallest set $\mathcal{T}^{L,\alpha}$ s.t. $T^E \in \mathcal{T}^{L,\alpha}$. To this end, we consider the following three cases (see Figure~\ref{fig:select_alpha}):
\begin{enumerate}
\item overestimating the opponent strength, i.e., $1-\alpha > \frac{d_\mathrm{dyn} \br{T^E, T^L}}{2}$. 
\item perfect estimation of the opponent strength, i.e., $1-\alpha = \frac{d_\mathrm{dyn} \br{T^E, T^L}}{2}$.
\item underestimating the opponent strength, i.e., $1-\alpha < \frac{d_\mathrm{dyn} \br{T^E, T^L}}{2}$. 
\end{enumerate}

\begin{figure*}[h!]
\centering
\begin{subfigure}{0.4\textwidth}
\includegraphics[width=\linewidth]{plots/1.pdf}
\caption{\textsc{GridWorld-1}} \label{fig:maintr1-new}
\end{subfigure}
\begin{subfigure}{0.4\textwidth}
\includegraphics[width=\linewidth]{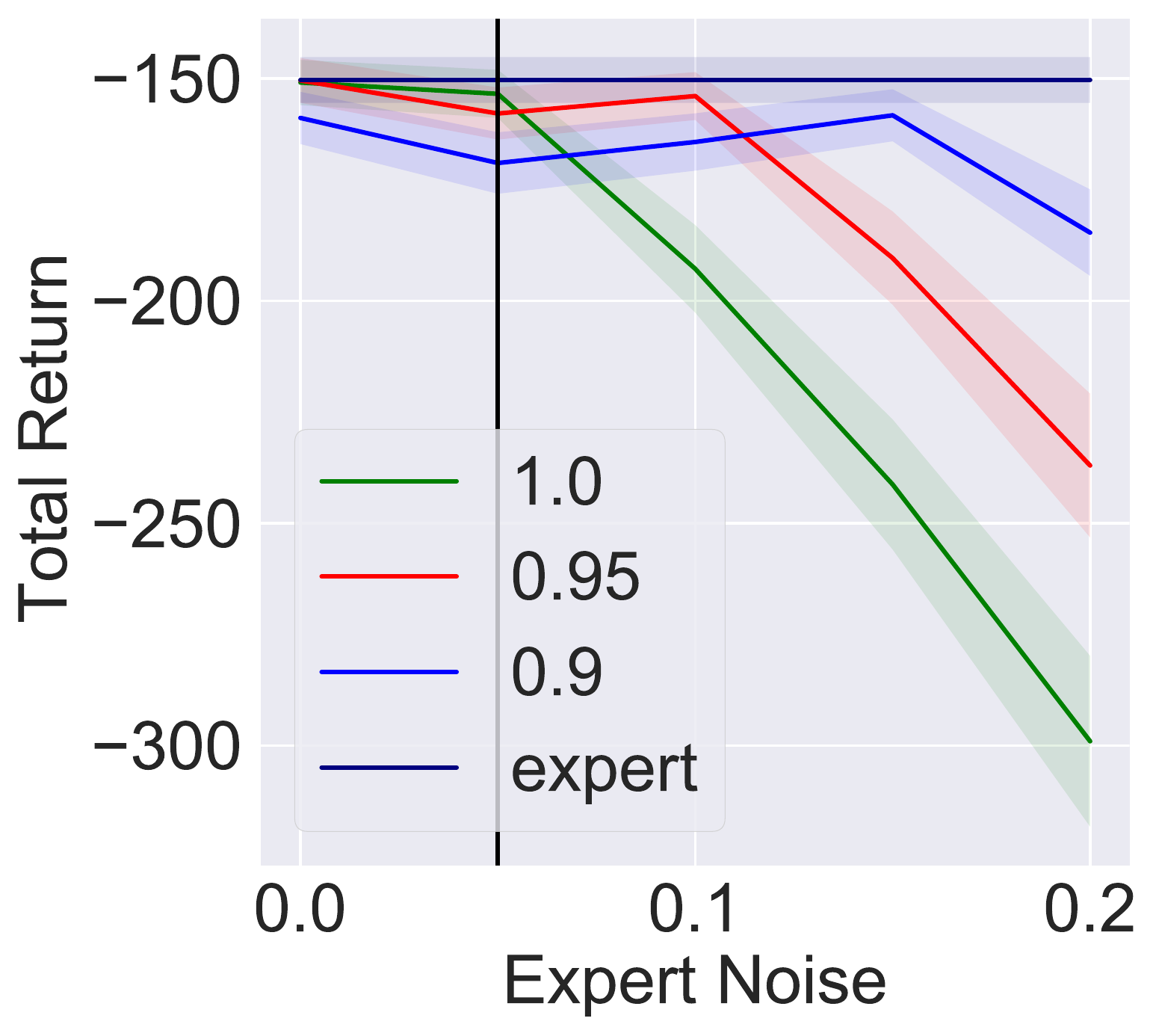}
\caption{$M^{L,\epsilon_L}$ with $\epsilon_L = 0.05$} \label{fig:maine1l0.05-new}
\end{subfigure}
\caption{Comparison of the performance our Algorithm~\ref{alg:MaxEntIRL} with different values of the player strength parameter $\alpha \in \bc{ 0.9, 0.95, 1.0}$, under different levels of mismatch: $\br{\epsilon_E, \epsilon_L} \in \bc{0.0, 0.05, 0.1, 0.15, 0.2} \times \bc{ 0.05}$. The values of $\alpha$ used for our Algorithm~\ref{alg:MaxEntIRL} are reported in the legend. Every point in the x-axis denotes an expert environment $M^{E,\epsilon_E}$ with the corresponding $\epsilon_E$. The vertical line indicates the position of the learner environment $M^{L,\epsilon_L}$ in the x-axis. Note that moving away from the vertical line increases the mismatch between the learner and the expert, i.e., $\abs{\epsilon_L - \epsilon_E}$.}
\label{fig:select_alpha_experiment}
\end{figure*}

Now, consider the experimental setup described in Section~\ref{sec:experiments}. Recall that, in this setup, the distance between the learner and the expert environment is given by $d_\mathrm{dyn} \br{T^{L,\epsilon_L}, T^{E, \epsilon_E}} = 2\br{1 - \frac{1}{\abs{\mathcal{S}}}}\abs{\epsilon_L - \epsilon_E}$. Thus, a reasonable choice for the opponent strength would be $1-\alpha \approx \abs{\epsilon_L - \epsilon_E}$. We note the following behavior in Figure~\ref{fig:select_alpha_experiment}:
\begin{itemize}
\item For $\alpha = 1.0$ ({\color{kagreen}---}), we observe a linear decay in the performance when moving away from the vertical line, i.e, with the increase of mismatch. Note that this curve corresponds to the MCE IRL algorithm. 
\item For $\alpha = 0.95$ ({\color{red}---}), we observe a linear decay in the performance when moving away from the vertical line, after $\epsilon_E = 0.10$. Note that, for $1 - \alpha \approx 0.05$, beyond $\epsilon_L \pm 0.05$ is underestimation region (here, $\epsilon_L = 0.05$). 
\item For $\alpha = 0.9$ ({\color{blue}---}), we observe a linear decay in the performance when moving away from the vertical line, after $\epsilon_E = 0.15$. Note that, for $1 - \alpha \approx 0.1$, beyond $\epsilon_L \pm 0.1$ is underestimation region (here, $\epsilon_L = 0.05$).
\item Within the overestimation region, choosing the larger value of $1-\alpha$ hinders the performance. For example, the region $\epsilon_L \pm 0.05$ is overestimation region for both $1 - \alpha \approx 0.05$ ({\color{red}---}) and $1 - \alpha \approx 0.1$ ({\color{blue}---}). Within this region, the performance of ({\color{blue}---}) curve is lower than that of ({\color{red}---}) curve. 
\end{itemize}

In addition, in Figure~\ref{fig:gridworld_diff_alpha}, we note the following:
\begin{itemize}
\item In general, the curves $\alpha = 1.0$ ({\color{kagreen}---}), $\alpha = 0.95$ ({\color{red}---}), and $1 - \alpha \approx 0.1$ demonstrated the above discussed behavior on the right hand side of the vertical line. Note that the right hand side of the vertical line represents the setting where the expert environment is more stochastic/noisy than the learner environment.
\item In general, the curves $\alpha = 1.0$ ({\color{kagreen}---}), $\alpha = 0.95$ ({\color{red}---}), and $1 - \alpha \approx 0.1$ demonstrated a stable and good performance on the left hand side of the vertical line. Note that the left hand side of the vertical line represents the setting where the expert environment is more deterministic than the learner environment. 
\end{itemize}

To choose the right value of $\alpha$, that depends on $d_\mathrm{dyn} \br{T^E, T^L}$, we need to have an estimate $\widehat T^E$ of the expert environment $T^E$. A few recent works~\cite{reddy2018you,Gong2020WhatII,Herman2016InverseRL} attempt to infer the expert's transition dynamics from the demonstration set or via additional information. Our robust IRL approach can be incorporated into this research vein to improve the IRL agent's performance further.

\newpage

\begin{table}[h!]
\caption{Hyperparameters for the \textsc{GridWorld} experiments}
\centering
\begin{tabular}{ll}
\hline
\hline
\textbf{Hyperparameter}              & \textbf{Value} \\ \hline
IRL Optimizer                        & Adam           \\
Learning rate                        & $0.5$          \\
Weight decay                         & $0.0$          \\
First moment exponential decay rate  & $0.9$          \\
Second moment exponential decay rate & $0.99$         \\
Numerical stabilizer                 & $1e-7$         \\
Number of steps                      & $200$          \\ 
Discount factor $\gamma$             & $0.99$          \\ \hline
\end{tabular}
\label{tab:Adamtable}
\end{table}

\begin{table}[h!]
\caption{Hyperparameters for the \textsc{ObjectWorld} experiments}
\centering
\begin{tabular}{ll}
\hline
\hline
\textbf{Hyperparameter}              & \textbf{Value} \\ \hline
IRL Optimizer                        & Adam           \\
Learning rate                        & $1e-3$          \\
Weight decay                         & $0.01$          \\
First moment exponential decay rate  & $0.9$          \\
Second moment exponential decay rate & $0.999$         \\
Numerical stabilizer                 & $1e-8$         \\
Number of steps                      & $200$          \\ 
Reward network & two 2D-CNN layers; layers size = number of input features; ReLu \\
Discount factor $\gamma$             & $0.7$          \\ \hline
\end{tabular}
\label{tab:InfHorNonlinearAdamtable}
\end{table}

\begin{table}[h!] 
\caption{Hyperparameters for the MDP solvers}
\centering
\begin{tabular}{ll}
\hline
\hline
\textbf{Hyperparameter}              & \textbf{Value} \\ \hline
Two-Player soft value iteration tolerance & $1e-10$           \\
Soft value iteration tolerance & $1e-10$           \\
Value iteration tolerance & $1e-10$           \\
Policy propagation tolerance & $1e-10$           \\
\hline
\end{tabular}
\label{tab:MDPtable}
\end{table}


\newpage

\begin{figure}[h!]  
\centering
\begin{subfigure}{0.2\textwidth}
\includegraphics[width=\linewidth]{plots/1.pdf}
\caption{\textsc{GridWorld-1}} \label{fig:grid1-repeat}
\end{subfigure}\hspace*{\fill}
\begin{subfigure}{0.2\textwidth}
\includegraphics[width=\linewidth]{plots/fillBetweenNotebookCompare_Alphas_Env_1_noise_L_0.0dim10alphaE1.0fix_startFalselegendTruepresentation.pdf}
\caption{$M^{L,\epsilon_L}$ with $\epsilon_L = 0$} 
\end{subfigure}\hspace*{\fill}
\begin{subfigure}{0.2\textwidth}
\includegraphics[width=\linewidth]{plots/fillBetweenNotebookCompare_Alphas_Env_1_noise_L_0.05dim10alphaE1.0fix_startFalselegendTruepresentation.pdf}
\caption{$M^{L,\epsilon_L}$ with $\epsilon_L = 0.05$} 
\end{subfigure}\hspace*{\fill}
\begin{subfigure}{0.2\textwidth}
\includegraphics[width=\linewidth]{plots/fillBetweenNotebookCompare_Alphas_Env_1_noise_L_0.1dim10alphaE1.0fix_startFalselegendTruepresentation.pdf}
\caption{$M^{L,\epsilon_L}$ with $\epsilon_L = 0.1$} 
\end{subfigure}
\medskip
\begin{subfigure}{0.2\textwidth}
\includegraphics[width=\linewidth]{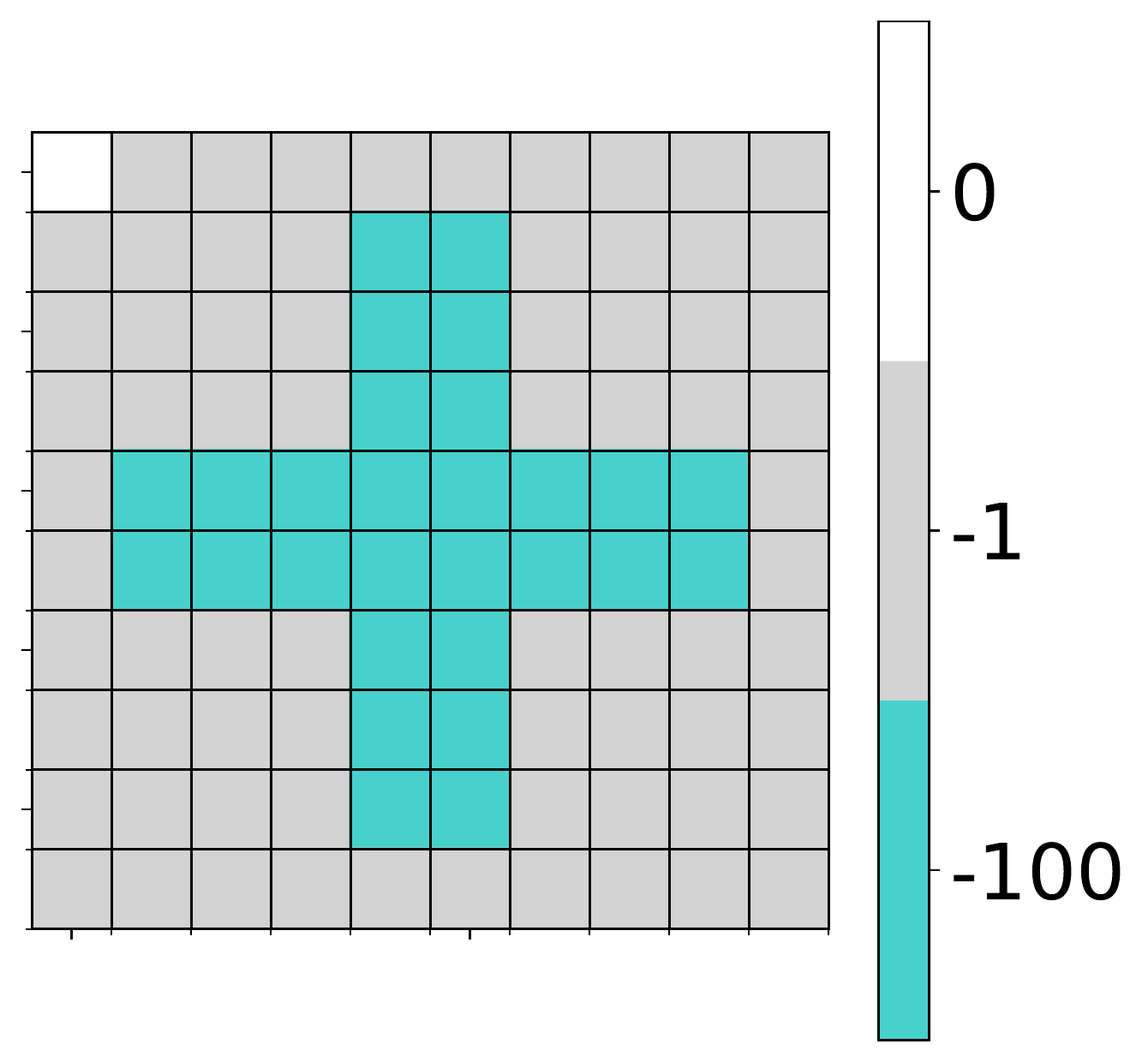}
\caption{\textsc{GridWorld-2}} \label{fig:grid2}
\end{subfigure}\hspace*{\fill}
\begin{subfigure}{0.2\textwidth}
\includegraphics[width=\linewidth]{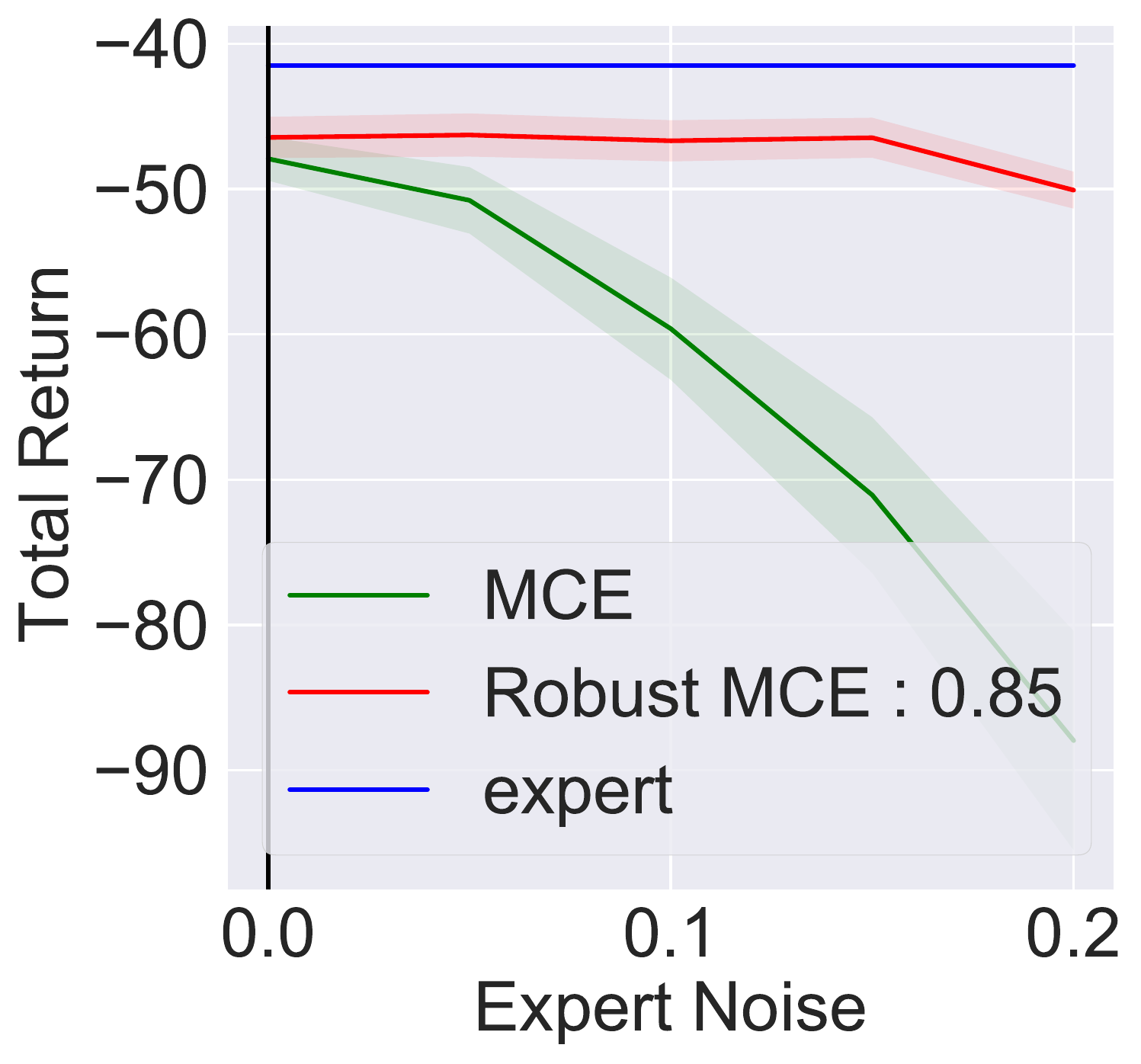}
\caption{$M^{L,\epsilon_L}$ with $\epsilon_L = 0$} \label{fig:e2l0.0pres}
\end{subfigure}\hspace*{\fill}
\begin{subfigure}{0.2\textwidth}
\includegraphics[width=\linewidth]{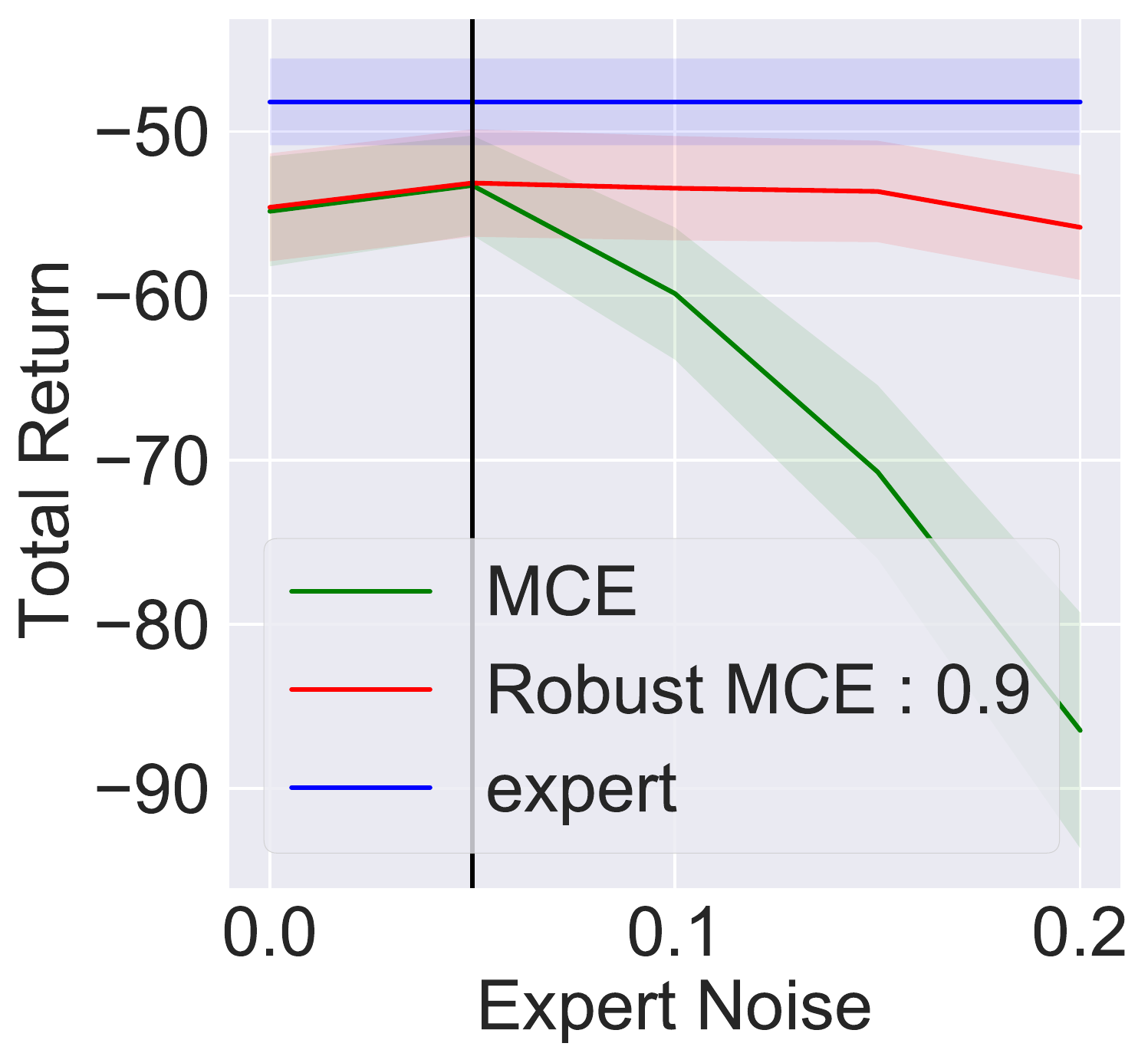}
\caption{$M^{L,\epsilon_L}$ with $\epsilon_L = 0.05$} \label{fig:e2l0.05pres}
\end{subfigure}\hspace*{\fill}
\begin{subfigure}{0.2\textwidth}
\includegraphics[width=\linewidth]{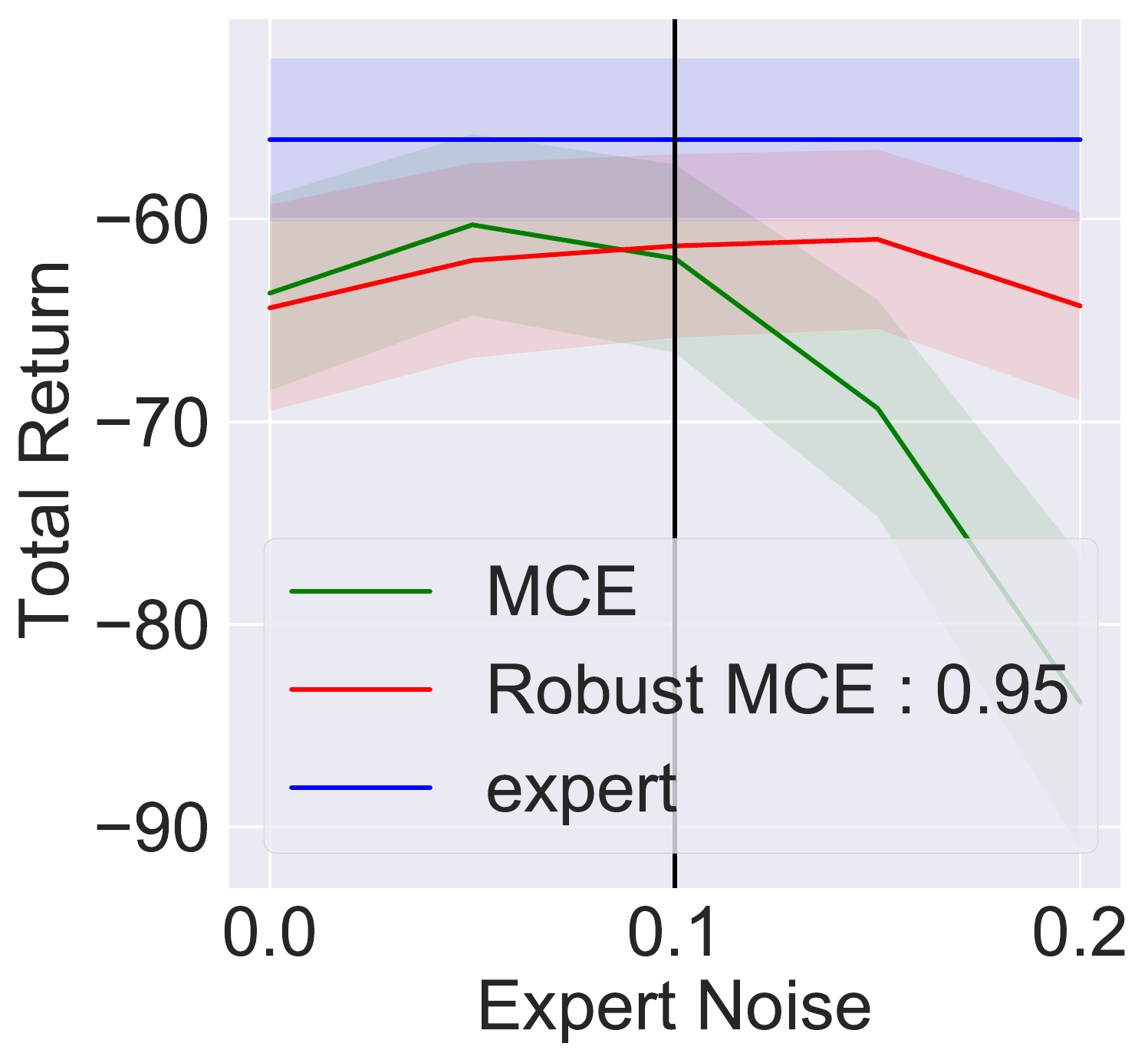}
\caption{$M^{L,\epsilon_L}$ with $\epsilon_L = 0.1$} \label{fig:e2l0.1pres}
\end{subfigure}
\medskip
\begin{subfigure}{0.2\textwidth}
\includegraphics[width=\linewidth]{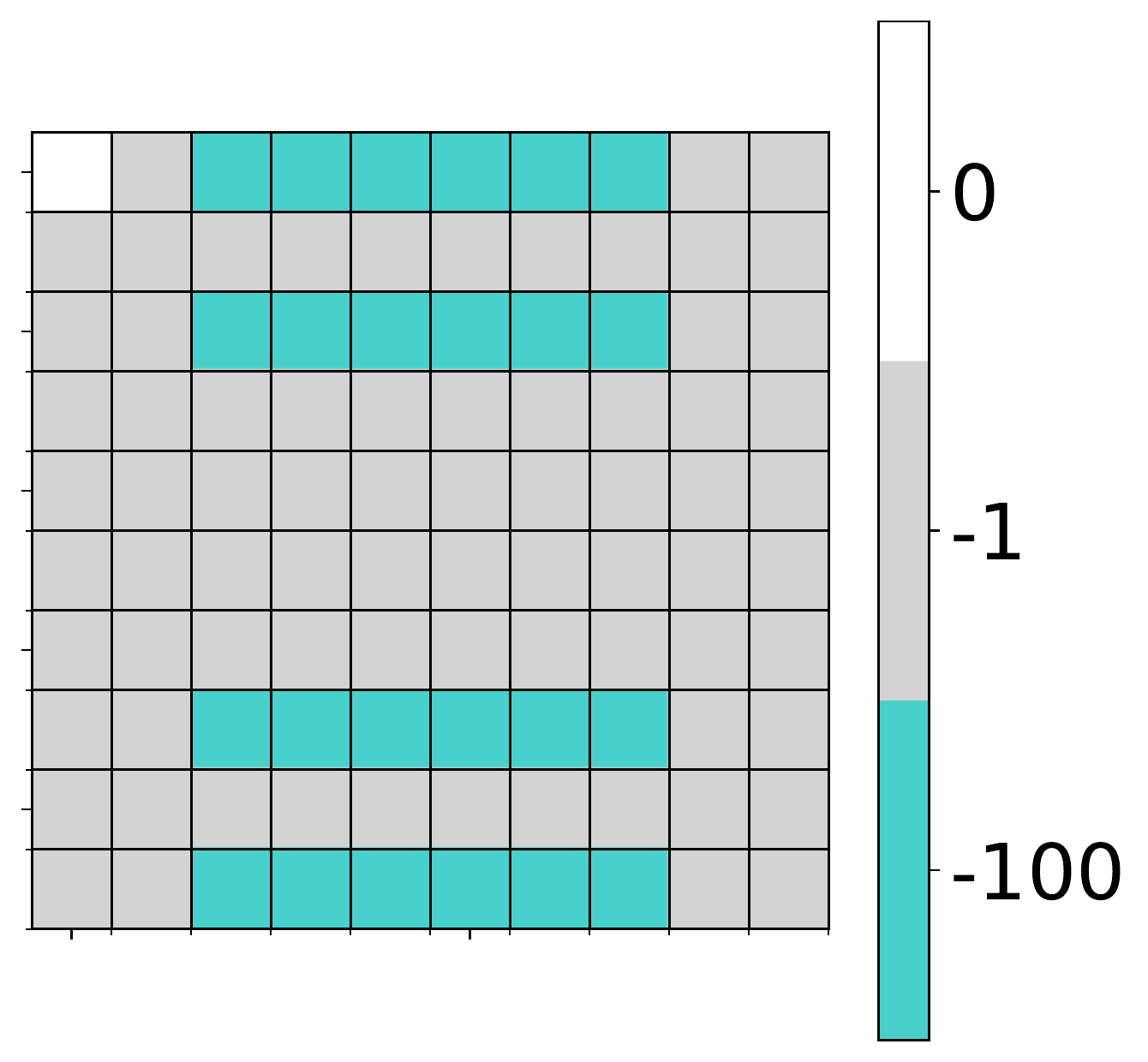}
\caption{\textsc{GridWorld-3}} \label{fig:grid3}
\end{subfigure}\hspace*{\fill}
\begin{subfigure}{0.2\textwidth}
\includegraphics[width=\linewidth]{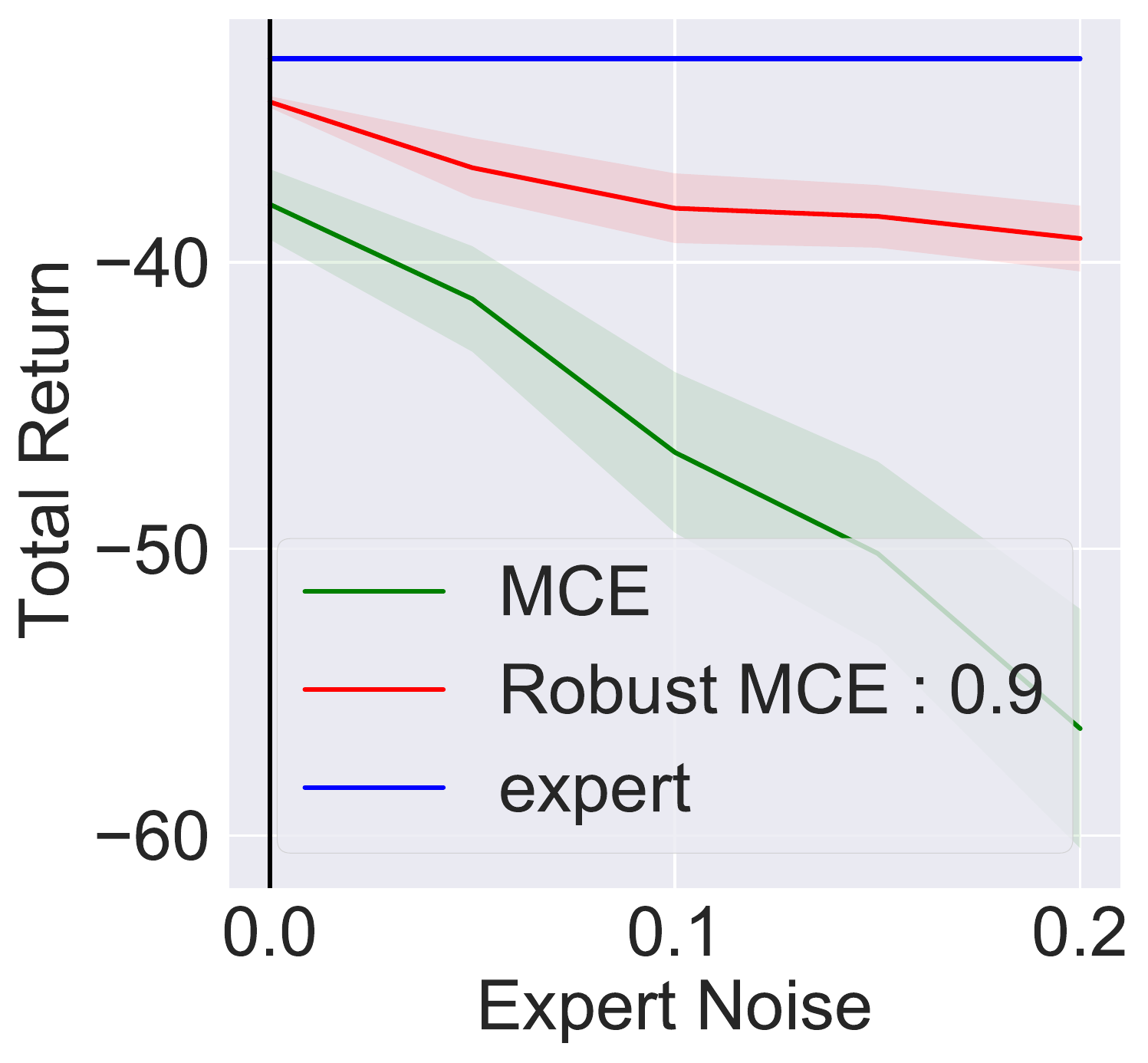}
\caption{$M^{L,\epsilon_L}$ with $\epsilon_L = 0$} \label{fig:e9l0.0pres}
\end{subfigure}\hspace*{\fill}
\begin{subfigure}{0.2\textwidth}
\includegraphics[width=\linewidth]{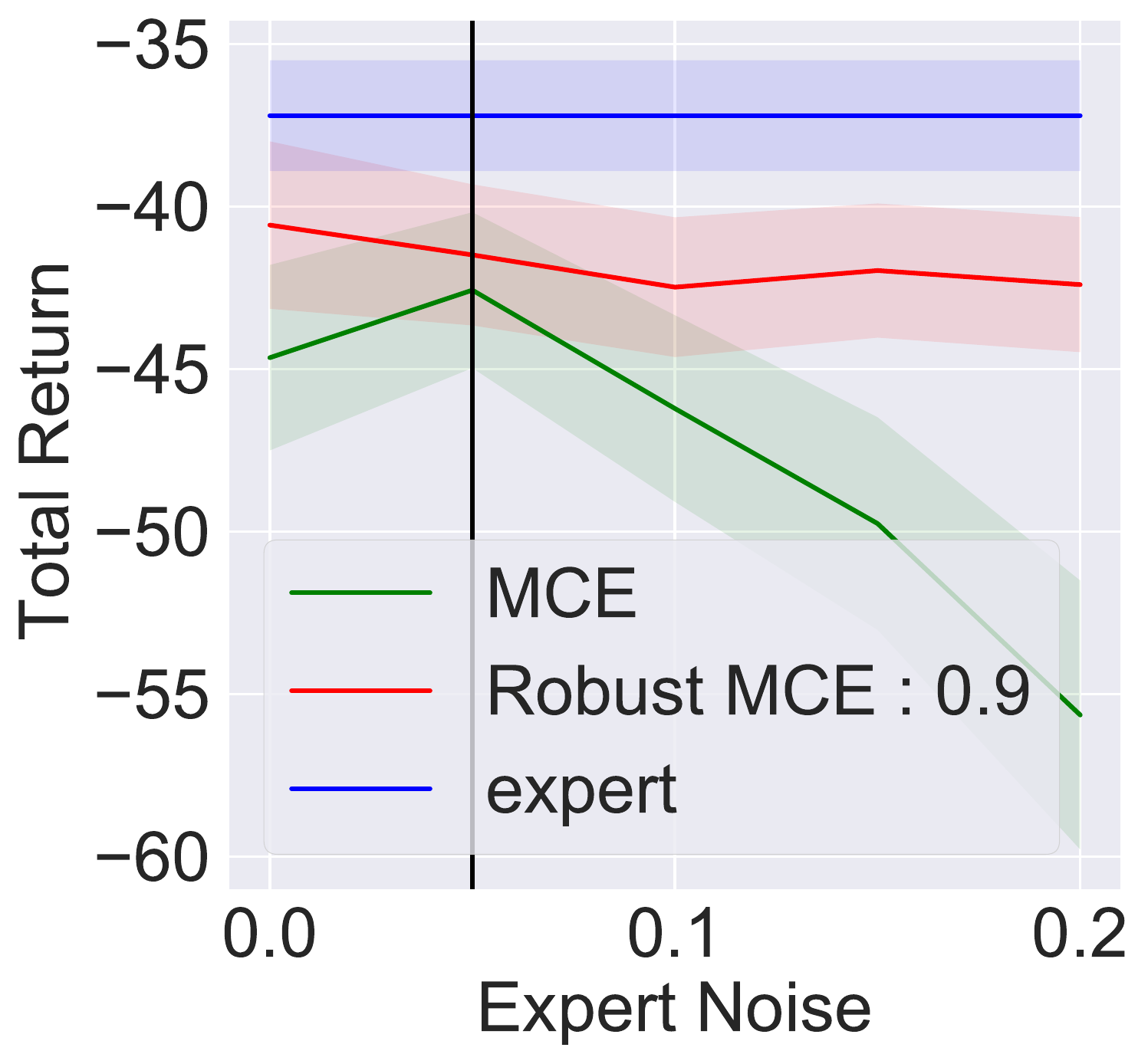}
\caption{$M^{L,\epsilon_L}$ with $\epsilon_L = 0.05$} \label{fig:e9l0.05pres}
\end{subfigure}\hspace*{\fill}
\begin{subfigure}{0.2\textwidth}
\includegraphics[width=\linewidth]{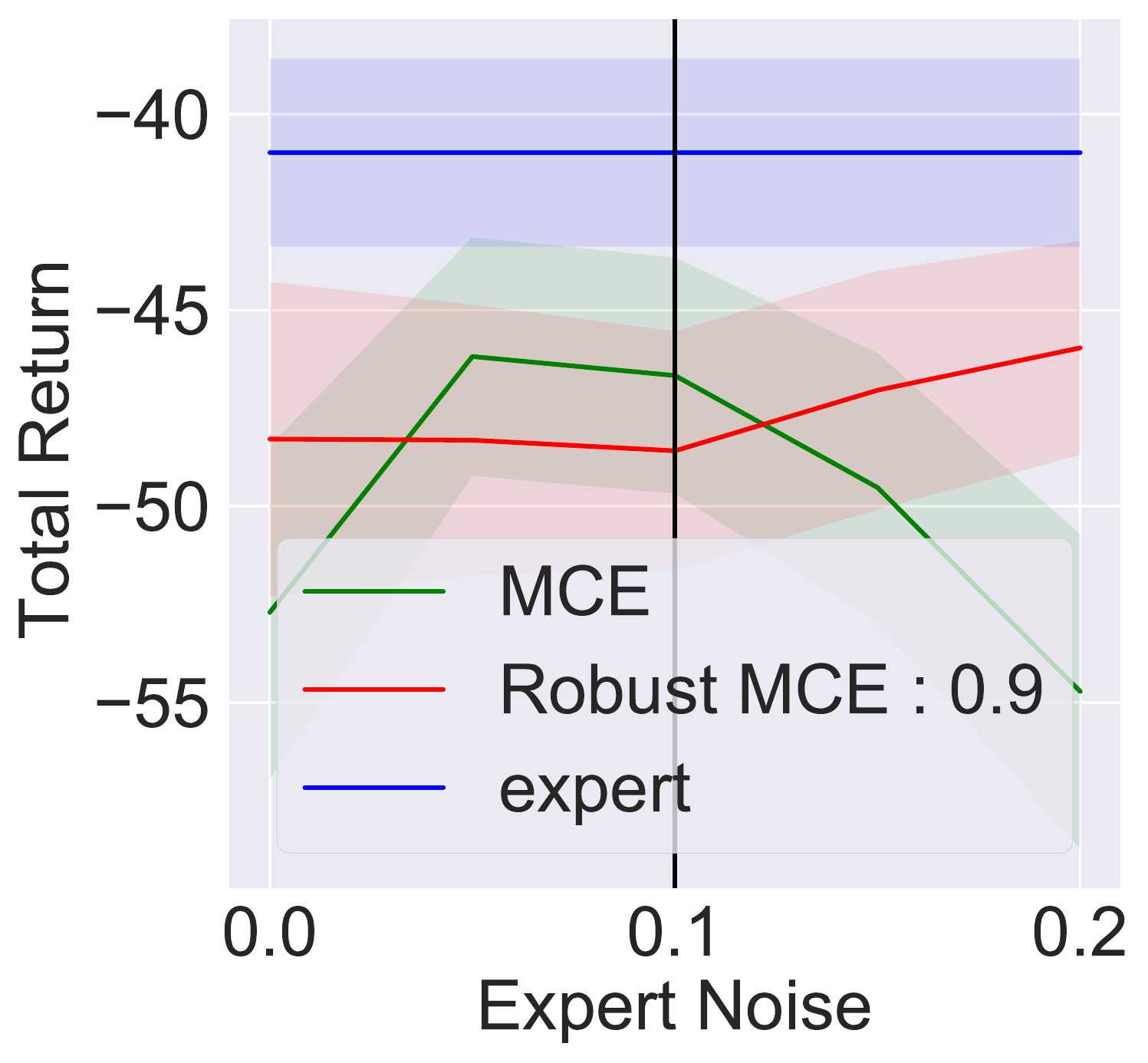}
\caption{$M^{L,\epsilon_L}$ with $\epsilon_L = 0.1$} \label{fig:e9l0.1pres}
\end{subfigure}
\medskip
\begin{subfigure}{0.2\textwidth}
\includegraphics[width=\linewidth]{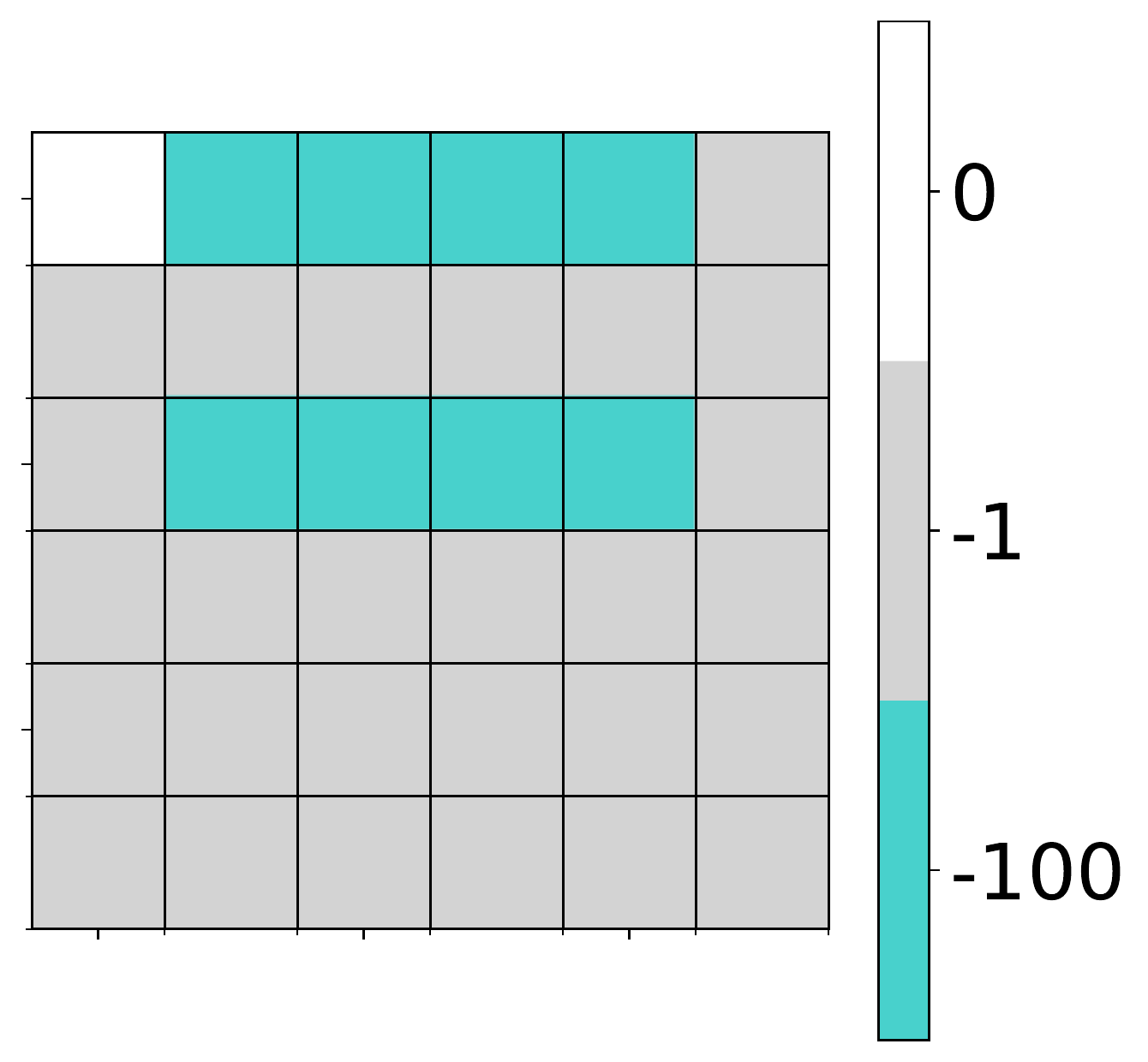}
\caption{\textsc{GridWorld-4}} \label{fig:grid4}
\end{subfigure}\hspace*{\fill}
\begin{subfigure}{0.2\textwidth}
\includegraphics[width=\linewidth]{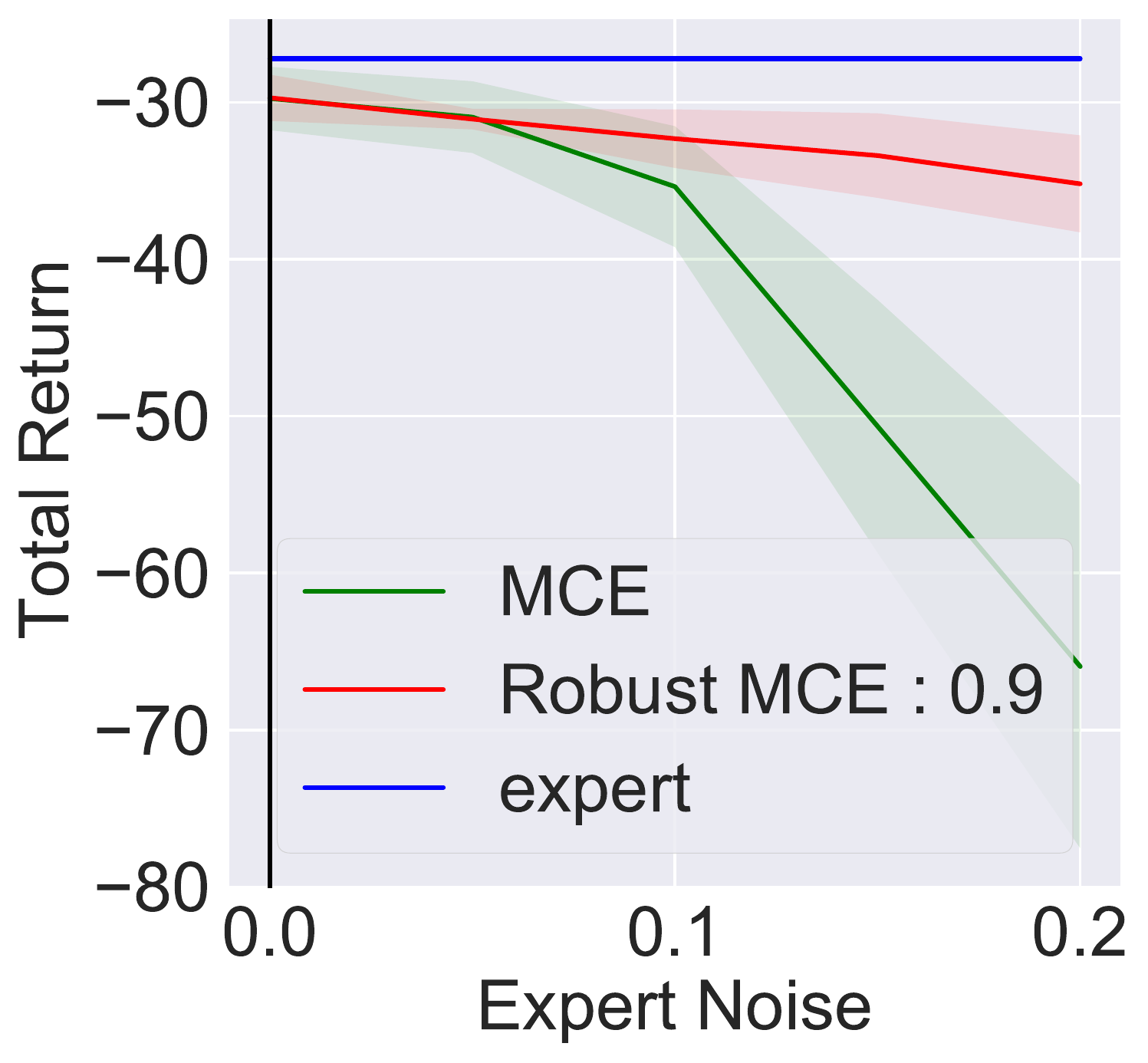}
\caption{$M^{L,\epsilon_L}$ with $\epsilon_L = 0$} \label{fig:e10l0.0pres}
\end{subfigure}\hspace*{\fill}
\begin{subfigure}{0.2\textwidth}
\includegraphics[width=\linewidth]{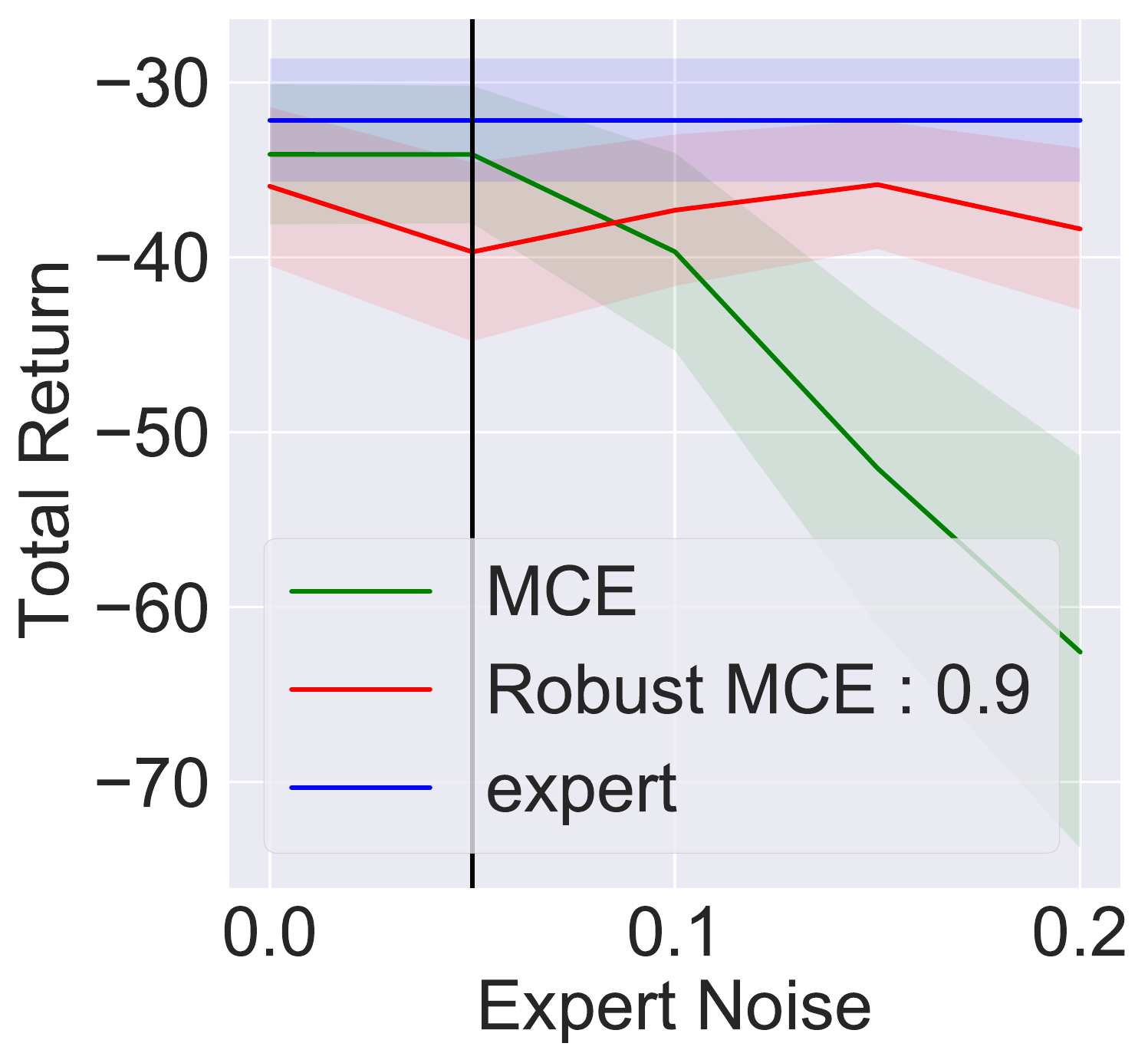}
\caption{$M^{L,\epsilon_L}$ with $\epsilon_L = 0.05$} \label{fig:e10l0.05pres}
\end{subfigure}\hspace*{\fill}
\begin{subfigure}{0.2\textwidth}
\includegraphics[width=\linewidth]{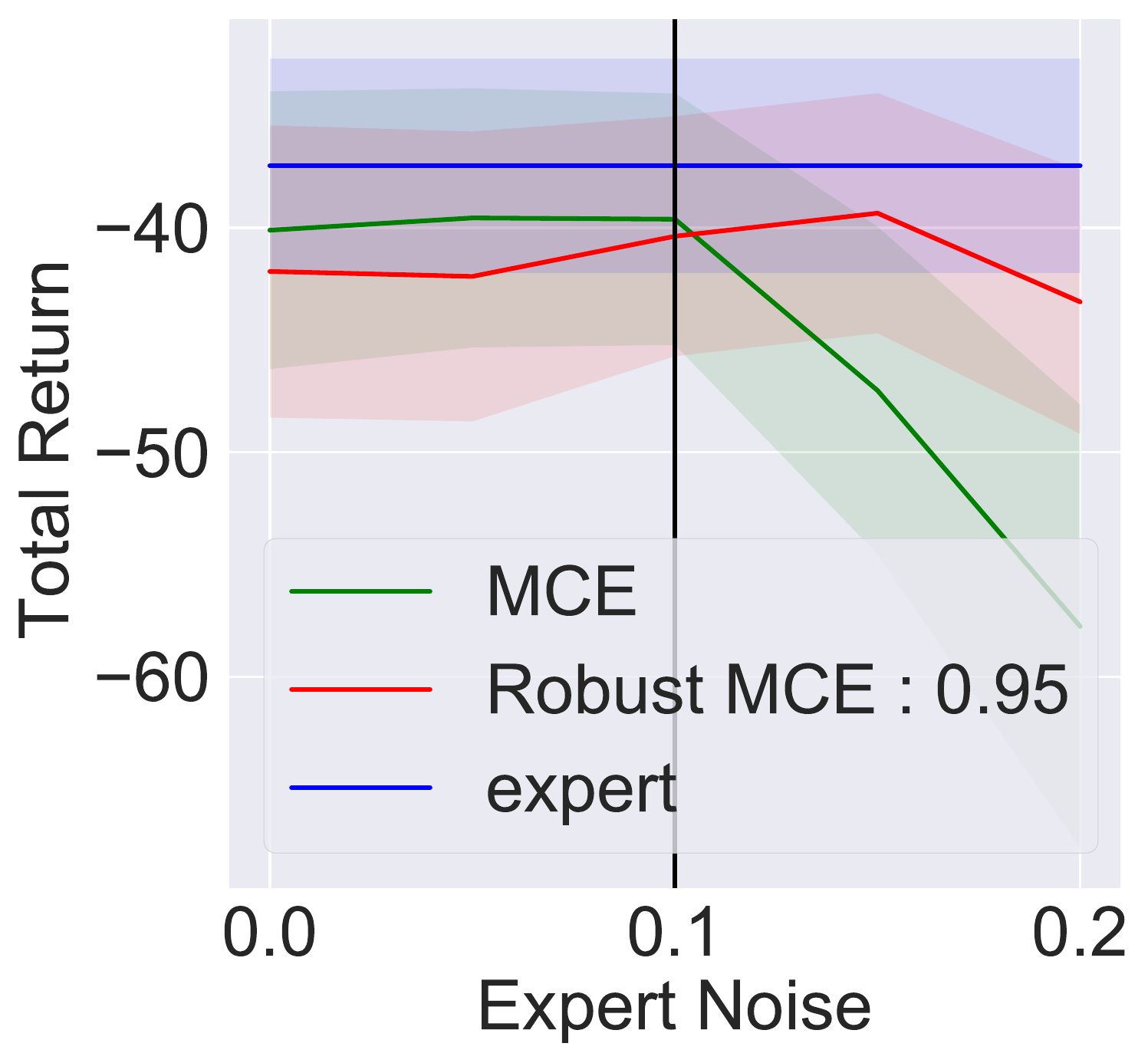}
\caption{$M^{L,\epsilon_L}$ with $\epsilon_L = 0.1$} \label{fig:e10l0.1pres}
\end{subfigure}
\medskip
\begin{subfigure}{0.2\textwidth}
\includegraphics[width=\linewidth]{plots/Obj10.pdf}
\caption{\textsc{ObjectWorld}} 
\end{subfigure}\hspace*{\fill}
\begin{subfigure}{0.2\textwidth}
\includegraphics[width=\linewidth]{plots/fillBetweenNotebookCompareAlphasEnvOWnoiseL0.0dim10alphaE1.0fix_startFalselegendTruepresentationFalse.pdf}
\caption{$M^{L,\epsilon_L}$ with $\epsilon_L = 0$} 
\end{subfigure}\hspace*{\fill}
\begin{subfigure}{0.2\textwidth}
\includegraphics[width=\linewidth]{plots/fillBetweenNotebookCompareAlphasEnvOWnoiseL0.05dim10alphaE1.0fix_startFalselegendTruepresentationFalse.pdf}
\caption{$M^{L,\epsilon_L}$ with $\epsilon_L = 0.05$} 
\end{subfigure}\hspace*{\fill}
\begin{subfigure}{0.2\textwidth}
\includegraphics[width=\linewidth]{plots/fillBetweenNotebookCompareAlphasEnvOWnoiseL0.1dim10alphaE1.0fix_startFalselegendTruepresentationFalse.pdf}
\caption{$M^{L,\epsilon_L}$ with $\epsilon_L = 0.1$} 
\end{subfigure}
\caption{Comparison of the performance our Algorithm~\ref{alg:MaxEntIRL} against the baselines, under different levels of mismatch: $\br{\epsilon_E, \epsilon_L} \in \bc{0.0, 0.05, 0.1, 0.15, 0.2} \times \bc{ 0.0, 0.05, 0.1}$. Each plot corresponds to a fixed leaner environment $M^{L,\epsilon_L}$ with $\epsilon_L \in \bc{ 0.0, 0.05, 0.1}$. The values of $\alpha$ used for our Algorithm~\ref{alg:MaxEntIRL} are reported in the legend. The vertical line indicates the position of the learner environment in the x-axis.}
\label{fig:all_gridworld_best_alpha}
\end{figure}

\begin{figure}[h!] 
\centering
\begin{subfigure}{0.2\textwidth}
\includegraphics[width=\linewidth]{plots/1.pdf}
\caption{\textsc{GridWorld-1}} \label{fig:maintr1}
\end{subfigure}\hspace*{\fill}
\begin{subfigure}{0.2\textwidth}
\includegraphics[width=\linewidth]{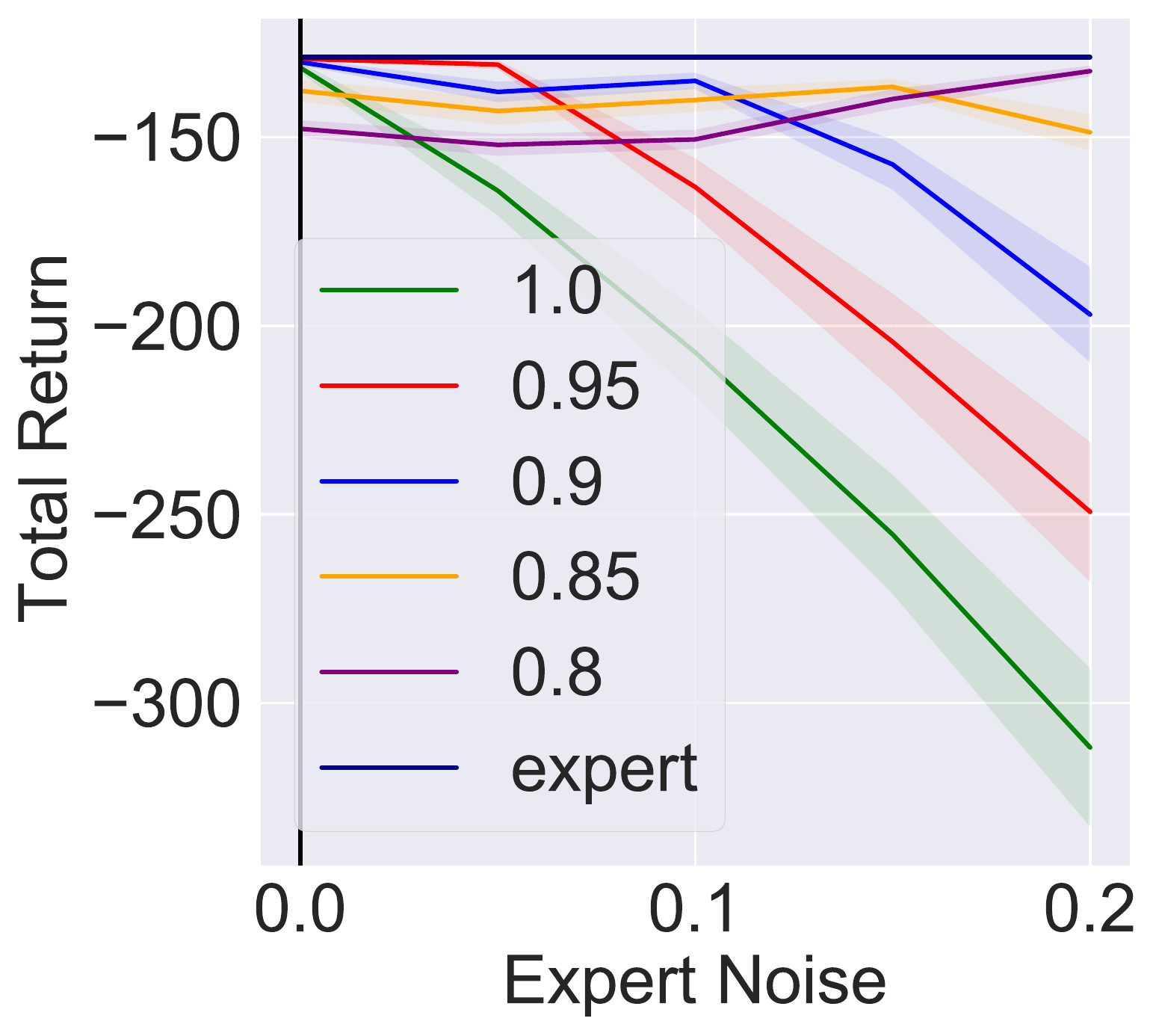}
\caption{$M^{L,\epsilon_L}$ with $\epsilon_L = 0$} \label{fig:maine1l0.0}
\end{subfigure}\hspace*{\fill}
\begin{subfigure}{0.2\textwidth}
\includegraphics[width=\linewidth]{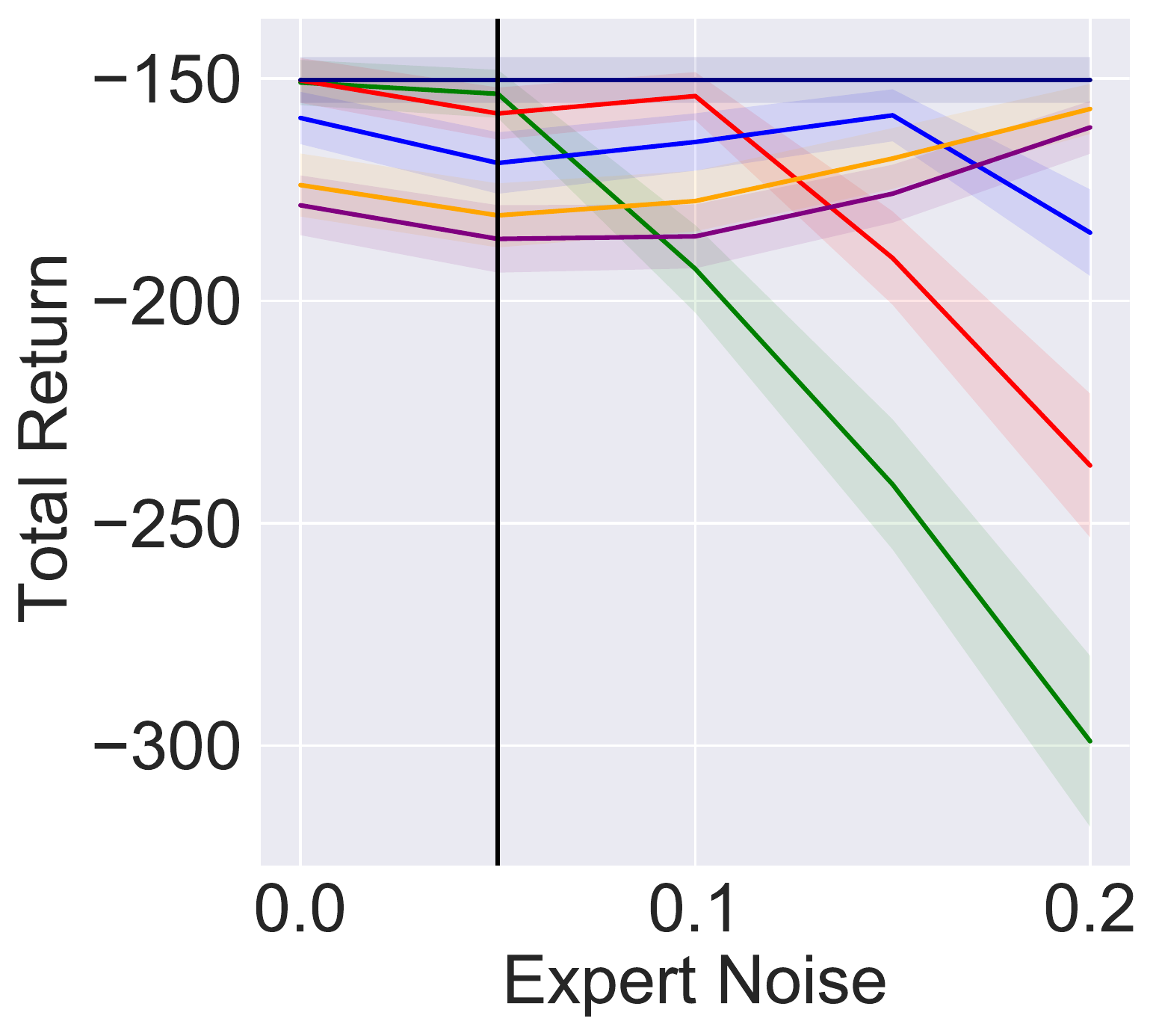}
\caption{$M^{L,\epsilon_L}$ with $\epsilon_L = 0.05$} \label{fig:maine1l0.05}
\end{subfigure}\hspace*{\fill}
\begin{subfigure}{0.2\textwidth}
\includegraphics[width=\linewidth]{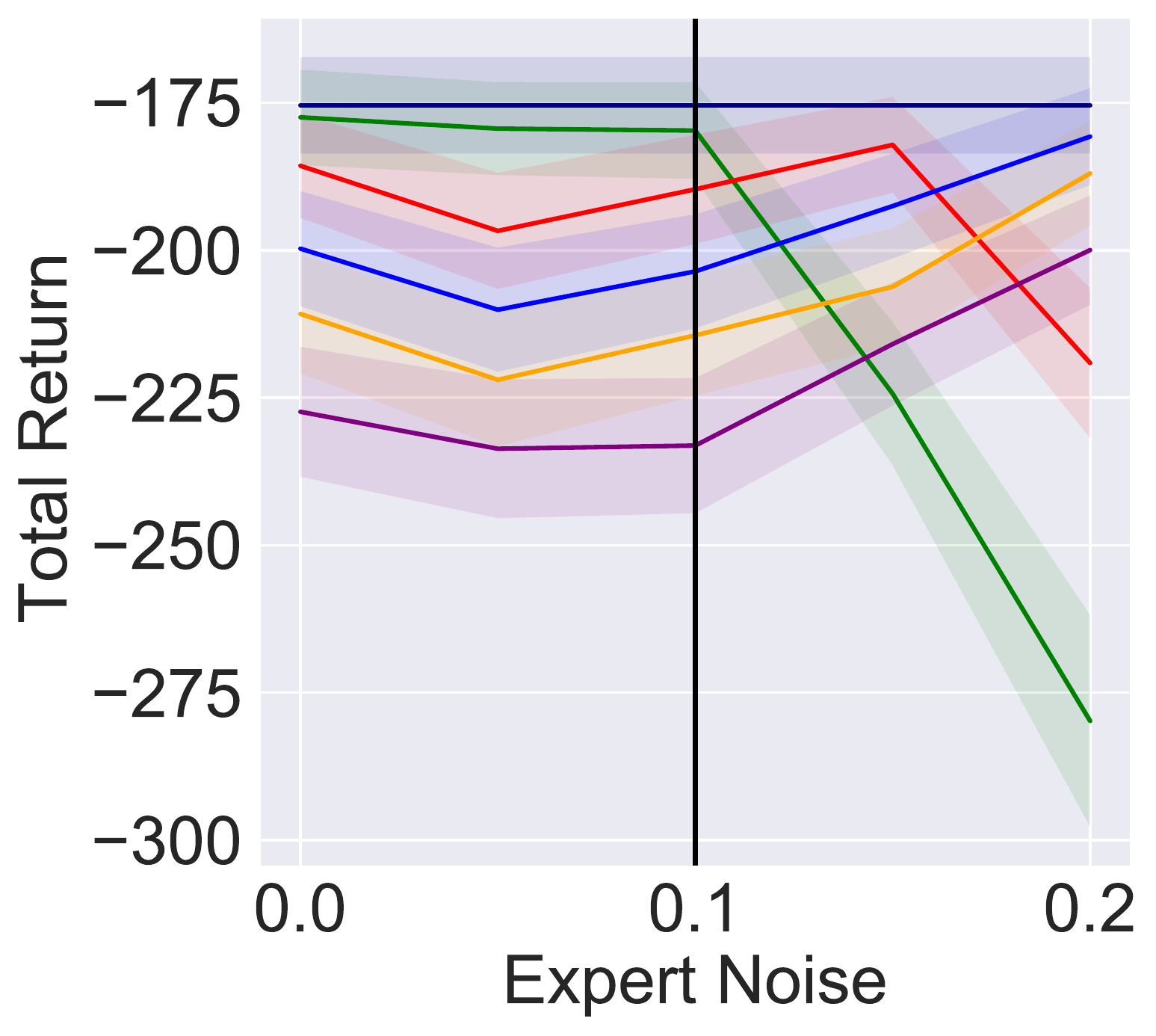}
\caption{$M^{L,\epsilon_L}$ with $\epsilon_L = 0.1$} \label{fig:maine1l0.1}
\end{subfigure}
\medskip
\begin{subfigure}{0.2\textwidth}
\includegraphics[width=\linewidth]{plots/2.pdf}
\caption{\textsc{GridWorld-2}} \label{fig:tr2}
\end{subfigure}\hspace*{\fill}
\begin{subfigure}{0.2\textwidth}
\includegraphics[width=\linewidth]{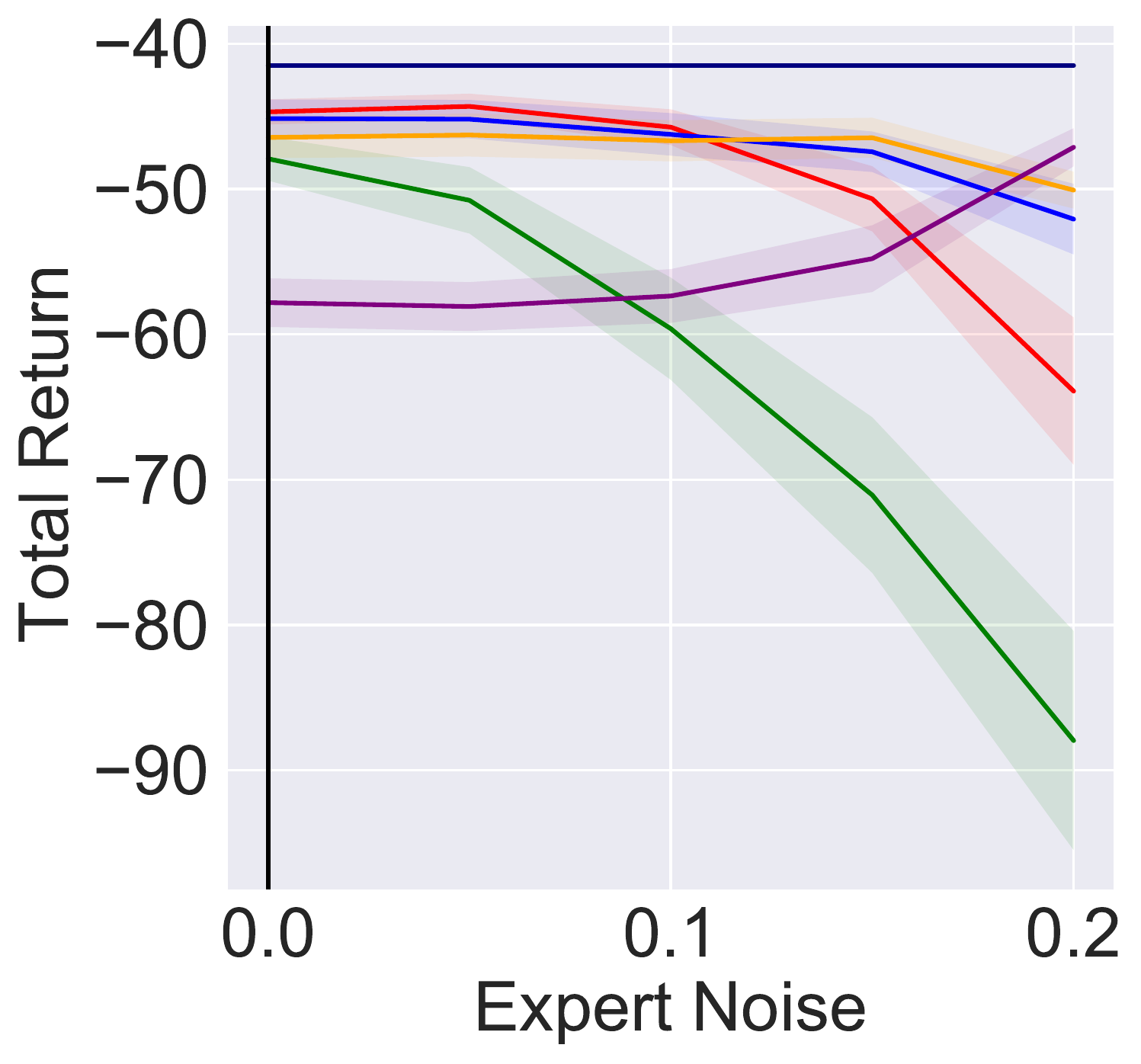}
\caption{$M^{L,\epsilon_L}$ with $\epsilon_L = 0$} \label{fig:e2l0.0abl}
\end{subfigure}\hspace*{\fill}
\begin{subfigure}{0.2\textwidth}
\includegraphics[width=\linewidth]{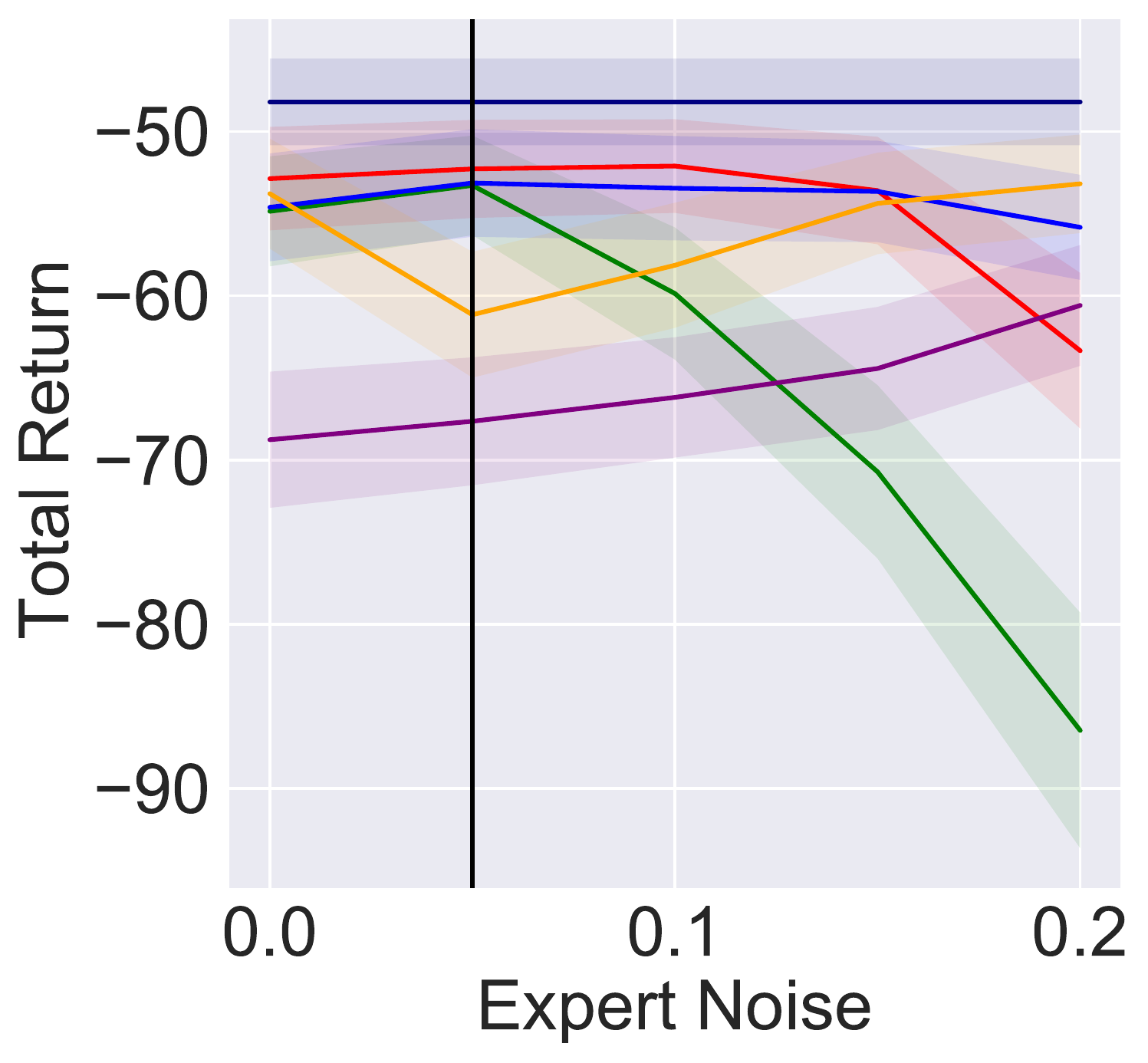}
\caption{$M^{L,\epsilon_L}$ with $\epsilon_L = 0.05$} \label{fig:e2l0.05abl}
\end{subfigure}\hspace*{\fill}
\begin{subfigure}{0.2\textwidth}
\includegraphics[width=\linewidth]{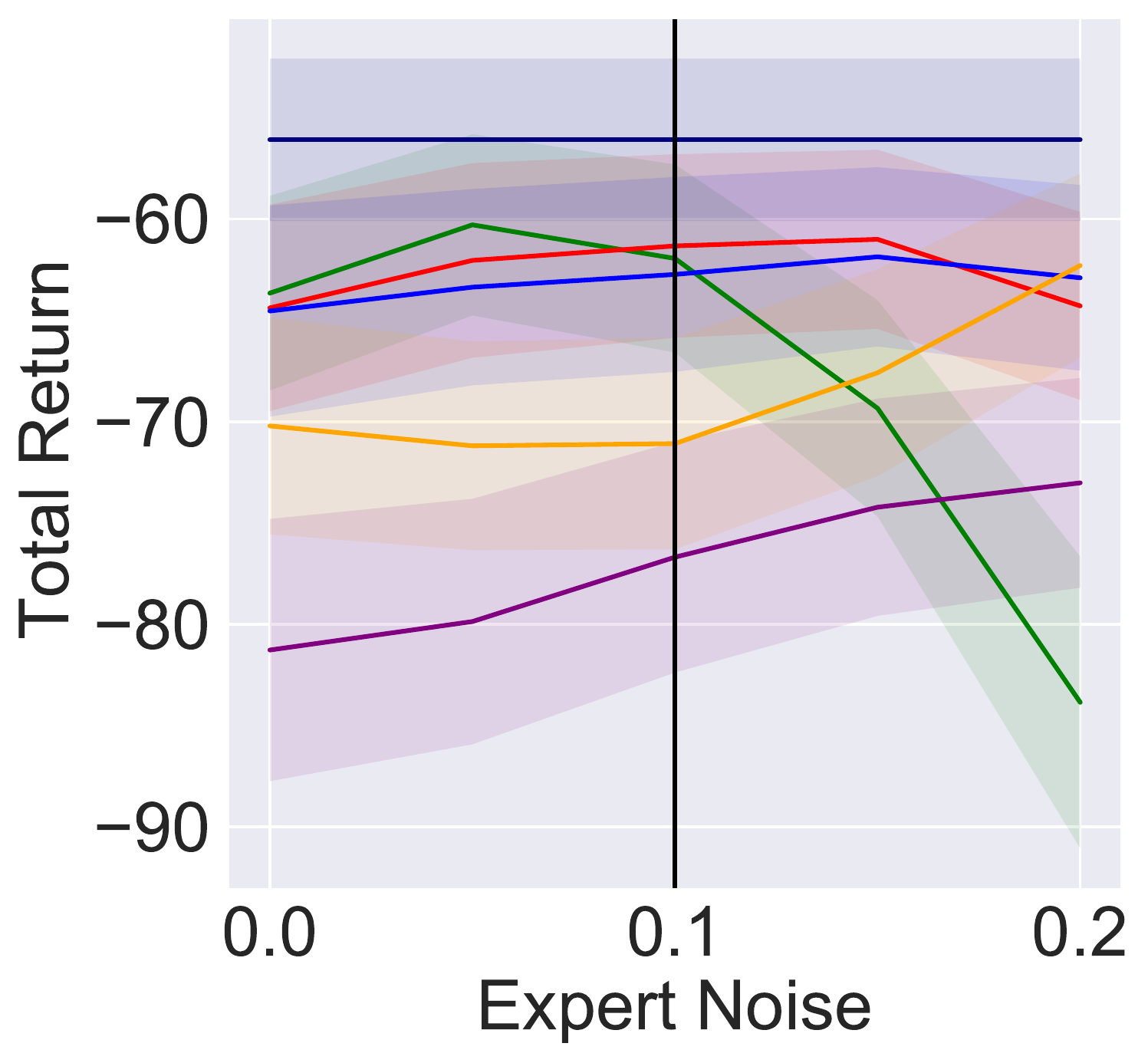}
\caption{$M^{L,\epsilon_L}$ with $\epsilon_L = 0.1$} \label{fig:e2l0.1abl}
\end{subfigure}
\medskip
\begin{subfigure}{0.2\textwidth}
\includegraphics[width=\linewidth]{plots/9.pdf}
\caption{\textsc{GridWorld-3}} \label{fig:tr9}
\end{subfigure}\hspace*{\fill}
\begin{subfigure}{0.2\textwidth}
\includegraphics[width=\linewidth]{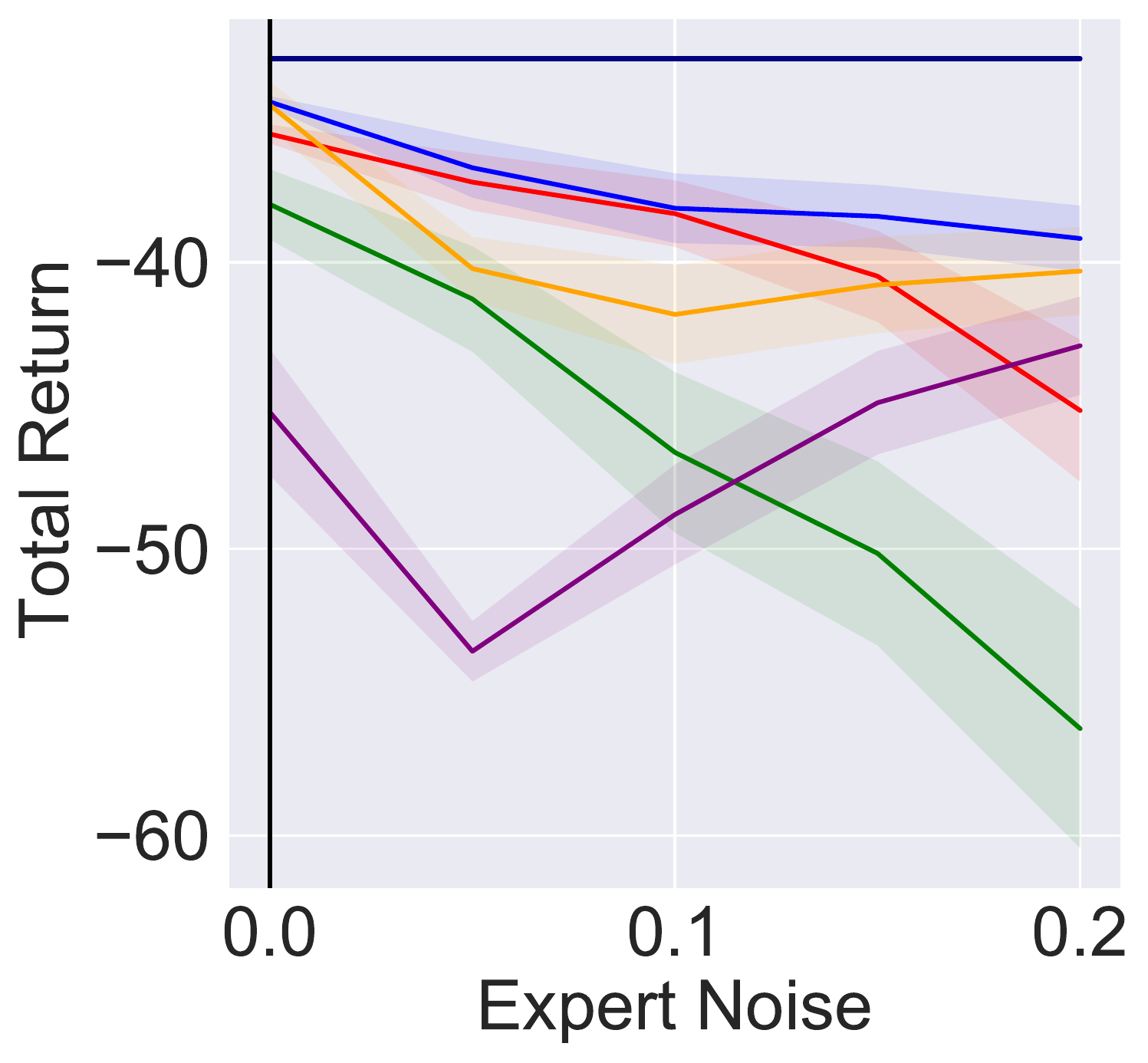}
\caption{$M^{L,\epsilon_L}$ with $\epsilon_L = 0$} \label{fig:e9l0.0abl}
\end{subfigure}\hspace*{\fill}
\begin{subfigure}{0.2\textwidth}
\includegraphics[width=\linewidth]{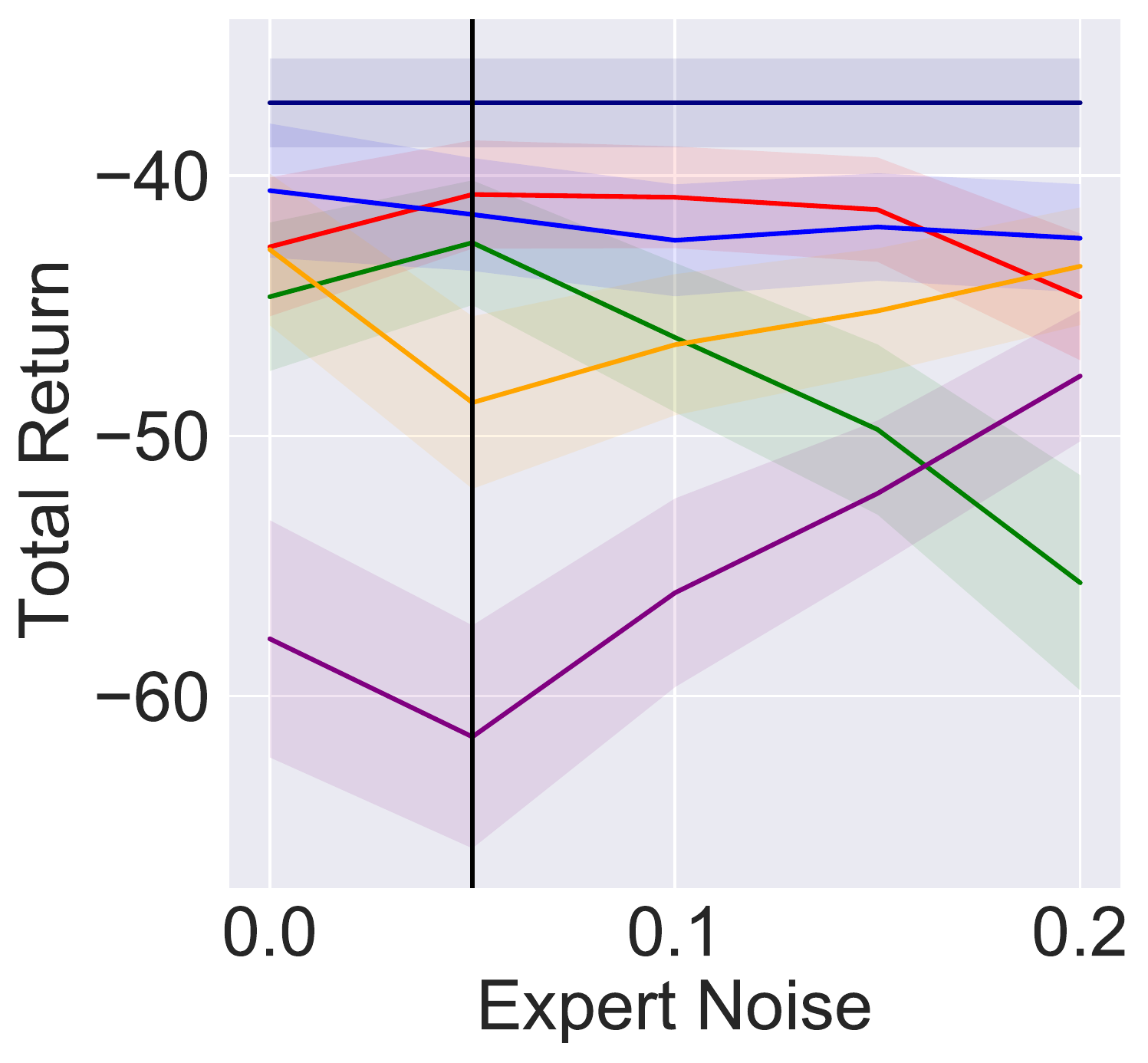}
\caption{$M^{L,\epsilon_L}$ with $\epsilon_L = 0.05$} \label{fig:e9l0.05abl}
\end{subfigure}\hspace*{\fill}
\begin{subfigure}{0.2\textwidth}
\includegraphics[width=\linewidth]{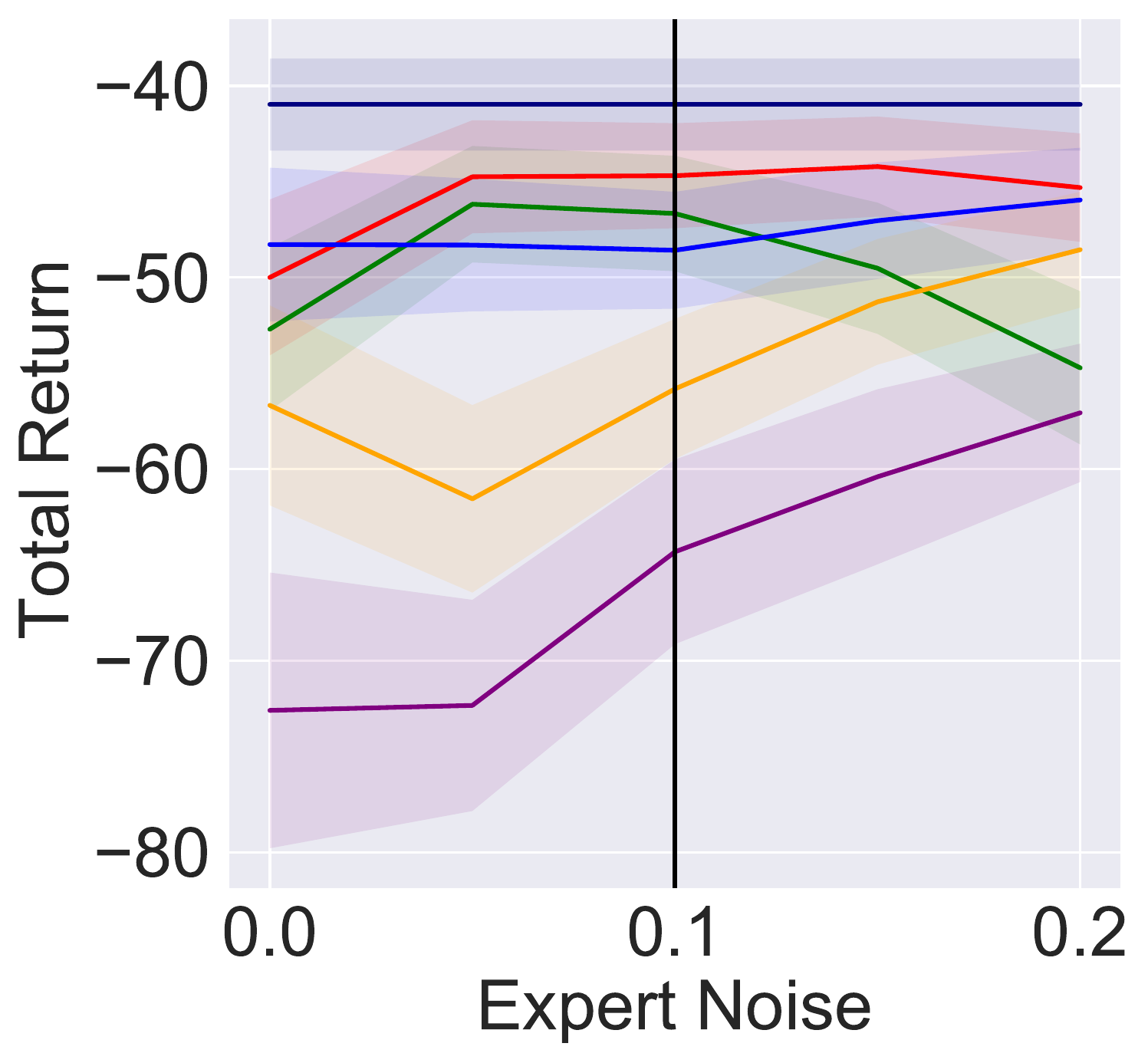}
\caption{$M^{L,\epsilon_L}$ with $\epsilon_L = 0.1$} \label{fig:e9l0.1abl}
\end{subfigure}
\medskip
\begin{subfigure}{0.2\textwidth}
\includegraphics[width=\linewidth]{plots/10.pdf}
\caption{\textsc{GridWorld-4}} \label{fig:tr10}
\end{subfigure}\hspace*{\fill}
\begin{subfigure}{0.2\textwidth}
\includegraphics[width=\linewidth]{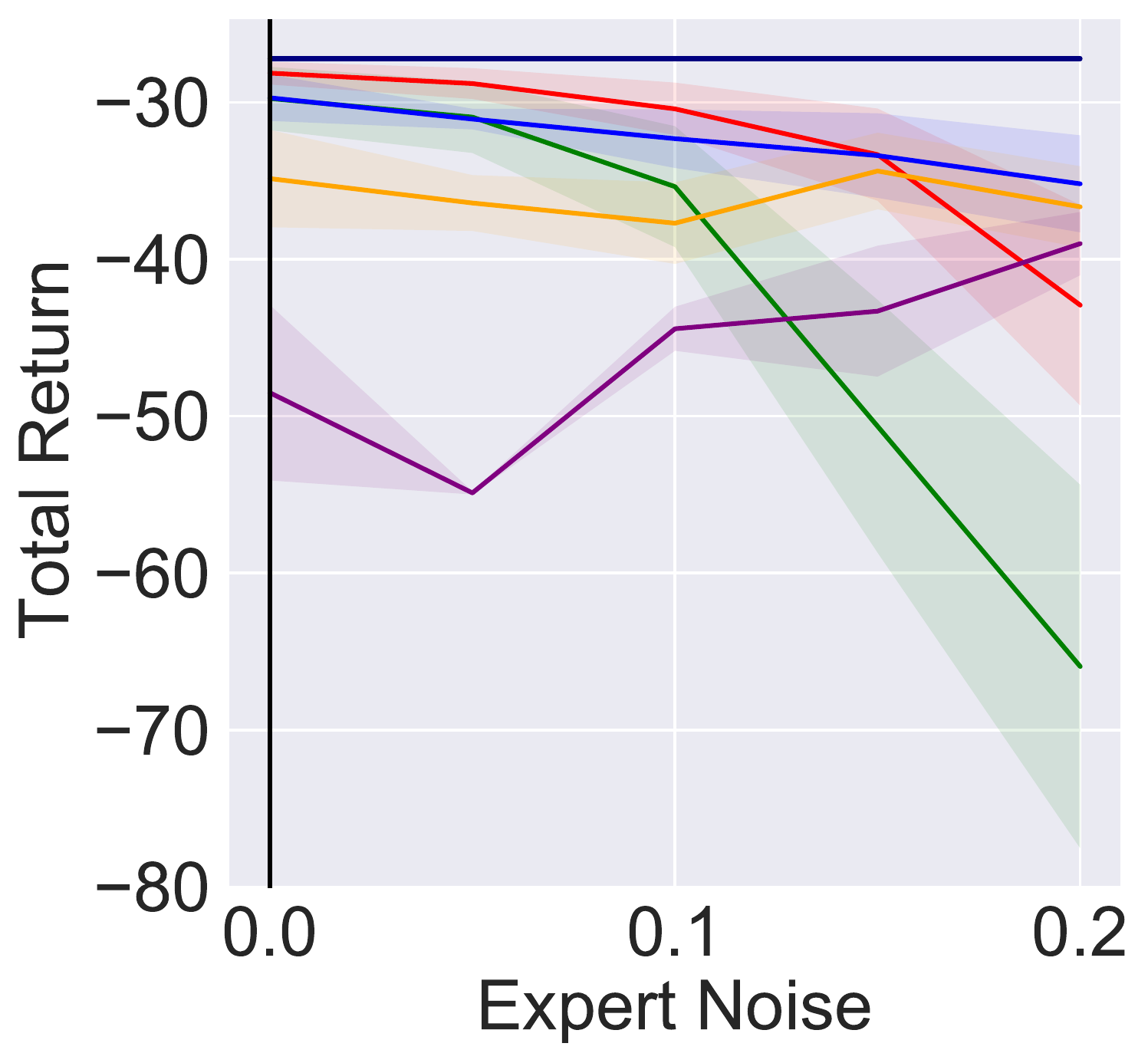}
\caption{$M^{L,\epsilon_L}$ with $\epsilon_L = 0$} \label{fig:e10l0.0abl}
\end{subfigure}\hspace*{\fill}
\begin{subfigure}{0.2\textwidth}
\includegraphics[width=\linewidth]{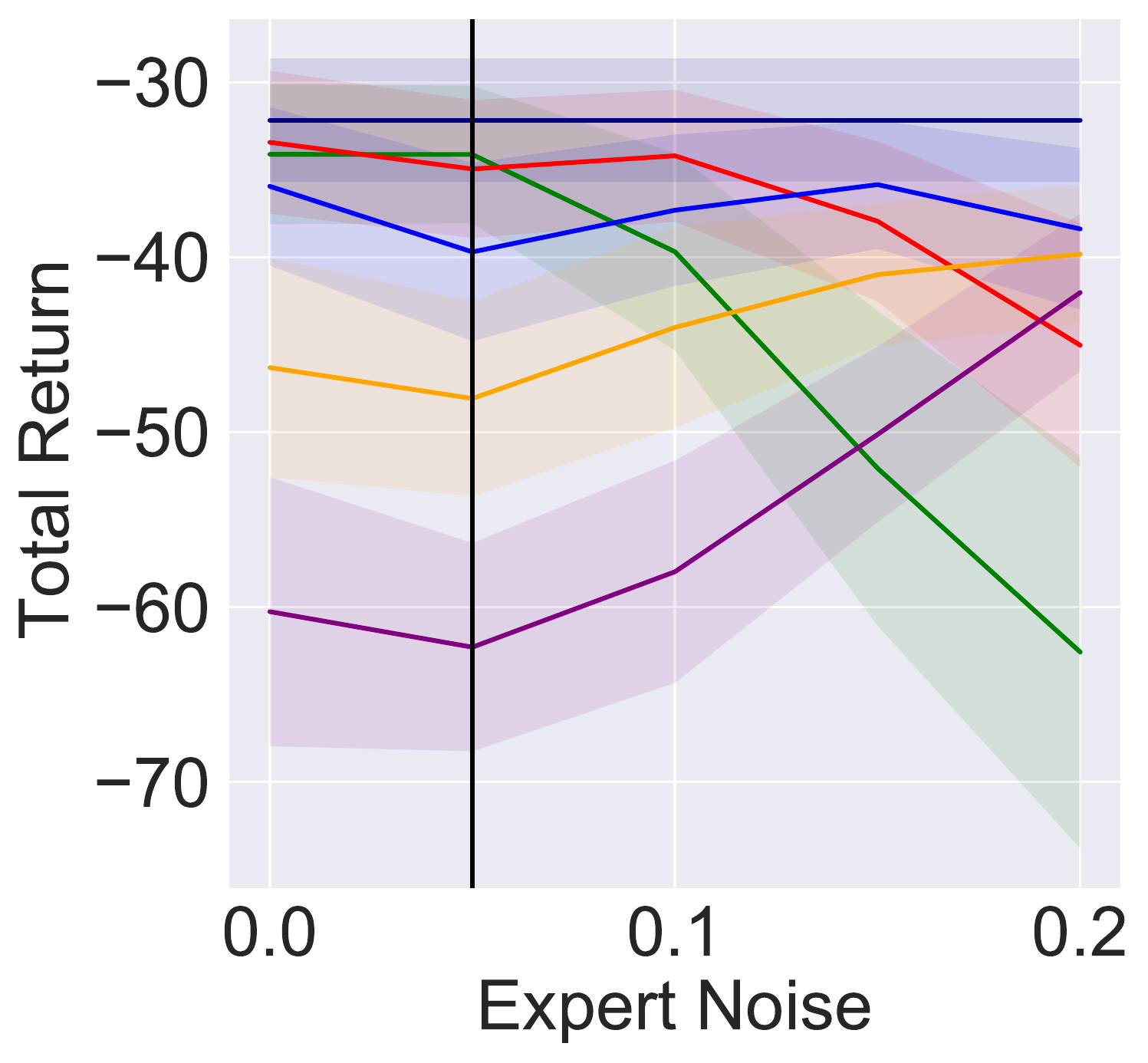}
\caption{$M^{L,\epsilon_L}$ with $\epsilon_L = 0.05$} \label{fig:e10l0.05abl}
\end{subfigure}\hspace*{\fill}
\begin{subfigure}{0.2\textwidth}
\includegraphics[width=\linewidth]{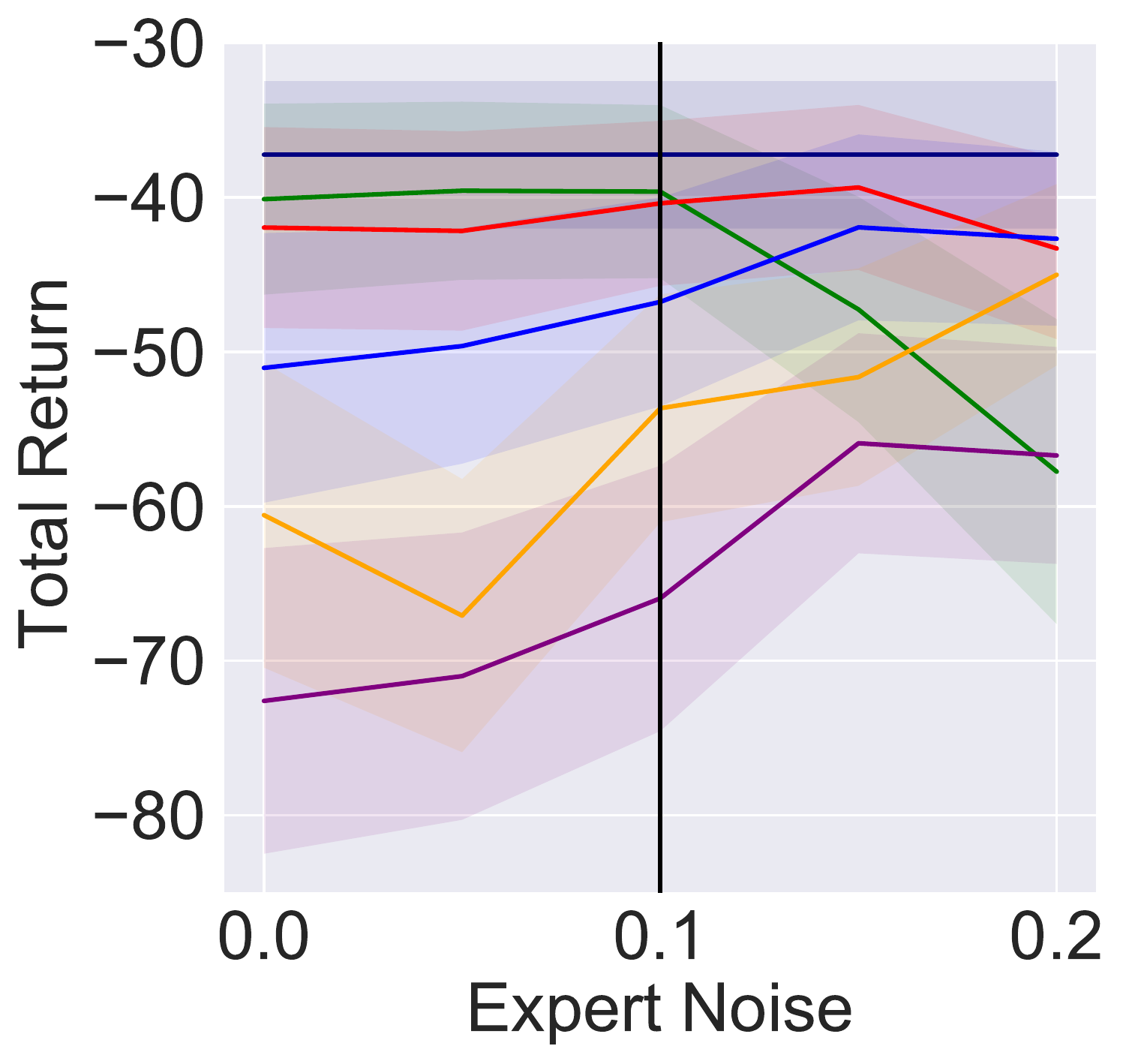}
\caption{$M^{L,\epsilon_L}$ with $\epsilon_L = 0.1$} \label{fig:e10l0.1abl}
\end{subfigure}
\medskip
\begin{subfigure}{0.2\textwidth}
\includegraphics[width=\linewidth]{plots/Obj10.pdf}
\caption{\textsc{ObjectWorld}} \label{fig:mainball_ow_reward}
\end{subfigure}\hspace*{\fill}
\begin{subfigure}{0.2\textwidth}
\includegraphics[width=\linewidth]{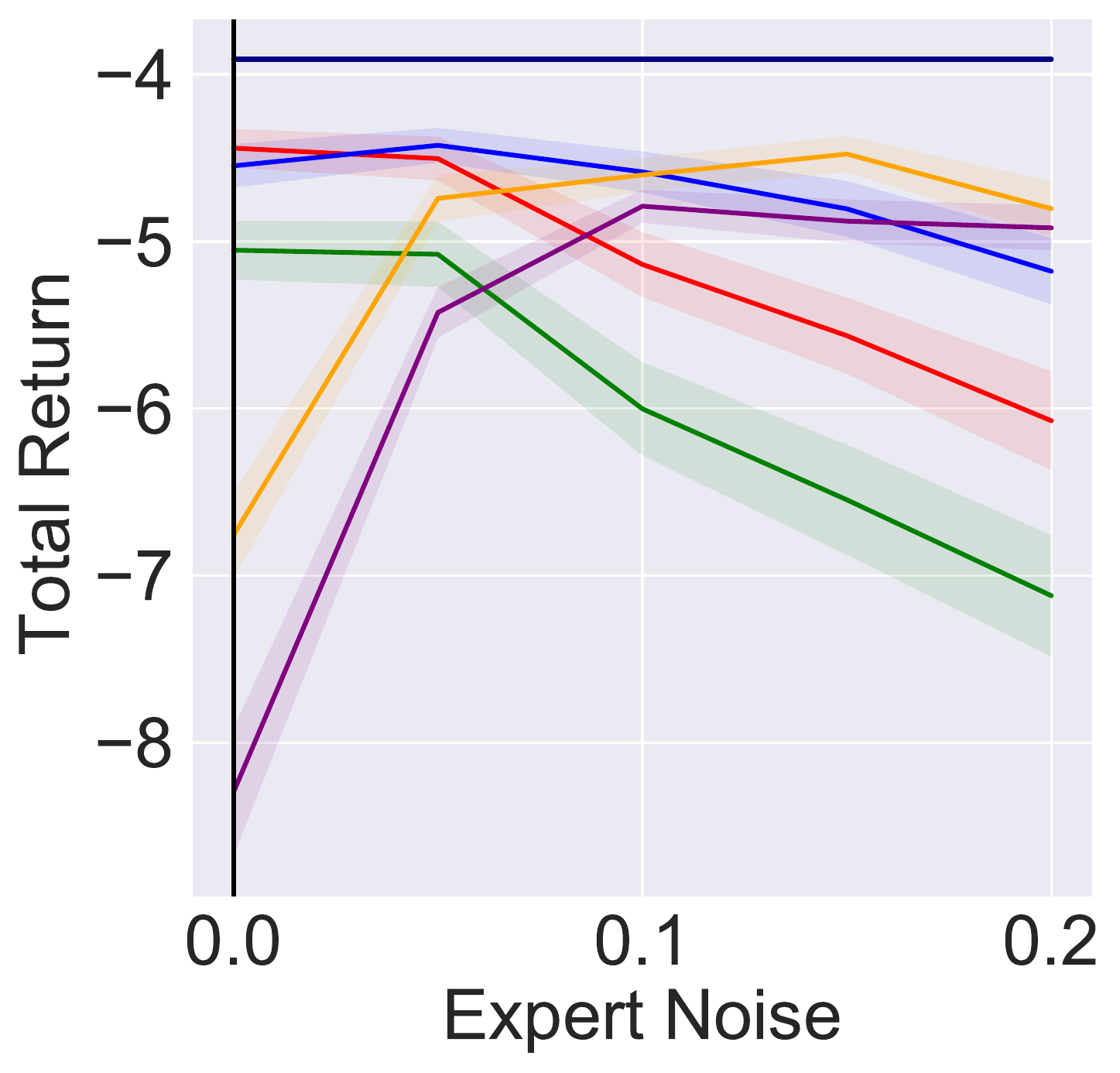}
\caption{$M^{L,\epsilon_L}$ with $\epsilon_L = 0$} \label{fig:mainowl0.0ball}
\end{subfigure}\hspace*{\fill}
\begin{subfigure}{0.2\textwidth}
\includegraphics[width=\linewidth]{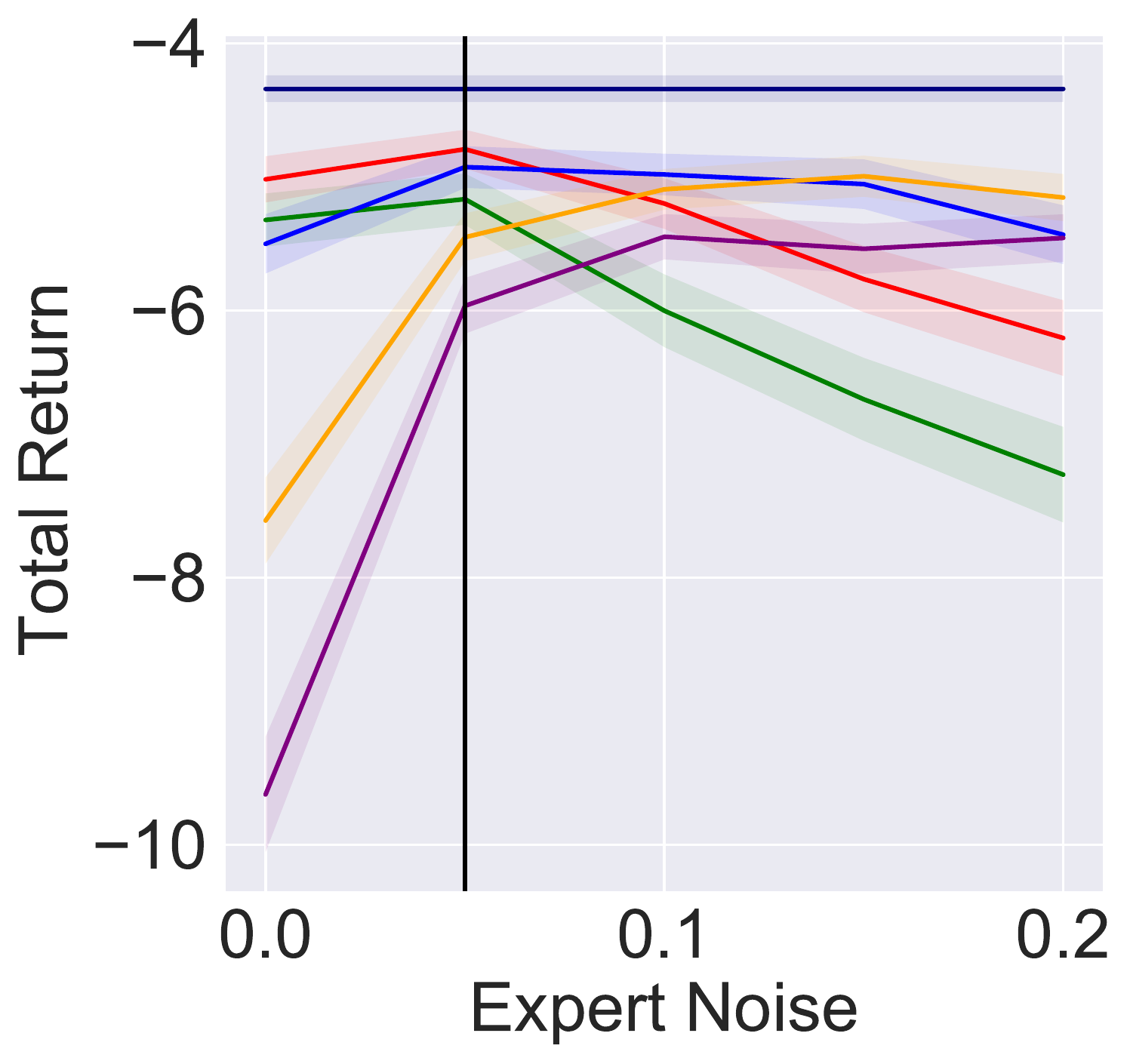}
\caption{$M^{L,\epsilon_L}$ with $\epsilon_L = 0.05$} \label{fig:mainowl0.05ball}
\end{subfigure}\hspace*{\fill}
\begin{subfigure}{0.2\textwidth}
\includegraphics[width=\linewidth]{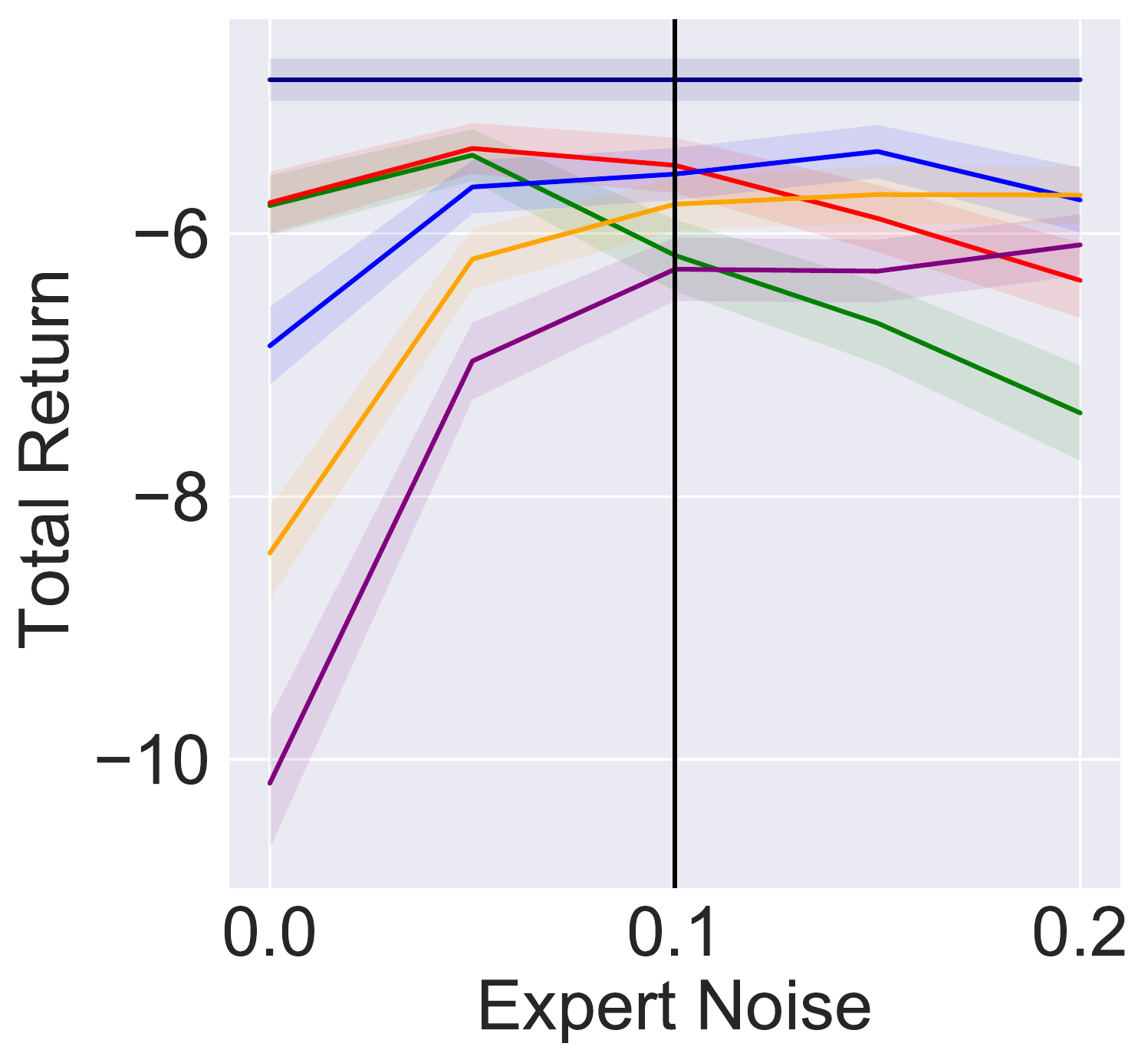}
\caption{$M^{L,\epsilon_L}$ with $\epsilon_L = 0.1$} \label{fig:mainowl0.1ball}
\end{subfigure}
\caption{Comparison of the performance our Algorithm~\ref{alg:MaxEntIRL} with different values of $\alpha$, under different levels of mismatch: $\br{\epsilon_E, \epsilon_L} \in \bc{0.0, 0.05, 0.1, 0.15, 0.2} \times \bc{ 0.0, 0.05, 0.1}$. Each plot corresponds to a fixed leaner environment $M^{L,\epsilon_L}$ with $\epsilon_L \in \bc{ 0.0, 0.05, 0.1}$. The values of $\alpha$ used for our Algorithm~\ref{alg:MaxEntIRL} are reported in the legend. The vertical line indicates the position of the learner environment in the x-axis.}
\label{fig:gridworld_diff_alpha}
\end{figure}

\clearpage
\section{Further Details of Section~\ref{sec:experiments_continuous_control_main}}
\label{sec:experiments_continuous_control}

\begin{algorithm}[h]
    \caption{Robust RE IRL via Markov Game}
    \label{alg:RobustREIRL}
    \begin{spacing}{0.8}
    \begin{algorithmic}
    \STATE \textbf{Input:} opponent strength $1-\alpha$, the expert's empirical feature occupancy measure $\bar{\phi}^E$
	\STATE \textbf{Initialize:} player policy parameters $\boldsymbol{w}^{\mathrm{pl}}$, opponent policy parameters $\boldsymbol{w}^{\mathrm{op}}$, reward parameters $\boldsymbol{\theta}$
    \STATE \textbf{Initialize:} uniform sampling policy $\pi$
        \WHILE{not converged}
            \STATE collect trajectories dataset $\mathcal{D}^{\pi}$ with the sampling policy $\pi$.
            \STATE estimate the features occupancy measure for each trajectory $\tau \in \mathcal{D}^{\pi}$ as $\bar{\phi}^{\tau} = \frac{1}{|\tau|}\sum_{s \in \tau} \phi(s)$.
            \FOR{$t = 1, \dots, N^\theta$}
            \STATE update the distribution over trajectories as:
            \begin{equation*}
                 P(\tau | \boldsymbol{\theta}) \propto \exp \br{{\ip{\boldsymbol{\theta}}{\bar{\phi}^{\tau}}}}
            \end{equation*}
            \STATE compute the gradient estimate for updating $\boldsymbol{\theta}$ as proposed in~\cite{boularias2011relative} (to tackle the unknown transition dynamics case):
            \begin{equation*}
     \nabla_{\boldsymbol{\theta}} g(\boldsymbol{\theta}) = \bar{\phi}^E - \sum_{\tau \in \mathcal{D}^{\pi}} P(\tau | \boldsymbol{\theta}) \cdot \bar{\phi}^{\tau}
            \end{equation*}
        \STATE update the reward parameter $\boldsymbol{\theta}$ with Adam~\cite{kingma2014adam} using the gradient estimate $\nabla_{\boldsymbol{\theta}} g(\boldsymbol{\theta})$.
        \ENDFOR
        
        \STATE use Algorithm~\ref{alg:TwoplayersPolicyGrad} with $R = R_{\boldsymbol{\theta}}$ to update $\pi^{\mathrm{pl}}$ and $\pi^{\mathrm{op}}$ s.t. they solve the following Markov Game approximately with policy gradient:
        \begin{equation*}
\max_{\pi^{\mathrm{pl}} \in \Pi} \min_{\pi^{\mathrm{op}} \in \Pi} \E{G \bigm| \pi^{\mathrm{pl}}, \pi^{\mathrm{op}}, M^{\mathrm{two},L,\alpha}}
        \end{equation*}
        \STATE update the sampling policy:
    \begin{equation*}
        \pi = \alpha \pi^{\mathrm{pl}} + (1-\alpha)\pi^{\mathrm{op}}
    \end{equation*}
        \ENDWHILE
        \STATE \textbf{Output:} player policy $\pi^{\mathrm{pl}}$
    \end{algorithmic}
    \end{spacing}
\end{algorithm}

\begin{algorithm}[h]
    \caption{Policy Gradient Method for Two-Player Markov Game}
    \label{alg:TwoplayersPolicyGrad}
    \begin{spacing}{0.8}
    \begin{algorithmic}
    \STATE \textbf{Input:} reward parameters $\boldsymbol{\theta}$
    \STATE \textbf{Initialize:} player policy parameters $\boldsymbol{w}^{\mathrm{pl}}$, opponent policy parameters $\boldsymbol{w}^{\mathrm{op}}$
            \FOR{$s = 1, \dots, N^\pi$}
            \STATE $\mathcal{D} = \bc{}$
            \FOR{$i = 1, \dots, N^{\mathrm{traj}}$}
        \STATE collect trajectory a with $a^{\mathrm{pl}}_t \sim \pi^{\mathrm{pl}}(\cdot|s_t)$, $a^{\mathrm{op}}_t \sim \pi^{\mathrm{op}}(\cdot|s_t)$, $s_{t+1} \sim T^{\mathrm{two}, L, \alpha}(\cdot |s_t, a^{\mathrm{pl}}_t, a^{\mathrm{op}}_t)$.
        \STATE store the trajectory $\tau^i := \bc{(s_t, a^{\mathrm{pl}}_t, a^{\mathrm{op}}_t)}_t$ in $\mathcal{D}$.
        \STATE compute the return-to-go at each step of the trajectory $\tau^i$ as $G_t^i = \sum^T_{k = t+1} \gamma^{k-t-1} R(s_k)$.
        \ENDFOR
        \STATE update the policy parameters (player and opponent) with the following gradient estimates:
        \begin{align*}	    
		\widehat{\nabla}_{\boldsymbol{w}^{\mathrm{pl}}} J(\boldsymbol{w}^{\mathrm{pl}}, \boldsymbol{w}^{\mathrm{op}})~=~& \frac{1}{\abs{\mathcal{D}}}  \sum_{\tau_i \in \mathcal{D}} \sum_{t}\gamma^t \nabla_{\boldsymbol{w}^{\mathrm{pl}}} \log \pi^\mathrm{pl}(a^{\mathrm{pl}}_t|s_t) G^i_t \\
		\widehat{\nabla}_{\boldsymbol{w}^{\mathrm{op}}} J(\boldsymbol{w}^{\mathrm{pl}}, \boldsymbol{w}^{\mathrm{op}})~=~& - \frac{1}{\abs{\mathcal{D}}} \sum_{\tau_i \in \mathcal{D}} \sum_{t}\gamma^t \nabla_{\boldsymbol{w}^{\mathrm{op}}} \log \pi^{\mathrm{op}}(a^{\mathrm{op}}_t|s_t) G^i_t
		\end{align*}
        \ENDFOR
        \STATE \textbf{Output:} player policy $\pi^{\mathrm{pl}} \gets \pi_{\boldsymbol{w}^{\mathrm{pl}}}$, opponent policy $\pi^{\mathrm{op}} \gets \pi_{\boldsymbol{w}^{\mathrm{op}}}$
    \end{algorithmic}
    \end{spacing}
\end{algorithm}

\paragraph{\textsc{GaussianGrid} Environment.}

We consider a 2D environment, where we denote the horizontal coordinate as $x \in [0,1]$ and vertical one as $y \in [0,1]$. The agent starts in the upper left corner, i.e., the coordinate $(0,1)$, and the episode ends when the agent reaches the lower right region defined by the indicator function $\mathbf{1}\{x\in [0.95, 1], y \in [-1, -0.95]\} $. The reward function is given by: $R(s) = R(x,y) = -(x-1)^2 -(y+1)^2 -80 \cdot e^{-8(x^2 + y^2)} + 10 \cdot \mathbf{1}\{x\in [0.95, 1], y \in [-1, -0.95]\}$. Note that the central region of the 2D environment represents a low reward area that should be avoided. The action space for the agent is given by $\mathcal{A} = [-0.5, 0.5]^2$, and the transition dynamics are given by:
\begin{equation*}
s_{t+1} = \begin{cases}
& s_t + \frac{a_t}{10} \quad \text{w.p.} \quad 1 - \epsilon \\
& s_t - \frac{s_t}{10 \norm{s_t}_2} \quad \text{w.p.} \quad \epsilon 
\end{cases}
\end{equation*}
Thus, with probability $\epsilon$, the environment does not respond to the action taken by the agent, but it takes a step towards the low reward area centered at the origin, i.e., $- \frac{s_t}{10 \norm{s_t}_2}$. The agent should therefore pass far enough from the origin. The parameter $\epsilon$ can be varied to create a dynamic mismatch, e.g., higher $\epsilon$ corresponds to a more difficult environment. We investigate the performance of our Robust RE IRL method with different choices of the parameter $\alpha$ under various mismatches given by pairs $(\epsilon_E, \epsilon_L)$. Let $\boldsymbol{\phi}(s) = \boldsymbol{\phi}(x,y) = \bs{x^2, y^2, x, y, e^{-8(x^2 + y^2)}, \mathbf{1}\bc{x\in [0.95, 1], y \in [-1, -0.95]}, 1}^T$. The parameterization for both the player and opponent policies are given by:
\begin{equation*}
    a_t^{\mathrm{pl}} \sim \mathcal{N}\br{(\boldsymbol{w}^{\mathrm{pl}})^T \phi(s_t), \Sigma^{\mathrm{pl}}}
\end{equation*}
\begin{equation*}
    a_t^{\mathrm{op}} \sim \mathcal{N}\br{(\boldsymbol{w}^{\mathrm{op}})^T \phi(s_t), \Sigma^{\mathrm{op}}}
\end{equation*}
The covariance matrices $\Sigma^{\mathrm{pl}}, \Sigma^{\mathrm{op}}$ are constrained to be diagonal, and the diagonal elements are included as part of the policy parameterization.

\begin{figure}[h!] 
\centering
\begin{subfigure}{0.24\textwidth}
\includegraphics[width=0.99\linewidth]{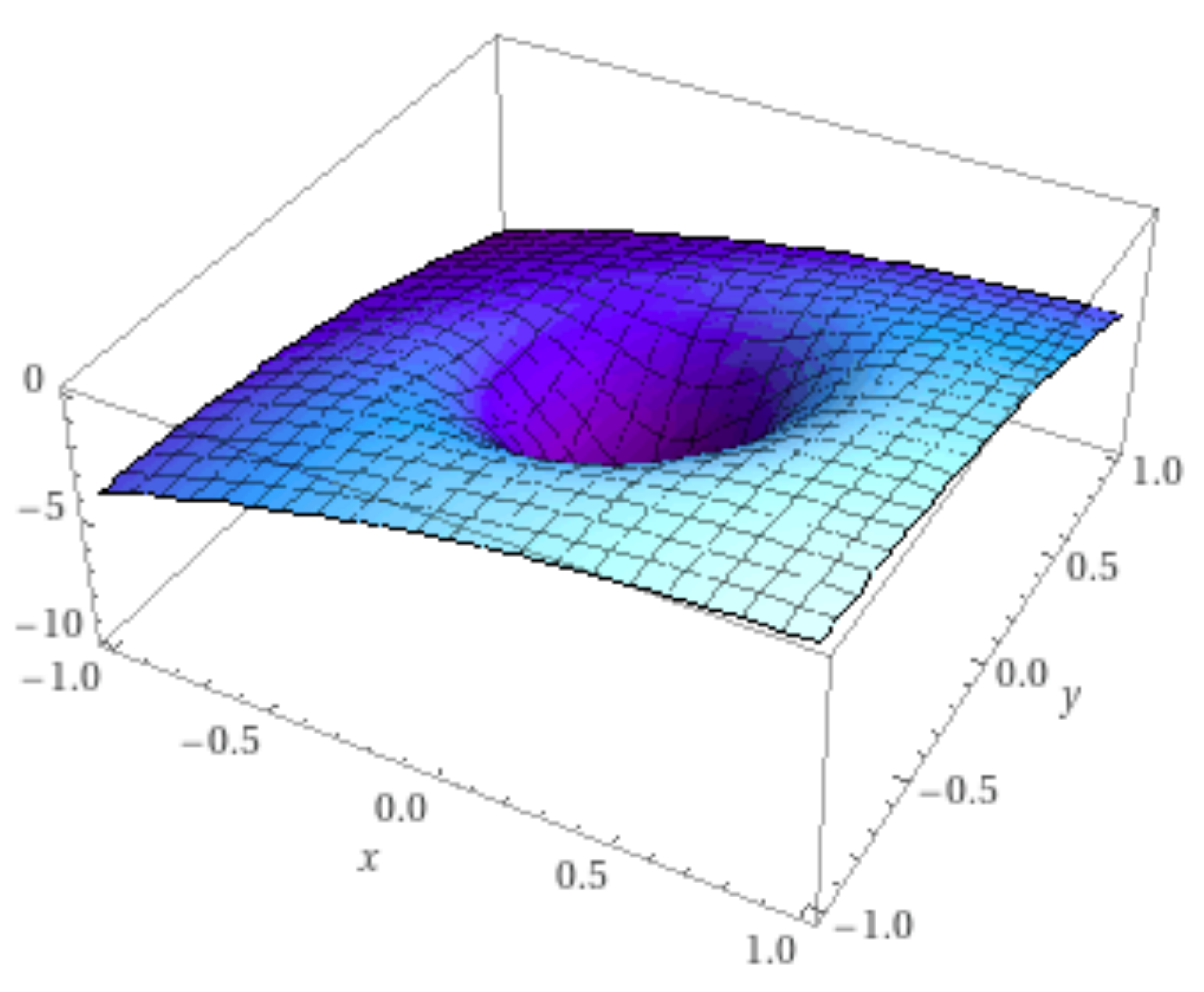}
\caption{\textsc{GaussianGrid}} \label{fig:gauss_3d}
\end{subfigure}
\begin{subfigure}{0.24\textwidth}
\includegraphics[width=\linewidth]{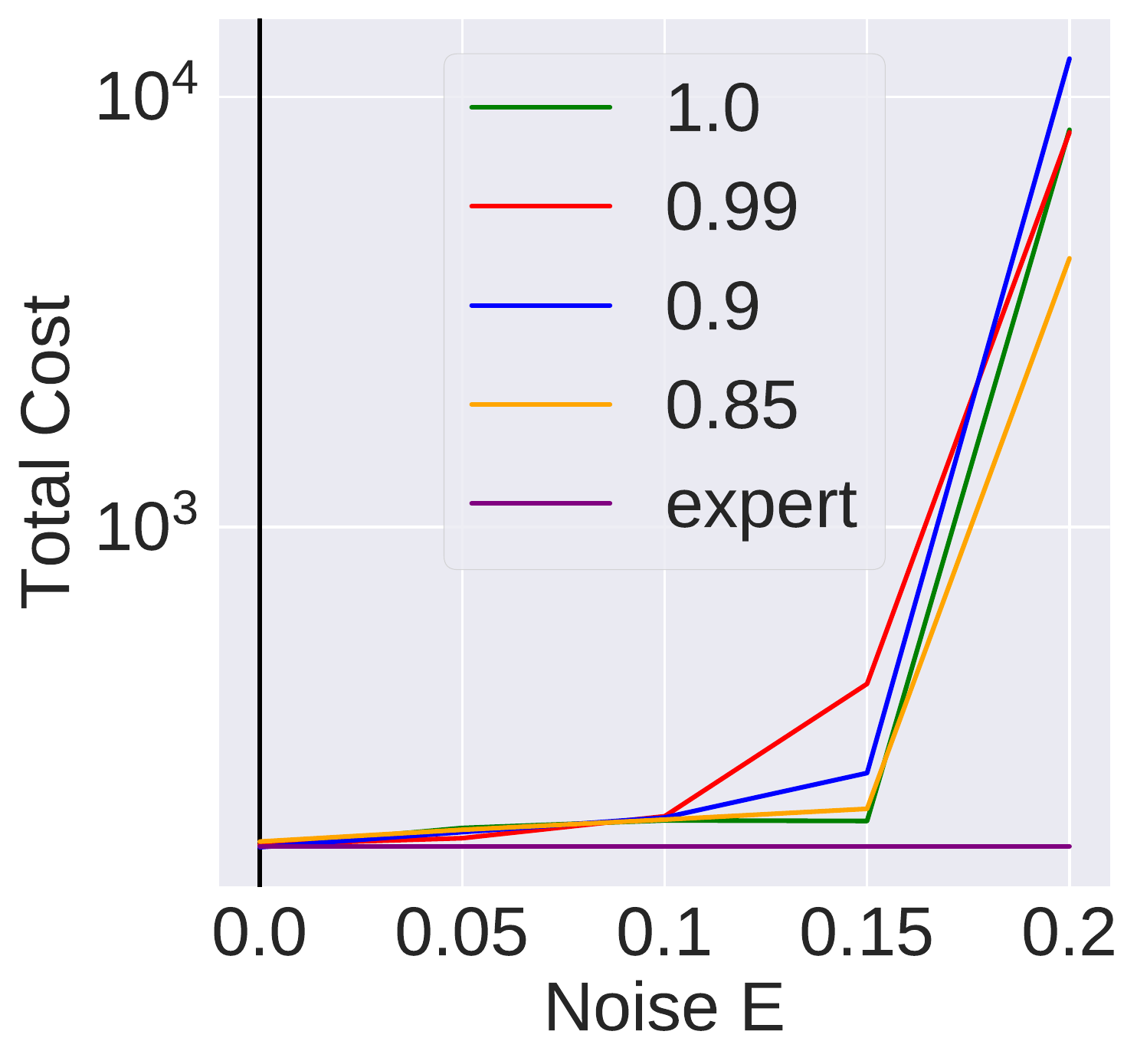}
\caption{$M^{L,\epsilon_L}$ with $\epsilon_L = 0$}
\label{fig:adversarial_gauss_gridl0.0}
\end{subfigure}\hspace*{\fill}
\begin{subfigure}{0.24\textwidth}
\includegraphics[width=\linewidth]{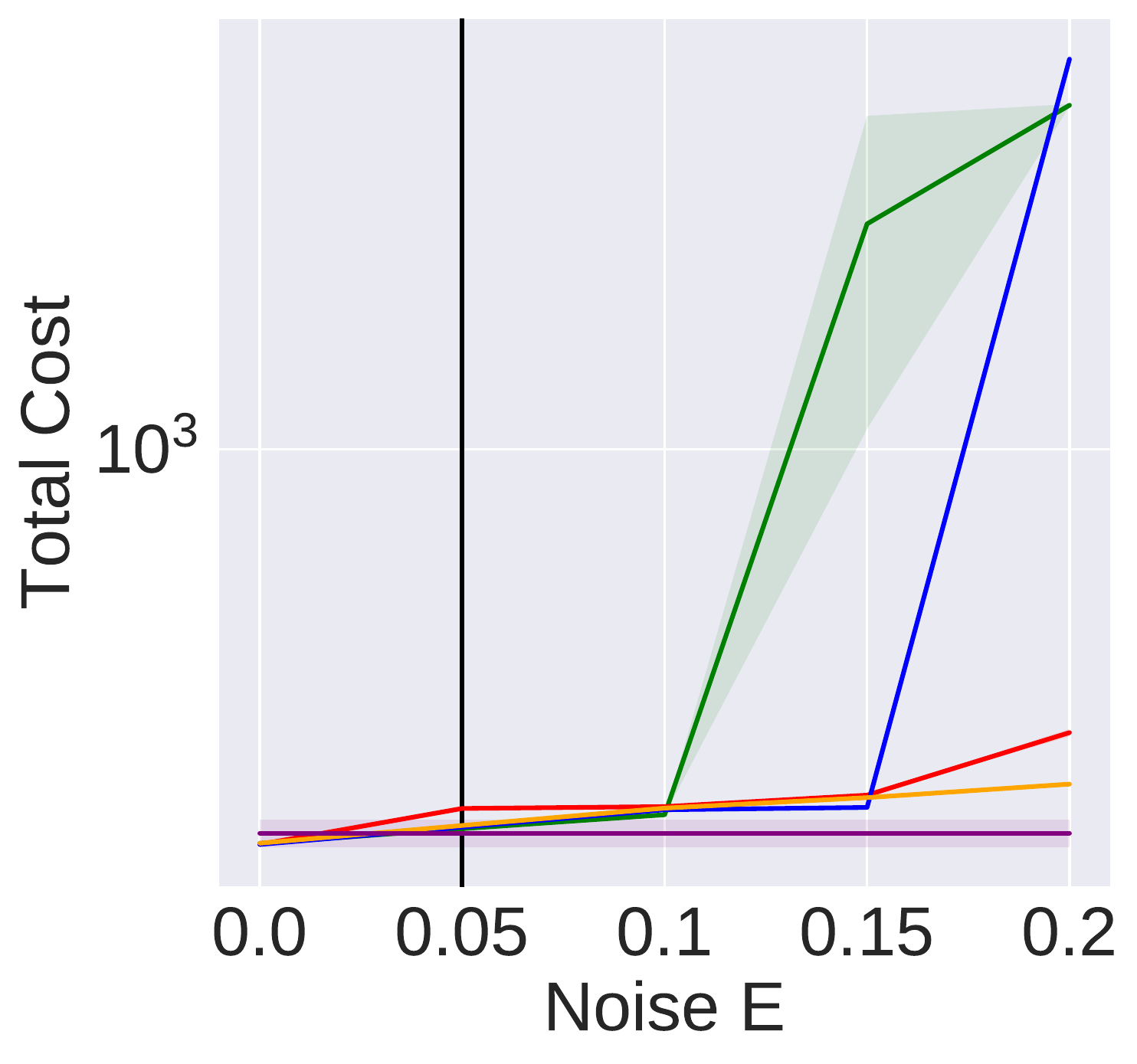}
\caption{$M^{L,\epsilon_L}$ with $\epsilon_L = 0.05$} \label{fig:adversarial_gauss_gridl0.05}
\end{subfigure}\hspace*{\fill}
\begin{subfigure}{0.24\textwidth}
\includegraphics[width=\linewidth]{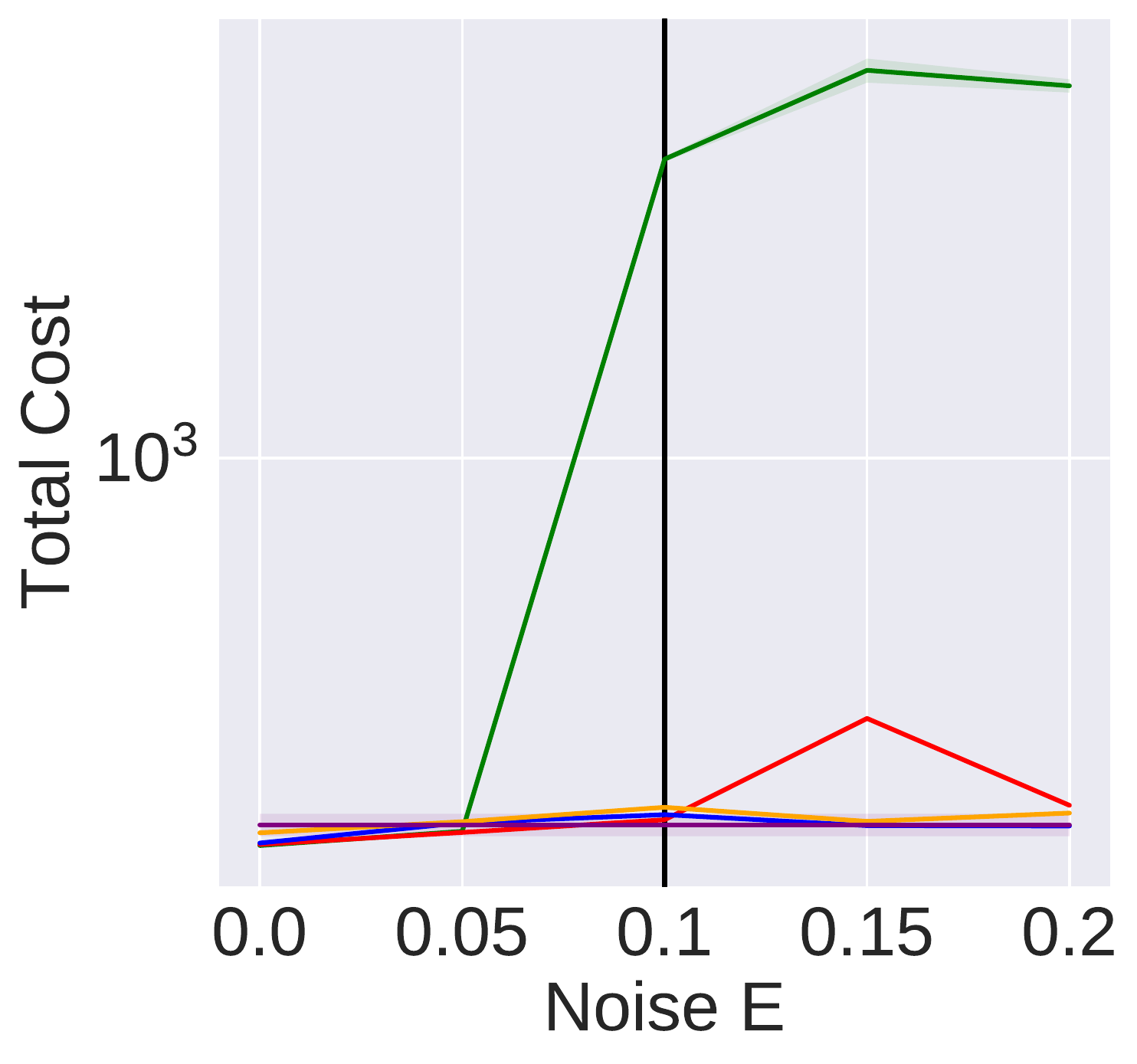}
\caption{$M^{L,\epsilon_L}$ with $\epsilon_L = 0.1$} \label{fig:adversarial_gauss_gridl0.1}
\end{subfigure}
\caption{Ablation of $\alpha$ in Algorithm~\ref{alg:RobustREIRL} under different levels of mismatch: $\br{\epsilon_E, \epsilon_L} \in \bc{0.0, 0.05, 0.1, 0.15, 0.2} \times \bc{ 0.0, 0.05, 0.1}$. Each plot corresponds to a fixed leaner environment $M^{L,\epsilon_L}$ with $\epsilon_L \in \bc{ 0.0, 0.05, 0.1}$. The values of $\alpha$ used in our Algorithm~\ref{alg:RobustREIRL} are reported in the legend. The vertical line indicates the position of the learner environment in the x-axis. The results are averaged across $5$ seeds.}
\label{fig:adversarial_gaussian-gridworld}
\end{figure}

\end{document}